\documentclass[11pt]{article}

\usepackage{amsmath,amssymb,amsthm,amsfonts,bm,dsfont,mathrsfs,mathtools,float,algorithm,algorithmicx,algpseudocode}
\usepackage{avant,array,geometry,fontenc,inputenc,natbib,hyperref,multirow,enumerate,rotating,booktabs,color,latexsym,times,stmaryrd,xr,physics,subcaption,graphicx,cleveref}
\usepackage{listings}
\usepackage{xr}

\usepackage{amsmath,amsfonts,bm}

\def\eqref#1{equation~\ref{#1}}

\def\ceil#1{\lceil #1 \rceil}
\def\floor#1{\lfloor #1 \rfloor}
\def\1{\bm{1}}

\def\rw{{\textnormal{w}}}

\def\mD{{\bm{D}}}

\def\mF{{\bm{F}}}
\def\mG{{\bm{G}}}

\def\mM{{\bm{M}}}

\def\mT{{\bm{T}}}

\def\mX{{\bm{X}}}
\def\mY{{\bm{Y}}}

\DeclareMathAlphabet{\mathsfit}{\encodingdefault}{\sfdefault}{m}{sl}
\SetMathAlphabet{\mathsfit}{bold}{\encodingdefault}{\sfdefault}{bx}{n}

\newcommand{\E}{\mathbb{E}}

\newcommand{\R}{\mathbb{R}}
\newcommand{\Z}{\mathbb{Z}}
\newcommand{\N}{\mathbb{N}^+}

\newcommand{\sigmoid}{\sigma}

\newcommand{\KL}{D_{\mathrm{KL}}}

\DeclareMathOperator*{\argmin}{arg\,min}

\DeclareMathOperator\conv{conv}

\def\zero{\boldsymbol{0}}

\def\e{\boldsymbol{e}}
\def\p{\boldsymbol{p}}
\def\q{\boldsymbol{q}}

\def\x{\boldsymbol{x}}
\def\y{\boldsymbol{y}}
\def\z{\boldsymbol{z}}
\def\h{\boldsymbol{h}}

\def\v{\boldsymbol{v}}

\def\P {\mathbb{P}}
\def\E {\mathbb{E}}

\newtheorem{theorem}{Theorem}
\newtheorem{proposition}{Proposition}
\newtheorem{lemma}{Lemma}
\newtheorem{corollary}{Corollary}
\theoremstyle{definition}
\newtheorem{definition}{Definition}
\newtheorem{assumption}{Assumption}
\theoremstyle{remark}

\def\P{\mathbb{P}}

\def\E{\mathbb{E}}

\def\calB{{\cal B}}

\renewcommand{\hat}{\widehat}
\allowdisplaybreaks

\def\TF{{\textnormal{TF}}}
\def\Attn{{\textnormal{Attn}}}
\def\pre{{\textnormal{pre}}}
\def\FFN{{\textnormal{FFN}}}
\def\id{{\textnormal{id}}}

\def\syn{{\textnormal{syn}}}

\def\seed{{\textnormal{seed}}}
\def\mn{{\textnormal{min}}}

\def\r{\textnormal{raw}}
\def\s{\textnormal{syn}}
\def\o{\textnormal{ovs}}
\def\a{\textnormal{aug}}
\def\b{\textnormal{bal}}
\def\rw{\textnormal{rwt}}
\def\T{{\textnormal{tot}}}

\def\Id{\mathbb{I}}
\def\calR{\mathcal{R}}
\def\calL{\mathcal{L}}
\def\calRhat{\mathcal{\hat \calR}}
\def\calLhat{\mathcal{\hat \calL}}
\def\mX{{\mathcal{X}}}
\def\mY{{\mathcal{Y}}}

\def\mD{{\mathcal{D}}}
\def\mF{{\mathcal{F}}}
\def\mG{{\mathcal{G}}}

\def\mT{{\mathcal{T}}}
\def\mM{{\mathcal{M}}}
\def\s{{\boldsymbol{s}}}
\def\p{{\boldsymbol{p}}}
\def\q{{\boldsymbol{q}}}
\def\u{{\boldsymbol{u}}}
\def\v{{\boldsymbol{v}}}
\def\z{{\boldsymbol{z}}}
\def\bzeta{{\boldsymbol{\zeta}}}
\def\btheta{{\boldsymbol{\theta}}}
\def\bgamma{{\boldsymbol{\gamma}}}

\begin{document}

\title{Theoretical Foundations of Synthetic Oversampling and Augmentation: A Transformer-Based Approach for Imbalanced Data}

\author{
    Ryumei Nakada${}^{1}$\footnote{equal contribution}
    \and
    Yichen Xu${}^{2}$$^*$
    \and
    Lexin Li${}^{3}$\footnote{co-corresponding author}
    \and
    Linjun Zhang${}^{4}$$^\dag$
}

\footnotetext[1]{Harvard University. Email: \href{mailto:ryumei_nakada@hms.harvard.edu}{ryumei\_nakada@hms.harvard.edu}.}
\footnotetext[2]{University of California, Berkeley. Email: \href{mailto:yichen_xu@berkeley.edu}{yichen\_xu@berkeley.edu}.}
\footnotetext[3]{University of California, Berkeley. Email: \href{mailto:lexinli@berkeley.edu}{lexinli@berkeley.edu}.}
\footnotetext[4]{Rutgers University. Email: \href{mailto:lz412@stat.rutgers.edu}{lz412@stat.rutgers.edu}.}

\maketitle

\begin{abstract}
Imbalanced classification and spurious correlation are common challenges in data science and machine learning. Both issues are linked to data imbalance, with certain groups of data samples significantly underrepresented, which in turn would compromise the accuracy, robustness, and generalizability of the learned models. Recent advances have proposed leveraging the flexibility and generative capabilities of large language models (LLMs), typically built on transformer architectures, to generate synthetic samples and to augment the observed data. In the context of imbalanced data, LLMs are used to oversample underrepresented groups and have shown promising improvements. However, there is a clear lack of theoretical understanding of such synthetic data approaches. In this article, we develop novel theoretical foundations to systematically study the roles of synthetic samples in addressing imbalanced classification and spurious correlation. Specifically, we first explicitly quantify the benefits of synthetic oversampling. Next, we analyze the scaling dynamics in synthetic augmentation, and derive the corresponding scaling law. Finally, we demonstrate the capacity of transformer models to generate high-quality synthetic samples. We further conduct extensive numerical experiments to validate the efficacy of the LLM-based synthetic oversampling and augmentation. 
\end{abstract}
\bigskip


\section{Introduction}
\label{sec:intro}

Imbalanced classification and spurious correlation are common challenges in data science and machine learning. Imbalanced classes appear when a minority class has a significantly smaller number of instances than a majority class, a phenomenon commonly encountered in applications such as rare disease diagnosis, risk management, and spam or fraud detection \citep{haixiang17imbalanced}. Spurious correlation refers to a situation where two variables appear to be related to each other, but the relationship is coincidental or confounded by an external variable. Spurious correlation often stems from class imbalance in the data, with certain majority group being over-represented, leading the model to rely on spurious correlation that does not always hold for the minority group \citep{ye2024spurious}. In both situations, statistical analyses may suffer from modeling bias favoring the majority group, potentially at the expense of the minority group, which would eventually lead to poor model performance, misleading insights, and compromised model generalizability and fairness.

Oversampling has been a common and effective strategy to address imbalanced classification and spurious correlation, which rebalances the data by increasing the representation of the minority class \citep[see][for reviews]{krawczyk2016learning, gosain2017handling}. A simple solution is to randomly oversample the minority group, though it could result in overfitting \citep{johnson2019survey}. Other conventional and popular solutions include synthetic minority oversampling technique \citep[SMOTE,][]{chawla02smote} along with its variants \citep{han2005borderline,bunkhumpornpat2009safe,douzas2018improving}, and adaptive synthetic sampling \citep[ADASYN,][]{he2008adasyn}. The former generates synthetic samples through linear interpolation of the nearest neighbors of the observed samples, whereas the latter further improves by adaptively creating more synthetic samples for harder-to-learn minority cases. 

In the era of artificial intelligence (AI), generative AI is rapidly gaining prominence and playing a crucial role across a diverse array of applications. There have been proposals utilizing generative AI techniques to produce synthetic and realistic data samples for the minority group. One family of such solutions employs generative adversarial networks \citep[GANs,][]{douzas2018effective, oh2019oversampling, jo2022obgan}. Another family employs large language models (LLMs), leveraging the powerful transformer architectures \citep{vaswani2017attention} and the knowledge learned from training on vast amounts of data, to generate coherent and contextually relevant data \citep{cloutier2023fine, yanglanguage, isomura4821750llmovertab}. We refer to these solutions collectively as \emph{synthetic oversampling}. However, while there are promising empirical evidence demonstrating the effectiveness of synthetic oversampling in addressing imbalanced classification and spurious correlation, there has been virtually no investigation to statistically quantify its benefits or to establish formal theoretical guarantees.  

In addition to oversampling, generative AI is also utilized to generate additional synthetic data samples for both majority and minority groups, and to further augment the observed data. Notably, one family focuses on generative pre-trained transformer (GPT) models to produce synthetic samples for tabular data, and has demonstrated strong empirical performances  \citep{hollmann2022tabpfn, borisov23GReaT, solatorio23realtabformer, zhang2023generative, zhao23tabula, gulati23tabmt, seedat24curated, qu2025icml}. As exemplified by the GReaT pipeline \citep{borisov23GReaT}, these approaches first convert numeric tabular values into natural-language sentences, then feed them into a GPT model, e.g., a fine-tuned GPT-2, to generate synthetic data. We refer to them collectively as \emph{synthetic augmentation}. Scaling laws offer critical insights into the effectiveness of such approaches, particularly by guiding how generative models should be scaled to optimize the performance while balancing the computational cost. While the current focus is primarily on pre-training and fine-tuning GPT to improve the quality of generated samples, there have emerged theoretical studies on synthetic augmentation. In particular, \citet{jain2024scaling} proposed a weighted empirical risk minimization method and introduced a scaling law that determines the optimal balance between the original and surrogate data. Their scaling law reveals that appropriately increasing the surrogate data volume can significantly reduce the test error, offering a systematic framework for leveraging synthetic data at scale. Meanwhile, \citet{shen2023boosting} introduced the synthetic data generation for analytics framework, which leverages finetuned diffusion models to generate synthetic data, addresses the data scarcity and privacy concerns, and explores the generational effect where the error rates evolve with the synthetic data volume. However, both works mostly focus on prediction, rather than imbalanced classification and spurious correlation.

In this article, we develop theoretical foundations to systematically study the roles of synthetic samples in addressing imbalanced classification and spurious correlation. Our theoretical contributions are three-fold. First, we present a general theoretical framework for synthetic oversampling, and demonstrate how oversampling mitigates biases in imbalanced classification and reduces the impact of spurious correlation. We achieve this by quantifying the risk specific to the majority group and the minority group, respectively, and investigating how the synthetic data affect such risks. Second, we extend our theoretical framework for synthetic augmentation, and provide the scaling law that describes how the balanced excess risk is impacted by progressively adding synthetic data across all groups. Moreover, we quantify the bias in terms of the data imbalance ratio and synthetic data quality. We also note that our theory differs from that of \citet{jain2024scaling} and \citet{shen2023boosting}, in that they consider a general prediction task, while we focus on imbalanced data and group-specific risks. Finally, as both oversampling and augmentation depend on the synthetic data quality, we provide formal theoretical results quantifying the quality of data generated by transformer architectures, and establish a foundation for the reliable use of LLMs in synthetic data generation. Our theory thus also differs from the existing theory on transformers, which has predominantly focused on prediction or inference in in-context learning \citep{xie2021explanation, garg2022can, zhang2024trained, bai2023transformers}. 

In summary, we specifically target imbalanced classification and spurious correlation problems. We utilize large language models as high-quality synthetic data generators. We establish the theoretical foundations for both synthetic oversampling for the minority group, and for synthetic augmentation for all groups of data. We implement the synthetic data generation by fine-tuning a GPT-2 language model or prompting a pretrained GPT-4 model. We conduct extensive numerical experiments with two goals. For oversampling, we empirically compare our synthetic approach to some common alternative oversampling solutions. For augmentation, we empirically verify the polynomial decay rates established in our theory. 

Throughout this article, we employ the following notation.  For two sequences of positive numbers $(a_k)_k$ and $(b_k)_k$ indexed by $k \in \mathcal{K}$, write $a_k\lesssim b_k$ if and only if there exists a constant $C >0$ independent of the index $k$ such that $\sup_{k \in \mathcal{K}} a_k / b_k < C$ holds. Write $b_k = \Omega(a_k)$ or $a_k = O(b_k)$ when $a_k \lesssim b_k$. For a vector $\v = (v_1, \dots, v_D)^\top \in \R^D$, write $(\v)_i = v_i$, and $(\v)_{n_1:n_2} = (v_{n_1+1}, v_{n+2}, \dots, v_{n_2})^\top \in \R^{n_2 - n_1}$. For a matrix $V = [\v_1; \dots; \v_D]$, write $(V)_i = \v_i$. Let $a \vee b$ and $a \wedge b$ denote $\max(a, b)$ and $\min(a, b)$. Let $\mathbb{B}_D(R)$ denote a ball in $\R^D$ with radius $R > 0$ centered at $\zero_D$. Let $\|\v\|$ and $\|f\|$ denote the $L_2$ norm of the vector $v$ and the function $f$, and $\|A\|$ denote the operator norm of the matrix $A$. 

The rest of the article is organized as follows. Section~\ref{sec:set-up} introduces the problem setup and outlines the synthetic data generation using GPT models. Section~\ref{sec: theory} presents the associated theory for synthetic oversampling and augmentation, respectively. Section~\ref{sec: theory transformers} investigates transformers as high-quality synthetic data generators. Section~\ref{sec:experiments} conducts the numerical experiments. The Supplementary Appendix presents the proofs and additional results.

\section{Synthetic Oversampling and Augmentation}
\label{sec:set-up}

In this section, we first introduce the problem setup and the key quantities of interest. We then describe a general algorithm for synthetic oversampling and augmentation.

\subsection{Problem setup}
\label{sec:problem}

Individual samples are often organized in groups. For instance, for imbalanced classification with a binary outcome, there are two groups, whereas for spurious correlation with a binary outcome and a binary-valued spurious feature, there are four groups. Let $\mG$ denote the set of groups. For each group $g \in \mG$, let $n_g$ be the number of observed original samples, and $m_g$ be the number of additionally generated synthetic samples. Let the total sample size $n_\T =  \sum_{g \in \mG} n_g$, and $m_\T =  \sum_{g \in \mG} m_g$. For {\it synthetic oversampling}, we only add synthetic samples to the minority group, whereas for {\it synthetic augmentation}, we further add $N$ synthetic samples to both minority and majority groups. To balance data across all groups in synthetic oversampling, we choose $m_g =  \max_{g' \in \mG} n_{g'} - n_g$.

Let $\{(\x_i^{(g)}, y_i^{(g)})\}_{i \in [n_g]}$ denote the observed original raw data, and $\{(\tilde \x_i^{(g)}, \tilde y_i^{(g)})\}_{i \in [m_g+N]}$ denote the added synthetic data for group $g$, where $[n]=\{1,2,\ldots,n\}$ for any positive integer $n$. Let $\ell(\btheta; \x, y)$, parameterized by $\btheta \in \Theta$, denote a general loss function. We define the population-level group-specific risk, and the balanced risk with equal weights on all groups as
\begin{equation*}\label{eq: R bal}
\calR^{(g)}(\btheta) =  \E\left[ \ell\left( \btheta; \x_1^{(g)}, y_1^{(g)} \right) \right], \ \ \text{ and } \ \ 
\calR_\b(\btheta) =  \frac{1}{|\mG|} \sum_{g \in \mG} \calR^{(g)}(\btheta).
\end{equation*}
We define the empirical risk with synthetic oversampling as
\begin{equation*}\label{eq: general R}
\calRhat_\o(\btheta) =  \frac{1}{n_\T + m_\T} \qty{\sum_{g \in \mG} \sum_{i \in [n_g]} \ell( \btheta; \x_i^{(g)}, y_i^{(g)} ) +  \sum_{g \in \mG} \sum_{i \in [m_g]} \ell( \btheta; \tilde \x_i^{(g)}, \tilde y_i^{(g)} )},
\end{equation*}
and the empirical risk with synthetic augmentation as
\begin{equation*}
\calRhat_\a(\btheta) =  \frac{1}{N |\mG|} \sum_{g \in \mG} \sum_{i \in [N+m_g]\setminus[m_g]} \ell(\btheta; \tilde \x_i^{(g)}, \tilde y_i^{(g)}),
\end{equation*}
We also define the empirical risk with both synthetic oversampling and augmentation as
\begin{equation}\label{eq: general R with data augmentation}
\calRhat(\btheta) =  (1 - \alpha) \calRhat_\o(\btheta) + \alpha \calRhat_\a(\btheta),
\end{equation}
where $\alpha \in [0, 1]$ is a weight parameter that balances the contribution of the additional $N$ augmented data for every group. Let 
\begin{equation*}
    \hat\btheta_\o =  \arg\min_{\btheta} \calRhat_\o(\btheta), \quad\quad \hat\btheta = \arg\min_{\btheta} \calRhat(\btheta), \quad \text{ and } \quad \btheta_\b =  \arg\min_{\btheta} \calR_\b(\btheta)
\end{equation*}
denote the corresponding minimizers. Note that $\btheta_\b$ can be viewed as the oracle solution that balances over all groups present in the dataset. 

We further introduce the notion of imbalance ratio for group $g \in \mG$ as
\begin{equation*}
\rho_g =  \frac{\max_{g' \in \mG} n_{g'} - n_g}{\max_{g' \in \mG} n_{g'}}.
\end{equation*}
Accordingly, we define the average imbalance ratio as $\rho =  (1/|\mG|) \sum_{g \in \mG} \rho_g$. Note that $\rho_g \in [0, 1)$ and $\rho \in [0, 1)$. Moreover, when the data is perfectly balanced, i.e., when $n_1 = n_2 = \dots = n_{|\mG|}$, we have $\rho_1 = \rho_2 = \dots = \rho_{|\mG|} = \rho = 0$.

In our theoretical analysis, we aim to investigate the effect of the synthetic data on the risk of both majority and minority groups. Specifically, for synthetic oversampling, we aim to quantify the group-specific risk $\calR^{(g)}(\hat\btheta_\o) - \calR^{(g)}(\btheta_\b)$, whereas for synthetic augmentation, we aim to derive the scaling law for the excess risk $\calR_\b(\hat\btheta) - \calR_\b(\btheta_\b)$.

\subsection{A general algorithm for synthetic data generation}
\label{sec:alg}

\begin{algorithm}[t!]
\small
\caption{\color{black}{LLM-powered Synthetic Oversampling and Augmentation}}
\label{alg}
\begin{algorithmic}[1]
\Statex \textbf{Input:} Set of groups $\mG$; raw data $\left\{(\mathbf{x}_i^{(g)}, y_i^{(g)})\right\}_{i \in [n_g]}$, and seed data $\left\{ (\mathbf{x}^{(g) \seed}_i, y_i^{(g) \seed}) \right\}_i$ for each group $g \in \mG$; the number of seed data $n_\seed$; the number of samples for additional augmentation $N$;

\Statex \textbf{Procedure:}
\State Select balanced seed data: $\left\{ (\mathbf{x}^{(g) \seed}_i, y_i^{(g) \seed}) \right\}_{i \in [n_\seed]}$ 
\State Generate and pool sufficiently many synthetic data using LLM given the seed data as $\left\{ (\tilde{\x}_i, \tilde{y}_i) \right\}_i$. 
\State \underline{Synthetic oversampling}: randomly select $m_g = \max_{g' \in \mG} n_{g'} - n_g$ generated samples from the pool for each group $g \in \mG$, $\left\{ (\tilde{\x}^{(g)}_i, \tilde{y}^{(g)}_i) \right\}_{i \in [m_g]}$.
\State \underline{Synthetic augmentation}: randomly select $N$ generated samples from the pool for each group $g \in \mG$, $\left\{ (\tilde{\x}_i^{(g)}, \tilde{y}_i^{(g)}) \right\}_{i \in [N+m_g]\setminus[m_g]}$.
\Statex \textbf{Output:} Augmented data $\left\{ (\x_i^{(g)}, y_i^{(g)}) \right\}_{i \in [n_g]} \cup \left\{ (\tilde{\x}_i^{(g)}, \tilde y_i^{(g)}) \right\}_{i \in [m_g]} \cup \left\{ (\tilde \x_i^{(g)}, \tilde y_i^{(g)}) \right\}_{i \in [N+m_g]\setminus[m_g]}$ for all $g \in \mG$.
\end{algorithmic}
\end{algorithm}

We present a general LLM-powered algorithm for synthetic oversampling and augmentation. We illustrate our approach with the setting of two groups, the minority group $g_{\text{min}}$ and the majority group $g_{\text{maj}}$ in Algorithm~\ref{alg}. More specifically, we first randomly select a subset of seed data by group as input for the LLM, so to ensure that the input data is balanced (line 1). We next consider two ways of generating the synthetic data using the LLM. One is to fine-tune a GPT-2 language model, and the other is to directly prompt a pretrained GPT-4 model without any additional fine-tuning. To prepare the tabular data to feed into the LLM, we follow \citet{borisov23GReaT} and convert the numeric values into the sentence serialized format, [$f_{j}$  ``is'' $v_{ij}$], where $f_{j}$ is the $j$th feature name and $v_{ij}$ is the value of the $j$th feature of the $i$th sample. After synthetic data generation, we deserialize the generated data. We then choose $m_g$ generated samples for group $g$ for synthetic oversampling (line 3), and $N$ generated samples for all groups for synthetic augmentation (line 4). We provide a schematic plot and more discussions of the algorithm in Section~\ref{supp-sec: add-algo} of the Appendix.

\section{Theory for Synthetic Oversampling and Augmentation}
\label{sec: theory}

In this section, we first derive the excess risk for synthetic oversampling for the minority group, then the scaling law for synthetic augmentation for all groups, all under the regime where $n_\T \to \infty$. Our theory is general and does not depend on any particular generative mechanism. It provides a broad theoretical foundation for understanding statistical behaviors of synthetic oversampling and augmentation, assuming only that the bias between the real and synthetic distributions diminishes as the sample size grows. The framework encompasses a range of generative models, and we specialize it to the transformer architecture later in Section~\ref{sec: theory transformers}.

\subsection{Regularity Conditions}
\label{sec: conditions}

We begin with a set of regularity conditions. Denote the bias of the risk for group $g$ as
\begin{align*}
\calB^{(g)}(\btheta) =  \E\left[ \ell\left( \btheta; \tilde \x_1^{(g)}, \tilde y_1^{(g)} \right) \right] - \E\left[ \ell\left( \btheta; \x_1^{(g)}, y_1^{(g)} \right) \right], 
\end{align*} 
and denote the covariance of gradients for each group $g \in \mG$ as
\begin{small}
\begin{align*}
\Sigma_g(\btheta) &= \E\left[ \nabla \ell\left(\btheta; \x_1^{(g)}, y_1^{(g)}\right) \nabla \ell\left(\btheta; \x_1^{(g)}, y_1^{(g)}\right)^\top \right],\ \ \tilde \Sigma_g(\btheta) = \E\left[ \nabla \ell\left( \btheta; \tilde \x_1^{(g)}, \tilde y_1^{(g)} \right) \nabla \ell\left( \btheta; \tilde \x_1^{(g)}, \tilde y_1^{(g)} \right)^\top \right].
\end{align*}
\end{small}\noindent

\begin{assumption}\label{asm: differentiability}
    Suppose for every $g\in\mathcal G$, $\ell(\cdot;z)$ is twice continuously differentiable around $\btheta_\b$ under the distributions of $\left( \tilde \x_1^{(g)}, \tilde y_1^{(g)} \right)$ and $\left( \x_1^{(g)}, y_1^{(g)} \right)$ almost surely with integrable envelopes; $\sum_{g\in\mathcal G}\nabla^2 \calR^{(g)}(\btheta)$ is strictly positive definite around $\btheta_\b$ with bounded eigenvalues; and $\Sigma_g,\tilde\Sigma_g$ are locally Lipschitz around $\btheta_\b$, and satisfy $\sup_{\btheta\in\Theta}\|\Sigma_g(\btheta)-\tilde\Sigma_g(\btheta)\|=o(1)$ as $n_\T \to\infty$.
\end{assumption}

\begin{assumption}\label{asm: bias o(1)}
Suppose $\max_{g \in \mG} \sup_{\btheta \in \Theta} | \calB^{(g)}(\btheta)| \vee \|\nabla \calB^{(g)}(\btheta)\| \vee \|\nabla^2 \calB^{(g)}(\btheta)\| = o(1)$ as $n_\T \to \infty$.
\end{assumption}

\begin{assumption}\label{asm: uniform convergence R syn}
Suppose $\sup_{\btheta \in \Theta_\o} |\calRhat_\o(\btheta) - \calR_\o(\btheta)| = o_p(1)$ as $n_\T \to \infty$, where $\calR_\o(\btheta) = \E[\calRhat_\o(\btheta)]$, and $\btheta_\o =  \argmin_\btheta \calR_\o(\btheta)$.
\end{assumption}

\begin{assumption}\label{asm: identifiability theta}
Suppose $\inf_{\btheta \in \Theta: \|\btheta - \btheta_\o\| \geq \epsilon} \calR_\o(\btheta) > \calR_\o(\btheta_\o)$, and $\inf_{\btheta \in \Theta: \|\btheta - \btheta_\b\| \geq \epsilon}$ $\calR_\b(\btheta) > \calR_\b(\btheta_\b)$.
\end{assumption}

\noindent
We make some remarks. Assumption~\ref{asm: differentiability} requires that the loss function is sufficiently smooth, which ensures the model behaves well in local neighborhoods around $\btheta_\b$, its Hessian is consistent, and the covariance structure of gradients under the synthetic distribution matches the raw data  distribution near $\btheta_\b$. A similar condition was imposed in \citet{jain2024scaling} to study the scaling law for standard supervised learning. Assumption~\ref{asm: bias o(1)} specifies a regime where the bias for each group $g \in \mathcal{G}$ diminishes as the amount of seed data grows. It requires that $\calB^{(g)}$, along with its first- and second-order derivatives, has vanishing magnitude. Note that, for a sufficiently smooth loss function, we have the bound, $|\calB^{(g)}(\btheta_\b)| \vee \|\nabla \calB^{(g)}(\btheta_\b)\| \vee \|\nabla^2 \calB^{(g)}(\btheta_\b)\| \lesssim d_{\operatorname{TV}}(\mathcal{D}_\r, \mathcal{D}_\syn)$, where $\mathcal{D}_\r$ and $\mathcal{D}_\syn$ are the raw and synthetic data distributions, respectively. Thus Assumption~\ref{asm: bias o(1)} holds if the synthetic data distribution approaches the raw data distribution as $n_\T \to \infty$. In Section~\ref{sec: theory transformers}, we will show that transformer-based generators indeed satisfy this condition when the seed data size grows together with $n_\T$. Assumption~\ref{asm: uniform convergence R syn} is a uniform law of large numbers for the oversampled risk assuming that the empirical oversampling objective $\hat \calR_\o$ concentrates around its expectation $\calR_\o$, which ensures that minimizing $\hat \calR_\o$ is asymptotically equivalent to minimizing its population counterpart. Assumption~\ref{asm: identifiability theta} is an identification requirement, where both $\calR_\o$ and $\calR_\b$ admit unique, well-separated minimizers. Both conditions are standard in proving the consistency and asymptotic normality of the estimators in statistical learning \citep{van2000asymptotic,jain2024scaling}. Moreover, in Section~\ref{supp-sec: verification} of the Appendix, we verify that Assumptions~\ref{asm: differentiability}, \ref{asm: uniform convergence R syn}, and \ref{asm: identifiability theta} hold for linear regressions with squared loss and logistic regressions with sub-Gaussian covariates.

We also remark that, for our theory development, we assume that both the raw data and the synthetic data are i.i.d., and they are independent of each other. To ensure independence between the raw and synthetic data, one may use sample splitting, i.e., using part of training samples as the seed data to generate synthetic samples, while using another part as the raw data, together with the synthetic data, for model training. To ensure independence within the synthetic samples, one may generate each pair $\left\{ (\tilde x_i^{(g)}, \tilde y_i^{(g)}) \right\}_{i \in [m_g+N]}$ in separate decoding runs with independent random seeds but the same seed-prompt derived from $\left\{ (x_i^{(g)}, y_i^{(g)}) \right\}_{i \in [n_g]}$, making the synthetic samples conditionally i.i.d.\ given the prompt. In our implementation, we adopt the GReaT pipeline \citep{borisov23GReaT}, which tokenizes each data record as a structured text sequence, and generates samples via autoregressive decoding. Each sample is produced in a separate decoding run with a fresh random seed, approximately achieving independence across samples. That is, the LLM samples from the learnt distribution without conditioning on previously generated samples. \citet{borisov23GReaT} further validates this independence pattern empirically through some numerical analyses.

\subsection{Theory for synthetic oversampling}
\label{sec: theory oversampling}

We first establish the theoretical guarantees for synthetic oversampling, i.e., when we add synthetic samples to the minority group on the purpose of balancing the data. More specifically, we begin with a general setting to understand the effect of the synthetic data bias on the estimator $\hat\btheta_\o$, by measuring the risk of each individual group $g$, i.e., $\calR^{(g)}(\hat\btheta_\o) - \calR^{(g)}(\btheta_\b)$. We then study the specific settings of imbalanced classification and spurious correlation. 

We have the following general theorem, which formally establishes the excess risk for synthetic oversampling stated in terms of the bias and statistical error. 

\begin{theorem}\label{thm: synthetic data risk}
Suppose Assumptions~\ref{asm: differentiability}-\ref{asm: identifiability theta} hold for any $g \in \mG$. Then, 
\begin{align} \label{eq: synthetic oversampling risk ub}
\calR^{(g)}(\hat \btheta_\o) - \calR^{(g)}(\btheta_\b) &= - \left\{\nabla \calR^{(g)}(\btheta_\b)\right\}^\top \left\{\nabla^2 \calR_\b(\btheta_\b)\right\}^{-1} \boldsymbol{b} + R,
\end{align}
where
\begin{align*}
\boldsymbol{b} &= \frac{1}{|\mG|} \sum_{g' \in \mG} \rho_{g'} \nabla \calB^{(g')}(\btheta_\b),\\ 
R &= O_p\left( \frac{v_g}{\sqrt{|\mG| \max_{g' \in \mG} n_{g'}}} + \frac{1}{|\mG|} \sum_{g' \in \mG} \rho_{g'} \left\{ \|\nabla \calB^{(g')}(\btheta_\b)\|^2 \vee \|\nabla^2 \calB^{(g')}(\btheta_\b)\|^2 \right\} \right), \\
v_g^2 &= \{\nabla \calR^{(g)}(\btheta_\b)\}^\top \{\nabla^2 \calR_\b(\btheta_\b)\}^{-1} \left\{ \frac{1}{|\mG|} \sum_{g' \in \mG} \Sigma_{g'}(\btheta_\b) \right\} \{\nabla^2 \calR_\b(\btheta_\b)\}^{-1} \nabla \calR^{(g)}(\btheta_\b).
\end{align*}
\end{theorem}

\noindent
Theorem~\ref{thm: synthetic data risk} shows that the excess risk is driven by the group specific bias term $\{\nabla \calR^{(g)}(\btheta_\b)\}^\top$ $\{\nabla^2 \calR_\b(\btheta_\b)\}^{-1} \boldsymbol{b}$ and the residual term $R$, both of which depend on the imbalance ratio and the synthetic data quality. Note the term $\boldsymbol{b}$ in (\ref{eq: synthetic oversampling risk ub}) consists of the group-specific bias term $\nabla \calB^{(g')}(\btheta_\b) = \nabla \E[ \ell( \btheta_\b; \tilde \x_1^{(g')}, \tilde y_1^{(g')} ) ] - \nabla \E[ \ell( \btheta_\b; \x_1^{(g')}, y_1^{(g')} ) ]$ weighted by the imbalance ratio $\rho_{g'}$ of the group. Hence $\{\nabla \calR^{(g)}(\btheta_\b)\}^\top \{\nabla^2 \calR_\b(\btheta_\b)\}^{-1} \boldsymbol{b}$ projects the discrepancy between the synthetic data distribution $\boldsymbol{b}$ onto the geometry defined by the Hessian ${\nabla^2 \calR_\b(\btheta_\b)}^{-1}$, capturing how the discrepancy affects the group specific risk under distribution shifts. The term $R$ in (\ref{eq: synthetic oversampling risk ub}) is the residual from the variation of the synthetic and raw data, as well as possible nonlinearity of the model and the loss function. Specifically, $v_g/\sqrt{|\mG| \max_{g' \in \mG} n_{g'}}$ accounts for the variance of the gradients, arising from the finite sample variation of $\hat{\mathcal{R}}_\o$, whereas $\|\nabla \calB^{(g')}(\btheta_\b)\|^2$ and $\|\nabla^2 \calB^{(g')}(\btheta_\b)\|^2$ arise from possible nonlinearity of the loss function. Moreover, when $\|\nabla \calB^{(g')}(\btheta_\b)\|$ and $\|\nabla^2 \calB^{(g')}(\btheta_\b)\|$ have the same order,
\begin{align*}
|\calR^{(g)}(\hat \btheta_\o) - \calR^{(g)}(\btheta_\b) | = O\left(\max_{g' \in \mG} \|\nabla \calB^{(g')}(\btheta_\b)\|\right) + O_p\left( \qty(\max_{g' \in \mG} n_{g'})^{-1/2} \right).
\end{align*}
Since the bias term on the right-hand side does not decrease as the number of raw data increases, reducing the bias in the synthetic data is crucial for reducing the oversampling risk, which in turn suggests the importance of having a high-quality synthetic data generator in synthetic oversampling.

We next apply the general result of Theorem~\ref{thm: synthetic data risk} to the specific settings of imbalanced classification and spurious correlation, where we explicitly specify the groups $\mathcal{G}$ and the synthetic sample size $m_g$ for each group $g \in \mathcal{G}$.

\medskip
\noindent
\textbf{Imbalanced classification.} 
We first consider a binary imbalanced classification problem with the label set $\mG = \mathcal{Y} = \{0, 1\}$. Suppose group $0$ is the minority group and group $1$ the majority group. We add the synthetic data only to the minority group $g=0$, such that the total number of samples for each group are equal, i.e., $n_0 + m_0 = n_1$. For this setting, we are interested in the performance of $\hat \btheta_\o$ compared to $\btheta_\b$, measured in the worst group excess risk $\max_{y \in \mY} \calR^{(y)}(\hat\btheta_\o) - \max_{y \in \mY} \calR^{(y)}(\btheta_\b)$, where $\calR^{(y)}(\btheta) =  \E\left[ \ell(\btheta; \x_1, y_1) | y_1 = y \right]$. We have the following corollary for the imbalanced classification problem. 

\begin{corollary}\label{cor: synthetic data risk imbalanced}
Suppose Assumptions~\ref{asm: differentiability}-\ref{asm: identifiability theta} hold. If $1 \leq n_0 \leq c n_1$ for some constant $c \in (0, 1)$, then,
\vspace{-0.1in}
\begin{align*}
\max_{y \in \mY} \calR^{(y)}(\hat\btheta_\o) - \max_{y \in \mY} \calR^{(y)}(\btheta_\b) &\leq \frac{n_1-n_0}{2 n_1} \max_{y \in \mY} \abs{b_{y,0}}\\
   +  & O_p\left( \frac{1}{\sqrt{n_1}} (v_0 + v_1) + \|\nabla \calB^{(0)}(\btheta_\b)\|^2 + \|\nabla^2 \calB^{(0)}(\btheta_\b)\|^2 \right),
\end{align*}
where
\vspace{-0.1in}
\begin{align*}
b_{y,0} &= \{\nabla \calR^{(y)}(\btheta_\b)\}^\top \{\nabla^2 \calR_\b(\btheta_\b)\}^{-1} \nabla \calB^{(0)}(\btheta_\b),\\
v_y^2 &= \{\nabla \calR^{(y)}(\btheta_\b)\}^\top \{\nabla^2 \calR_\b(\btheta_\b)\}^{-1} \frac{\Sigma_0(\btheta_\b) + \Sigma_1(\btheta_\b)}{2} \{\nabla^2 \calR_\b(\btheta_\b)\}^{-1} \nabla \calR^{(y)}(\btheta_\b).
\end{align*}
\end{corollary}
\noindent
The upper bound for the excess risk involves the term $\{(n_1 - n_0) / (2 n_1)\} (b_{0,0} \vee b_{1,0})$ that represents the bias introduced by the addition of synthetic data for group $g=0$, while the rest contains the variance term proportional to $n_1^{-1/2}$ and higher order error terms. This corollary thus shows that $\hat\btheta_\o$ achieves a similar minority-group performance as $\btheta_\b$, when the bias of the minority group introduced by the synthetic data generation, quantified by $b_{0,0}$, is small. 

\medskip
\noindent
\textbf{Spurious correlation.} 
We next consider the spurious correlation problem with a binary outcome $y \in \mathcal{Y} = \{-1,1\}$ and a discrete-valued spurious feature $\s \in \mathcal{S} = \{-\bgamma, \bgamma\}$ for some $\bgamma \in \R^q$. A similar setting of spurious correlations was also studied in \citet{arjovsky2019invariant,ye2023freeze}. The observed data consist of $\x_i = (\z_i, \s_i)$ and $y_i$, where $\z_i \in \R^p$ is the core feature and $\s_i \in \mathcal{S}$ is the spurious feature. {Suppose the core feature $\z_i$ only depends on the outcome label $y_i$, whereas $y_i$ is spuriously correlated with $\s_i$. In particular, the groups $(1,\bgamma),(-1, -\bgamma)$ form the majority, while $(1,-\bgamma), (-1,\bgamma)$ form the minority. This imbalance induces a spurious association between $y$ and $\s$, making $\s$ appear spuriously predictive even though it does not drive the actual label generation. Reweighing the distribution of $(y,\s)$ pairs can help mitigate this artificial correlation and reduce the risk of learning a spurious predictor.} For simplicity, suppose $n_{(-1,\bgamma)} = n_{(1,-\bgamma)} = n_{\text{min}} < n_{\text{maj}} = n_{(1,\bgamma)} = n_{(-1,-\bgamma)}$, so that groups $(-1,-\bgamma)$ and $(1,\bgamma)$ are the majority groups. We choose the synthetic data size for group $g = (y, \s) \in \mG$ by $m_{(y,\s)} = (n_{\text{maj}} - n_{\text{min}})\Id\{y \bgamma \neq \s\}$, so to make the raw and synthetic data size equal for each group. Define the reweighted risk and its minimizer by
\begin{align*}
\calR_\rw(\btheta) =  \frac{1}{2} \sum_{y} \E[\ell(\btheta; \x_1', y_1) | y_1 = y], \ \ \text{ and } \ \ \btheta_\rw =  \argmin_{\btheta \in \Theta} \calR_\rw(\btheta).
\end{align*}
where $\x_1' = (\z_1, \s_1')$, with $\s_1' \sim \operatorname{Uniform}(\{-\bgamma, \bgamma\})$ independent of $y_1$. This reweighted risk is useful for handling spurious correlation, and has been studied in \citet{shimodaira2000improving,byrd2019effect,sagawa2020investigation} to decouple the correlation between $y_i$ and $\s_i$ and mitigate the effect of spurious correlation. For this setting, we are interested in the performance of $\hat \btheta_\o$ compared to $\btheta_\rw$, measured in the reweighted risk $\calR_\rw(\btheta)$. We have the following corollary for the spurious correlation problem. 

\begin{corollary}\label{cor: synthetic data risk spurious}
    Suppose Assumptions~\ref{asm: differentiability}-\ref{asm: identifiability theta} hold. If $1 \leq n_{\text{min}} \leq c n_{\text{maj}}$ for some constant $c \in (0, 1)$, then,
    \vspace{-0.1in}
    \begin{align*}
    \calR_\rw(\hat \btheta_\o) - \calR_\rw(\btheta_\rw) = O_p\biggl( \|\nabla \calB^{(-1,\bgamma)}(\btheta_\rw)\| + \|\nabla^2 \calB^{(-1,\bgamma)}(\btheta_\rw)\|\\
    + \|\nabla \calB^{(1,-\bgamma)}(\btheta_\rw)\| + \|\nabla^2 \calB^{(1,-\bgamma)}(\btheta_\rw)\| + \frac{\max_{g \in \mG} v_g}{\sqrt{n_{\text{maj}}}} \biggr),
    \end{align*}
    where
    \vspace{-0.1in}
    \begin{align*}
    v_g^2 =  \{\nabla \calR^{(g)}(\btheta_\rw)\}^\top \{\nabla^2 \calR_\rw(\btheta_\rw)\}^{-1} \left\{\frac{1}{4} \sum_{g' \in \mG} \Sigma_{g'}(\btheta_\rw) \right\} \{\nabla^2 \calR_\rw(\btheta_\rw)\}^{-1} \nabla \calR^{(g)}(\btheta_\rw).
    \end{align*}
\end{corollary}

\noindent
Again, the difference between $\calR_\rw(\hat \btheta_\o)$ and $\calR_\rw(\btheta_\rw)$ involves the bias term for the minority group, and a variance term of order $n_{\text{maj}}^{-1/2}$. This corollary thus shows that the reweighted risk of $\hat\btheta_\o$ becomes close to the reweighted risk of $\btheta_\rw$ when the bias of the minority group introduced by the synthetic data generation decreases.

\subsection{Theory for synthetic augmentation}
\label{sec: scaling law}

We next derive the scaling law for the excess risk associated with synthetic augmentation, where we continue adding synthetic samples to all groups following synthetic oversampling. We consider a parametric model to derive a general result here and present additional results for nonparametric regression models regarding imbalanced classification and spurious correlation in Section~\ref{supp-sec: additional scaling law} of the Appendix. 

Consider the empirical risk minimizer $\hat \btheta = \argmin_{\btheta} \calRhat(\btheta)$. For the  population balanced risk $\calR_\b$, we are interested in deriving the scaling law for $\calR_\b(\hat\btheta) - \calR_\b(\btheta_\b)$. We have the following theorem on the scaling law in terms of $n_\T$, $N$, and the bias term. 

\begin{theorem}\label{thm: scaling law}
Suppose the same conditions in Theorem~\ref{thm: synthetic data risk} hold. Recall that $\alpha \in [0, 1]$ is the weight parameter defined in (\ref{eq: general R with data augmentation}). Then, 
\begin{align}
\calR_\b(\hat\btheta) - \calR_\b(\btheta_\b) &= O\left( \frac{1}{|\mG|} \sum_{g \in \mG} \{(1 - \alpha) \rho_g + \alpha\} \{\|\nabla \calB^{(g)}(\btheta_\b)\|^2 + \|\nabla^2 \calB^{(g)}(\btheta_\b)\|^3\} \right) \nonumber \\
&\quad\quad + O_p\left( (1 - \alpha)^2 (1 - \rho) \frac{\tr(\Sigma)}{n_\T} + \alpha^2 \frac{\tr(\Sigma')}{N |\mG|} \right), \label{eq: scaling law}
\end{align}
where $\Sigma =  {|\mG|}^{-1} \sum_{g \in \mG} \qty{(1 - \rho_g) \Sigma_g(\btheta_\b) + \rho_g \tilde \Sigma_g(\btheta_\b)}$, and $\Sigma' =  {|\mG|}^{-1} \sum_{g \in \mG} \tilde \Sigma_g(\btheta_\b)$.
\end{theorem}

\noindent
Theorem~\ref{thm: scaling law} shows that the upper bound of the excess risk can be decomposed into the bias term from synthetic oversampling and data augmentation, i.e., the first term of the right-hand side in (\ref{eq: scaling law}), and the variance term that scales with the number of raw data and size of data augmentation, i.e., the second term in (\ref{eq: scaling law}). It is also noteworthy that the bias is captured by the first and second order derivatives of the difference $\calB^{(g)}(\btheta) =  \E[\ell(\btheta; \tilde \x_1^{(g)}, \tilde y_1^{(g)})] - \E[\ell(\btheta; \x_1^{(g)}, y_1^{(g)})]$, whose contribution is proportional to the ratio of the total added synthetic sample size over all samples, $(1 - \alpha) \rho_g + \alpha \in [0, 1]$. Moreover, the variance term also depends on the data augmentation weight $\alpha$ as well as $\Sigma$ and $\tilde \Sigma$. We further study nonparametric regressions regarding imbalanced classification and spurious correlations in Theorems~\ref{thm: gaussian sequence scaling law} and \ref{thm: nonparametric scaling law balanced} in the Appendix, and observe a similar polynomial decay in variance with respect to $n_\T$ and $N$.

We next apply the general result of Theorem~\ref{thm: scaling law} to the specific settings of imbalanced classification and spurious correlation

\medskip
\noindent
\textbf{Imbalanced classification.}
In the imbalanced classification setting, the scaling law provides insights into the excess risk reduction through both synthetic oversampling and synthetic augmentation. We consider a  similar setting as in Section~\ref{sec: theory oversampling}, while after we add the synthetic samples to the minority group for oversampling such that $n_0 + m_0 = n_1$, we further add $N$ synthetic samples for both minority and majority groups for augmentation. We have the following corollary for the imbalanced classification problem.

\begin{corollary}\label{cor: scaling law imb}
Suppose the same conditions in Theorem~\ref{thm: synthetic data risk} hold. For any $\alpha \in [0, 1]$,
\begin{align*}
\calR_\b(\hat\btheta) - \calR_\b(\btheta_\b) &= O\qty(\qty{(1 - \alpha) \frac{n_1-n_0}{n_1} + \alpha} \{\|\nabla \calB^{(0)}(\btheta_\b)\|^2 + \|\nabla^2 \calB^{(0)}(\btheta_\b)\|^3\}) \\ 
&\quad+ O_p\qty( (1 - \alpha)^2 \frac{\tr(\Sigma)}{n_1} + \alpha^2 \frac{\tr(\tilde \Sigma)}{N |\mG|}), 
\end{align*}
where $\Sigma = (1/2) (n_0/n_1) \Sigma_0(\btheta_\b) + (1/2) (1 - n_0/n_1) \tilde \Sigma_0(\btheta_\b) + (1/2)\Sigma_1(\btheta_\b)$, and $\tilde \Sigma = (1/2) \tilde \Sigma_0(\btheta_\b) + (1/2)\tilde \Sigma_1(\btheta_\b)$.
\end{corollary}

\noindent 
This scaling law shows the dependence of the balanced risk on the imbalance ratio and the synthetic data quality. The first term on the right-hand side quantifies the effect of bias introduced by synthetic data. It scales with the imbalance ratio $\rho_0 = (n_1 - n_0)/n_1$ for the minority group and the quality of synthetic data measured through $\|\nabla \calB^{(0)}(\btheta_\b)\|$ and $\|\nabla^2 \calB^{(0)}(\btheta_\b)\|$. The second term on the right-hand side is the variance term, which decreases with the total raw data size $n_\T = n_0+n_1$ and synthetic augmentation size $N$ increases. For the minority group $g=0$, the variance scaling is determined by the covariance matrices $\Sigma_0$ and $\tilde{\Sigma}_0$.

\medskip
\noindent
\textbf{Spurious correlation.} 
In the spurious correlation setting, the scaling law applies to disentangling the spurious feature $\mathcal{S}$ from the label $\mathcal{Y}$. We consider a similar setting as in Section~\ref{sec: theory oversampling}, while we add synthetic samples for oversampling then for augmentation. We have the following corollary for the spurious correlation problem. 

\begin{corollary}\label{cor: scaling law spu}
    Suppose the same conditions in Theorem~\ref{thm: synthetic data risk} hold. For any $\alpha \in [0, 1]$, 
    \begin{small}
    \begin{align*}
    \calR_\rw(\hat\btheta) - \calR_\rw(\btheta_\rw) &= O\qty(\sum_{y \in \mathcal{Y}} \qty{(1 - \alpha) \qty(1 - \frac{n_{\text{min}}}{n_{\text{maj}}}) + \alpha} \left\{\|\nabla \calB^{(y,-y\bgamma)}(\btheta_\rw)\|^2 + \|\nabla^2 \calB^{(y,-y\bgamma)}(\btheta_\rw)\|^3 \right\}) \\ 
        &\quad+ O\qty(\sum_{y \in \mathcal{Y}} \alpha \{\|\nabla \calB^{(y,y\bgamma)}(\btheta_\rw)\|^2 + \|\nabla^2 \calB^{(y,y\bgamma)}(\btheta_\rw)\|^3\}) \nonumber\\
        &\quad+ O_p\qty( (1 - \alpha)^2 \frac{\tr(\Sigma)}{n_{\text{maj}}} + \alpha^2 \frac{\tr(\Sigma')}{N |\mG|}),\nonumber
    \end{align*}
    \end{small}
    where $\Sigma = (1/4) \sum_{y \in \mathcal{Y}} \qty{({n_{\text{min}}}/{n_{\text{maj}}}) \Sigma_{(y,-y\bgamma)}(\btheta_\rw) + \qty(1 - ({n_{\text{min}}}/{n_{\text{maj}}})) \tilde \Sigma_{(y,-y\bgamma)}(\btheta_\rw)} + (1/4)$ $\sum_{y \in \mathcal{Y}} \Sigma_{(y,y\bgamma)}(\btheta_\rw)$, and $\Sigma' = (1/4) \sum_{g \in \mG} \tilde \Sigma_g(\btheta_\rw)$. 
\end{corollary}

\noindent
This scaling law shows that reducing spurious correlations requires high-quality synthetic data that aligns well with the reweighted distribution. It further emphasizes the role of $N$, i.e., the synthetic augmentation size, in variance reduction for the minority groups. As $N$ grows, the synthetic data variance $\mathrm{tr}(\tilde{\Sigma}_g)/N$ diminishes, improving the performance accordingly.

\section{Theory for Transformer as Synthetic Data Generators}
\label{sec: theory transformers}

The general theoretical results of Section~\ref{sec: theory} require that the bias between the real and synthetic distributions diminishes as $n_\T$ grows. In other words, synthetic oversampling and augmentation rely on high-quality synthetic data generators. In this section, we show that transformer-based generators indeed satisfy this property, provided that the in-context sample size grows.

\subsection{Data generating process}
\label{sec: dgp}

Recall that in Algorithm~\ref{alg}, we start with the balanced seed data from the raw data to generate the synthetic data. For the simplicity of presentation, in the theoretical analysis of this section, we directly start with a dataset that is already balanced and omit the ``seed'' superscript for the raw data. Suppose we have $n$ i.i.d.\ seed data $\mD_n =  (X_i, Y_i)_{i \in [n]}$. Given $\mD_n$ provided in-context as the prompt (seed data), we aim to show that the transformers can generate the high-quality synthetic data $(\tilde X_1, \tilde Y_1), (\tilde X_2, \tilde Y_2), \dots$ that mimic the distribution of $(X_1, Y_1)$. 

We introduce the data generating process for $(X_i, Y_i)_{i \in [n]}$. Suppose $X_i$ and $Y_i$ can take discretized values in a finite set $\mX$. Without loss of generality, we assume $\mX = [d]$, reflecting the treatment of tabular data as a collection of finite words for language models. We define a ``subject'' as any background information providing context for each tabular data. For instance, the data on heart failure could be specified by the subject ``heart failure''. Following the literature on word representations \citep{arora2015latent,khalife2021further,li2022brief}, we consider a Bayesian setting for the embeddings of $\mX$. Specifically, for discretized values $\mX = [d]$, the corresponding embeddings $\u_1, \dots, \u_d \in \R^r$ are modeled as i.i.d.\ realizations of $\text{Normal}(0, (1/r) I_r)$. Let $U = [\u_1, \ldots, \u_d]^\top$. 

We then introduce the data generating process for tabular data by combining generative models for covariates $X_i$ and discriminative models for labels given the covariates $Y_i\mid X_i$. Let $\mT$, $\mM$ be finite possible indices of subjects and discriminative functions. Specifically, $(X_i, Y_i) \in \mathcal{X}^2$ given subject $T=t \in \mT$ and discriminative function index $M=m \in \mM$ follows a multinomial distribution defined as,
\begin{align} \label{model: dgp}
\begin{split}
\P(X_i = x ; T = t, U, \eta) &\propto {\exp}({\eta^{-1} \langle \z^{(t)}, \u_x \rangle}), \\
\P(Y_i = y | X_i = x; M = m, U, \eta) &\propto {\exp}({\eta^{-1} \langle f^{(m)}(\u_x), \u_y \rangle}),
\end{split}
\end{align}
where $\eta > 0$ is the temperature, $\u_x \in \R^r$ is the embedding of the value $x \in \mX$, $\z^{(t)} \in \R^r$ is the representation of subject $t$, and $f^{(m)} : \R^r \to \R^r$ is a discriminative function with index $m$. We consider the case where $|\mM| \geq |\mT|$ for simplicity, and assume $|\mM| \lesssim d^\alpha$ for some positive constant $\alpha = O(1)$. We also remark that (\ref{model: dgp}) accommodates the subject-dependent label mechanism. This can be done by letting the discriminative index carry the subject, i.e., replacing $m \in \mM$ with $(m,t) \in \mM \times \mT$ and using functions $f^{(m,t)}$. Equivalently, when subjects select from a shared pool of functions, one may absorb $\mT$ into $\mM$, by e.g., redefining $\mM \leftarrow \mM \cup \mT$. A similar model has been commonly employed in the literature on language models \citep{mnih2007three, arora2015latent, arora2017simple, shi2017jointly, khalife2021further}. For simplicity, denote the joint distribution of $X_i$ and $Y_i$ by $P_{X_i,Y_i;T=t,M=m,U,\eta}$. 

We remark that our model differs from those in the in-context learning theory papers. For instance, \citet{xie2021explanation,garg2022can,zhang2024trained,bai2023transformers} considered settings where the seed data is directly contained in tokens. In contrast, our approach is motivated by the use of proprietary text-based LLMs with serialised tabular data, rather than training tabular transformers from scratch. Therefore, in our setting, it is more natural to consider tokens that \emph{indirectly} embed the seed data.

\subsection{Background on transformers}
\label{sec:transformer-background}

We provide some background on transformers. We follow the notation in \citet{akyurek2022learning,von2023transformers,bai2023transformers}. A transformer layer typically consists of two primary types of layers: a self-attention layer, which computes the attention scores between elements of the input sequence to capture relationships and dependencies, regardless of their distance in the sequence, and a feedforward layer, which consists of fully connected neural networks that process the output of the self-attention layer independently for each position in the sequence. More specifically, suppose the input is a sequence of $N$ tokens, denoted by $H = [\h_1,\h_2,\dots,\h_N]\in\R^{D\times N}$, where $\h_s\in\R^D$ is a column vector denoting the embedding of the $s$-th token, for $s\in[N]$. Here $D$ is the dimension of token, larger than the embedding dimension $r$. A self-attention layer with $J$ heads and parameters $\mu=\{(Q_j, K_j, V_j)\}_{j=1}^J$ takes $H$ as input, and outputs $\Attn_\mu(H)_s =  \h_s + \sum_{j \in [J]} \sum_{s' \in [N]} \sigma(\langle Q_j \h_s, K_j \h_{s'} \rangle) V_j \h_{s'}, \ \ s \in [N]$, where $\sigma$ is the ReLU activation function, and $Q, K, V \in \R^{D \times D}$. A feedforward layer with parameters $\nu=(W_1,W_2) \in \R^{D'\times D} \times \R^{D \times D'}$ takes $H$ as input, and outputs $\FFN_\nu(H)= H+W_2\sigma(W_1 H)$. Then a transformer layer is the composition of a self-attention layer $\Attn_\mu$ and a feedforward layer $\FFN_\nu$, i.e., $\TF_\psi = \FFN_\nu \circ \Attn_\mu$, with $\psi = (\mu, \nu)$. With a slight abuse of notation, we write multiple transformer layers as $\TF_{(\psi_1, \ldots, \psi_L)} =  \TF_{\psi_L} \circ \dots \circ \TF_{\psi_1}$. In all, the (multi-layer) transformation function maps $\R^{D\times N}$ to $\R^{D\times N}$

Next, we discuss input tokens for tabular data. Specifically, suppose tokens $(\h^Y_i)_{i \in [n]}$ correspond to $(Y_i)_{i \in [n]}$, and $(\h^X_i)_{i \in [n]}$ correspond to $(X_i)_{i \in [n]}$.  The input of the transformer is given by $H_n = [\h_1^X; \h_1^Y; \h_2^X; \h_2^Y; \dots; \h_n^X; \h_n^Y]$. We consider the composite type positional encoding, i.e., the forms of $\h_i^X$ and $\h_i^Y$ are given by $\h_i^X = (\u_{X_i}^\top, \zero^\top, \p_{2i-1,n}^\top)^\top$ and $\h_i^Y = (\u_{Y_i}^\top, \zero^\top, \p_{2i,n}^\top)^\top$, respectively. That is, for $i \in [n]$, the positional encoding $\p_{s,n} \in \R^4$ is defined as $\p_{s,n} = \qty(\lceil \frac{s}{2} \rceil, (s \operatorname{mod} 2), 2n, 1)^\top$, where $s \operatorname{mod} 2$ is $0$ for even $s$ and $1$ for odd $s$. 

Next, we discuss the output distribution. Given the initial input tokens $H_n \in \R^{D \times 2n}$, a transformer parameterized by $\Psi = (\psi_1, \ldots, \psi_L)$ sequentially outputs $\h_{2n+1}, \h_{2n+2}, \dots$ corresponding to the synthetic data from a categorical distribution given the last output from the transformer layers. More specifically, at each step $\ell \in \N$, given all previous tokens $H_n$ and $\h_{2n+1}, \h_{2n+2}$, $\dots$, $\h_{2n+\ell-1}$, the next token $\h_{2n+\ell}$ is given by $\h_{2n+\ell} = (\v_{2n+\ell}^\top, \zero^\top, \p_{2n+\ell,n}^\top)^\top$, where $\v_{2n+\ell} \in \R^r$ follows (\ref{model: dgp}) and satisfies a categorical distribution with softmax probability over all possible tokens $\u_1, \dots, \u_d \in \R^r$. For example, for odd $\ell$, $\P(\v_{2n+\ell} = \u_x) \propto {\exp}(\tau^{-1} \langle \u_x, (\tilde \h_{2n+\ell-1})_{1:r} \rangle)$, where $\tilde \h_{2n+\ell-1} =  (\TF_{\Psi}([H_n, \h_{2n+1}, \h_{2n+2}, \dots, \h_{2n+\ell-1}]))_{2n+\ell-1}$, and $\tau > 0$ is the temperature parameter. Since we expect the outputs $\h_{2n+1}, \h_{2n+2}, \h_{2n+3}, \h_{2n+4}, \dots$ from a transformer $\TF_\Psi$ correspond to $\tilde X_1, \tilde Y_1, \tilde X_2, \tilde Y_2, \dots$, we write the joint distribution of $(\v_{2n+2s-1}, \v_{2n+2s})$ as 
\begin{align} \label{eq: Q joint}
Q_{\tilde X_s,\tilde Y_s;\Psi,\tau,\mD_n}(x,y) =  \P(\v_{2n+2s-1} = \u_x, \v_{2n+2s} = \u_y).
\end{align}
Table~\ref{tab: tokens} shows the overview of the input and output tokens sequentially input to transformers. A table for the summary of notation can be found in Table~\ref{tab: notation transformer} in the Appendix.

\begin{table}[t!]
\centering
\caption{An overview of the input and output tokens.}
\label{tab: tokens}
\begin{tabular}{c|ccccccc|ccccc}
        \toprule
        & \multicolumn{7}{c|}{Input Tokens} & \multicolumn{4}{c}{Output Tokens} \\
        \midrule
        Index $s$ & $1$ & $2$ & $3$ & $4$ & $\cdots$ & $2n-1$ & $2n$ & $\cdots$ & $2n+2s-1$ & $2n+2s$ & $\cdots$\\
        \midrule
        Token & $\h_1^X$ & $\h_1^Y$ & $\h_2^X$ & $\h_2^Y$ & $\cdots$ & $\h_n^X$ & $\h_n^Y$ & $\cdots$ & $\h_{2n+2s-1}$ & $\h_{2n+2s}$ & $\cdots$\\
        \midrule
        Datum & $X_1$ & $Y_1$ & $X_2$ & $Y_2$ & $\cdots$ & $X_n$ & $Y_n$ & $\cdots$ & $\tilde X_s$ & $\tilde Y_s$ & $\cdots$\\
        $(\p_{s,n})_1$ & $1$ & $1$ & $2$ & $2$ & $\cdots$ & $n$ & $n$ & $\cdots$ & $n+s$ & $n+s$ & $\cdots$\\
        $(\p_{s,n})_2$ & $0$ & $1$ & $0$ & $1$ & $\cdots$ & $0$ & $1$ & $\cdots$ & $0$ & $1$ & $\cdots$\\
        \bottomrule
\end{tabular}
\end{table}

\subsection{Theory for transformers in generating high-quality synthetic data}
\label{sec:theory-transformers}

We next provide a formal theory to quantify the data quality generated by transformers. We first introduce two regularity conditions.

\begin{assumption}\label{asm: known subjects}
    Suppose the transformer knows the candidate of subject embeddings $\{\z^{(t)}\}_{t \in \mT}$ and candidate of functions $\{f^{(m)}\}_{m \in \mM} \subset \mF(L_0, r_0)$, where 
    \begin{align*}
    \mF(L_0, r_0) \subset \Big\{ \FFN_{\nu_{\pre,L_0}} \circ \FFN_{\nu_{\pre,L_0-1}} \circ \dots \circ \FFN_{\nu_{\pre,1}}: \nu_{\pre,\ell} \in \R^{r_0 \times r} \times \R^{r \times r_0} \text{ for } \ell \in [L_0] \Big\}
    \end{align*}
    with fixed $L_0, r_0 \in \N$. Suppose $\|\z^{(t)}\| = 1$ for any $t \in \mT$, and $\sup_{\u \in \mathbb{B}_r(\log d)} \|f^{(m)}(\u)\| \leq 1$, for identifiability.
\end{assumption}

\begin{assumption}\label{asm: separability}
Suppose there exist constants $\delta_\mT > 0$ and $\delta_\mM > 0$, such that $1 - \max_{t, t' \in \mT: t \neq t'}$ $\langle \z^{(t)}, \z^{(t')}\rangle \geq \delta_\mT$, and $\E\qty[\|f^{(m)}(\u_{X_1})\|^2 | T=t, U, \eta] - \max_{m': m' \neq m} \E[\langle f^{(m')}(\u_{X_1}), f^{(m)}(\u_{X_1}) \rangle |$ $T=t, U, \eta] \geq \delta_\mM$ hold with probability $1 - \exp\{-\Omega(\log^2 d)\}$ under the randomness of $U$.    
\end{assumption}

\noindent
Assumption~\ref{asm: known subjects} requires that the transformer has access to the candidate of subject embeddings and candidate of discriminative functions expressed as a multilayer neural networks, which reflects the fact that the transformer is pre-trained on vast amount of data. Assumption~\ref{asm: separability} is a margin-based identifiability condition. The first part requires subject embeddings to be angularly well separated, with a cosine margin $\delta_\mT$, which ensures that a unique subject best explains the covariates. The second part requires that different functions remain distinguishable in expectation under the random token embeddings. Together these margins guarantee the identifiability and stability of the argmax steps in our transformer construction.

Here we consider the regime where $d$, $r$ both grow with $n$ and $r = o(\log d)$. 

\begin{theorem}\label{thm: generative and discriminative}
    Suppose Assumptions~\ref{asm: known subjects} and \ref{asm: separability} hold. Fix any $r \in \N$, $(\z^{(t)})_{t \in \mT}$ and $(f^{(m)})_{m \in \mM}$. Choose $n$ and $d$ sufficiently large, such that
    \begin{align} \label{eq: sample margin condition}
    C_0 \qty( \eta \sqrt{r} \frac{\log^2 d}{\sqrt{n}} + \eta r \frac{\log d}{\sqrt{d}} ) &< \delta_\mM \wedge \delta_\mT
    \end{align}
    holds for some constant $C_0 > 0$. Then, there exist transformer layers $\TF_{\Psi^*}$ with $O(|\mM|)$ attention heads, such that, for any $t \in \mT$, $m \in \mM$, and $\eta \geq r^{-1/2} \log d$, 
    \begin{align} \label{eq:transformer-upperbound}
    \begin{split}
    \min_{\tau > 0} \E\left[ \KL\left( P_{X_1,Y_1;T=t,M=m,U,\eta} \| Q_{\tilde X_s, \tilde Y_s;\Psi^*,\tau,\mD_n} \right) | T=t,M=m,\eta \right] \\
    = \exp\left\{ -\Omega\qty(\frac{\delta_\mT^2 \wedge \delta_\mM^2}{\eta^2 r} n) \right\}
    \end{split}
    \end{align}
    holds for all $s \in \N$.    
\end{theorem}

\noindent
Theorem~\ref{thm: generative and discriminative} shows an exponential convergence of the KL divergence between the true joint distribution $P_{X_1,Y_1;T=t,M=m,U,\eta}$ and the transformer-generated distribution $Q_{\tilde X_s,\tilde Y_s;\Psi^*,\tau,\mD_n}$ under the separability margins. More specifically, for a suitable temperature, e.g., $\tau=\eta$, there exists a transformer $\Psi^*$, such that its generated distribution approximates the true distribution with the KL error equal to the right-hand-side of (\ref{eq:transformer-upperbound}). This error decays exponentially in the number of in-context seed samples $n$, and improves by larger separability margins $(\delta_{\mM}, \delta_{\mT})$ and a smaller latent dimension $r$. 
This result holds in the in-context regime, where $r=o(\log d)$, $\eta \ge r^{-1/2}\log d$, and $n,d$ are sufficiently large so that (\ref{eq: sample margin condition}) holds. This requires $n$ to dominate $(\eta^{2}r)\log^{4}d/(\delta_\mM \wedge \delta_\mT)^2$, and $d$ to dominate $(\eta^{2}r^{2})\log^{2}d/(\delta_\mM \wedge \delta_\mT)^2$, which ensure the separability margins exceed the stochastic error terms. In practice, $d$ corresponds to the size of a dictionary or token set, and is typically large in our tabular-tokenization scheme.

We remark that, Theorems~\ref{thm: synthetic data risk} and \ref{thm: scaling law} require that $\|\nabla \calB^{(g)}(\btheta_\b)\| \vee \|\nabla^2 \calB^{(g)}(\btheta_\b)\| = o(1)$, and show that the upper bound of the excess risk depends on $\|\nabla \calB^{(g)}(\btheta_\b)\|$. Recall that, when the loss function is sufficiently smooth, we have $\|\nabla \calB^{(g)}(\btheta_\b)\| \vee \|\nabla^2 \calB^{(g)}(\btheta_\b)\| \lesssim d_{\operatorname{TV}}(\mathcal{D}_\r, \mathcal{D}_\syn) \leq \sqrt{(1/2) \KL(\mathcal{D}_\r, \mathcal{D}_\syn)}$, where the last relation holds due to Pinsker's inequality. By Theorem~\ref{thm: generative and discriminative}, $\|\nabla \calB^{(g)}(\btheta_\b)\| \vee \|\nabla^2 \calB^{(g)}(\btheta_\b)\|$ is small when $d$ and $n$ are both large. Therefore, the capacity of transformers to produce high-quality synthetic data supports the use of LLMs for synthetic oversampling and augmentation. 

We also remark that, Theorem~\ref{thm: generative and discriminative} essentially shows that there \emph{exists} a transformer capable of learning the target distribution from in-context examples at an exponential rate, thereby generating high-quality synthetic data. Our analysis focuses on the representational capacity of the transformer architecture, rather than the optimization dynamics that guide a pretrained model to attain optimal parameters. This type of existence statement is standard in the literature on transformer expressivity and in-context learning \citep{akyurek2022learning, mahankali2023one, bai2023transformers, zhang2024trained}.  We give a specific construction of the desired transformer in Proposition~\ref{prop: transformer as good data generator}, and establish the generative and discriminative capacity of this constructed transformer in Theorem~\ref{thm: transformer as good data generator 2}, as shown in Section~\ref{supp-sec: theory transformers proof} of the Appendix.

\section{Numerical Experiments}
\label{sec:experiments}

We carry out numerical experiments to demonstrate the effectiveness of our LLM-based synthetic method. In particular, for oversampling, we empirically compare our approach to some common alternative oversampling solutions. For augmentation, we empirically verify the polynomial decay rates established in our theory.

\subsection{Experiment setup}
\label{sec:setup}

We consider one simulation dataset and three real datasets, each with a binary outcome, for our numerical experiments. For each dataset, we choose one class as the minority group and the other as the majority group, and we obtain a subset of samples accordingly. We also choose a possible spurious feature, and specify the corresponding minority and majority groups. 

\noindent \textbf{Craft}. This simulated dataset involves $8000$ samples and 9 features. 
\begin{align*}
& X_1 \sim \mathcal{N}(0,1), \quad
X_2 \sim \mathcal{N}(0,1), \quad
X_3 = 0.5X_1 + 0.3X_2 + \epsilon_3,  \epsilon_3 \sim \mathcal{N}(0, 0.5), \\
& X_4 = X_1 \cdot \epsilon_4,  \epsilon_4 \sim \mathcal{N}(0,1), \quad
X_5 = 0.5X_3 + \epsilon_5,  \epsilon_5 \sim \mathcal{N}(0,1), \quad
X_6 \sim \text{Uniform}\{-1, 1\}, \\
& X_7 \sim \mathcal{N}(0,1), \quad
X_8 = X_2 \cdot X_3, \quad
X_9 = X_1 \cdot X_2, 
\end{align*}
Moreover, we set $Z_i = f(X_1^i, X_2^i, X_3^i, X_9^i) + \epsilon_i$, where {$f(X_1, X_2, X_3, X_9) = 1.5 + 0.7X_1 - 0.6X_2 + 0.8X_3 + 0.4X_9$, and  $\epsilon_i \sim \text{Normal}(0, 1)$,} and generate the binary outcome as $ Y_i = \Id\left\{ Z_i > \text{median}(\{Z_j\}_{j=1}^{n}) \right\}.$ We choose $X_6$ as a potential spurious feature, where the minority group includes samples with $\{X_6 = 1, Y = 0\}$ or $\{X_6 = -1, Y = 1\}$, and the majority group includes samples with $\{X_6 = -1, Y = 0\}$ or $\{X_6 = 1, Y = 1\}$.

\noindent \textbf{Gender}. This dataset \citep{gender} involves $5001$ samples, $7$ features related to facial characteristics, and one binary outcome {\tt gender}. We choose ${\tt long hair}$ as a potential spurious feature, where the minority group includes female subjects without long hair, $\{\text{\tt gender} = 1, \text{\tt long hair} = 0\}$, and male subjects with long hire, $\{\text{\tt gender} = 0, \text{\tt long hair} = 1\}$, and the majority group includes female subjects with long hair, and male subjects without long hair, $\{\text{\tt gender} = 0, \text{\tt long hair} = 0\}$.

\noindent \textbf{Diabetes}. This dataset \citep{diabetes} collects ten years of clinical care records from 1999 to 2008 at $130$ U.S.\ hospitals and delivery networks to predict whether the patients would be readmitted to hospital. It involves $101766$ samples and $47$ features. We choose {\tt gender} as a possible spurious feature, where the minority group includes female patients who are readmitted, $\{ \text{\tt readmitted} = \text{YES}, \text{\tt gender} = \text{Female}\}$, and male patients who are not readmitted, $\{ \text{\tt readmitted} = \text{NO}, \text{\tt gender} = \text{Male}\}$, {and the majority group includes male patients who are readmitted, $\{ \text{\tt readmitted} = \text{YES}, \text{\tt gender} = \text{Male}\}$,} and female patients who are not readmitted, $\{ \text{\tt readmitted} = \text{NO}, \text{\tt gender} = \text{Female}\}$.

\noindent \textbf{Adult}. This dataset \citep{census} collects the demographic and employment  information from the 1994 U.S.\ Census to predict whether an individual's annual income exceeds $\$50K$. It involves $48842$ samples and $14$ features. There is no obvious spurious feature in this dataset. 

For each dataset, we randomly leave out $30\%$ samples as the testing data, and use the rest $70\%$ as the raw training data. We consider the synthetic data generation by fine-tuning a GPT-2 language model. We use all the training data for GPT-2 fine-tuning. To set up the imbalanced data, we choose the sample size for the minority group as $n_{\text{min}} = 100$, and vary the sample size of the majority group so that the ratio ${n_{\text{maj}}} / {n_{\text{min}}} = \{1, 2, \ldots, 10\}$. We then randomly sample $n_{\text{min}}$ and $n_{\text{maj}}$ observations from the training data. For synthetic oversampling, we add $n_{\text{maj}} - n_{\text{min}}$ generated samples by LLM to the minority group. For synthetic augmentation, we further add $N$ generated samples to both the minority and majority groups, with $N = \{400, 800, \ldots, 3600\}$. We choose the weight parameter $\alpha = 1/3$ for the empirical risk in (\ref{eq: general R with data augmentation}). 

For both imbalanced classification and spurious correlation, we choose the binary outcome classification as the downstream task, and employ XGBoost \citep{chen16} as the classifier. We have experimented with other classifiers such as random forests and obtained similar qualitative results. We evaluate the performance on the test data using the balanced cross-entropy loss, and the cross-entropy loss for the minority group \citep{brodersen10}:
\vspace{-0.01in}
\begin{align*}
\ell_{\text{balanced}} 
&= \frac{1}{| \mathcal{G} |} \sum_{g \in \mathcal{G}} \frac{1}{n_g} \sum_{i \in [n_g]} 
\left\{ - y_i \log \hat{y}_i - (1 - y_i) \log (1 - \hat{y}_i) \right\}, \\
\ell_{\text{minority}} 
&= \frac{1}{n_{g_{\text{min}}}} \sum_{i \in [n_{g_{\text{min}}}]} 
\left\{ - y_i \log \hat{y}_i - (1 - y_i) \log (1 - \hat{y}_i) \right\}.
\end{align*}

We compare our method with a number of alternative solutions, including the method with no oversampling (denoted as RAW), random oversampling \citep[][denoted as ROS]{he09}, synthetic minority oversampling technique \citep[][denoted as SMOTE]{chawla02smote}, and adaptive synthetic sampling \citep[][denoted as ADASYN]{he2008adasyn}. For our proposed method, we consider two versions, one only adding synthetic samples to the minority group (denoted as LLM oversampling), and the other further adding synthetic samples to all groups of data after oversampling the minority group (denoted as LLM oversampling + augmentation). 

We report more details on fine-tuning of the GPT-2 model in Section~\ref{supp-sec: add-gpt2}, and report the details and results using a pretrained GPT-4 model in Section~\ref{supp-sec: add-gpt4} of the Appendix.

\subsection{Synthetic oversampling}
\label{sec:exp-oversampling}

We first investigate the empirical performance of synthetic oversampling for the imbalanced classification setting and the spurious correlation setting, respectively. 

For imbalanced classification, Figure~\ref{fig:ratio-imb-xgb} reports the two cross-entropy losses for the four datasets based on 5 data replications, with ${n_{\text{maj}}} / {n_{\text{min}}} = \{1, 2, \ldots, 10\}$, whereas Table~\ref{tab:imb-vary-ratio} reports the results when ${n_{\text{maj}}} / {n_{\text{min}}} = 6$. For the LLM oversampling method, we add $n_{\text{maj}} - n_{\text{min}}$ synthetic samples to the minority group, and for the LLM oversampling + augmentation method, we further add $N = 600$ synthetic samples to all groups after oversampling. From both the plot and the table, we observe that our LLM-based synthetic oversampling method clearly outperforms all the alternative solutions. When the ${n_{\text{maj}}} / {n_{\text{min}}}$ ratio increases, the improvement becomes more substantial. Moreover, oversampling coupled with additional synthetic augmentation could further improve the performance moderately. 

\begin{figure}[t!]
\centering
    \begin{subfigure}{0.4\textwidth}
        \centering
        \includegraphics[width=\linewidth,height=1.18in]{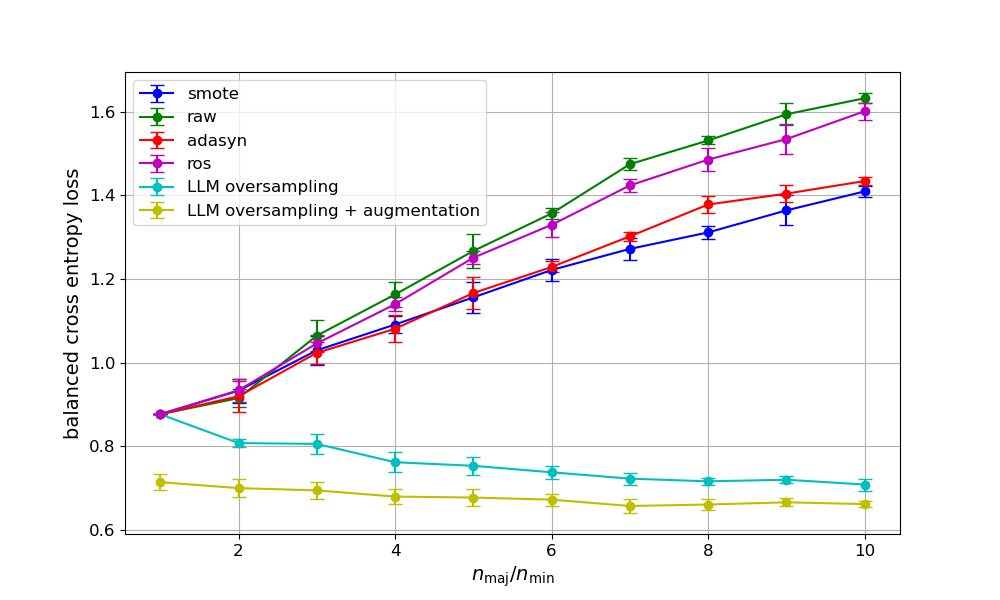}
        \caption{Craft - balanced cross-entropy loss}
    \end{subfigure}
    \begin{subfigure}{0.4\textwidth}
        \centering
        \includegraphics[width=\linewidth,height=1.18in]{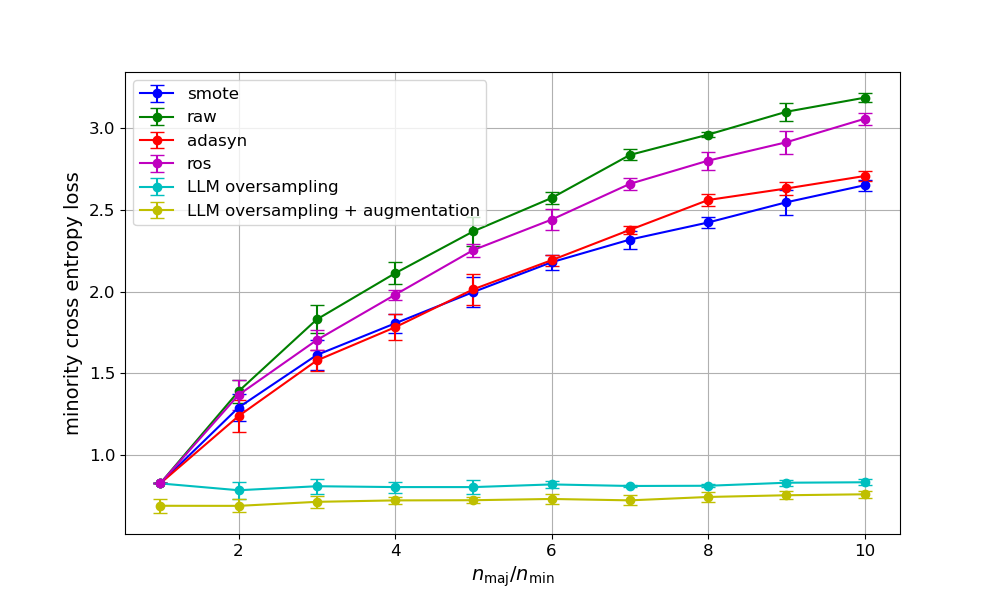}
        \caption{Craft - minority cross-entropy loss}
    \end{subfigure}
    \begin{subfigure}{0.4\textwidth}
        \centering
        \includegraphics[width=\linewidth,height=1.18in]{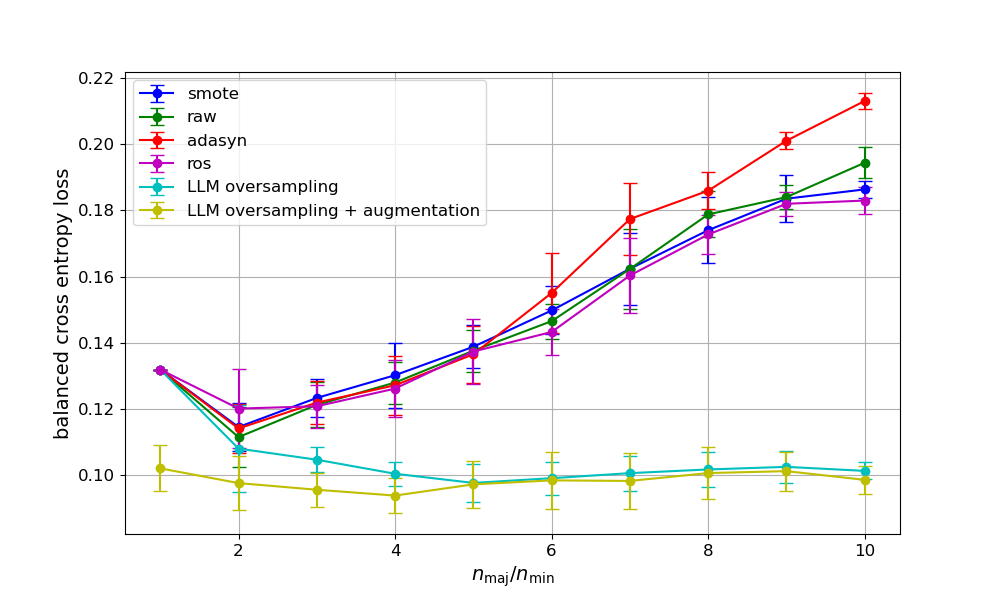}
        \caption{Gender - balanced cross-entropy loss}
    \end{subfigure}
    \begin{subfigure}{0.4\textwidth}
        \centering
        \includegraphics[width=\linewidth,height=1.18in]{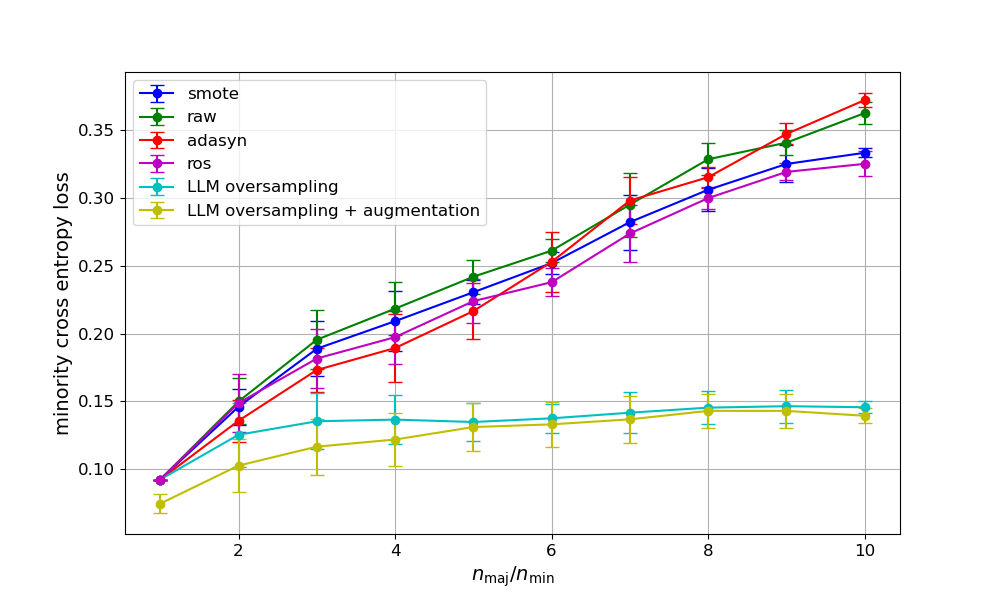}
        \caption{Gender - minority cross-entropy loss}
    \end{subfigure}
    \begin{subfigure}{0.4\textwidth}
        \centering
        \includegraphics[width=\linewidth,height=1.18in]{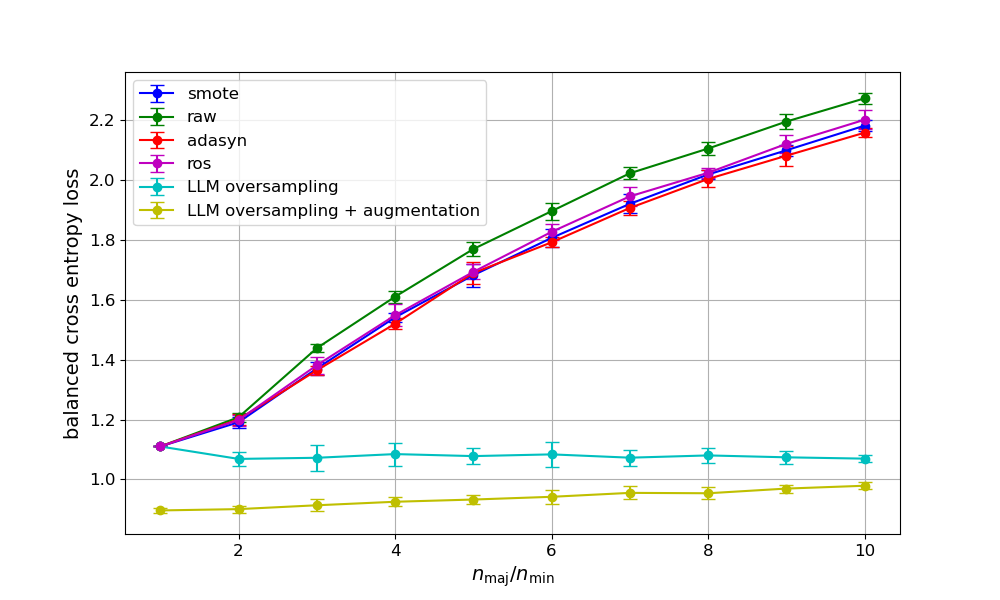}
        \caption{Diabetes - balanced cross-entropy loss}
    \end{subfigure}
    \begin{subfigure}{0.4\textwidth}
        \centering
        \includegraphics[width=\linewidth,height=1.18in]{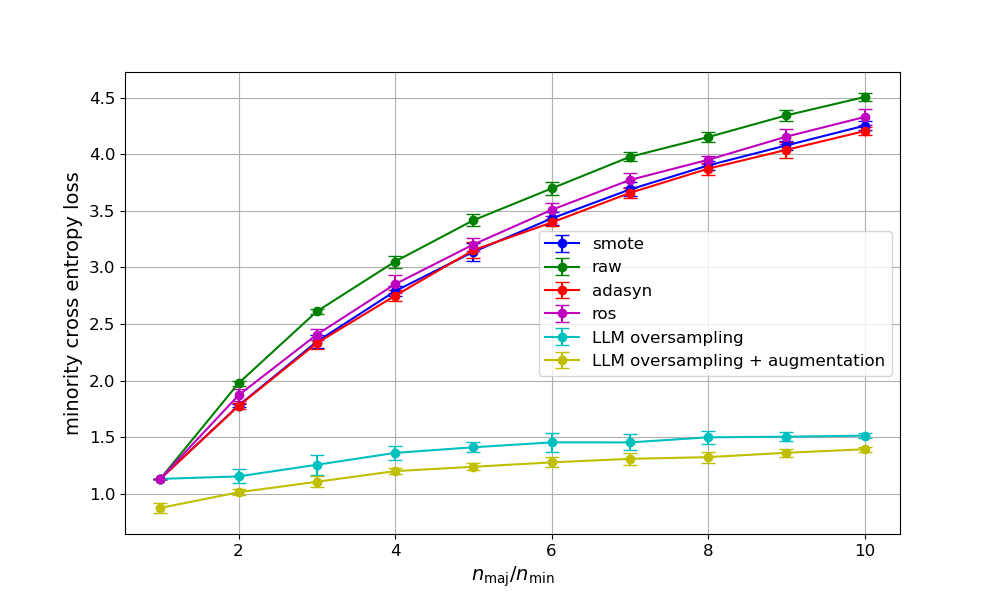}
        \caption{Diabetes - minority cross-entropy loss}
    \end{subfigure}
    \begin{subfigure}{0.4\textwidth}
        \centering
        \includegraphics[width=\linewidth,height=1.18in]{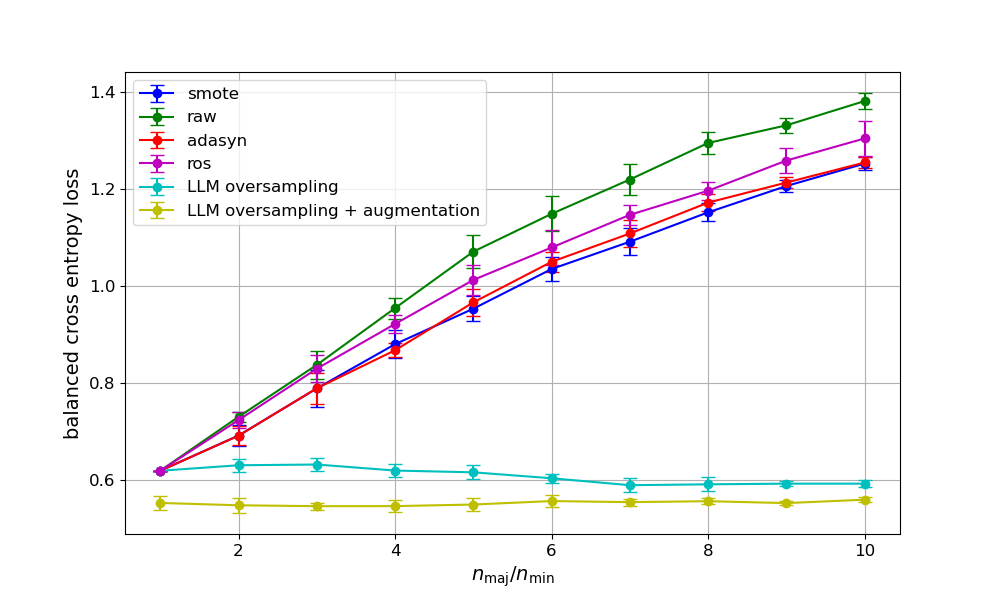}
        \caption{Adult - balanced cross-entropy loss}
    \end{subfigure}
    \begin{subfigure}{0.4\textwidth}
        \centering
        \includegraphics[width=\linewidth,height=1.18in]{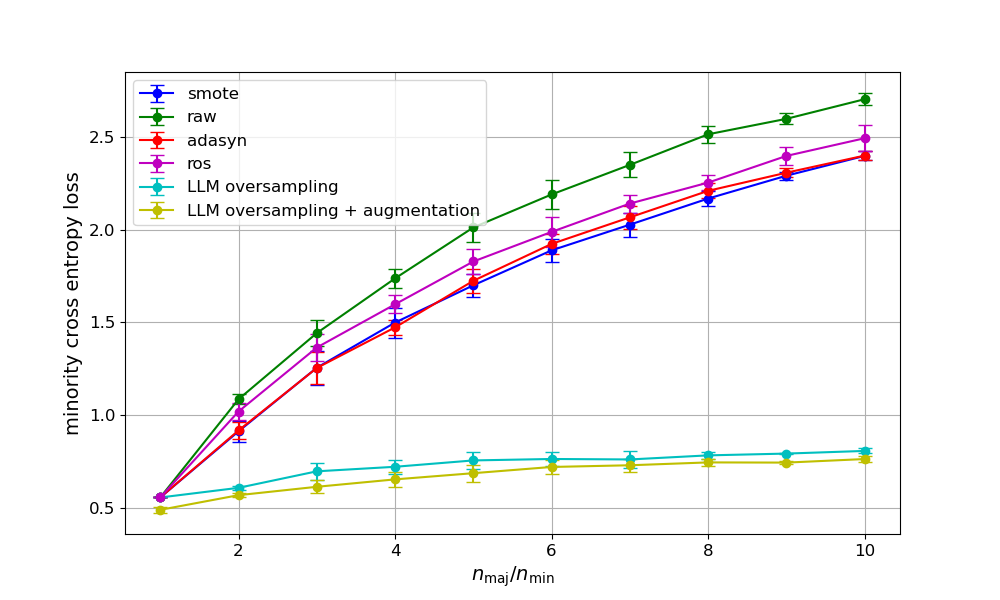}
        \caption{Adult - minority cross-entropy loss}
    \end{subfigure}
\caption{Synthetic oversampling for imbalanced classification. Six methods are compared: no oversampling (RAW), random oversampling (ROS), synthetic minority oversampling technique (SMOTE), adaptive synthetic sampling (ADASYN), the proposed synthetic oversampling (LLM oversampling), and the proposed synthetic oversampling with additional synthetic augmentation (LLM oversampling + augmentation). Two cross-entropy losses are reported as the sample size of the majority and minority groups ${n_{\text{maj}}} / {n_{\text{min}}}$ varies.}
\label{fig:ratio-imb-xgb}
\end{figure}

\begin{table}[t!]
\centering
\caption{Synthetic oversampling for imbalanced classification. Six methods are compared: no oversampling (RAW), random oversampling (ROS), synthetic minority oversampling technique (SMOTE), adaptive synthetic sampling (ADASYN), the proposed synthetic oversampling (LLM oversampling), and the proposed synthetic oversampling with additional synthetic augmentation (LLM oversampling + augmentation). Two cross-entropy losses are reported with the sample size of the majority and minority groups ${n_{\text{maj}}} / {n_{\text{min}}} = 6$.}
\label{tab:imb-vary-ratio}
\resizebox{\textwidth}{!}{%
\begin{tabular}{l l c c c c c c}
\toprule
\textbf{Data Name} & \textbf{Metric} & \textbf{RAW} & \textbf{ROS} & \textbf{SMOTE} & \textbf{ADASYN} & \textbf{LLM bal} & \textbf{LLM bal+aug} \\
\midrule
\multirow{2}{*}{craft} & Balanced & 1.357 $\pm$ 0.013 & 1.330 $\pm$ 0.030 & 1.222 $\pm$ 0.027 & 1.229 $\pm$ 0.013 & 0.738 $\pm$ 0.016 & \textbf{0.673 $\pm$ 0.014} \\
                       & Min      & 2.574 $\pm$ 0.036 & 2.441 $\pm$ 0.063 & 2.180 $\pm$ 0.046 & 2.192 $\pm$ 0.033 & 0.817 $\pm$ 0.020 & \textbf{0.728 $\pm$ 0.030} \\
\midrule
\multirow{2}{*}{gender} & Balanced & 0.147 $\pm$ 0.005 & 0.143 $\pm$ 0.007 & 0.150 $\pm$ 0.007 & 0.155 $\pm$ 0.012 & 0.099 $\pm$ 0.005 & \textbf{0.098 $\pm$ 0.008} \\
                        & Min      & 0.261 $\pm$ 0.009 & 0.238 $\pm$ 0.010 & 0.252 $\pm$ 0.008 & 0.253 $\pm$ 0.022 & 0.138 $\pm$ 0.011 & \textbf{0.133 $\pm$ 0.017} \\
\midrule
\multirow{2}{*}{diabetes} & Balanced & 1.897 $\pm$ 0.028 & 1.827 $\pm$ 0.028 & 1.807 $\pm$ 0.029 & 1.793 $\pm$ 0.017 & 1.084 $\pm$ 0.042 & \textbf{0.942 $\pm$ 0.023} \\
                          & Min      & 3.697 $\pm$ 0.058 & 3.508 $\pm$ 0.058 & 3.432 $\pm$ 0.057 & 3.396 $\pm$ 0.033 & 1.454 $\pm$ 0.083 & \textbf{1.278 $\pm$ 0.045} \\
\midrule
\multirow{2}{*}{adult} & Balanced & 1.149 $\pm$ 0.036 & 1.080 $\pm$ 0.036 & 1.036 $\pm$ 0.025 & 1.050 $\pm$ 0.021 & 0.604 $\pm$ 0.009 & \textbf{0.557 $\pm$ 0.013} \\
                       & Min      & 2.189 $\pm$ 0.077 & 1.987 $\pm$ 0.083 & 1.888 $\pm$ 0.062 & 1.923 $\pm$ 0.052 & 0.765 $\pm$ 0.037 & \textbf{0.722 $\pm$ 0.040} \\
\bottomrule
\end{tabular}
}
\end{table}

For spurious correlation, Figure~\ref{fig:ratio-spurious-xgb} reports the two cross-entropy losses based on 5 data replications, with ${n_{\text{maj}}} / {n_{\text{min}}} = \{1, 2, \ldots, 10\}$, whereas Table~\ref{tab:spurious-corr-vary-ratio} reports the results when ${n_{\text{maj}}} /$ $ {n_{\text{min}}} = 6$. There is no obvious spurious feature for the Adult dataset, so we only report the results for the other three datasets. Again, we observe that our LLM-based synthetic oversampling method clearly outperforms all the alternative solutions, and the improvement is more substantial when the minority group has a smaller sample size compared to the majority group. 

\begin{figure}[t!]
\centering
    \begin{subfigure}{0.4\textwidth}
        \centering
        \includegraphics[width=\linewidth,height=1.18in]{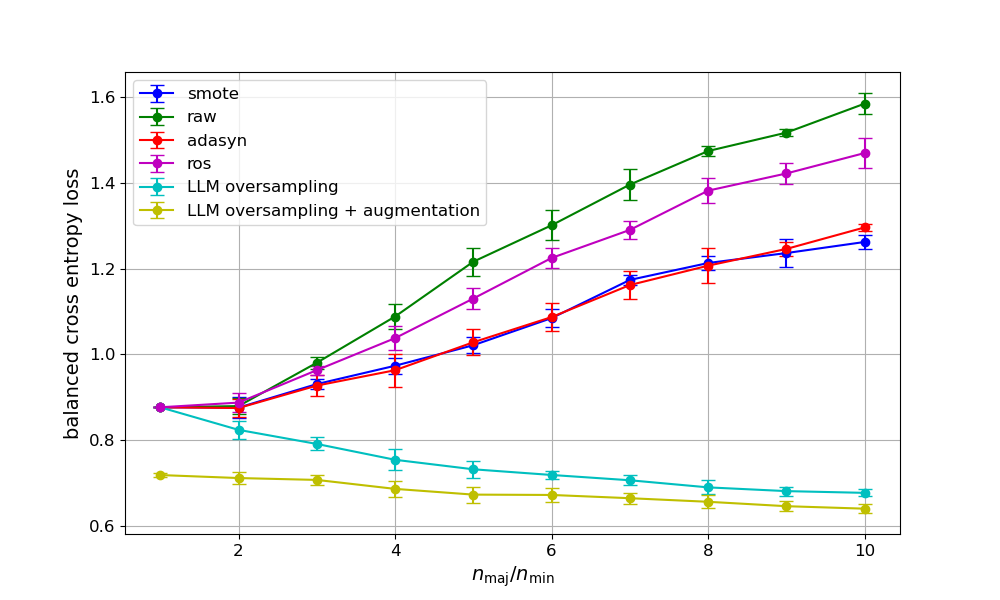}
        \caption{Craft - balanced cross-entropy loss}
    \end{subfigure}
    \begin{subfigure}{0.4\textwidth}
        \centering
        \includegraphics[width=\linewidth,height=1.18in]{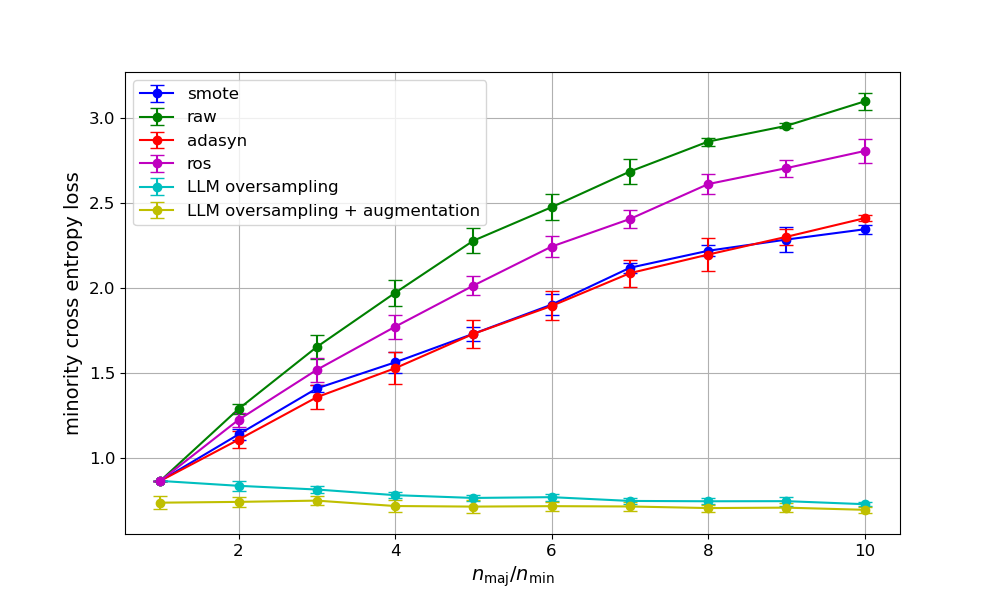}
        \caption{Craft - minority cross-entropy loss}
    \end{subfigure}
    \begin{subfigure}{0.4\textwidth}
        \centering
        \includegraphics[width=\linewidth,height=1.18in]{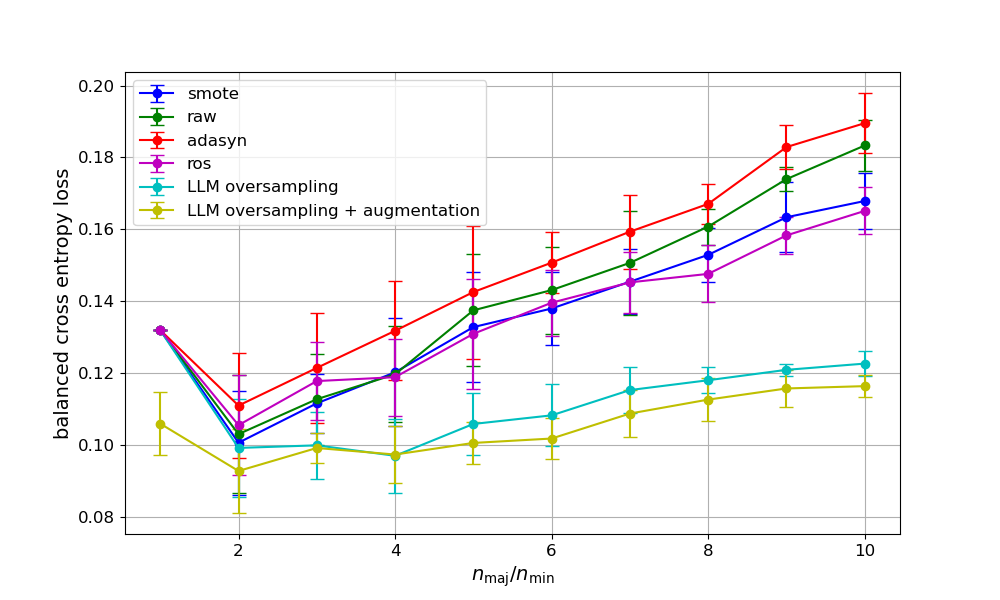}
        \caption{Gender - balanced cross-entropy loss}
    \end{subfigure}
    \begin{subfigure}{0.4\textwidth}
        \centering
        \includegraphics[width=\linewidth,height=1.18in]{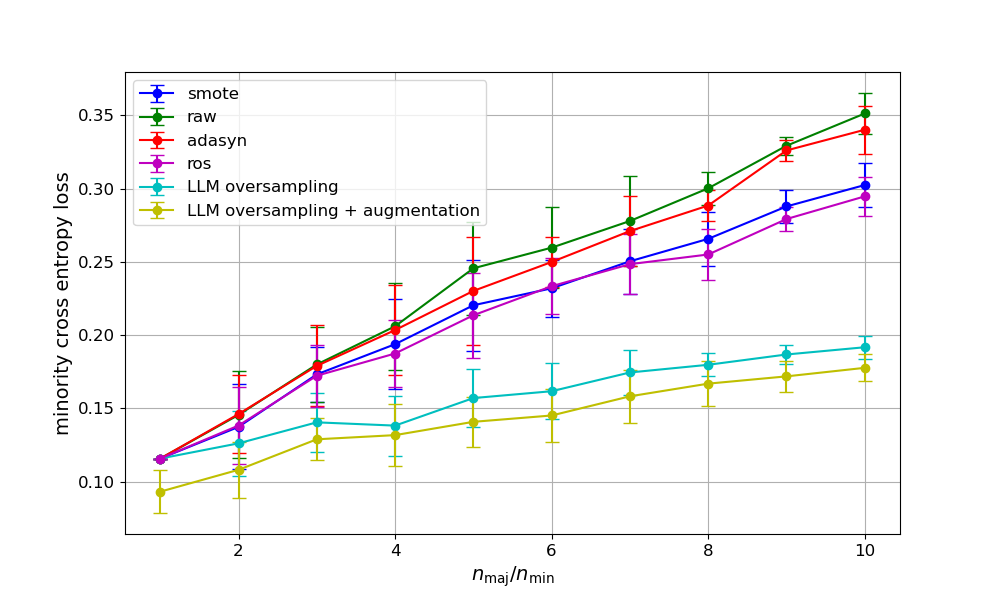}
        \caption{Gender - minority cross-entropy loss}
    \end{subfigure}
    \begin{subfigure}{0.4\textwidth}
        \centering
        \includegraphics[width=\linewidth,height=1.18in]{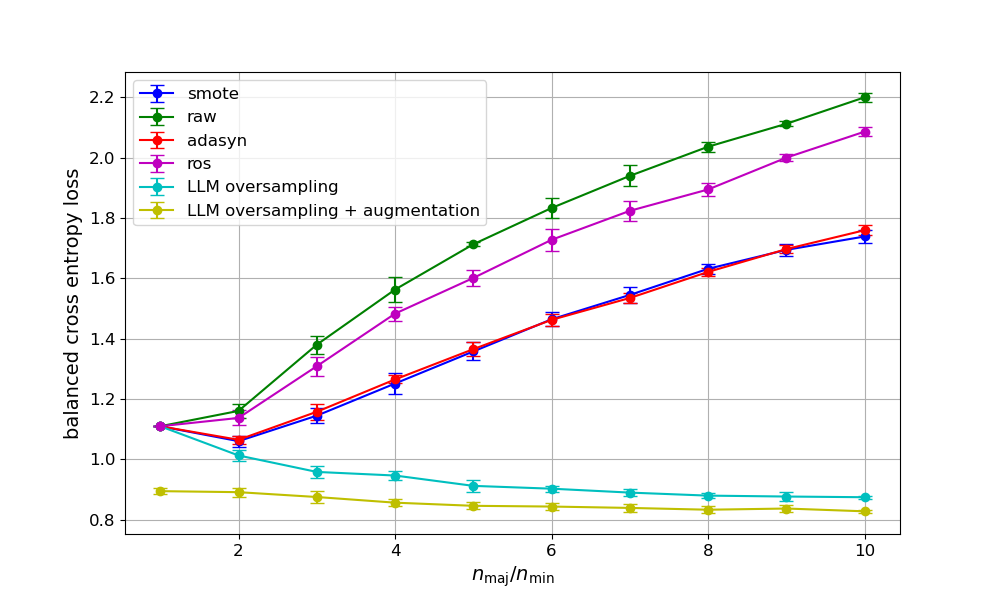}
        \caption{Diabetes - balanced cross-entropy loss}
    \end{subfigure}
    \begin{subfigure}{0.4\textwidth}
        \centering
        \includegraphics[width=\linewidth,height=1.18in]{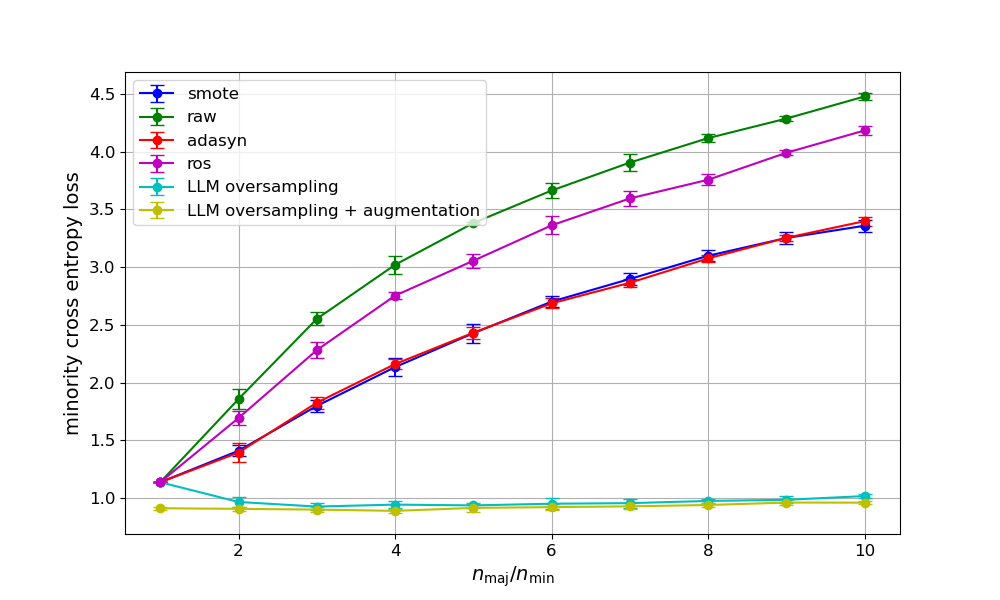}
        \caption{Diabetes - minority cross-entropy loss}
    \end{subfigure}
\caption{Synthetic oversampling for spurious correlation with a similar setup as Figure~\ref{fig:ratio-imb-xgb}.}
\label{fig:ratio-spurious-xgb}
\end{figure}

\begin{table}[t!]
\centering
\caption{Synthetic oversampling for spurious correlation with a similar setup as Table~\ref{tab:imb-vary-ratio}. Six methods are compared: no oversampling (RAW), random oversampling (ROS), synthetic minority oversampling technique (SMOTE), adaptive synthetic sampling (ADASYN), the proposed synthetic oversampling (LLM oversampling), and the proposed synthetic oversampling with additional synthetic augmentation (LLM oversampling + augmentation). Two cross-entropy losses are reported with ${n_{\text{maj}}}/{n_{\text{min}}} = 6$.}
\label{tab:spurious-corr-vary-ratio}
\resizebox{\textwidth}{!}{%
\begin{tabular}{l l c c c c c c}
\toprule
\textbf{Data Name} & \textbf{Metric} & \textbf{RAW} & \textbf{ROS} & \textbf{SMOTE} & \textbf{ADASYN} & \textbf{LLM bal} & \textbf{LLM bal+aug} \\
\midrule
\multirow{2}{*}{craft} & Balanced & 1.301 $\pm$ 0.035 & 1.225 $\pm$ 0.024 & 1.085 $\pm$ 0.021 & 1.087 $\pm$ 0.033 & 0.719 $\pm$ 0.009 & \textbf{0.673 $\pm$ 0.016} \\
                       & Min      & 2.476 $\pm$ 0.078 & 2.244 $\pm$ 0.060 & 1.902 $\pm$ 0.064 & 1.894 $\pm$ 0.085 & 0.767 $\pm$ 0.020 & \textbf{0.714 $\pm$ 0.026} \\
\midrule
\multirow{2}{*}{gender} & Balanced & 0.143 $\pm$ 0.012 & 0.140 $\pm$ 0.009 & 0.138 $\pm$ 0.010 & 0.151 $\pm$ 0.008 & 0.108 $\pm$ 0.009 & \textbf{0.102 $\pm$ 0.006} \\
                        & Min      & 0.260 $\pm$ 0.028 & 0.233 $\pm$ 0.019 & 0.232 $\pm$ 0.019 & 0.250 $\pm$ 0.017 & 0.162 $\pm$ 0.019 & \textbf{0.145 $\pm$ 0.018} \\
\midrule
\multirow{2}{*}{diabetes} & Balanced & 1.833 $\pm$ 0.032 & 1.728 $\pm$ 0.037 & 1.464 $\pm$ 0.023 & 1.463 $\pm$ 0.020 & 0.903 $\pm$ 0.009 & \textbf{0.844 $\pm$ 0.011} \\
                          & Min      & 3.663 $\pm$ 0.067 & 3.361 $\pm$ 0.079 & 2.701 $\pm$ 0.044 & 2.686 $\pm$ 0.043 & 0.953 $\pm$ 0.046 & \textbf{0.923 $\pm$ 0.026} \\
\bottomrule
\end{tabular}
}
\end{table}

\subsection{Synthetic augmentation}
\label{sec:exp-scaling-law}

\begin{figure}[t!]
\centering
    \begin{subfigure}{0.4\textwidth}
        \centering
        \includegraphics[width=\linewidth,height=1.18in]{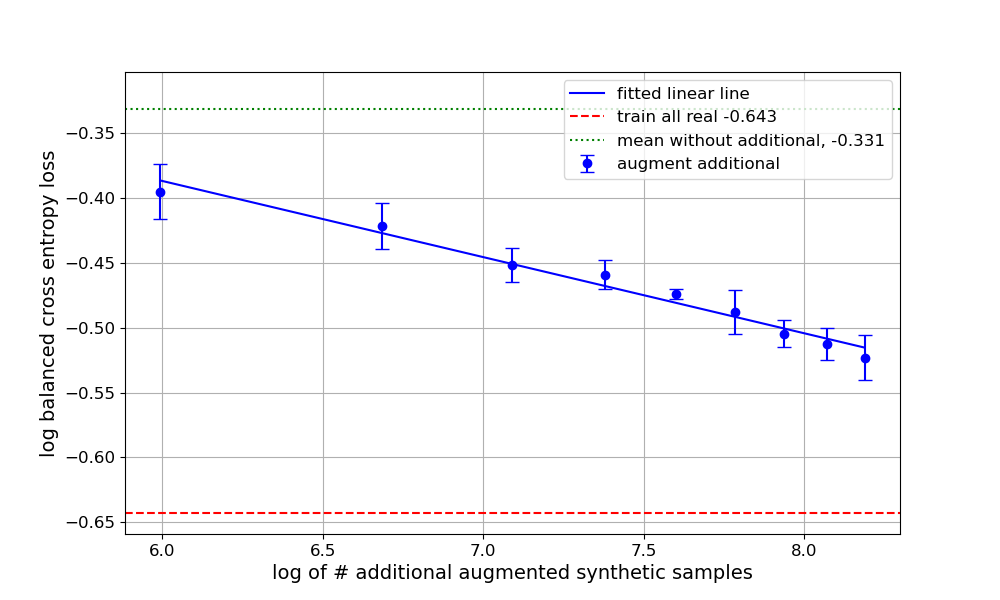}
        \caption{Craft - balanced cross-entropy loss}
    \end{subfigure}
    \begin{subfigure}{0.4\textwidth}
        \centering
        \includegraphics[width=\linewidth,height=1.18in]{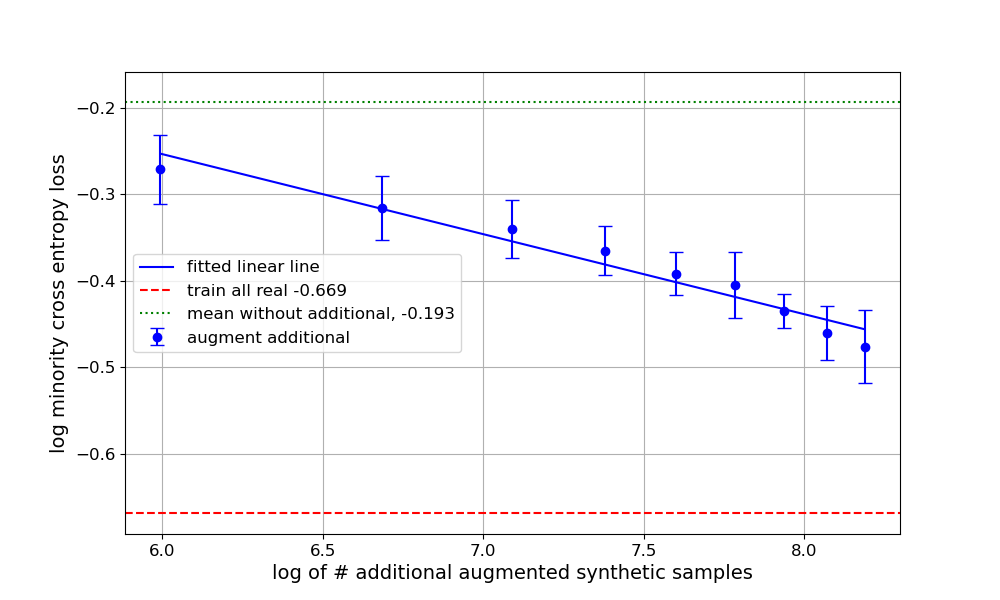}
        \caption{Craft - minority cross-entropy loss}
    \end{subfigure}
    \begin{subfigure}{0.4\textwidth}
        \centering
        \includegraphics[width=\linewidth,height=1.18in]{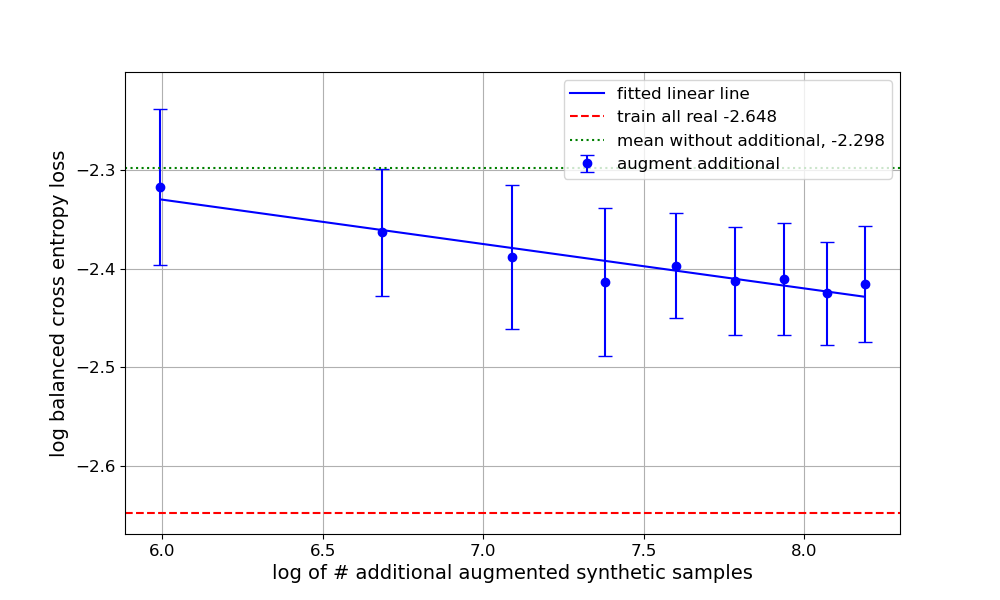}
        \caption{Gender - balanced cross-entropy loss}
    \end{subfigure}
    \begin{subfigure}{0.4\textwidth}
        \centering
        \includegraphics[width=\linewidth,height=1.18in]{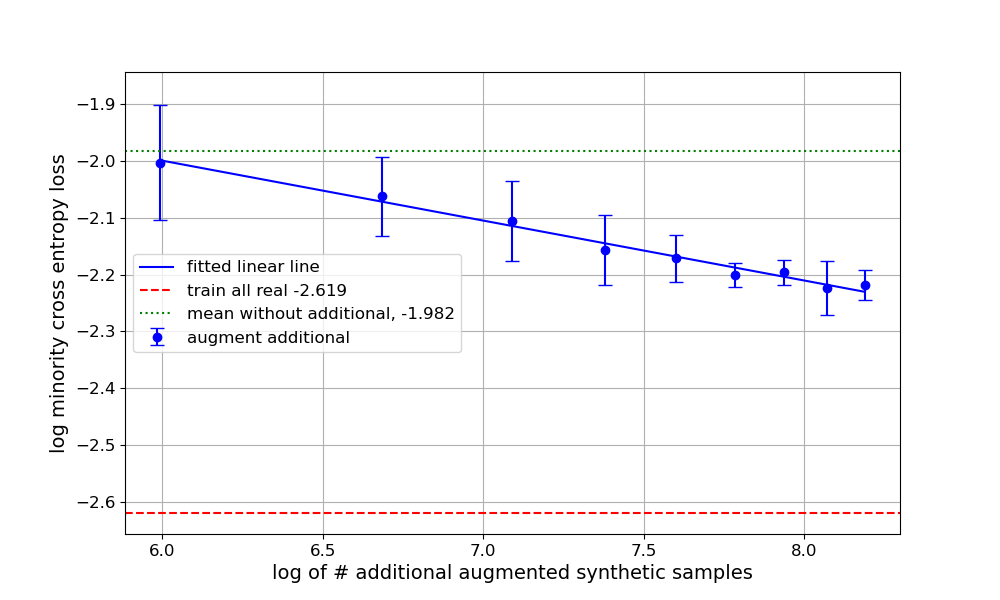}
        \caption{Gender - minority cross-entropy loss}
    \end{subfigure}
    \begin{subfigure}{0.4\textwidth}
        \centering
        \includegraphics[width=\linewidth,height=1.18in]{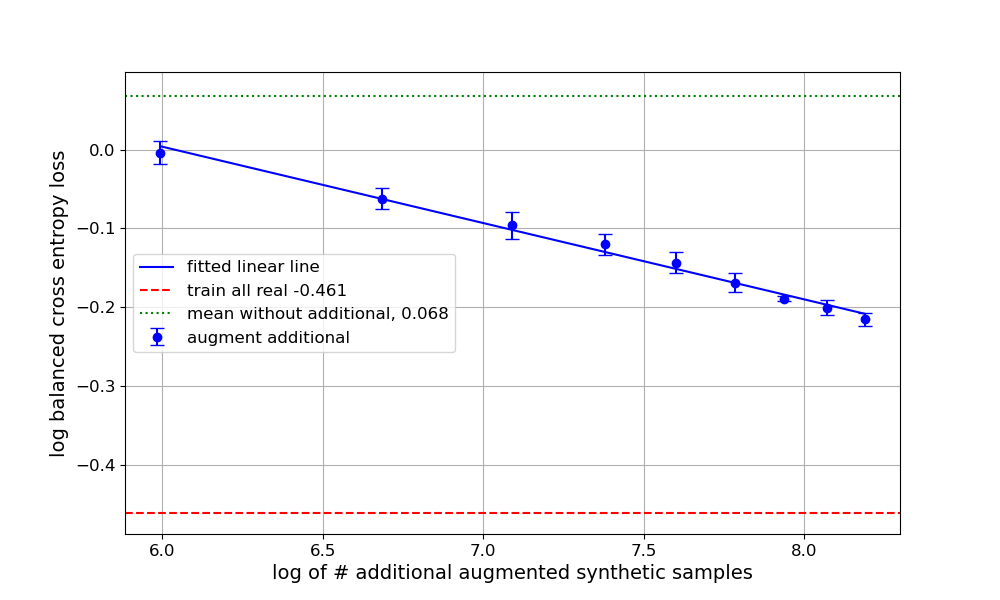}
        \caption{Diabetes - balanced cross-entropy loss}
    \end{subfigure}
    \begin{subfigure}{0.4\textwidth}
        \centering
        \includegraphics[width=\linewidth,height=1.18in]{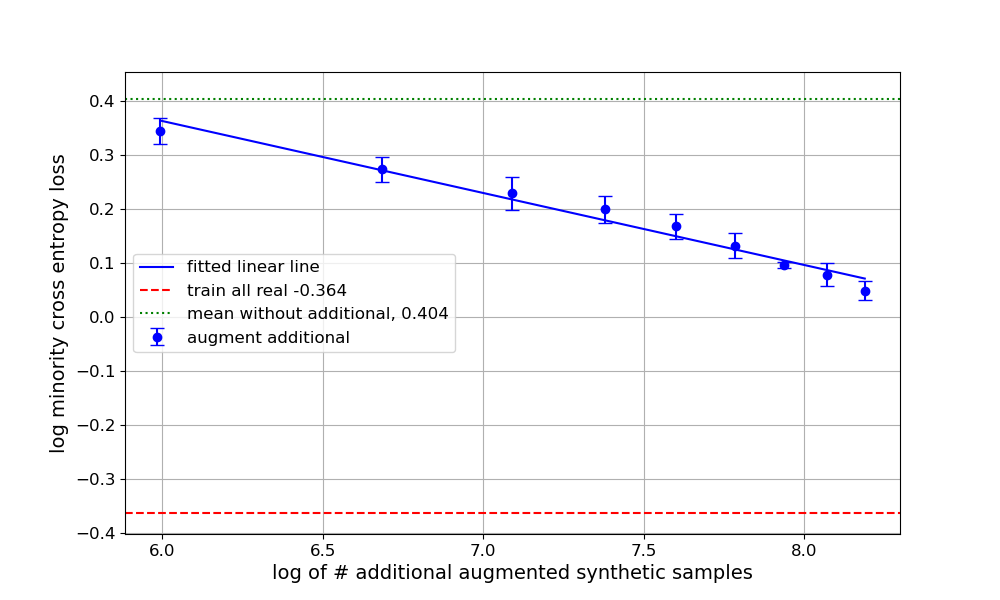}
        \caption{Diabetes - minority cross-entropy loss}
    \end{subfigure}
    \begin{subfigure}{0.4\textwidth}
        \centering
        \includegraphics[width=\linewidth,height=1.18in]{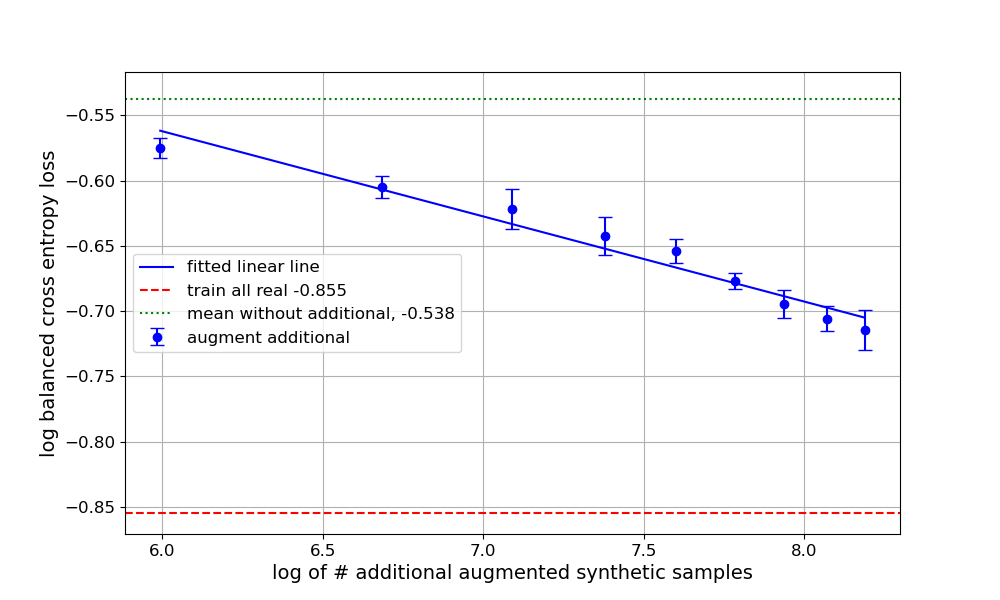}
        \caption{Adult - balanced cross-entropy loss}
    \end{subfigure}
    \begin{subfigure}{0.4\textwidth}
        \centering
        \includegraphics[width=\linewidth,height=1.18in]{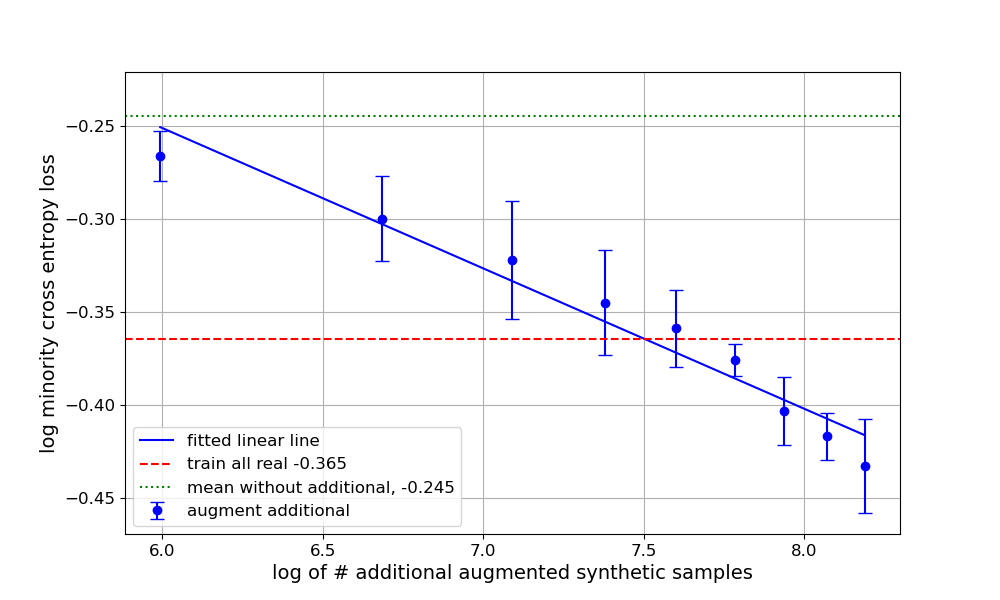}
        \caption{Adult - minority cross-entropy loss}
    \end{subfigure}
\caption{Synthetic augmentation for imbalanced classification. Two cross-entropy losses (blue) are reported as the size of the additional augmented samples $N$ increase, all in the log scale. Two benchmark lines are also included, one without augmentation (green), and the other utilizing the original data (red).}
\label{fig:scale-imb-xgb-loglog}
\end{figure}

\begin{figure}[t!]
\centering
    \begin{subfigure}{0.4\textwidth}
        \centering
        \includegraphics[width=\linewidth,height=1.18in]{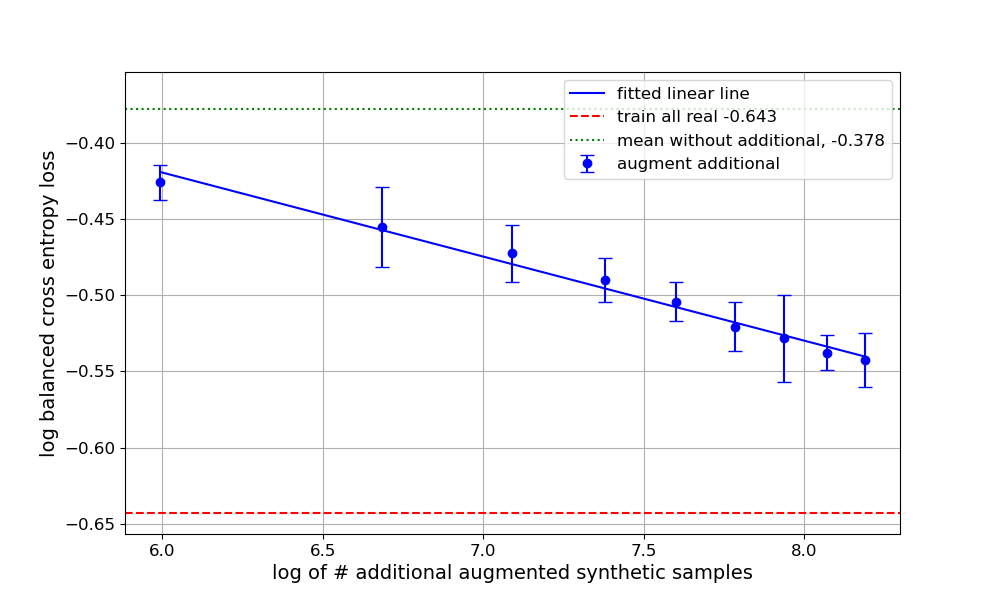}
        \caption{Craft - balanced cross-entropy loss}
    \end{subfigure}
    \begin{subfigure}{0.4\textwidth}
        \centering
        \includegraphics[width=\linewidth,height=1.18in]{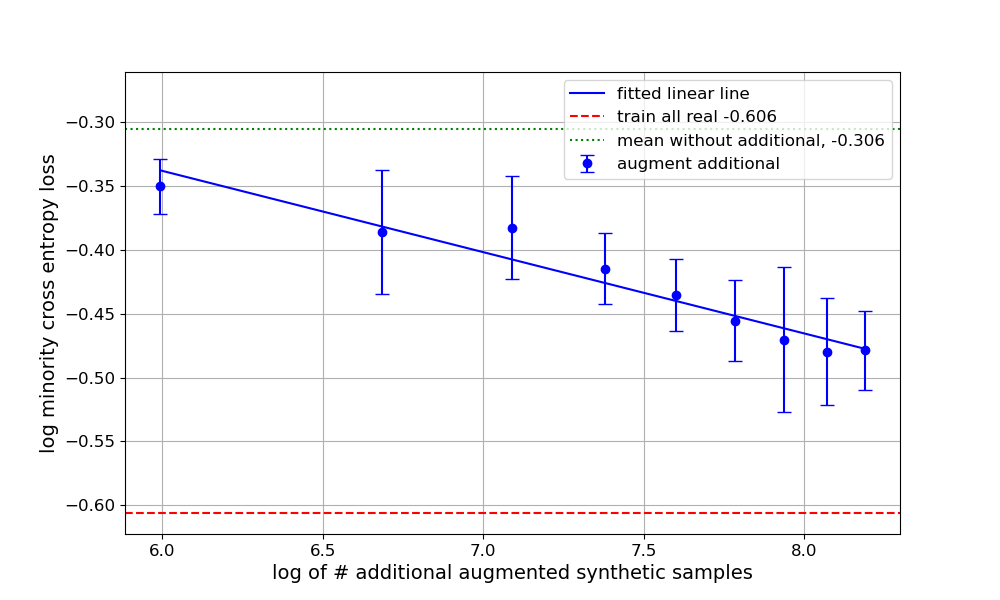}
        \caption{Craft - minority cross-entropy loss}
    \end{subfigure}
    \begin{subfigure}{0.4\textwidth}
        \centering
        \includegraphics[width=\linewidth,height=1.18in]{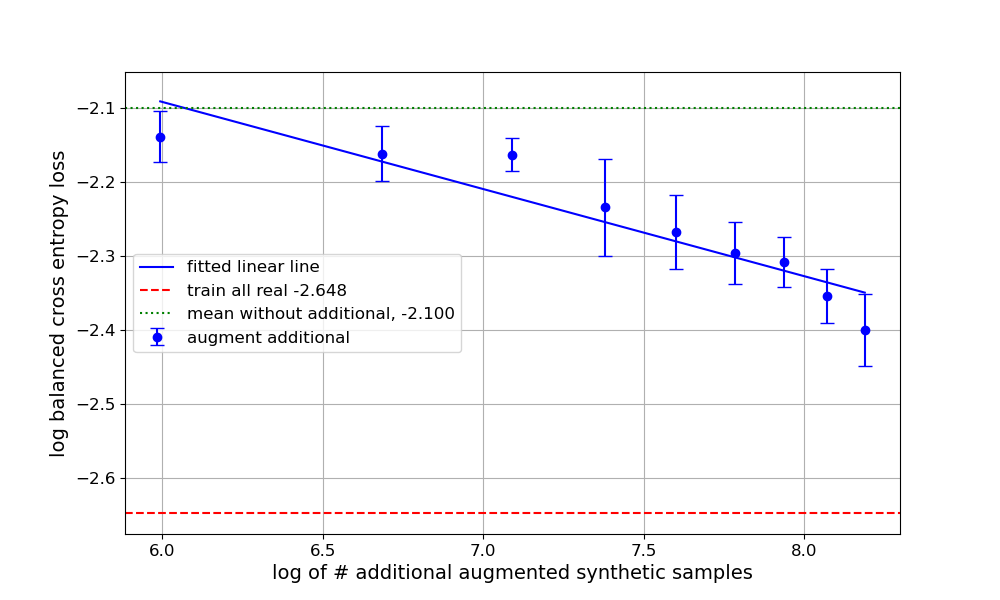}
        \caption{Gender - balanced cross-entropy loss}
    \end{subfigure}
    \begin{subfigure}{0.4\textwidth}
        \centering
        \includegraphics[width=\linewidth,height=1.18in]{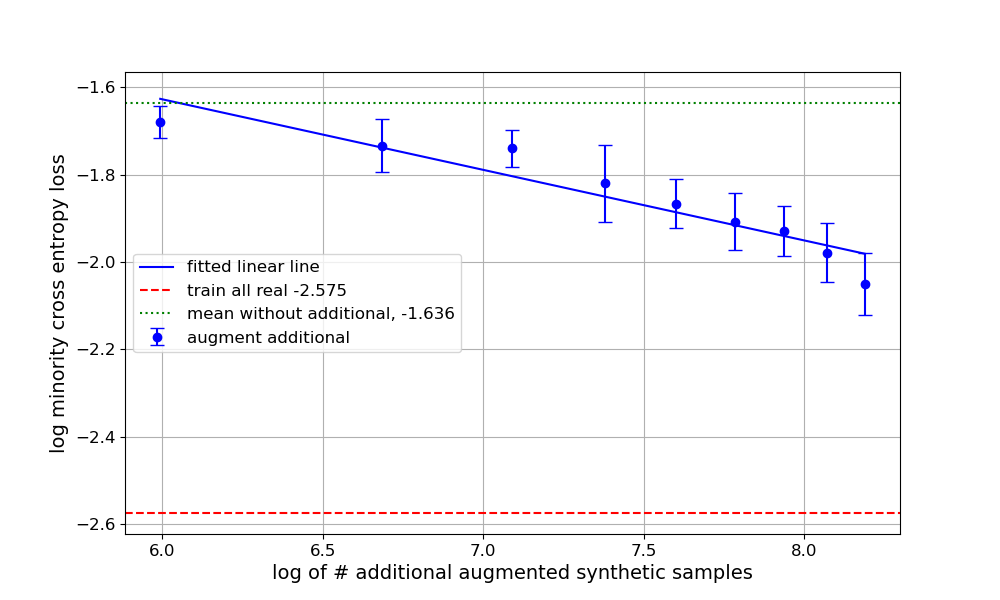}
        \caption{Gender - minority cross-entropy loss}
    \end{subfigure}
    \begin{subfigure}{0.4\textwidth}
        \centering
        \includegraphics[width=\linewidth,height=1.18in]{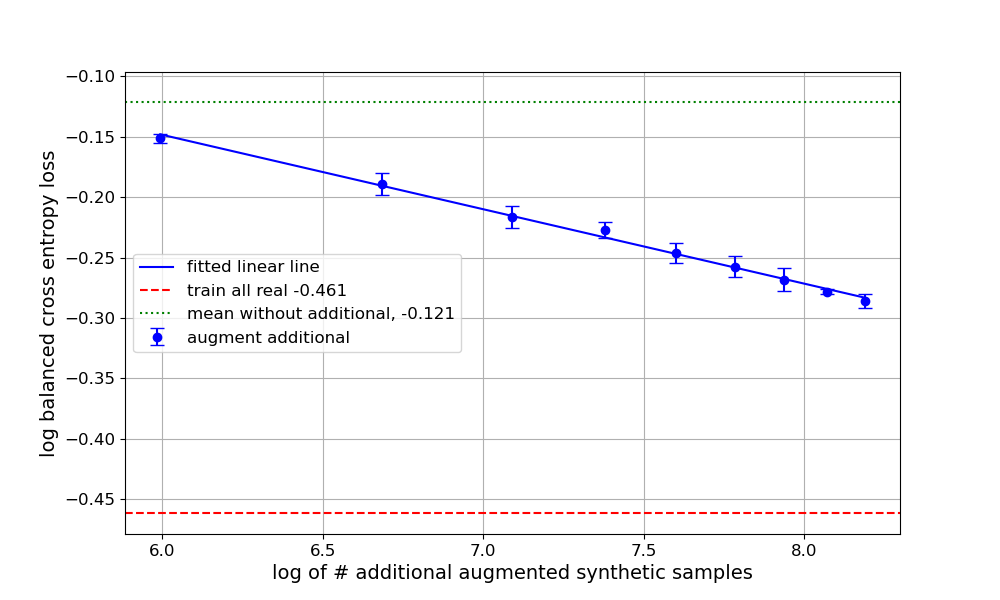}
        \caption{Diabetes - balanced cross-entropy loss}
    \end{subfigure}
    \begin{subfigure}{0.4\textwidth}
        \centering
        \includegraphics[width=\linewidth,height=1.18in]{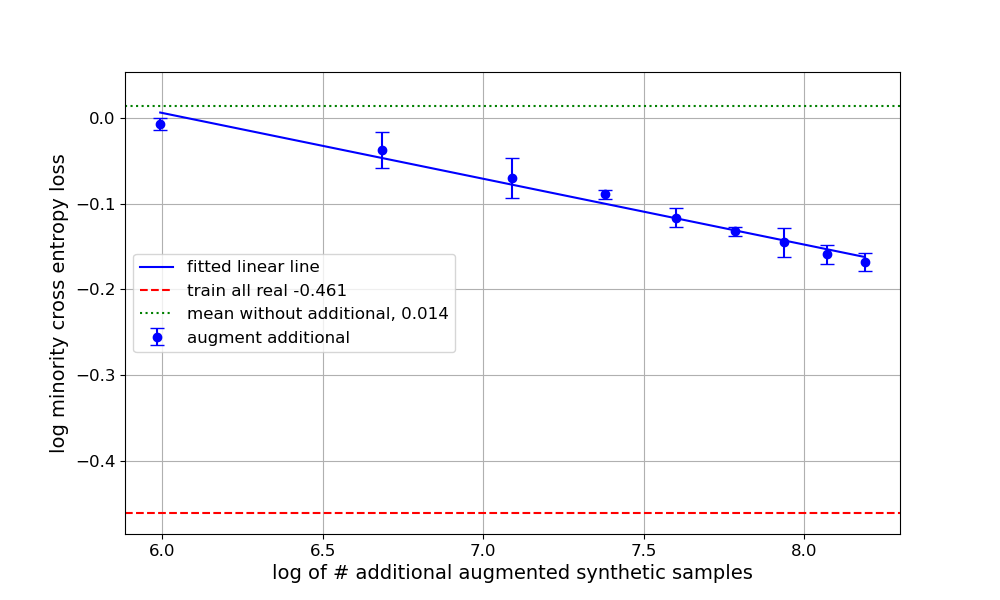}
        \caption{Diabetes - minority cross-entropy loss}
    \end{subfigure}
\caption{Synthetic augmentation for spurious correlation with a similar setup as Figure~\ref{fig:scale-imb-xgb-loglog}.}
\label{fig:scale-spurious-xgb-loglog}
\end{figure}

We next investigate the empirical performance of synthetic augmentation. Figures~\ref{fig:scale-imb-xgb-loglog} and \ref{fig:scale-spurious-xgb-loglog} report the two entropy losses based on 5 data replications for the imbalanced classification setting and the spurious correlation setting, respectively, when the number of additional samples for all groups $N = \{400, 800, \ldots, 3600\}$. We also include two benchmarks, one without any additional synthetic samples, and the other with all the training samples from the original data. We observe that both entropy losses decrease as $N$ increases, and the decay rates are approximately linear in the log scale, which agrees well with our theory. Moreover, the synthetic augmentation method clearly outperforms the solution without augmentation, but falls short compared to the one utilizing the original data. {This phenomenon has also been noted in the literature, where synthetic samples provide some benefits but are not as effective as the original data  \cite{Shumailov2024, alemohammad2023selfconsuminggenerativemodelsmad, zhang2024regurgitativetrainingvaluereal}.}

\section{Discussion}
\label{sec: discussion}

In this article, we have developed rigorous theoretical foundations that systematically characterize how LLM-generated synthetic data contributes to addressing imbalanced classification and spurious correlation. We have quantified the bias introduced by synthetic oversampling, derived the scaling law for synthetic augmentation, and showed that transformers can generate high-quality synthetic data that uphold the theoretical guarantees for oversampling and augmentation. We have run extensive numerical experiments to confirm the effectiveness of our approach. We make some additional remarks. 

First, if a model could perfectly learn the conditional distribution $P(Y|X)$, then the classification problem would be fully solved, and there is no need for an extra step of data oversampling or augmentation. In practice, however, the learned conditional distribution is only an approximation, especially under severe class imbalance or limited sample size. Our two-stage strategy, first generating synthetic data then integrating them with the observed data, is designed to leverage the generative power of LLMs, while preserving interpretability and ensuring generalizability in the final predictive model. 

More specifically, in terms of interpretability, directly using an LLM for prediction is largely a black-box process, offering little insight into how individual features affect the outcome. In contrast, our hybrid strategy, first generating synthetic data with an LLM and then fitting an interpretable model such as logistic regression, boosting, or random forest, offers transparent reasoning about feature effects and decision structure. This design thus leverages the LLM's capacity to enrich underrepresented regions of the data distribution while keeping the downstream classifier transparent, auditable, and domain-interpretable.

In terms of generalizability, our theoretical analysis formalizes when this extra step is beneficial, even when the synthetic generator produces data from a distribution $\mathcal{D}_{\mathrm{syn}}$ that differs from the raw distribution $\mathcal{D}_{\mathrm{raw}}$. We show that the balanced excess risk under oversampling and augmentation satisfies
\begin{align*}
\calR_{\mathrm{bal}}(\hat{\theta}) - \calR_{\mathrm{bal}}(\theta_{\mathrm{bal}}) &= O\left( \frac{1}{|\mathcal{G}|} \sum_{g\in\mathcal{G}} \{ (1-\alpha)\rho_g + \alpha \} \left\{ \|\nabla \calB^{(g)}(\theta_{\mathrm{bal}})\|^2 + \|\nabla^2 \calB^{(g)}(\theta_{\mathrm{bal}})\|^3 \right\} \right)\\
&\quad+ O_p\left( \frac{(1-\alpha)^2}{n_{\mathrm{tot}}} + \frac{\alpha^2}{N|\mathcal{G}|} \right),
\end{align*}
where $\rho_g$ is the imbalance ratio of group $g$, and $\alpha\in[0,1]$ controls the weight of the additional data augmentation, and therefore balances the bias-variance tradeoff. If one were to directly use an LLM for prediction, it corresponds to setting $\alpha=1$. This yields
\begin{align*}
\calR_{\mathrm{bal}}(\hat{\theta}) - \calR_{\mathrm{bal}}(\theta_{\mathrm{bal}}) &= O\left( \frac{1}{|\mathcal{G}|} \sum_{g\in\mathcal{G}} \left\{ \|\nabla \calB^{(g)}(\theta_{\mathrm{bal}})\|^2 + \|\nabla^2 \calB^{(g)}(\theta_{\mathrm{bal}})\|^3 \right\} \right) + O_p\left( \frac{1}{N|\mathcal{G}|} \right).
\end{align*}
When the learned conditional generator for $P(Y|X)$ is biased, the terms $\|\nabla \calB^{(g)}(\theta_{\mathrm{bal}})\|^2$ and $\|\nabla^2 \calB^{(g)}(\theta_{\mathrm{bal}})\|^3$ would impair the performance of $\hat\btheta$. 

Next, we note some potential limitations of our approach. Our theoretical framework requires a large token dimension, on the order of $O(r |\mathcal{M}|)$, because we use zeros in the input tokens as temporal memory to store intermediate results. However, for our implementation with a pretrained GPT-4 model, there is a practical token limit, which limits the effective data size we can handle. As a result, our numerical experiments use relatively small datasets, while GPT-4 still demonstrates promising performance under such limited sample sizes.

Finally, we point out several potential future directions. Although we have focused on tabular data, the proposed framework can be extended to other data modalities, e.g., images. In addition, we can further explore transfer learning scenarios, where labeled data are scarce and costly to obtain, while large amounts of unlabeled data are readily available. Leveraging LLMs to generate synthetic labels for those unlabeled samples offers a compelling strategy to improve model performance, whereas the theoretical properties remain to be rigorously established.

\bibliographystyle{apalike}
\bibliography{cite}

\newpage

\appendix

In this Supplementary Appendix, Section~\ref{supp-sec: theory} provides additional theoretical results and the proofs for Section~\ref{sec: theory} of the paper. It includes Section~\ref{supp-sec: verification} about verification of regularity conditions, Section~\ref{sec: oversampling theory proof} about the proofs for synthetic oversampling in Section~\ref{sec: theory oversampling} of the paper, Section~\ref{supp-sec: explicit quality form} about the explicit form of synthetic data quality term in Theorem~\ref{thm: synthetic data risk}, Section~\ref{supp-sec: scaling law proof} about the proofs for synthetic augmentation in Section~\ref{sec: scaling law} of the paper, and Section~\ref{supp-sec: additional scaling law} about scaling laws for additional statistical models. Section~\ref{supp-sec: theory transformers proof} provides additional theoretical results and the proofs for Section~\ref{sec: theory transformers} of the paper. It includes Section~\ref{supp-sec:tf_preliminaries} about notations, Section~\ref{supp-sec:auxiliary} about supporting auxiliary lemmas, and Section~\ref{supp-sec:construction} about the proofs for the transformer theory in Section~\ref{sec: theory transformers} of the paper. Section~\ref{supp-sec: add-numerical} provides additional numerical results. It includes Section~\ref{supp-sec: add-algo} about additional discussions of Algorithm~\ref{alg}, Section~\ref{supp-sec: add-gpt2} about additional results using GPT-2, and Section~\ref{supp-sec: add-gpt4} about additional results using GPT-4.

\section{Theory about synthetic oversampling and augmentation}
\label{supp-sec: theory}

\subsection{Verification of regularity conditions}
\label{supp-sec: verification}

We verify the regularity conditions, Assumptions~\ref{asm: differentiability}, \ref{asm: uniform convergence R syn}, and \ref{asm: identifiability theta}, in Section~\ref{sec: conditions} of the paper, under the settings of logistic and linear regressions with random oversampling. That is, within each group $g \in \mG$, samples are drawn with replacement from the observed in-group sample, class-conditional when applicable, to reach a target size. At the population level, this corresponds to sampling from the same underlying group distribution as the raw data. Henceforth, for every fixed $\btheta$, $\tilde\Sigma_g(\btheta)\equiv \Sigma_g(\btheta)$. Moreover, for any matrix $A$, let $\lambda_{\min}(A)$ denote the minimum singular value of $A$. We say an event $\mathcal{E}$ occurs with high probability when $\P(\mathcal{E}) = 1 - \exp\{-\Omega(\log^2 d)\}$.

\subsubsection{Logistic regression}

We first consider the binary logistic regression, where $y_i^{(g)} \in\{-1,+1\}$ and $\ell(\btheta;\x,y)=\log\{1+\exp(-y \btheta^\top \x)\}$. Suppose for each $g$, $\x_i^{(g)}$ is sub-Gaussian with $\E\|\x_i^{(g)}\|^2<\infty$ and $\E[\x_i^{(g)}\x_i^{(g) \top}]\succ 0$. Let $s(t) = (1+e^{-t})^{-1}$ be the sigmoid function.

For Assumption~\ref{asm: differentiability}, observe that the loss $\ell$ is infinitely differntiable in $\btheta$. Moreover,
$$
\nabla \ell(\btheta;\x,y)= - \sigmoid(-y \btheta^\top \x) y \x, \quad\quad  \nabla^2 \ell(\btheta;\x,y)= \sigmoid(\btheta^\top \x)\bigl(1-\sigmoid(\btheta^\top \x)\bigr) \x \x^\top,
$$
with $0<\sigmoid(1-\sigmoid)\le 1/4$. Hence, on any neighborhood of $\btheta_\b$, $\|\nabla \ell\|\le \|\x\|$, and $\|\nabla^2 \ell\|\le \frac{1}{4} \|\x \x^\top\|$, and sub-Gaussiannity of $\x_i^{(g)}$ ensures integrable envelopes. The population Hessian in group $g$ is $\nabla^2 \calR^{(g)}(\btheta)=\E\left[\sigmoid(1-\sigmoid) \x_i^{(g)}\x_i^{(g) \top}\right]$. Hence $\sum_{g}\nabla^2 \calR^{(g)}(\btheta)$ is strictly positive definite around $\btheta_\b$, and has bounded eigenvalues since $\E[\x_i^{(g)}\x_i^{(g) \top}]\succ0$ and $0<\sigmoid(1-\sigmoid)\le 1/4$. For the covariance map $\Sigma_g(\btheta)$, a standard dominated-convergence argument yields local Lipschitzness of $\Sigma_g(\btheta)$ and $\tilde \Sigma_g(\btheta)$. Since $\tilde\Sigma_g(\btheta)\equiv \Sigma_g(\btheta)$, we trivially have $\sup_{\btheta\in \Theta}\|\Sigma_g(\btheta)-\tilde\Sigma_g(\btheta)\|=o(1)$.
This verifies Assumption~\ref{asm: differentiability}. 

For Assumption~\ref{asm: uniform convergence R syn}, let $\mathcal{H} = \{\ell(\btheta;\cdot):\btheta\in\Theta\}$. We first show that this parametric class is Glivenko-Cantelli. Since $\{(\x, y) \mapsto \btheta^\top \x\}$ is a VC-subgraph, and $(x, y) \mapsto -y$ is a fixed measurable function, $\mathcal{H} = \{(\x, y) \mapsto \btheta^\top \x\} \cdot ((x, y) \mapsto -y)$ is also a VC-subgraph by \citet[][Lemma 2.6.18]{van1996weak}. Since it has a trivial integrable envelope $1$, $\mathcal{H}$ is Glivenko-Cantelli by \citet[][Theorem~2.6.7]{van1996weak}. Note that
\begin{align*}
    \calRhat_\o(\btheta) &= \frac{1}{n_\T + m_\T} \qty{\sum_{g \in \mG} \sum_{i \in [n_g]} \ell( \btheta; \x_i^{(g)}, y_i^{(g)} ) + \sum_{g \in \mG} \sum_{i \in [m_g]} \ell( \btheta; \tilde \x_i^{(g)}, \tilde y_i^{(g)} )}\\
    &= \frac{n_\T}{n_\T + m_\T} \underbrace{\frac{1}{n_\T} \sum_{g \in \mG} \sum_{i \in [n_g]} \ell( \btheta; \x_i^{(g)}, y_i^{(g)} )}_{=: \calRhat_{\o,1}(\btheta)} + \frac{m_\T}{n_\T + m_\T} \underbrace{ \frac{1}{m_\T} \sum_{g \in \mG} \sum_{i \in [m_g]} \ell( \btheta; \tilde \x_i^{(g)}, \tilde y_i^{(g)} )}_{=: \calRhat_{\o,2}(\btheta)}.
\end{align*}
By Glivenko-Cantelli theorem, we have $\sup_{\btheta \in \Theta} |\calRhat_{\o,1}(\btheta) - \E[\calRhat_{\o,1}(\btheta)]| \stackrel{p}{\to} 0$. Following a similar argument, we also have $\sup_{\btheta \in \Theta} |\calRhat_{\o,1}(\btheta) - \E[\calRhat_{\o,1}(\btheta)]| \stackrel{p}{\to} 0$.
Thus, we obtain that
\begin{small}
\begin{align*}
    \sup_{\btheta \in \Theta} |\calRhat_\o(\btheta) - \E[\calRhat_\o(\btheta)]| \leq \sup_{\btheta \in \Theta} |\calRhat_{\o,1}(\btheta) - \E[\calRhat_{\o,1}(\btheta)]| + \sup_{\btheta \in \Theta} |\calRhat_{\o,2}(\btheta) - \E[\calRhat_{\o,2}(\btheta)]| \stackrel{p}{\to} 0.
\end{align*}
\end{small}\noindent

For Assumption~\ref{asm: identifiability theta}, the balanced or oversampled population risk is strictly convex in a neighborhood of its minimizer whenever $\E[\x_i^{(g)}\x_i^{(g) \top}]\succ0$. Hence the minimizer is unique and the strict separation condition holds.

\subsubsection{Linear regression}

We next consider a linear regression model with squared loss, $\ell(\btheta;\x,y)=(1/2) (y-\btheta^\top \x)^2$. Consider a bounded parameter space $\Theta$. Suppose for each $g$, $\x_i^{(g)}$ is sub-Gaussian, with $\E[\x_i^{(g)}\x_i^{(g) \top}] \succ 0$, and $\E[y_i^{(g) 2} \x_i^{(g)} \x_i^{(g) \top}]<\infty$.

For Assumption~\ref{asm: differentiability}, we have
$$
    \nabla \ell(\btheta;\x,y)=-(y-\btheta^\top \x) \x,   \quad\quad  \nabla^2 \ell(\btheta;\x,y)=\x \x^\top.
$$
Note that $\ell$ is infinitely differentiable with integrable envelopes $\|\nabla\ell\|\le (|y|+\|\btheta\| \|\x\|)\|\x\|$ and $\|\nabla^2\ell\|=\|\x \x^\top\|$. The population Hessian is $\nabla^2 \calR^{(g)}(\btheta)=\E[\x_i^{(g)}\x_i^{(g) \top}] \succ 0$, and $\sum_g \nabla^2 \calR^{(g)}(\btheta)$ preserves strict positive definiteness with bounded eigenvalues around $\btheta_\b$. 
For covariance maps, note that
\begin{align*}
\Sigma_g(\btheta) &= \E\left[\nabla \ell(\btheta;\x_i^{(g)},y_i^{(g)}) \nabla \ell(\btheta;\x_i^{(g)},y_i^{(g)})^\top\right] =\E\left[(y_i^{(g)} - \btheta^\top \x_i^{(g)})^2 \x_i^{(g)} \x_i^{(g) \top}\right] \\
    &= \E[y_i^{(g) 2} \x_i^{(g)} \x_i^{(g) \top}] - 2 \E[y_i^{(g)} \x_i^{(g)} \x_i^{(g) \top} \btheta \x_i^{(g) \top}] + \E[\x_i^{(g)} \x_i^{(g) \top} \btheta \btheta^\top \x_i^{(g)} \x_i^{(g) \top} ].
\end{align*}
Hence $\Sigma_g(\btheta)$ is locally Lipschitz around $\btheta_\b$. Likewise, $\tilde \Sigma_g(\btheta)$ is locally Lipschitz around $\btheta_\b$. Since $\tilde\Sigma_g(\btheta)\equiv\Sigma_g(\btheta)$, we trivially have $\sup_{\btheta\in \Theta}\|\Sigma_g(\btheta)-\tilde\Sigma_g(\btheta)\|=o(1)$.

For Assumption~\ref{asm: uniform convergence R syn}, note that the linear function class $\{(\x, y) \mapsto y - \btheta^\top \x\}$ is a VC-subgraph. By \citet[][Lemma 2.6.18]{van1996weak}, the class $\mathcal{H} := \{(\x, y) \mapsto (y - \btheta^\top \x)^2\}$ is also a VC-subgraph. Since $\Theta$ is bounded, $\x_i^{(g)}$ is sub-Gaussian, and $\E[y_i^{(g)} \x_i^{(g)}]$ exists, the class $\mathcal{H}$ has an integrable envelope. Following a similar argument as in logistic regression yields $\sup_{\btheta\in\Theta}\bigl|\widehat{\calR}_\o(\btheta)-\calR_\o(\btheta)\bigr|=o_p(1)$.

For Assumption~\ref{asm: identifiability theta}, $\calR^{(g)}(\btheta)= (1/2) \E[(y_i^{(g)} - \btheta^\top \x_i^{(g)})^2]$ is globally strictly convex when $\E[\x_i^{(g)}\x_i^{(g) \top}]\succ0$. Thus $\calR_\b$ and $\calR_\o$ are both strictly convex, so the minimizer is unique and the strict separation holds for both ${\calR}_\o$ and ${\calR}_\b$.

\subsection{Proofs of theory on synthetic oversampling in Section~\ref{sec: theory oversampling}}
\label{sec: oversampling theory proof}

We first see how oversampling improves the performance on minority group in general. Let $n_g$ be the number of raw samples observed for group $g \in \mG$. For each group $g$, we generate $m_g$ synthetic data $m_g$. Denote $n_\T := \sum_{g} n_g$ and $m_\T := \sum_g m_g$ by the total number of raw data and the total number of synthetic data, respectively.

Denote the observed raw data as $\{(\x_1, \dots, \x_{n_g})\}_{g \in \mG}$ and generated synthetic data as $\{(\tilde \x_1,$ $\dots, \tilde \x_{m_g})\}_{g \in \mG}$.
We assume the independence between raw data and synthetic data. When directly using raw data as a reference to synthetic data, the independence does not hold. However, we can always split the raw data, with one half used as raw data, and the other half used as a reference. Given a loss function $\ell(\btheta; x, y) \in \R$ we define the empirical risk with raw and synthetic data as follows
\begin{align*}
    \calRhat_\o(\btheta) &= \frac{1}{n_\T+m_\T} \sum_{g \in \mG} \sum_{i \in [n_g]} \ell(\btheta; \x_i^{(g)}, y_i^{(g)}) + \frac{1}{n_\T + m_\T} \sum_{g \in \mG} \sum_{i \in [m_g]} \ell(\btheta; \tilde \x_i^{(g)}, \tilde y_i^{(g)}).
\end{align*}
Let $\calR_\o(\btheta) = \E[\calRhat_\o(\btheta)]$ be the population version of $\calRhat_\o$.
We also define the balanced risk
\begin{align*}
    \calR_\b(\btheta) = \frac{1}{|\mG|} \sum_{g \in \mG} \calR^{(g)}(\btheta),
\end{align*}
where $\calR^{(g)}(\btheta) = \E[\ell(\btheta; \x_1^{(g)}, y_1^{(g)})]$ is the group specific risk.
Let the minimizers of $\calRhat_\o$, $\calR_\o$, and $\calR_\b$ be $\hat\btheta_\o$, $\btheta_\o$, $\btheta_\b$ respectively.
Define the bias term for group $g$ as
\begin{align*}
    \calB^{(g)}(\btheta) := \E[\ell(\btheta; \tilde \x_1^{(g)}, \tilde y_1^{(g)})] - \E[\ell(\btheta; \x_1^{(g)}, y_1^{(g)})].
\end{align*}
Note that $\btheta_\b$ is the ideal estimator, balancing the parameters over all groups in the dataset.

Our goal is to see the effect of bias present in the synthetic data to the estimator $\hat\btheta_\o$ by measuring the risk of group $g$, $\calR^{(g)}(\hat\btheta_\o) - \calR^{(g)}(\btheta_\b)$.

\begin{proof}[\textbf{Proof of Theorem~\ref{thm: synthetic data risk}}]
In this proof, we aim to approximate $\calR^{(g)}(\hat \btheta_\o) - \calR^{(g)}(\btheta_\b)$. We divide the proof into three steps. In Step 1, we derive the bias between $\btheta_\b$ and $\btheta_\o$. In Step 2, we approximate $\hat\btheta_\o - \btheta_\o$. Finally, in Step 3, we take the effect of group-specific risk into consideration.
    
To ease notation, define $H_{\o,\o} = \nabla^2 \calR_\o(\btheta_\o)$ and $H_{\b,\b} = \nabla^2 \calR_\b(\btheta_\b)$. We note that
\begin{align}
        \calR_\o(\btheta) &= \calR_\b(\btheta) + \sum_{g \in \mG} \frac{m_g}{n_\T + m_\T} \calB^{(g)}(\btheta).\label{eq: R decomp}
\end{align}    

\paragraph{Step 1.}    
Note that $\sup_\btheta |\calR_\o(\btheta) - \calR_\b(\btheta)| = o(1)$ follows from Assumption~\ref{asm: bias o(1)} and \eqref{eq: R decomp}. Then, the convergence $\|\btheta_\o - \btheta_\b\| = o(1)$ follows by a standard Taylor expansion argument and Assumption~\ref{asm: identifiability theta}. See, for example, Theorem 5.7 of \citet{van2000asymptotic}.
From \eqref{eq: R decomp}, we obtain
    \begin{align}
        \nabla \calR_\o(\btheta_\b) &= \sum_{g' \in \mG} \nabla \calR^{(g')}(\btheta_\b) + \sum_{g' \in \mG} \frac{m_{g'}}{n_\T + m_\T} \nabla \calB^{(g')}(\btheta_\b)\nonumber\\
        &= \sum_{g' \in \mG} \frac{m_{g'}}{n_\T + m_\T} \nabla \calB^{(g')}(\btheta_\b),\label{eq: grad R syn}
    \end{align}
    where we used $\sum_{g'} \nabla \calR^{(g')}(\btheta_\b) = 0$.
    By Taylor expansion, there exists some $\btheta'$ in a line segment between $\btheta_\b$ and $\btheta_\o$ such that
    \begin{align}
        \nabla \calR_\o(\btheta_\b) = \underbrace{\nabla \calR_\o(\btheta_\o)}_{=0} + \nabla^2 \calR_\o(\btheta') (\btheta_\b - \btheta_\o).\label{eq: R s theta b expansion}
    \end{align}
    This yields
    \begin{align*}
        \{\nabla \calR_\o(\btheta_\b)\}^\top (\btheta_\b - \btheta_\o) = (\btheta_\b - \btheta_\o)^\top \nabla^2 \calR_\o(\btheta') (\btheta_\b - \btheta_\o).
    \end{align*}
    Since $\|\btheta_\b - \btheta_\o\| = o(1)$, and $\nabla^2 \calR_\o(\btheta)$ is Lipschitz continuous around $\btheta_\b$ with its smallest eigenvalue bounded below (Assumption~\ref{asm: differentiability}), we have
    \begin{align*}
        \|\btheta_\b - \btheta_\o\| &\leq \frac{1}{\lambda_{\min}(\nabla^2 \calR_\o(\btheta_\b)) + o(1)} \|\nabla \calR_\o(\btheta_\b)\| \lesssim \|\nabla \calR_\o(\btheta_\b)\|.
    \end{align*}
    Using \eqref{eq: grad R syn} and \eqref{eq: R s theta b expansion}, we have
    \begin{align*}
        \btheta_\b - \btheta_\o &= \{\nabla^2 \calR_\o(\btheta')\}^{-1} \nabla \calR_\o(\btheta_\b) = H_{\o,\o}^{-1} \nabla \calR_\o(\btheta_\b) + R_1,
    \end{align*}
    where $\|R_1\| = O(\|\nabla \calR_\o(\btheta_\b)\|^2)$.
    By Assumption~\ref{asm: differentiability} and $\|\btheta_\b - \btheta_\o\| = O(\|\nabla \calR_\o(\btheta_\b)\|)$, 
    \begin{align*}
        \|H_{\o,\o}^{-1} - H_{\b,\b}^{-1}\| &\lesssim \norm{\nabla^2 \calR_\b(\btheta_\o) + \sum_{g \in \mG} \frac{m_g}{n_\T + m_\T} \nabla^2 \calB^{(g)}(\btheta_\o) - \nabla^2 \calR_\b(\btheta_\b)}\\
        &\lesssim \|\btheta_\o - \btheta_\b\| + \norm{\sum_{g \in \mG} \frac{m_g}{n_\T + m_\T} \nabla^2 \calB^{(g)}(\btheta_\b)},\\
        &\lesssim \|\nabla \calR_\o(\btheta_\b)\| + \norm{\sum_{g \in \mG} \frac{m_g}{n_\T + m_\T} \nabla^2 \calB^{(g)}(\btheta_\b)},
    \end{align*}
Therefore,
\begin{align}
\btheta_\b - \btheta_\o &= H_{\b,\b}^{-1} \nabla \calR_\o(\btheta_\b) + R_1' = \sum_{g \in \mG} \frac{m_g}{n_\T + m_\T} H_{\b,\b}^{-1} \nabla \calB^{(g)}(\btheta_\b) + R_1',\label{eq: theta bal theta o distance}
\end{align}
where $\|R_1'\| = O(\|\nabla \calR_\o(\btheta_\b)\|^2 + \|\sum_{g \in \mG} m_g/(n_\T + m_\T) \nabla^2 \calB^{(g)}(\btheta_\b)\|^2)$.

\paragraph{Step 2.}
The consistency $\|\hat\btheta_\o - \btheta_\o\| = o_p(1)$ follows by a standard argument of $M$-estimators combined with Assumptions~\ref{asm: uniform convergence R syn} and \ref{asm: identifiability theta}. (Theorem 5.7 in \citet{van2000asymptotic}.)
    We follow the argument in the proof of Proposition 3.1 of \citet{jain2024scaling}.
    By a standard argument in $M$-estimation, combined with Assumption~\ref{asm: differentiability}, and $\|\btheta_\b - \btheta_\o\| = o(1)$, we have
    \begin{align*}
        \hat\btheta_\o &= \btheta_\o - H_{\o,\o}^{-1} \nabla \calRhat_\o(\btheta_\o) + O(\|\hat\btheta_\o - \btheta_\o\|^2),\\
        &= \btheta_\o - H_{\b,\b}^{-1} \nabla \calRhat_\o(\btheta_\o) + \underbrace{O(\|\hat\btheta_\o - \btheta_\o\|^2 + \|\btheta_\o - \btheta_\b\| \|\nabla \calRhat_\o(\btheta_\o)\|)}_{=: R_2}.
    \end{align*}
    Using Cauchy-Schwarz inequality, we can bound $R_2$ as
    \begin{align*}
        \|R_2\| &\lesssim \|\nabla \calRhat_\o(\btheta_\o)\|^2 + \|\btheta_\o - \btheta_\b\|^2\\
        &= O_p\qty( \frac{\max_{g' \in \mG} \tr(\Sigma_{g'}(\btheta_\o)) \vee \tr(\tilde \Sigma_{g'}(\btheta_\o))}{n_\T + m_\T}) + O\qty(\norm{\sum_{g \in \mG} \frac{m_g}{n_\T + m_\T} \nabla \calB^{(g)}(\btheta_\b)}^2 + \|R_1'\|^2 )\\
        &= O_p\qty( \frac{\max_{g' \in \mG} \tr(\Sigma_{g'}(\btheta_\b)) \vee \tr(\tilde \Sigma_{g'}(\btheta_\b))}{n_\T + m_\T}) + O\qty(\norm{\sum_{g \in \mG} \frac{m_g}{n_\T + m_\T} \nabla \calB^{(g)}(\btheta_\b)}^2 + \|R_1'\|^2 ).
    \end{align*}

\paragraph{Step 3.}
Combining the results from Step 1 and Step 2, we have
\begin{align*}
\hat \btheta_\o - \btheta_\b = - H_{\b,\b}^{-1} \nabla \calRhat_\o(\btheta_\o) - \sum_{g' \in \mG} \frac{m_{g'}}{n_\T + m_\T} H_{\b,\b}^{-1} \nabla \calB^{(g')}(\btheta_\b) + R_3,
\end{align*}
where $\|R_3\| \lesssim \|R_2\| + \|R_1'\|$. Now we measure the performance of $\hat\btheta_\o$ for group $g$. From Assumption~\ref{asm: differentiability}, $\|\nabla^2 \calR^{(g)}(\btheta')\| = O(1)$ for any $\btheta'$ in the line segment between $\btheta_\o$ and $\btheta_\b$, since $\|\btheta_\b - \btheta_\o\| = o(1)$ by Step 1.
Using Taylor expansion, we have
\begin{align*}
&\calR^{(g)}(\hat \btheta_\o) - \calR^{(g)}(\btheta_\b) = \{\nabla \calR^{(g)}(\btheta_\b)\}^\top (\hat \btheta_\o - \btheta_\b) + O(\|\hat\btheta_\o - \btheta_\b\|^2)\\
& =: - \{\nabla \calR^{(g)}(\btheta_\b)\}^\top H_{\b,\b}^{-1} \nabla \calRhat_\o(\btheta_\o) - \sum_{g' \in \mG} \frac{m_{g'}}{n_\T + m_\T} \{\nabla \calR^{(g)}(\btheta_\b)\}^\top H_{\b,\b}^{-1} \nabla \calB^{(g')}(\btheta_\b) + R,
\end{align*}
where
\begin{align*}
        \|R\| &\lesssim \|R_3\| + \|\nabla \calRhat_\o(\btheta_\o)\|^2\\
        &= O_p\qty( \frac{\max_{g' \in \mG} \tr(\Sigma_{g'}(\btheta_\o)) \vee \tr(\tilde \Sigma_{g'}(\btheta_\o))}{n_\T + m_\T})\\
        &\quad+ O\qty(\norm{\sum_{g \in \mG} \frac{m_g}{n_\T + m_\T} \nabla \calB^{(g)}(\btheta_\b)}^2 + \|\nabla \calR_\o(\btheta_\b)\|^2 + \norm{\sum_{g \in \mG} \frac{m_g}{n_\T + m_\T} \nabla^2 \calB^{(g)}(\btheta_\b)}^2 ).
    \end{align*}
    with $v_g$ and $b_{g,g'}$ defined as
    \begin{align}
        v_g^2 &= \{\nabla \calR^{(g)}(\btheta_\b)\}^\top H_{\b,\b}^{-1} \qty(\frac{1}{|\mG|} \sum_{g' \in \mG} \Sigma_{g'}(\btheta_\b)) H_{\b,\b}^{-1} \nabla \calR^{(g)}(\btheta_\b),\nonumber\\
        b_{g,g'} &= \{\nabla \calR^{(g)}(\btheta_\b)\}^\top H_{\b,\b}^{-1} \nabla \calB^{(g')}(\btheta_\b).\label{eq: b g g}
    \end{align}
    By Assumption~\ref{asm: identifiability theta}, we have $\|\tilde \Sigma_g(\btheta_\o) - \Sigma_g(\btheta_\o)\| = o(1)$ and $\|\Sigma_g(\btheta_\o) - \Sigma_g(\btheta_\b)\| = o(1)$.
    Therefore, we have
    \begin{align*}
        &\calR^{(g)}(\hat \btheta_\o) - \calR^{(g)}(\btheta_\b) = - \sum_{g' \in \mG} \frac{m_{g'}}{n_\T + m_\T} b_{g,g'}\\
        &\quad\quad+ O_p\qty(\frac{1}{\sqrt{n_\T+m_\T}} v_g + \frac{\max_{g' \in \mG} \tr(\Sigma_{g'}(\btheta_\b)) \vee \tr(\tilde \Sigma_{g'}(\btheta_\b))}{m_\T + n_\T})\\
        &\quad\quad+ O\qty(\|\nabla \calR_\o(\btheta_\b)\|^2 + \norm{\sum_{g \in \mG} \frac{m_g}{n_\T + m_\T} \nabla \calB^{(g)}(\btheta_\b)}^2 + \norm{\sum_{g \in \mG} \frac{m_g}{n_\T + m_\T} \nabla^2 \calB^{(g)}(\btheta_\b)}^2),
    \end{align*}
where $O_p$ hides constants in the assumption. The conclusion follows from \eqref{eq: grad R syn}, and $1/(n_\T + m_\T) = o(1)$. This completes the proof of Theorem~\ref{thm: synthetic data risk}.
\end{proof}

\subsubsection{Imbalanced classification}

Consider the binary classification setting in Section~\ref{sec: theory}.

\begin{proof}[\textbf{Proof of Corollary~\ref{cor: synthetic data risk imbalanced}}]
Directly applying Theorem~\ref{thm: synthetic data risk} gives
\begin{align*}
        \calR^{(0)}(\hat \btheta_\o) &= \calR^{(0)}(\btheta_\b) - \frac{n_1-n_0}{2 n_1} b_{0,0} + O_p\qty(\frac{1}{\sqrt{n_1}} v_0 + \norm{\nabla \calB^{(0)}(\btheta_\b)}^2 + \norm{\nabla^2 \calB^{(0)}(\btheta_\b)}^2),\\
        \calR^{(1)}(\hat \btheta_\o) &= \calR^{(1)}(\btheta_\b) - \frac{n_1-n_0}{2 n_1} b_{1,0} + O_p\qty(\frac{1}{\sqrt{n_1}} v_1 + \norm{\nabla \calB^{(0)}(\btheta_\b)}^2 + \norm{\nabla^2 \calB^{(0)}(\btheta_\b)}^2),
\end{align*}
where $b_{g,g'}$ is defined in \eqref{eq: b g g}. Using $|\max(a,b) - \max(c,d)| \leq |a-c| \vee |b-d|$ for any $a,b,c,d \in \R$ completes the proof of Corollary~\ref{cor: synthetic data risk imbalanced}. 
\end{proof}

\subsubsection{Spurious correlation}

Let $\mG = \mathcal{Y} \times \mathcal{S}$. Consider the spurious correlation setting in Section~\ref{sec: theory}.
We are interested in the performance of $\hat \btheta_\o$ against the minimizer $\btheta_\rw$ of $\calR_\rw$, measured in minority group risk $\calR^{(-1,1)}(\btheta) := \E[\ell(\btheta; (\z_1, \s_1), \y_1)| y_1 = -1, \s_1 = \bgamma]$.

\begin{proof}[\textbf{Proof of Corollary~\ref{cor: synthetic data risk spurious}}]
We first show that $\btheta_\b = \btheta_\rw$. Let $\s_i'$ be an independent copy of $\s_i$. Since $\z_i$ only depends on the label $y_i$,
    \begin{align*}
        \calR_\b(\btheta) &= \frac{1}{4} \sum_{y, \s} \E[\ell(\btheta; (\z_i, \s)^\top, y)] = \frac{1}{4} \sum_{y, \s} \E[\ell(\btheta; (\z_i, \s)^\top, y_i) | y_i = y]\\
        &= \frac{1}{2} \sum_{y} \E[\ell(\btheta; (\z_i, \s_i')^\top, y_i) | y_i = y] = \calR_\rw(\btheta).
    \end{align*}
    Thus $\btheta_\b = \btheta_\rw$ by definition.
    For any $g \in \mG$, Theorem~\ref{thm: synthetic data risk} gives
    \begin{align*}
        &\calR^{(g)}(\hat \btheta_\o)\\
        &\quad= \calR^{(g)}(\btheta_\rw) - \sum_{g' \in \mG} \frac{m_{g'}}{n_\T + m_\T} b_{g,g'}\\
        &\quad\quad+ O_p\biggl(\frac{1}{\sqrt{n_\T + m_\T}} v_g + \frac{1}{m_\T \wedge n_\T} + \sum_{g' \in \mG} \frac{m_{g'}}{n_\T + m_\T} (\|\nabla \calB^{(g')}(\btheta_\rw)\|^2 \vee \|\nabla^2 \calB^{(g')}(\btheta_\rw)\|^2)\biggr),
    \end{align*}
    where $b_{g,g'}$ is defined in \eqref{eq: b g g}. Using $\calR_\b(\btheta) = \calR_\rw(\btheta) = (1/|\mG|) \sum_{g \in \mG} \calR^{(g)}(\btheta)$, we obtain
    \begin{align*}
        &|\calR_\rw(\hat \btheta_\o) - \calR_\rw(\btheta_\rw)|\\
        &\quad= O_p\biggl(\sum_{g \in \mG} \sum_{g' \in \mG} \frac{m_{g'}}{n_\T + m_\T} b_{g,g'} + \sum_{g \in \mG} \frac{1}{\sqrt{n_\T + m_\T}} v_g\\
        &\quad\quad\quad\quad+ \frac{1}{m_\T \wedge n_\T} + \sum_{g' \in \mG} \frac{m_{g'}}{n_\T + m_\T} (\|\nabla \calB^{(g')}(\btheta_\rw)\|^2 \vee \|\nabla^2 \calB^{(g')}(\btheta_\rw)\|^2)\biggr)\\
        &\quad= O_p\biggl(\frac{n_{\text{maj}}-n_{\text{min}}}{n_{\text{maj}}} (\|\nabla \calB^{(1,-\bgamma)}(\btheta_\rw)\| + \|\nabla \calB^{(-1,\bgamma)}(\btheta_\rw)\| + \|\nabla^2 \calB^{(1,-\bgamma)}(\btheta_\rw)\|^2 + \|\nabla^2 \calB^{(-1,\bgamma)}(\btheta_\rw)\|^2)\\
        &\quad\quad+ \frac{1}{\sqrt{n_{\text{maj}}}} \max_{g \in \mG} v_g\biggr),
    \end{align*}
    where the last inequality follows since $n_\mn \leq c n_{\text{maj}}$ for some $c < 1$, and $|b_{g,g'}| \lesssim \|\nabla \calB^{(g')}(\btheta_\rw)\|$.
    This completes the proof of Corollary~\ref{cor: synthetic data risk spurious}.
\end{proof}

\subsection{Explicit form of the synthetic data quality term}
\label{supp-sec: explicit quality form}

We derive the explicit form of the synthetic data quality term in Theorem~\ref{thm: synthetic data risk}, for group $g \in \mG$, under the settings of logistic and linear regressions, i.e., 
$$ 
q^{(g)} := \{\nabla \calR^{(g)}(\btheta_\b)\}^\top \{\nabla^2 \calR_\b(\btheta_\b)\}^{-1} \boldsymbol{b},
$$
where $\boldsymbol{b} = {|\mG|}^{-1} \sum_{g' \in \mG} \rho_{g'} \nabla \calB^{(g')}(\btheta_\b)$.

\subsubsection{Logistic regression}

Let $\ell(\btheta;\x,y) = -y\x^\top\btheta + \log\bigl(1+e^{\x^\top\btheta}\bigr)$ with $y\in\{0,1\}$, $s(t) := (1+e^{-t})^{-1}$, and $s_\b(\x) := s(\x^\top\btheta_\b)$. Suppose both the true and synthetic label conditionals follow logistic models,
$$
\P(Y^{(g)}=1\mid \x_i^{(g)}=\x) = s(\x^\top\btheta^{(g)}), \quad\quad 
\P(\tilde Y^{(g)}=1\mid \tilde \x_i^{(g)}=\x) = s(\x^\top\tilde\btheta^{(g)}),
$$
with possibly different covariate laws for $\x_i^{(g)}$ and $\tilde \x_i^{(g)}$. The group-$g$ risk, gradient, and Hessian are, respectively, 
\begin{align*}
    \calR^{(g)}(\btheta) &= \E\left[ -Y^{(g)}\x_i^{(g) \top}\btheta+\log(1+e^{\x_i^{(g) \top}\btheta}) \right],\\ 
    \nabla\calR^{(g)}(\btheta) &= \E\left[ \x_i^{(g)}\{s(\x_i^{(g) \top}\btheta)-Y^{(g)}\} \right],\\
    \nabla^2\calR^{(g)}(\btheta) &= \E\left[ \x_i^{(g)}\x_i^{(g) \top}s(\x_i^{(g) \top}\btheta)\{1-s(\x_i^{(g) \top}\btheta)\} \right].
\end{align*}
Evaluated at $\btheta_\b$, write 
$$
s^{(g)} := \E[\x_i^{(g)} s_\b(\x_i^{(g)})],  \mu^{(g)}_{xy} := \E[\x_i^{(g)}Y^{(g)}],  H_\b^{(g)}:=\E[\x_i^{(g)}\x_i^{(g) \top}s_\b(\x_i^{(g)})\{1-s_\b(\x_i^{(g)})\}].
$$
Then
\begin{align*}
    \nabla \calR^{(g)}(\btheta_\b) &= s^{(g)}-\mu^{(g)}_{xy},\ \ \nabla^2 \calR_\b(\btheta_\b) = \frac{1}{|\mG|}\sum_{g'\in\mG} H_\b^{(g')} =: H_\b.
\end{align*}

For the bias component, define the analogous synthetic moments $\tilde s^{(g)}:=\E[\tilde \x_i^{(g)}s_\b(\tilde \x_i^{(g)})]$ and $\tilde \mu^{(g)}_{xy}:=\E[\tilde \x_i^{(g)}\tilde Y^{(g)}]$.
A direct calculation gives $\nabla \calB^{(g)}(\btheta_\b) = \Big(\tilde s^{(g)}-s^{(g)}\Big) - \Big(\tilde \mu^{(g)}_{xy}-\mu^{(g)}_{xy}\Big)$. Therefore, $\boldsymbol{b} = \frac{1}{|\mG|}\sum_{g'\in\mG}\rho_{g'}\left[\big(\tilde s^{(g')}-s^{(g')}\big)-\big(\tilde \mu^{(g')}_{xy}-\mu^{(g')}_{xy}\big)\right]$. 

Finally, the group-specific quality term is
\begin{align*}
    q^{(g)} = \big(s^{(g)}-\mu^{(g)}_{xy}\big)^\top H_\b^{-1} \left\{\frac{1}{|\mG|}\sum_{g'\in\mG}\rho_{g'}\left[\big(\tilde s^{(g')}-s^{(g')}\big)-\big(\tilde \mu^{(g')}_{xy}-\mu^{(g')}_{xy}\big)\right]\right\}.
\end{align*}

\noindent
Since $\mu^{(g)}_{xy} = \E[\x_i^{(g)}s(\x_i^{(g) \top}\btheta^{(g)})]$ under the logistic model, the group gradient can be written as $\nabla \calR^{(g)}(\btheta_\b) = \E[\x_i^{(g)}\{s_\b(\x_i^{(g)})-s(\x_i^{(g) \top}\btheta^{(g)})\}]$, i.e., a score imbalance between the balanced model and the group's true conditional; $q^{(g)}$ then couples this imbalance with the synthetic-real mismatch in the score and label moments, scaled by the inverse Fisher metric $H_\b^{-1}$. We also note that the mismatch terms $\big(\tilde s^{(g')}-s^{(g')}\big)$ and $\big(\tilde \mu^{(g')}_{xy}-\mu^{(g')}_{xy}\big)$ measure how the synthetic data deviates from the real data in terms of feature-score alignment and label-feature alignment. Thus $q^{(g)}$ acts as a bias term measuring how much this synthetic-real distributional mismatch influences the downstream risk via its interaction with the group gradient and the curvature of the risk function through $H_\b^{-1}$.

\subsubsection{Linear regression}

Let $\ell(\btheta;\x,y)= (1/2) (y-\x^\top\btheta)^2$. Suppose the data generating model for the true and synthetic distribution data follow
\begin{align*}
y_i^{(g)} &= \x_i^{(g) \top} \btheta^{(g)} + \epsilon_i^{(g)},\quad i=1,\ldots,n_g,\\
\tilde y_i^{(g)} &= \tilde \x_i^{(g) \top} \tilde \btheta^{(g)} + \tilde \epsilon_i^{(g)},\quad i=1,\ldots,m_g.
\end{align*}
Suppose $\x_i^{(g)}$ is a mean-zero random vector, and $\epsilon_i^{(g)}$ is a mean-zero random variable, independent of $\x_i^{(g)}$. Likewise, $\tilde \x_i^{(g)}$ is a mean-zero random vector and $\tilde \epsilon_i^{(g)}$ is a mean-zero random variable, independent of $\tilde \x_i^{(g)}$. Suppose $S^{(g)} := \E[\x_i^{(g)} \x_i^{(g) \top}]$, and $\tilde S^{(g)} := \E[\tilde \x_i^{(g)} \tilde \x_i^{(g) \top}]$ exist and are symmetric positive definite matrices.
Then, 
\begin{align*}
    \calR^{(g)}(\btheta) &= \frac{1}{2} \mathbb{E}[(\x_i^{(g) \top} \btheta - \x_i^{(g) \top} \btheta^{(g)})^2] + \frac{1}{2} \E[\epsilon_i^{(g) 2}]\\
    &= \frac{1}{2} (\btheta - \btheta^{(g)})^\top S^{(g)} (\btheta - \btheta^{(g)}) + \frac{1}{2} \E[\epsilon_i^{(g) 2}].
\end{align*}
Thus, $\nabla \calR^{(g)}(\btheta) = S^{(g)} (\btheta - \btheta^{(g)}),\ \ \nabla^2 \calR^{(g)}(\btheta) = S^{(g)}$. This gives
\begin{align*}
    \nabla \calR^{(g)}(\btheta_\b) = \x_i^{(g)} ( \btheta_\b - \btheta^{(g)}),\ \ \nabla^2 \calR_\b(\btheta_\b) = \frac{1}{|\mG|} \sum_{g' \in \mG} S^{(g')}.
\end{align*}
By a similar argument, we have $\nabla \calB^{(g)}(\btheta_\b) = \tilde S^{(g)} (\btheta_\b - \tilde \btheta^{(g)}) - S^{(g)} (\btheta_\b - \btheta^{(g)})$. This gives
\begin{align*}
    \boldsymbol{b} = \frac{1}{|\mG|} \sum_{g' \in \mG} \rho_{g'} \left( \tilde S^{(g')} (\btheta_\b - \tilde \btheta^{(g')}) - S^{(g')} (\btheta_\b - \btheta^{(g')}) \right).
\end{align*}
Hence
\begin{align*}
    q^{(g)} = (\btheta_\b - \btheta^{(g)})^\top S^{(g)} \left(\frac{1}{|\mG|} \sum_{g' \in \mG} S^{(g')}\right)^{-1} \qty{ \frac{1}{|\mG|} \sum_{g' \in \mG} \rho_{g'} \left( (\tilde S^{(g')} - S^{(g')}) \btheta_\b - (\tilde \x_i^{(g')} \tilde \btheta^{(g')} - S^{(g')} \btheta^{(g')}) \right) }.
\end{align*}
When the covariance matrices are the same across all groups for synthetic data and raw data, i.e., $S^{(g)} = S$ and $\tilde S^{(g)} = S$, we can further simplify $q^{(g)}$ as
\begin{align*}
q^{(g)} = - \frac{1}{|\mG|} \sum_{g' \in \mG} \rho_{g'} (\btheta_\b - \btheta^{(g)})^\top S (\tilde \btheta^{(g')} - \btheta^{(g')}).
\end{align*}
Similar to the case of logistic regressoin, $q^{(g)}$ measures the mismatch between the synthetic data and the real data.

\subsection{Proofs of theory on synthetic augmentation in Section~\ref{sec: scaling law}}
\label{supp-sec: scaling law proof}

We prove Theorem~\ref{thm: scaling law}, the scaling law of the risk when we feed additional synthetic data equally to all groups. Corollaries~\ref{cor: scaling law imb} and \ref{cor: scaling law spu} follow directly and thus their proofs are omitted.

\begin{proof}[\textbf{Proof of Theorem~\ref{thm: scaling law}}]
In this proof, we aim to show the scaling law for the balanced excess risk $\calR_\b(\hat\btheta) - \calR_\b(\btheta_\b)$ in terms of $n_\T$ and $N$, where $\hat\btheta$ is the minimizer of the risk in \eqref{eq: general R with data augmentation}. We divide the proof into 4 steps. Step 1 approximates $\btheta^* - \btheta_\b$, where $\btheta^* := \argmin_\btheta \E[\calR(\btheta)]$ and $\calR(\btheta) := \E[\calRhat(\btheta)]$. Step 2 bounds the difference between $H := \nabla^2 \calR(\btheta^*)$ and $H_{\b,\b} := \nabla^2 \calR_\b(\btheta_\b)$. Step 3 approximates $\hat\btheta - \btheta^*$. Step 4 combines the previous steps and derives the approximation of $\hat\btheta - \btheta_\b$, and $\calR_\b(\hat\btheta) - \calR_\b(\btheta_\b)$.

For notational convenience, define $\calR_\o(\btheta) = \E[\calRhat_\o(\btheta)]$, $\calR_\a(\btheta) = \E[\calRhat_\a(\btheta)]$, and $H_\o := \nabla^2 \calR(\btheta_\o)$. Also define
\begin{align*}
\alpha_g := (1 - \alpha) \rho_g + \alpha,\ \  \kappa^2 := \frac{1}{|\mG|} \sum_{g \in \mG} \alpha_g \{\|\nabla \calB^{(g)}(\btheta_\b)\|^2 + \|\nabla^2 \calB^{(g)}(\btheta_\b)\|^2\}.
\end{align*}

    \paragraph{Step 1.}
    In this step, we approximate $\btheta^* - \btheta_\b$.
    By the same argument as in the proof of Theorem~\ref{thm: synthetic data risk}, we have $\|\btheta^* - \btheta_\b\| = o(1)$.
    By Taylor expansion, we have
    \begin{align*}
        0 = \nabla \calR(\btheta^*) = \nabla \calR(\btheta_\b) + \nabla^2 \calR(\btheta') (\btheta^* - \btheta_\b),
    \end{align*}
    where $\btheta'$ is in the line segment between $\btheta^*$ and $\btheta_\b$.
    By Assumption~\ref{asm: differentiability}, we have
    \begin{align*}
        0 = \nabla \calR(\btheta_\b) + \nabla^2 \calR(\btheta_\b) (\btheta^* - \btheta_\b) + R_1,
    \end{align*}
    where $\|R_1\| \lesssim \|\btheta^* - \btheta_\b\|^2$. This gives
    \begin{align}
        0 &= \nabla \calR(\btheta_\b) + \nabla^2 \calR_\b(\btheta_\b) (\btheta^* - \btheta_\b)\nonumber\\
        &\quad+ \underbrace{R_1 + \{\nabla^2 \calR(\btheta_\b) - \nabla^2 \calR_\b(\btheta_\b)\}(\btheta^* - \btheta_\b)}_{=: R_1'}.\label{eq: R1 prime}
    \end{align}
    Now, observe that for any $\btheta$,
    \begin{align}
        \calR(\btheta) &= (1 - \alpha) \sum_{g \in \mG} \frac{n_g}{n_\T + m_\T} \calR^{(g)}(\btheta) + (1 - \alpha) \sum_{g \in \mG} \frac{m_g}{n_\T + m_\T} \{\calR^{(g)}(\btheta) + \calB^{(g)}(\btheta)\}\nonumber\\
        &\quad+ \alpha \sum_{g \in \mG} \frac{1}{|\mG|} \{\calR^{(g)}(\btheta) + \calB^{(g)}(\btheta)\}\nonumber\\
        &= \calR_\b(\btheta) + \frac{1}{|\mG|} \sum_{g \in \mG} \{(1 - \alpha) \rho_g + \alpha\} \calB^{(g)}(\btheta).\label{eq: calR and calR bal}
    \end{align}
    Thus, we can bound $R_1'$ as $\|R_1'\| \lesssim \|\btheta^* - \btheta_\b\|^2 + \|\btheta^* - \btheta_\b\| \kappa \lesssim \|\btheta^* - \btheta_\b\|^2 + \kappa^2$. Furthermore, \eqref{eq: R1 prime} gives
    \begin{align}
        \btheta^* - \btheta_\b = - H_{\b,\b}^{-1} \{\nabla \calR(\btheta_\b) + R_1'\}.\label{eq: theta * - theta b}
    \end{align}
    Again from Assumption~\ref{asm: differentiability}, we obtain $\|\btheta^* - \btheta_\b\| = O(\kappa)$ and hence $\|R_1'\| = O(\kappa^2)$.
    
    \paragraph{Step 2.}
    In this step, we bound $\|H - H_{\b,\b}\|$. 
    From \eqref{eq: calR and calR bal},
    \begin{align*}
        H - H_{\b,\b} &= \nabla^2 \calR_\b(\btheta^*) + \frac{1}{|\mG|} \sum_{g \in \mG} \{(1 - \alpha) \rho_g + \alpha\} \nabla^2 \calB^{(g)}(\btheta^*) - \nabla^2 \calR_\b(\btheta_\b)\\
        &\quad+ \frac{1}{|\mG|} \sum_{g \in \mG} \{(1 - \alpha) \rho_g + \alpha\} \{\nabla^2 \calB^{(g)}(\btheta_\b) - \nabla^2 \calB^{(g)}(\btheta^*)\}.
    \end{align*}
    From Step 1, $\|\btheta_\b - \btheta^*\| \lesssim \kappa$. 
    By Assumption~\ref{asm: differentiability},
    \begin{align}
        \|H - H_{\b,\b}\| &\lesssim \|\btheta^* - \btheta_\b\| + \frac{1}{|\mG|} \sum_{g \in \mG} \alpha_g \|\nabla^2 \calB^{(g)}(\btheta_\b)\| \lesssim \kappa.\label{eq: H o and H bal bal}
    \end{align}
    and thus $\|H^{-1} - H_{\b,\b}^{-1}\| \lesssim \kappa$.
    
    \paragraph{Step 3.}
    In this step, we derive the approximation of $\hat\btheta - \btheta^*$.
    From a standard argument in $M$-estimation as in the proof of Theorem~\ref{thm: synthetic data risk}, we have $\|\hat \btheta - \btheta^*\| = o_p(1)$, and
    \begin{align*}
        \hat \btheta - \btheta^* &= - H^{-1} \nabla \calRhat(\btheta^*) + R_2,
    \end{align*}
    where
    \begin{align*}
        \|R_2\| = O_p\qty(\frac{(1-\alpha)^2 \tr(\Sigma)}{n_\T + m_\T} + \frac{\alpha^2 \tr(\Sigma')}{N|\mG|}).
    \end{align*}
    Using the result from Step 2, we have
    \begin{align}
        \hat \btheta - \btheta^* &= - (1-\alpha) H_{\b,\b}^{-1} \qty{\frac{1}{n_\T + m_\T} \sum_{g \in \mG} \sum_{i \in [n_g]} \nabla \ell(\btheta^*; z_i^{(g)}) + \frac{1}{n_\T + m_\T} \sum_{g \in \mG} \sum_{i \in [m_g]} \nabla \ell(\btheta^*; \tilde z_i^{(g)})}\nonumber\\
        &\quad- \alpha H_{\b,\b}^{-1} \frac{1}{N|\mG|} \sum_{g \in \mG} \sum_{i \in [N+m_g]\setminus[m_g]} \nabla \ell(\btheta^*; \tilde z_i^{(g)}) + R_2'\nonumber\\
        &= - (1-\alpha) H_{\b,\b}^{-1} \qty{\frac{1}{n_\T + m_\T} \sum_{g \in \mG} \sum_{i \in [n_g]} \nabla \ell(\btheta^*; z_i^{(g)}) + \frac{1}{n_\T + m_\T} \sum_{g \in \mG} \sum_{i \in [m_g]} \nabla \ell(\btheta^*; \tilde z_i^{(g)})}\nonumber\\
        &\quad- \alpha H_{\b,\b}^{-1} \frac{1}{N|\mG|} \sum_{g \in \mG} \sum_{i \in [N+m_g]\setminus[m_g]} \nabla \ell(\btheta^*; \tilde z_i^{(g)}) + R_2',\label{eq: hat theta - theta *}
    \end{align}
    where
    \begin{align*}
        \|R_2'\| &\lesssim \|R_2\| + \kappa \|\nabla \calRhat(\btheta^*)\| = O_p\qty(\frac{(1-\alpha)^2 \tr(\Sigma)}{n_\T + m_\T} + \frac{\alpha^2 \tr(\Sigma')}{N|\mG|}) + O(\kappa^2).
    \end{align*}

    \paragraph{Step 4.}
    Lastly, from \eqref{eq: theta * - theta b} and \eqref{eq: hat theta - theta *}, we have
    \begin{small}
    \begin{align*}
        \hat\btheta - \btheta_\b &= \underbrace{- \frac{1}{|\mG|} \sum_{g \in \mG} \{(1 - \alpha) \rho_g + \alpha\} H_{\b,\b}^{-1} \nabla \calB^{(g)}(\btheta_\b)}_{=: T}\\
        &\quad \underbrace{- (1-\alpha) \frac{1}{n_\T + m_\T} \sum_{g \in \mG} \sum_{i \in [n_g]} H_{\b,\b}^{-1} \qty(\nabla \ell(\btheta^*; z_i^{(g)}) - \E[\nabla \ell(\btheta^*; z_i^{(g)})])}_{=: Q_1}\\
        &\quad \underbrace{- (1 - \alpha) \frac{1}{n_\T + m_\T} \sum_{g \in \mG} \sum_{i \in [m_g]} H_{\b,\b}^{-1} \qty( \nabla \ell(\btheta^*; \tilde z_i^{(g)}) - \E[\nabla \ell(\btheta^*; \tilde z_i^{(g)})])}_{=: Q_2}\\
        &\quad \underbrace{- \alpha \frac{1}{N|\mG|} \sum_{g \in \mG} \sum_{i \in [N+m_g]\setminus[m_g]} H_{\b,\b}^{-1} \qty(\nabla \ell(\btheta^*; \tilde z_i^{(g)}) - \E[\nabla \ell(\btheta^*; \tilde z_i^{(g)})])}_{=: Q_3} + R,
    \end{align*}
    \end{small}\noindent
    where
    \begin{align*}
        \|R\| = O_p\qty(\frac{(1 - \alpha)^2 \tr(\Sigma)}{n_\T + m_\T} + \frac{\alpha^2 \tr(\Sigma')}{N|\mG|}) + O(\kappa^2).
    \end{align*}
    Note that $\|T\| = O(\kappa)$, $\E[Q_1] = \E[Q_2] = \E[Q_3] = 0$ and $Q_1$, $Q_2$, $Q_3$ are independent sums of independent random vectors.
    For the excess risk, we have
    \begin{align*}
        \calR_\b(\hat\btheta) - \calR_\b(\btheta_\b) &= \frac{1}{2} (\hat\btheta - \btheta_\b)^\top \nabla^2 \calR_\b(\btheta'') (\hat\btheta - \btheta_\b),
    \end{align*}
    where $\btheta''$ lies between $\btheta_\b$ and $\hat\btheta$. This gives
    \begin{align*}
        \calR_\b(\hat\btheta) - \calR_\b(\btheta_\b) &= \frac{1}{2} (\hat\btheta - \btheta_\b)^\top H_{\b,\b} (\hat\btheta - \btheta_\b) + O(\|\hat\btheta - \btheta_\b\|^3).
    \end{align*}
    Observe that for any $j$-th standard basis vector $\e_j \in \R^d$,
    \begin{align*}
        \{e_j^\top H_{\b,\b}^{1/2} (Q_1 + Q_2 + Q_3)\}^2 &= O_p\biggl(\sum_{g \in \mG} \frac{(1 - \alpha)^2}{(n_\T + m_\T)^2} n_g \e_j^\top H_{\b,\b}^{-1/2} \Sigma_g H_{\b,\b}^{-1/2} \e_j\\
        &\quad+ \sum_{g \in \mG} \frac{(1 - \alpha)^2}{(n_\T + m_\T)^2} m_g \e_j^\top H_{\b,\b}^{-1/2} \tilde \Sigma_g H_{\b,\b}^{-1/2} \e_j\\
        &\quad+ \sum_{g \in \mG} \frac{\alpha^2}{N^2 |\mG|^2} N \e_j^\top H_{\b,\b}^{-1/2} \tilde \Sigma_g H_{\b,\b}^{-1/2} \e_j \biggr).
    \end{align*}
    Therefore, by Assumption~\ref{asm: differentiability},
    \begin{align*}
        &\calR_\b(\hat\btheta) - \calR_\b(\btheta_\b)\\
        &\quad= \frac{1}{2} (T + R)^\top H_{\b,\b} (T + R) + \frac{1}{2} (Q_1 + Q_2 + Q_3)^\top H_{\b,\b} (Q_1 + Q_2 + Q_3) + O(\|\hat\btheta - \btheta_\b\|^3)\\
        &\quad\leq T^\top H_{\b,\b} T + O(\|R\|^2) + O(\|\hat\btheta - \btheta_\b\|^3)\\
        &\quad\quad+ O_p\biggl(\frac{(1-\alpha)^2}{n_\T + m_\T} \sum_{g \in \mG} \frac{n_g}{n_\T + m_\T} \tr(H_{\b,\b}^{-1} \Sigma_g(\btheta_\b))\\
        &\quad\quad+ \frac{(1 - \alpha)^2}{n_\T + m_\T} \sum_{g \in \mG} \frac{m_g}{n_\T + m_\T} \tr(H_{\b,\b}^{-1} \tilde \Sigma_g(\btheta_\b))\\
        &\quad\quad+ \frac{\alpha^2}{N |\mG|} \sum_{g \in \mG} \frac{1}{|\mG|} \tr(H_{\b,\b}^{-1} \tilde \Sigma_g(\btheta_\b))\biggr)\\
        &\quad= T^\top H_{\b,\b} T + O_p\qty(\frac{(1-\alpha)^2}{\max_{g' \in \mG} n_{g'} |\mG|} \tr(\Sigma) + \frac{\alpha^2}{N |\mG|} \tr(\Sigma')) + R',
    \end{align*}
    where
    \begin{align*}
        |R'| = O_p\qty(\frac{(1 - \alpha)^3 \tr(\Sigma)^{3/2}}{(\max_{g' \in \mG} n_{g'} |\mG|)^{3/2}} + \frac{\alpha^3 \tr(\Sigma')^{3/2}}{(N|\mG|)^{3/2}}) + O(\kappa^3).
    \end{align*}
    Using Cauchy-Schwarz inequality combined with $\alpha_g \leq 1$,
    \begin{align*} 
        \kappa^3 &= \qty[\frac{1}{|\mG|} \sum_{g \in \mG} \alpha_g \{\|\nabla \calB^{(g)}(\btheta_\b)\|^2 + \|\nabla^2 \calB^{(g)}(\btheta_\b)\|^2\}]^{3/2}\\
        &\lesssim \frac{1}{|\mG|} \sum_{g \in \mG} \alpha_g \{\|\nabla \calB^{(g)}(\btheta_\b)\|^3 + \|\nabla^2 \calB^{(g)}(\btheta_\b)\|^3\}.
    \end{align*}
Hence, by Assumption~\ref{asm: differentiability}, 
\vspace{-0.01in}
    \begin{align*}
        \calR_\b(\hat\btheta) - \calR_\b(\btheta_\b) &\leq \norm{\frac{1}{|\mG|} \sum_{g \in \mG} \{(1 - \alpha) \rho_g + \alpha\} H_{\b,\b}^{-1/2} \nabla \calB^{(g)}(\btheta_\b)}^2\\
        &\quad+ O_p\qty( (1 - \alpha)^2 (1 - \rho) \frac{\tr(\Sigma)}{n_\T} + \alpha^2 \frac{\tr(\Sigma')}{N |\mG|}) + O(\kappa^3)\\
        &\lesssim \frac{1}{|\mG|} \sum_{g \in \mG} \{(1 - \alpha) \rho_g + \alpha\} \{\|\nabla \calB^{(g)}(\btheta_\b)\|^2 + \|\nabla \calB^{(g)}(\btheta_\b)\|^3\}\\
        &\quad+ O_p\qty( (1 - \alpha)^2 (1 - \rho) \frac{\tr(\Sigma)}{n_\T} + \alpha^2 \frac{\tr(\Sigma')}{N |\mG|}).
    \end{align*}
This completes the proof of Theorem~\ref{thm: scaling law}.
\end{proof}

\subsection{Scaling laws for additional statistical models}
\label{supp-sec: additional scaling law}

We provide theory of the scaling laws for additional models, including a Gaussian mixture model for imbalanced classification, and a nonparametric model for spurious correlation.

\subsubsection{Gaussian mixture model for imbalanced classification}
\label{sec: additional scaling law imb}

We consider a Gaussian mixture model with binary labels $\mG = \mathcal{Y} = \{0, 1\}$. Suppose that we observe raw data $(\z_i^{(y)})_{i \in [n_y]}$ for each label $y \in \{0, 1\}$, and we have access to synthetic data $(\tilde \z_i^{(y)})_{i \in [N + m_y]}$ for each label $y \in \{0, 1\}$.
Let $\P(y_i = 1) = \pi_1 > 0$ and $\P(y_i = 0) = \pi_0 > 0$.
We assume that $\z_i^{(y)}$ and $\tilde \z_i^{(y)}$ follow:
\begin{align}
    \z_i^{(y)} &= (-1)^y \btheta^* + \sigma \bzeta_i^{(y)},\ \ \tilde \z_i^{(y)} = (-1)^y \tilde \btheta^* + \sigma \tilde \bzeta_i^{(y)},\label{model: gaussian sequence}
\end{align}
where $\bzeta_i^{(y)}$, and $\tilde \bzeta_i^{(y)}$ across all $y \in \mathcal{Y}$ are i.i.d.\ standard Normal random vectors, and $\sigma > 0$.

In practice, $\btheta^*$ can be estimated using observed data; $1/(n_0 + n_1) \sum_{i \in [n_0 + n_1]} (-1)^{y_i} z_i^{(y_i)}$. This estimate is then incorporated into the decision rule. However, in this section, we instead oversample and augment the imbalanced data, and investigate the decision rule 
\begin{align*}
    \1\{\hat\btheta^\top \z < 0\}
\end{align*}
with an estimated $\hat\btheta$. We obtain $\hat\btheta$ by minimizing the combined loss with regularization,
\begin{small}
\begin{align}
    \calLhat(\btheta) &:= (1 - \alpha) \qty{\frac{1}{n_\T + m_\T} \sum_{y \in \mathcal{Y}} \sum_{i \in [n_y]} \|(-1)^y \z_i^{(y)} - \btheta\|^2 + \frac{1}{n_\T + m_\T} \sum_{y \in \mathcal{Y}} \sum_{i \in [m_y]} \|(-1)^y \tilde \z_i^{(y)} - \btheta\|^2}\nonumber\\
    &\quad+ \alpha \frac{1}{N|\mG|} \sum_{y \in \mathcal{Y}} \sum_{i \in [N+m_y]\setminus[m_y]} \|(-1)^y \tilde \z_i^{(y)} - \btheta\|^2 + \lambda \sum_{j\geq 1} \omega_j (\btheta)_j^2,\label{loss: gaussian sequence}
\end{align}
\end{small}\noindent
where $\omega_j > 0$ is a regularization weight for the $j$-th component of $\btheta$.
Let $\hat\btheta$ be the minimizer of $\calLhat$.
To assess the effectiveness of the decision rule, we analyze the balanced misclassification error defined as follows:
\begin{align*}
    \calR_\b(\btheta) := \frac{1}{2} \E[\1\{\btheta^\top \z_1^{(0)} < 0\}] + \frac{1}{2} \E[\1\{\btheta^\top \z_1^{(1)} \geq 0\}].
\end{align*}
We compare the balanced misclassification error of $\hat \btheta$ with the misclassification error of $\btheta^*$.

We introduce the following assumption.

\begin{assumption}\label{asm: gaussian sequence theta and tilde theta}
Suppose $\sup_{j \geq 1} j^{2r+1} (\btheta^*)_j^2 \vee \sup_{j \geq 1} j^{2r+1} (\tilde \btheta^*)_j^2 = O(1)$ holds for some integer $r \geq 1$.
\end{assumption}

Assumption~\ref{asm: gaussian sequence theta and tilde theta} ensures a controlled decay rate for the components of $\btheta^*$. The same assumption is imposed in the analysis of Gaussian sequence model \citep{jain2024scaling}. Define 
\begin{align}
R_{n_\T,N} := (1 - \alpha)^2 \frac{\sigma^2}{n_\T} (1 - \rho) + \alpha^2 \frac{\sigma'^2}{N |\mG|}.\label{eq: R n N}
\end{align}
Define
\begin{align*}
    \sigma^2 &:= \frac{1}{|\mG|} \sum_{g \in \mG} \qty{(1 - \rho_g) \sigma_g^2 + \rho_g \tilde \sigma_g^2},\ \ \sigma'^2 := \frac{1}{|\mG|} \sum_{g \in \mG} \tilde \sigma_g^2.
\end{align*}

We then state the following result on the scaling law of the balanced misclassification error.

\begin{theorem}\label{thm: gaussian sequence scaling law}
Suppose Assumption~\ref{asm: gaussian sequence theta and tilde theta} holds, and $R_{n_\T,N} = o(1)$. Choose $(\omega_j)_{j \geq 1}$, such that $\omega_j \asymp j^p$ holds with some integer $p$ satisfying $p \geq 2$ and $p \neq r$. Let $\beta = 2r' / (2r' + 1)$ with $r' := p \wedge r$. Then, there exists a constant $C > 0$, such that, for any $\delta > 0$, choosing $\lambda \asymp R_{n_\T,N}^{p/(2r' + 1)}$ yields
\begin{align*}
\E[\mathcal{R}_\b(\hat\btheta) - \mathcal{R}_\b(\btheta^*)] &\leq \qty(1 + \frac{\delta}{2}) C R_{n_\T,N}^{\beta/2} + \qty(1 + \frac{1}{2\delta}) \norm{\frac{1}{|\mG|} \sum_{y \in \mathcal{Y}} \qty{(1 - \alpha) \rho_y + \alpha} \{\btheta^{(y)} - \tilde \btheta^{(y)}\}}.
\end{align*}
\end{theorem}

Theorem~\ref{thm: gaussian sequence scaling law} highlights that the exponent of the statistical error depends on the regularization decay rate $p$ and the decay rate $r$ of $\btheta^*$. Namely, the statistical error component depends on $n_T$ and $N$ by $O\left( (1/n_\T + 1/N)^{\beta/2} \right)$. Note that the bias term depends on the synthetic data quality $\tilde \btheta^{(y)} - \btheta^{(y)}$ weighted by the ratio of the total added synthetic data size over all samples $(1 - \alpha) \rho_g + \alpha \in [0, 1]$. In addition, following a similar proof, the same upper bound extends to the minority misclassification error $\calR_0(\btheta) := \E[\1\{\btheta^\top \z_1^{(0)} < 0\}]$.

We also remark that, if we consider Theorem~\ref{thm: gaussian sequence scaling law} in the regime where the synthetic bias term $\|\btheta^{(y)} - \tilde\btheta^{(y)}\|^2$ is negligible relative to the statistical term $R_{n_\T,N}^{\beta}$, and if we choose $\alpha \simeq \sqrt{N|\mG|}/(\sqrt{N|\mG|} + \sqrt{n_\T})$, then we minimize the leading variance contribution, and obtain the statistical error
\begin{align*}
\left\{\frac{\sigma^2 + \sigma'^2}{n_\T + N|\mG|}\right\}^{\beta}.
\end{align*}
Therefore, if $\max_{y \in \mY} \|\btheta^{(y)} - \tilde\btheta^{(y)}\|^2 = o(R_{n_\T,N}^{\beta-\upsilon})$ for some $\upsilon>0$, we obtain the scaling law that
\begin{align*}
        \log \E[\mathcal{R}_\b(\hat\btheta) - \mathcal{R}_\b(\btheta^*)] \leq - \frac{\beta}{2} \log(n_\T + N|\mG|) + C,
\end{align*}
where $C$ is a constant independent of $n_\T$ and $N$. Moreover, in the augmentation-dominated regime $N|\mG| \gg n_\T$, it further simplifies to a linear decay in $\log N$.

\begin{proof}[\textbf{Proof of Theorem~\ref{thm: gaussian sequence scaling law}}]
We first bound the balanced misclassification error by $\|\btheta - \btheta^*\|$. For any $\btheta$, since $\z_1^{(0)} = \btheta^* + \sigma \bzeta_1^{(0)}$, we have
    \begin{align}
        \E[\1\{\btheta^\top \z_1^{(0)} < 0\}] &= \P(\btheta^\top \z_1^{(0)} < 0) = \P\qty(\frac{\btheta^\top \z_1^{(0)} - \btheta^\top \btheta^*}{\sigma \|\btheta\|} < -\frac{\btheta^\top \btheta^*}{\sigma \|\btheta\|}) = \Phi\qty(-\frac{\btheta^\top \btheta^*}{\sigma \|\btheta\|}).
    \end{align}
    This gives
    \begin{align}
        \E[\1\{\btheta^\top \z_1^{(0)} < 0\}] - \E[\1\{\btheta^{* \top} \z_1^{(0)} < 0\}] = \Phi\qty(-\frac{\btheta^\top \btheta^*}{\sigma \|\btheta\|}) - \Phi\qty(- \frac{\|\btheta^*\|}{\sigma}).
    \end{align}
    Since $\Phi$ is a $(1/\sqrt{2\pi})$-Lipschitz monotone increasing function, and $\btheta^\top \btheta^*/\|\btheta\| \leq \|\btheta^*\|$ by Cauchy-Schwarz inequality, we obtain
    \begin{align}
        0 &\leq \E[\1\{\btheta^\top \z_1^{(0)} < 0\}] - \E[\1\{\btheta^{* \top} \z_1^{(0)} < 0\}] \nonumber \\
        &\leq \frac{1}{\sqrt{2\pi}} \qty(\frac{\|\btheta^*\|}{\sigma} - \frac{\btheta^\top \btheta^*}{\sigma \|\btheta\|}) \nonumber \\
        &\leq \frac{1}{\sigma \sqrt{2\pi}} \qty{\qty(\frac{1}{\|\btheta^*\|} - \frac{1}{\|\btheta\|}) \btheta^\top \btheta^* - \frac{1}{\|\btheta^*\|} (\btheta - \btheta^*)^\top \btheta^*} \nonumber \\
        &\leq \frac{1}{\sigma \sqrt{2\pi}} \qty{\qty(\frac{|\|\btheta\| - \|\btheta^*\||}{\|\btheta\| \|\btheta^*\|}) |\btheta^\top \btheta^*| + \frac{1}{\|\btheta^*\|} |(\btheta - \btheta^*)^\top \btheta^*|}  \nonumber \\
        &\leq \sqrt{\frac{2}{\sigma^2 \pi}} \|\btheta - \btheta^*\|,\label{eq: z0}
    \end{align}

    By a similar argument, we have
    \begin{align}
        0 &\leq \E[\1\{\btheta^\top \z_1^{(1)} \geq 0\}] - \E[\1\{\btheta^{* \top} \z_1^{(1)} \geq 0\}] \leq \sqrt{\frac{2}{\sigma^2 \pi}} \|\btheta - \btheta^*\|.\label{eq: z1}
    \end{align}
    Taking the expectation over an average of \eqref{eq: z0} and \eqref{eq: z1} gives
    \begin{align*}
        \E[\calR_\b(\btheta) - \calR_\b(\btheta^*)] \lesssim \E[\|\btheta - \btheta^*\|] \leq \qty(\E[\|\btheta - \btheta^*\|^2])^{1/2}.
    \end{align*}
Then, the conclusion follows from Theorem~\ref{thm: gaussian sequence scaling law ap}. This completes the proof of Theorem~\ref{thm: gaussian sequence scaling law}. 
\end{proof}

\begin{theorem}\label{thm: gaussian sequence scaling law ap}
Suppose Assumption~\ref{asm: gaussian sequence theta and tilde theta} holds, and $R_{n_\T,N} = o(1)$. Choose $(\omega_j)_{j \geq 1}$ satisfying $\omega_j \asymp j^p$ with some integer $p \in \{2, 3, 4, \dots\} \setminus \{r\}$. Let $\beta = 2r' / (2r' + 1)$ with $r' := p \wedge r$.  Then, there exists a constant $C > 0$, such that, for any $\delta > 0$, choosing $\lambda \asymp R_{n_\T,N}^{p/(2r' + 1)}$ yields
\begin{align*}
        \E[\|\hat \btheta - \btheta^*\|^2] &\leq (1 + \delta) C R_{n_\T,N}^\beta + \qty(1 + \frac{1}{\delta}) \norm{\frac{1}{|\mY|} \sum_{y \in \mY} \qty{(1 - \alpha) \rho_y + \alpha} \{\btheta^{(y)} - \tilde \btheta^{(y)}\}}^2.
\end{align*}
Furthermore, there exists a constant $C > 0$, such that, for any $\upsilon > 0$, choosing $\lambda \asymp R_{n_\T,N}^{p/(2r' + 1) + \upsilon p}$ yields
\begin{align*}
\E[\|\hat \btheta - \btheta^*\|^2] &\gtrsim C' R_{n_\T,N}^{\beta - \upsilon} + \norm{\frac{1}{|\mY|} \sum_{y \in \mY} \qty{(1 - \alpha) \rho_y + \alpha} \{\btheta^{(y)} - \tilde \btheta^{(y)}\}}^2.
\end{align*}
\end{theorem}

\begin{proof}[\textbf{Proof of Theorem~\ref{thm: gaussian sequence scaling law ap}}]
We first write the risk $\E[\|\hat \btheta - \btheta^*\|^2]$ in terms of $\tilde\btheta$ and $\btheta^*$. We then provide the proofs of the upper bound and lower bound separately. This proof is analogous to that of Theorem~\ref{thm: nonparametric scaling law ap}.

Let $\hat \btheta$ be the minimizer of $\calRhat$ defined in \eqref{loss: gaussian sequence}. Define
    \begin{align*}
        \z^{(y)} := (-1)^y \frac{1}{n_y} \sum_{i \in [n_y]} \z_i^{(y)}, \ \ \tilde \z^{(y)} := (-1)^y \frac{1}{m_y} \sum_{i \in [m_y]} \z_i^{(y)}, \check \z^{(y)} := (-1)^y \frac{1}{N} \sum_{i \in [N+m_y]\setminus[m_y]} \tilde \z_i^{(y)}.
    \end{align*}
    Note that
    \begin{small}
    \begin{align*}
        (\hat \btheta)_j &= \frac{1}{1 + \lambda \omega_j} \sum_{y \in \mY} \qty{(1 - \alpha) \frac{1}{n_\T + m_\T} \sum_{i \in [n_y]} (-1)^y \z_i^{(y)} + (1 - \alpha) \frac{1}{n_\T + m_\T} \sum_{i \in [m_y]} (-1)^y \tilde \z_i^{(y)}}\\
        &\quad+ \frac{1}{1 + \lambda \omega_j} \sum_{y \in \mY} \alpha \frac{1}{|\mY| N} \sum_{i \in [N+m_y]\setminus[m_y]} (-1)^y \tilde \z^{(y)},\\
        &= \frac{1}{1 + \lambda \omega_j} \frac{1}{|\mY|} \sum_{y \in \mY} \qty{(1 - \alpha) \frac{n_y}{\max_{y' \in \mY} n_{g'}} \z^{(y)} + (1 - \alpha) \qty(1 - \frac{n_y}{\max_{y' \in \mY} n_{g'}}) \tilde \z^{(y)} + \alpha \tilde \z^{(y)}},
    \end{align*}
    \end{small}\noindent
    where $m_y = \max_{y' \in \mY} n_{g'} - n_y$ and $n_\T + m_\T = |\mY| \max_{y' \in \mY} n_{g'}$.
    Observe that
    \begin{align}
        \E[\|\hat\btheta - \btheta^*\|] &= \E\qty[\sum_j |(\hat \btheta)_j - (\btheta^*)_j|^2]\nonumber\\
        &\quad= \underbrace{\sum_j \qty[(s_j - 1) \btheta^* + s_j \frac{1}{|\mY|} \sum_{y \in \mY} \qty{(1 - \alpha) \rho_y + \alpha} ((\tilde \btheta)_j - (\btheta^*)_j)]^2}_{=: T_1}\nonumber\\
        &\quad\quad+ \underbrace{\sum_j s_j^2 \frac{1}{|\mY|^2} \sum_{y \in \mY} \qty{(1 - \alpha)^2 \frac{\sigma^2}{n_y} (1 - \rho_y)^2 + (1 - \alpha)^2 \frac{\tilde \sigma^2}{m_y} \rho_y^2 + \alpha^2 \frac{\tilde \sigma^2}{N}}}_{=: T_2},\label{eq: risk decomposition gaussian sequence}
    \end{align}
    where $s_j := 1/(1 + \lambda \omega_j)$.

We next establish the upper bound and the lower bound, respectively. 
    
\paragraph{Proof of the upper bound.}
We first bound $T_2$. Write $T_2 =: R_{n_\T,N} \sum_j s_j^2$. Note that, when $\lambda < 1$, we have
    \begin{align*}
        \sum_j s_j^2 &= \sum_{j \leq \floor{\lambda^{-1/p}}} s_j^2 + \sum_{j > \floor{\lambda^{-1/p}}} s_j^2\\
        &\leq \sum_{j \leq \floor{\lambda^{-1/p}}} 1 + 1 + \sum_{j > \ceil{\lambda^{-1/p}}} \frac{1}{1 + \lambda \omega_j} 
        \leq \lambda^{-1/p} + 1 + \int_{\lambda^{-1/p}}^\infty \frac{1}{c_2 \lambda x^p} \dd{x} \lesssim \lambda^{-1/p}.
    \end{align*}
Therefore, 
    \begin{align}
        T_2 \lesssim R_{n_\T,N} \lambda^{-1/p}.\label{eq: T2 gaussian sequence}
    \end{align}

Next, we bound $T_1$. Note that for any $\delta > 0$, $(a + b)^2 \leq (1 + \delta) a^2 + (1 + \delta^{-1}) b^2$. Thus,
    \begin{align*}
        T_1 &\leq (1 + \delta) \underbrace{\sum_j (s_j - 1)^2 (\btheta^*)_j^2}_{=: T_{1,1}}\\
        &\quad+ (1 + \delta^{-1}) \underbrace{\sum_j s_j^2 \qty[\frac{1}{|\mY|} \sum_{y \in \mY} \qty{(1 - \alpha) \qty(1 - \frac{n_y}{\max_{y' \in \mY} n_{g'}}) + \alpha} ((\tilde \btheta^{(y)})_j - (\btheta^{(y)})_j)]^2}_{=: T_{1,2}}.
    \end{align*}
Let $\eta > 1$ be some number which will be decided later. For the term $T_{1,1}$, we have 
    \begin{align}
        T_{1,1} &= \sum_{j: j \leq \floor{\eta}} \qty(\frac{\lambda \omega_j}{1 + \lambda \omega_j})^2 (\btheta^*)_j^2 + \sum_{j: j > \floor{\eta}} \qty(\frac{\lambda \omega_j}{1 + \lambda \omega_j})^2 (\btheta^*)_j^2
        \leq \sum_{j: j \leq \floor{\eta}} (\lambda \omega_j)^2 (\btheta^*)_j^2 + \sum_{j: j > \floor{\eta}}(\btheta^*)_j^2\nonumber\\
        &\lesssim \lambda^2 \sum_{j: j \leq \floor{\eta}} j^{2p} j^{-2r-1} + \sum_{j: j > \floor{\eta}} j^{-2r-1}
        \lesssim \lambda^2 \eta^{(2p-2r) \vee 0} + \eta^{-2r},
    \end{align}
where we use $(\btheta^*)_j \lesssim j^{-2r-1}$ and $\omega_j \lesssim j^p$ in the second last inequality. The third inequality follows since (1) $\sum_{j > \floor{\eta}} j^{-2r-1} \leq \eta^{-2r}$, and (2) $\sum_{j: j \leq \floor{\eta}} j^{2p} j^{-2r-1} \leq \eta^{2p-2r}$ when $p-r \geq 1$ and $\sum_{j: j \leq \floor{\eta}} j^{2p} j^{-2r-1} \lesssim 1$ when $p-r \leq -1$.
    Choosing $\eta \gets \lambda^{-1/(p \vee r)}$ gives
    \begin{align}
        T_{1,1} \lesssim \lambda^{2r/(p \vee r)}.\label{eq: T11 gaussian sequence}
    \end{align}
    Denote $\alpha_y := (1 - \alpha) \rho_y + \alpha$.
    For the term $T_{1,2}$, write $T_{1,2} =: \sum_j s_j^2 b_j^2$, where
    \begin{align}
        b_j := \frac{1}{|\mY|} \sum_{y \in \mY} \alpha_y \{(\tilde \btheta^{(y)})_j - (\btheta^{(y)})_j\}.\label{eq: bj definition}
    \end{align}
    We bound $T_{1,2}$ as
    \begin{align}
        T_{1,2} \leq \sum_j b_j^2.\label{eq: T12 gaussian sequence}
    \end{align}
    
Combining \eqref{eq: T2 gaussian sequence}, \eqref{eq: T11 gaussian sequence}, and \eqref{eq: T12 gaussian sequence}, we have
    \begin{align*}
        \E[\|\hat \btheta - \btheta^*\|^2] \lesssim \qty{ (1 + \delta) \lambda^{2r/(p \vee r)} +  R_{n_\T,N} \lambda^{-1/p} } + (1 + \delta^{-1}) \sum_j b_j^2.
    \end{align*}
    Choosing $\lambda \asymp R_{n_\T,N}^{p/(2r' + 1)}$, where $r' := r \wedge p$ gives
    \begin{align*}
        \E[\|\hat \btheta - \btheta^*\|^2] &\leq C (1 + \delta) R_{n_\T,N}^{2r'/(2r' + 1)} + (1 + \delta^{-1}) \sum_j b_j^2,
    \end{align*}
    where $C$ is some constant. 
    
Following the same argument as in the proof of Theorem~\ref{thm: nonparametric scaling law ap},
    \begin{align*}
        R_{n_\T,N} &= (1 - \alpha)^2 \frac{\sigma^2}{n_\T} (1 - \rho) + \alpha^2 \frac{\sigma'^2}{N |\mY|}.
    \end{align*}
This completes the upper bound.

\paragraph{Proof of the lower bound.}
Recall that $\E[\|\hat \btheta - \btheta^*\|^2] = T_1 + T_2$ by \eqref{eq: risk decomposition gaussian sequence}. We first bound $T_2$ from below. Write $T_2 =: R_{n_\T,N} \sum_j s^2$. Note that when $\lambda = o(1)$, we have
    \begin{align*}
        \sum_j s_j^2 &= \sum_j \frac{1}{(1 + \lambda \omega_j)^2} \geq \sum_{j \leq \floor{\lambda^{-1/p}}} \frac{1}{(1 + \lambda \omega_j)^2} \gtrsim \lambda^{-1/p} - 1 \gtrsim \lambda^{-1/p},
    \end{align*}
    where we used $\omega_j \lesssim j^p$ by assumption.
    Hence
    \begin{align}
        T_2 \gtrsim R_{n_\T,N} \lambda^{-1/p}.\label{eq: T2 gaussian sequence lb}
    \end{align}

    Next we bound $T_1$ from below. Note that, for any $a, b \in \R$, $(a + b)^2 \geq - a^2 + (1/2) b^2$ holds. Therefore,
    \begin{align*}
        T_1 &\geq -\underbrace{\sum_j (s_j - 1)^2 (\btheta^*)_j^2}_{= T_{1,1}}\\
        &\quad+ \frac{1}{2} \underbrace{\sum_j s_j^2 \qty[\frac{1}{|\mY|} \sum_{y \in \mY} \qty{(1 - \alpha) \qty(1 - \frac{n_y}{\max_{y' \in \mY} n_{g'}}) + \alpha} ((\tilde \btheta^{(y)})_j - (\btheta^{(y)})_j)]^2}_{= T_{1,2}}.
    \end{align*}

For the term $T_{1,1}$, we have $T_{1,1} \lesssim \lambda^{2r/(r \vee p)}$ from \eqref{eq: T11 gaussian sequence}.

For the term $T_{1,2} =: \sum_j s_j^2 b_j^2$, where $b$ is defined in \eqref{eq: bj definition}, it follows that, when $\lambda < 1$,
    \begin{align}
        T_{1,2} &\geq \sum_{j \leq \floor{\lambda^{-1/(p \vee r)}}} \frac{1}{(1 + \lambda \omega_j)^2} b_j^2 \gtrsim \sum_{j \leq \floor{\lambda^{-1/(p \vee r)}}} b_j^2 = \sum_j b_j^2 - \sum_{j > \floor{\lambda^{-1/(p \vee r)}}} b_j^2,\label{eq: T12 gaussian sequence lb 0}
    \end{align}
    where the second inequality follows from $\lambda \omega_j = O(1)$ when $j \leq \floor{\lambda^{-1/(p \vee r)}}$, since $\omega_j \lesssim j^p$.
    Using Cauchy-Schwarz inequality, we have
    \begin{align*}
        \sum_{j > \floor{\lambda^{-1/(p \vee r)}}} b_j^2 &\leq \sum_{j > \floor{\lambda^{-1/(p \vee r)}}} \frac{1}{|\mY|} \sum_{y \in \mY} \alpha_y^2 ((\tilde \btheta^{(y)})_j - (\btheta^{(y)})_j)^2\\
        &\leq 2 \sum_{j > \floor{\lambda^{-1/(p \vee r)}}} \frac{1}{|\mY|} \sum_{y \in \mY} (\tilde \btheta^{(y)})_j^2 + 2\sum_{j > \floor{\lambda^{-1/(p \vee r)}}} \frac{1}{|\mY|} \sum_{y \in \mY} (\btheta^{(y)})_j^2.
    \end{align*}
    Henceforth, by $\sum_{j \geq k} (\btheta^*)_j^2 \vee \sum_{j \geq k} (\tilde \btheta^*)_j^2 \lesssim k^{-2r}$, we have
    \begin{align*}
        \sum_{j > \floor{\lambda^{-1/(p \vee r)}}} b_j^2 &= O(\lambda^{2r/(p \vee r)}).
    \end{align*}
    Thus \eqref{eq: T12 gaussian sequence lb 0} gives,
    \begin{align}
        T_{1,2} &\gtrsim \sum_j b_j^2 - O(\lambda^{2r/(p \vee r)}).\label{eq: T12 gaussian sequence lb}
    \end{align}
     
    Combining \eqref{eq: T2 gaussian sequence lb}, \eqref{eq: T11 gaussian sequence}, and \eqref{eq: T12 gaussian sequence lb}, we have
    \begin{align*}
        \E[\|\hat \btheta - \btheta^*\|^2] \gtrsim - O(\lambda^{2r/(r \vee p)}) + \sum_j b_j^2 - O(\lambda^{2r/(p \vee r)}) + R_{n_\T,N} \lambda^{-1/p}.
    \end{align*}
    Choosing $\lambda \asymp R_{n_\T,N}^{p/(2r' + 1) + p \upsilon}$ with some $\upsilon > 0$ yields
    \begin{align*}
        \E[\|\hat \btheta - \btheta^*\|^2] &\gtrsim - O(R_{n_\T,N}^{2r'/(2r'+1)} R_{n_\T,N}^{2r' \upsilon}) + \sum_j b_j^2 - O(R_{n_\T,N}^{2r'/(2r'+1)} R_{n_\T,N}^{2r' \upsilon}) + R_{n_\T,N}^{2r'/(2r' + 1)} R_{n_\T,N}^{- \upsilon}\\
        &= \sum_j b_j^2 + \qty(1 - O(R_{n_\T,N}^{2r'\upsilon})) R_{n_\T,N}^{2r'/(2r' + 1) - \upsilon}\\
        &\gtrsim R_{n_\T,N}^{2r'/(2r' + 1) - \upsilon} + \sum_j b_j^2,
    \end{align*}
    where the last inequality follows since $R_{n_\T,N} = o(1)$.

This completes the proof of Theorem~\ref{thm: gaussian sequence scaling law ap}. 
\end{proof}

\subsubsection{Spurious correlation}
\label{sec: additional scaling law spu}

Spurious correlations often mislead models to rely on irrelevant features. We address this issue in the following setting. Consider a two-group setting, i.e., $\mG = \{-1, 1\}$.
For simplicity of analysis, we consider the following white noise model \citep{tsybakov2009nonparametric,gine2016mathematical,jain2024scaling} with spurious correlations: for $\x \in [0, 1]^d$,
\begin{align} \label{model: continuous}
\begin{split}
    \dd{Y}^{(g)} &= h(G(\x), g A(\x)) \dd{\x} + \frac{\sigma_g}{\sqrt{n_g}} \dd{B^{(g)}(\x)},\\
    \dd{\tilde Y}^{(g)} &= \tilde F^{(g)}(\x) \dd{\x} + \frac{\tilde \sigma_g}{\sqrt{m_g}} \dd{\tilde B^{(g)}(\x)},\\
    \dd{\check Y}^{(g)} &= \tilde F^{(g)}(\x) \dd{\x} + \frac{\tilde \sigma_g}{\sqrt{N}} \dd{\check B^{(g)}(\x)},
\end{split}
\end{align}
where $B^{(g)}$, $\tilde B^{(g)}$ and $\check B^{(g)}$ across all $g \in \mG$ are independent Brownian motions, and $\sigma_g, \tilde \sigma_g > 0$ for all $g \in \mG$. Here, $h(G(\x), g A(\x))$ models the combined contribution of core and spurious features, where $g A(\x)$ is a spurious feature depending on the group $g$, $G(\x)$ is a core feature, and $h: \R^2 \to \R$ is a smooth function that determines how the spurious feature $gA(\x)$ affects the core signal $G(\x)$ in the observation $Y^{(g)}$.
In this model, $g \in \mG$ determines the spurious feature $g A(\x)$, and the data imbalance between groups can be interpreted as the imbalance of spurious features.
Similar spurious correlation models under classification settings have been used in \citet{arjovsky2019invariant,sagawa2020investigation}.
Our goal is to decorrelate $Y^{(g)}$ with the spurious feature $g A^{(g)}$ by oversampling and data augmentation.

The white noise model~\ref{model: continuous} assumes that we have access to raw data $Y^{(g)}$, corresponding to $y^{(g)}$, and synthetic data $\tilde Y^{(g)}$ and $\check Y^{(g)}$ corresponding to $\tilde y^{(g)}$ and $\check y^{(g)}$, respectively. 

We remark that we address the asymptotic equivalence between the continuous and discrete models. Specifically, consider the discrete model $y_i = f(x_i) + \epsilon_i$ for $i \in [n]$, where $\epsilon_i \sim \mathcal{N}(0, \sigma^2)$. When $n^{1/d} \in \mathbb{N}$, and $x_i$ are equally spaced points in $[0, 1)^d$, this model is (asymptotically) equivalent to the following white noise model: $\dd{Y} = F(\x) \dd{\x} + (\sigma/\sqrt{n}) \dd{B(\x)}$, where $B(\x)$ is a Brownian motion, and $F$ is a Fourier-transform of $g$. See details in \citet{brown1996asymptotic,reiss2008asymptotic}. In the white noise model~\ref{model: continuous}, we have two sources of synthetic data: $\tilde Y^{(g)}$ and $\check Y^{(g)}$. The asymptotic equivalence of the white noise model to the discrete model for synthetic data is guaranteed by the amount of information in the raw data: $n_g$, and the amount of information in the synthetic data: $m_g$ and $N$, respectively.

Then, we minimize the penalized data augmentation loss over smooth functions:
\begin{align}
    \calLhat(F) &:= (1 - \alpha) \qty{\frac{n_\T}{n_\T + m_\T} \sum_{g \in \mG} \frac{n_g}{n_\T} \|Y^{(g)} - F\|^2 + \frac{m_\T}{n_\T + m_\T} \sum_{g \in \mG} \frac{m_g}{m_\T} \|\tilde Y^{(g)} - F\|^2}\nonumber\\
    &\quad+ \alpha \frac{1}{|\mG|} \sum_{g \in \mG} \|\check Y^{(g)} - F\|^2 + \lambda \|F\|_{p,2}^2,\label{eq: general R with data augmentation continuous}
\end{align}
which is an analogy to the loss in \eqref{eq: general R with data augmentation} with an additional penalty term $\lambda \|F\|_{p,2}^2$ for a non-negative integer $p$, and $\lambda > 0$ is the regularization parameter. The term $\|F\|_{p,2}^2$ enforces smoothness in $F$ to mitigate overfitting. We compare the risk of $\hat F$ with the reweighted regressor $F_\rw = \E_{g' \sim \operatorname{Unif}(\{-1, 1\})}[h(G(\x), g' A(\x))]$, where $g'$ is a Rademacher random variable. Note that $F_\rw$ decorrelates the effect of spurious correlation to the signal. Define
\begin{align*}
    \sigma^2 &:= \frac{1}{|\mG|} \sum_{g \in \mG} \qty{(1 - \rho_g) \sigma_g^2 + \rho_g \tilde \sigma_g^2},\ \ \sigma'^2 := \frac{1}{|\mG|} \sum_{g \in \mG} \tilde \sigma_g^2.
\end{align*}

Let $\calR_\b(F) := (1/2) \E[\|Y^{(-1)} - F\|^2] + (1/2) \E[\|Y^{(1)} - F\|^2]$ be the balanced risk of $F$. We compare the balanced risk $\calR_\b(\hat F)$ with $\calR_\b(F_\rw)$.

We introduce an assumption about the smoothness of $F^{(g)}$ and $\tilde F^{(g)}$ for all $g \in \mG$.
\begin{assumption}\label{asm: nonparametric F and tilde F}
Let $F^{(g)}(\x) := h(G(\x), g A(\x))$. Suppose $\|F^{(1)}\|_{r,2}^2 \vee \|F^{(-1)}\|_{r,2}^2 \leq 1$ and $\|\tilde F^{(1)} - F^{(1)}\|_{r,2}^2 \vee \|\tilde F^{(-1)} - F^{(-1)}\|_{r,2}^2 = O(1)$ for some integer $r \geq 1$.
\end{assumption}
Assumption~\ref{asm: nonparametric F and tilde F} ensures the smoothness of $F^{(g)}$ and its closeness to $\tilde F^{(g)}$, which is crucial for bounding the Fourier coefficients in the proof.

Let $R_{n_\T,N}$ be the rate defined in \eqref{eq: R n N}. We then introduce the following result on the scaling law for the nonparametric models in terms of $n_\T$, $N$, and the bias term.

\begin{theorem}\label{thm: nonparametric scaling law balanced}
Suppose Assumption~\ref{asm: nonparametric F and tilde F} holds, and $R_{n_\T,N} = o(1)$. Fix any $\upsilon > 0$. Let $\beta = 2r' / (2r' + d) - \upsilon$ with $r' := 2 p \wedge r$. If $2 p > d$, then there exist constants $C_d$ and $C_d'$ depending on $d$, such that for any $\delta > 0$, choosing $\lambda \asymp R_{n_\T,N}^{2p/(2r' + d) - 2\upsilon p / d}$ yields
\begin{small}
\begin{align*}
        \E[\calR_\b(\hat F) - \calR_\b(F_\rw)] &\leq (1 + \delta) C_d R_{n_\T,N}^\beta\\
        &\quad+ \qty(1 + \frac{1}{\delta}) \norm{\frac{1}{|\mG|} \sum_{g \in \mG} \qty{(1 - \alpha) \rho_g + \alpha} \{h(G(\x), A^{(g)}(\x)) - \tilde F^{(g)}(\x)\}}^2,\\
        \E[\calR_\b(\hat F) - \calR_\b(F_\rw)] &\gtrsim C_d' R_{n_\T,N}^\beta  + \norm{\frac{1}{|\mG|} \sum_{g \in \mG} \qty{(1 - \alpha) \rho_g + \alpha} \{h(G(\x), A^{(g)}(\x)) - \tilde F^{(g)}(\x)\}}^2.
\end{align*}
\end{small}
\end{theorem}

Theorem~\ref{thm: nonparametric scaling law balanced} shows that the excess risk decreases with the quality of synthetic data, the smoothness parameter $r$, dimensionality $d$, raw data size $n_\T$, and synthetic data size $N$. Namely, the statistical error component depends on $n_T$ and $N$ by $O((1/n_\T + 1/N)^{\beta/2})$. Note that the bias term depends on the synthetic data quality $\tilde F^{(g)} - F^{(g)}$ weighted by the ratio of total added synthetic data size over all samples $(1 - \alpha) \rho_g + \alpha \in [0, 1]$. The bias term reflects how well the synthetic data $\tilde F^{(g)}$ approximates the true data $F^{(g)}$. Higher quality synthetic data leads to smaller bias and better performance.

We remark that the synthetic data generator $\tilde F^{(g)}$ can be obtained via nonparametric regression on another set of raw data, which can be obtained by sample splitting. A similar argument under proper regularization with the condition that imbalance ratio is bounded above, i.e., $\min_{g' \in \mG} \rho_{g'} < C \leq 1$ holds for some $C > 0$, we obtain $\max_{g \in \mG} \E[\|\tilde F^{(g)} - F^{(g)}\|^2] = O(n_\T^{-2r/(2r+d)})$, where the expectation is taken for another set of raw data. Therefore, ignoring the dependence on $d$, we have
\begin{align*}
    \E[\calR_\b(\hat F) - \calR_\b(F_\rw)] &\approx \alpha^2 \frac{1}{N^\beta} + \frac{1}{n_\T^\beta}.
\end{align*}
In general, when the quality of synthetic data is not as good as the raw data, i.e., $\max_{g \in \mG} \E[\|\tilde F^{(g)} - F^{(g)}\|^2] = O(n_\T^{-\gamma})$ for some $\gamma < \beta$, we have
\begin{align*}
    \E[\calR_\b(\hat F) - \calR_\b(F_\rw)] &\approx \alpha^2 \frac{1}{N^\beta} + \frac{1}{n_\T^\gamma}.
\end{align*}

{Moreover, similar to the remark after Theorem~\ref{thm: gaussian sequence scaling law}, when the bias $\|h(G(\x), A^{(g)}(\x)) - \tilde F^{(g)}(\x)\|^2$ is negligible compared to $R_{n_\T,N}^\beta$, we can derive the scaling law of $\log\E[\calR_\b(\hat F) - \calR_\b(F_\rw)]$ with respect to $\log(n_\T + N|\mG|)$.

\begin{proof}[\textbf{Proof of Theorem~\ref{thm: nonparametric scaling law balanced}}]
Let $F^{(g)}(\x) := h(G(\x), g A(\x))$. Since $F_\rw = (1/2) F^{(1)} + (1/2) F^{(-1)}$,
    \begin{align}
        &\frac{1}{2} \|\hat F - F^{(-1)}\|^2 + \frac{1}{2} \|\hat F - F^{(1)}\|^2 - \frac{1}{2} \|F_\rw - F^{(-1)}\|^2 - \frac{1}{2} \|F_\rw - F^{(1)}\|^2\\
        &\quad= \frac{1}{2} \|\hat F - F^{(-1)}\|^2 + \frac{1}{2} \|\hat F - F^{(1)}\|^2 - \|(1/2) F^{(1)} - (1/2) F^{(-1)}\|^2\\
        &\quad= \|\hat F - F_\rw\|^2,
    \end{align}
where we used the parallelogram law in the second equality. The conclusion follows from Theorem~\ref{thm: nonparametric scaling law ap}. This completes the proof of  Theorem~\ref{thm: nonparametric scaling law balanced}. 
\end{proof}

\begin{theorem}\label{thm: nonparametric scaling law ap}
Suppose Assumption~\ref{asm: nonparametric F and tilde F} holds and $R_{n_\T,N} = o(1)$. Fix any $\upsilon > 0$. Let $\beta = 2r' / (2r' + d) - \upsilon$ with $r' := 2p \wedge r$. If $2p > d$, then there exist constants $C_d$ and $C_{d'}$ depending on $d$ such that for any $\delta > 0$, choosing $\lambda \asymp R_{n_\T,N}^{2p/(2r' + d) - 2p \upsilon/d}$ yields
\begin{align*}
\E[\|\hat F - F_\rw\|^2] &\leq (1 + \delta) C_d R_{n_\T,N}^\beta + \qty(1 + \frac{1}{\delta}) \norm{\frac{1}{|\mG|} \sum_{g \in \mG} \qty{(1 - \alpha) \rho_g + \alpha} \{F^{(g)}(\x) - \tilde F^{(g)}(\x)\}}^2,\\
\E[\|\hat F - F_\rw\|^2] &\gtrsim C_d' R_{n_\T,N}^\beta + \norm{\frac{1}{|\mG|} \sum_{g \in \mG} \qty{(1 - \alpha) \rho_g + \alpha} \{F^{(g)}(\x) - \tilde F^{(g)}(\x)\}}^2.
\end{align*}
\end{theorem}

\begin{proof}[\textbf{Proof of Theorem~\ref{thm: nonparametric scaling law ap}}]
We first write the risk $\E[\|\hat F - F_\rw\|^2]$ in terms of Fourier coefficients of $h(G(\x), gA(\x))$. We then provide the proofs of the upper bound and lower bound separately.
    
Let $F^{(g)}(\x) := h(G(\x), g A(\x))$. We first apply Fourier transform to the model~\ref{model: continuous}. For $\q \in \{2\pi \boldsymbol{j} : \boldsymbol{j} \in \Z^d\}$, define $\theta^{(g)}(\q) := \int F^{(g)}(\x) \exp(-\sqrt{-1} \langle \q, \x \rangle) \dd{\x}$, and $\tilde \theta^{(g)}(\q) := \int \tilde F^{(g)}(\x) \exp(-\sqrt{-1} \langle \q, \x \rangle) \dd{\x}$.  Let $z^{(g)}$, $\tilde z^{(g)}$ and $\check z^{(g)}$ be the Fourier transform of $Y$, $\tilde Y$ and $\check Y$, respectively. Then, we have
    \begin{align}
        z^{(g)}(\q) &= \theta^{(g)}(\q) + \frac{\sigma_g}{\sqrt{n_g}} \xi^{(g)}(\q),\label{eq: fourier transform continuous 1}\\
        \tilde z^{(g)}(\q) &= \tilde \theta^{(g)}(\q) + \frac{\tilde \sigma_g}{\sqrt{m_g}} \tilde \xi^{(g)}(\q),\label{eq: fourier transform continuous 2}\\
        \check z^{(g)}(\q) &= \tilde \theta^{(g)}(\q) + \frac{\tilde \sigma_g}{\sqrt{N}} \check \xi^{(g)}(\q),\label{eq: fourier transform continuous 3}
    \end{align}
where $\xi^{(g)}$, $\tilde \xi^{(g)}$ and $\check \xi^{(g)}$ across $g \in \mG$ are independent standard Gaussian random variables. Let $\hat \theta$ be Fourier transform of $\hat F$. Then, we have $\hat \theta = \argmin_\theta \calLhat(\theta)$, where
    \begin{align*}
        \calLhat(\theta) &= (1 - \alpha) \frac{n_\T}{n_\T + m_\T} \sum_{g \in \mG} \frac{n_g}{n_\T} \sum_\q |\theta(\q) - z^{(g)}(\q)|^2\\
        &\quad+ (1 - \alpha) \frac{m_\T}{n_\T + m_\T} \sum_{g \in \mG} \frac{m_g}{m_\T} \sum_\q |\theta(\q) - \tilde z^{(g)}(\q)|^2\\
        &\quad+ \alpha \frac{1}{|\mG|} \sum_{g \in \mG} \sum_\q |\theta(\q) - \check z^{(g)}(\q)|^2 + \lambda \sum_\q (1 + \|\q\|^{2p}) |\theta(\q)|^2.
    \end{align*}
    This gives
    \begin{small}
    \begin{align*}
        \hat \theta(\q) &= \frac{1}{1 + \lambda (1 + \|\q\|^{2p})} \sum_{g \in \mG} \qty{(1 - \alpha) \frac{n_g}{n_\T + m_\T} z^{(g)}(\q) + (1 - \alpha) \frac{m_g}{n_\T + m_\T} \tilde z^{(g)}(\q) + \alpha \frac{1}{|\mG|} \check z^{(g)}(\q)},\\
        &= \frac{1}{1 + \lambda (1 + \|\q\|^{2p})} \frac{1}{|\mG|} \sum_{g \in \mG} \qty{(1 - \alpha) \frac{n_g}{\max_{g' \in \mG} n_{g'}} z^{(g)}(\q) + (1 - \alpha) \qty(1 - \frac{n_g}{\max_{g' \in \mG} n_{g'}}) \tilde z^{(g)}(\q) + \alpha \check z^{(g)}(\q)}
    \end{align*}
    \end{small}\noindent
    where we used $m_g = \max_{g' \in \mG} n_{g'} - n_g$ and $n_\T + m_\T = |\mG| \max_{g' \in \mG} n_{g'}$.
    Define $s(\q) := \{1 + \lambda(1 + \|\q\|^{2p})\}^{-1}$. Using \eqref{eq: fourier transform continuous 1}, \eqref{eq: fourier transform continuous 2}, and \eqref{eq: fourier transform continuous 3} combined with the definition of the imbalance ratio $\rho_g$, we have
    \begin{align*}
        \hat \theta(\q) &= s(\q) \frac{1}{|\mG|} \sum_{g \in \mG} \qty{(1 - \alpha) (1 - \rho_g) \theta^{(g)}(\q) + (1 - \alpha) \rho_g \tilde \theta^{(g)}(\q) + \alpha \tilde \theta^{(g)}(\q)}\\
        &\quad+ s(\q) \frac{1}{|\mG|} \sum_{g \in \mG} \qty{(1 - \alpha) \frac{\sigma_g}{\sqrt{n_g}} (1 - \rho_g) \xi^{(g)}(\q) + (1 - \alpha) \frac{\tilde \sigma_g}{\sqrt{m_g}} \rho_g \tilde \xi^{(g)}(\q) + \alpha \frac{\tilde \sigma_g}{\sqrt{N}} \check \xi^{(g)}(\q)}\\
        &=: s(\q) \frac{1}{|\mG|} \sum_{g \in \mG} \qty{\theta^{(g)}(\q) + (1 - \alpha) \rho_g (\tilde \theta^{(g)}(\q) - \theta^{(g)}(\q)) + \alpha (\tilde \theta^{(g)}(\q) - \theta^{(g)}(\q))} + Q(\q).
    \end{align*}

Let $\theta_\rw$ be the fourier transform of $F_\rw$. Then, $\theta_\rw(\q) = {|\mG|}^{-1} \sum_{g \in \mG} \theta^{(g)}(\q)$. This gives
    \begin{align*}
        \hat \theta(\q) - \theta_\rw(\q) &= (s(\q) - 1) \frac{1}{|\mG|} \sum_{g \in \mG} \theta^{(g)}(\q) + s(\q) \frac{1}{|\mG|} \sum_{g \in \mG} \qty{(1 - \alpha) \rho_g + \alpha} (\tilde \theta^{(g)}(\q) - \theta^{(g)}(\q)) + Q(\q).
    \end{align*}
    Thus
    \begin{align}
        &\E[\|\hat F - F_\rw\|^2] = \E\qty[\sum_\q |\hat \theta(\q) - \theta_\rw(\q)|^2]\nonumber\\
        &\quad= \underbrace{\sum_\q \qty[(s(\q) - 1) \frac{1}{|\mG|} \sum_{g \in \mG} \theta^{(g)}(\q) + s(\q) \frac{1}{|\mG|} \sum_{g \in \mG} \qty{(1 - \alpha) \rho_g + \alpha} (\tilde \theta^{(g)}(\q) - \theta^{(g)}(\q))]^2}_{=: T_1}\nonumber\\
        &\quad\quad+ \underbrace{\sum_\q \{s(\q)\}^2 \frac{1}{|\mG|^2} \sum_{g \in \mG} \qty{(1 - \alpha)^2 \frac{\sigma_g^2}{n_g} (1 - \rho_g)^2 + (1 - \alpha)^2 \frac{\tilde \sigma_g^2}{m_g} \rho_g^2 + \alpha^2 \frac{\tilde \sigma_g^2}{N}}}_{=: T_2}.\label{eq: risk decomposition nonparametric}
    \end{align}
 
We next establish the upper bound and the lower bound, respectively. 
    
\paragraph{Proof of the upper bound.}
    We first bound $T_2$. 
    By definition of $\rho_g$, we have $\max_{g' \in \mG} n_{g'}$ $= n_g / (1 - \rho_g)$ and $n_\T = (1 - \rho) |\mG| \max_{g' \in \mG} n_{g'}$. This gives
    \begin{align*}
        &\frac{1}{|\mG|^2} \sum_{g \in \mG} \qty{(1 - \alpha)^2 \frac{\sigma_g^2}{n_g} (1 - \rho_g)^2 + (1 - \alpha)^2 \frac{\tilde \sigma_g^2}{\max_{g' \in \mG} n_{g'} - n_g} \rho_g^2 + \alpha^2 \frac{\tilde \sigma_g^2}{N }}\\
        &\quad= \frac{1}{|\mG|^2} \sum_{g \in \mG} \qty{(1 - \alpha)^2 \frac{\sigma_g^2}{\max_{g' \in \mG} n_{g'}} (1 - \rho_g) + (1 - \alpha)^2 \frac{\tilde \sigma_g^2}{\max_{g' \in \mG} n_{g'}} \rho_g + \alpha^2 \frac{\tilde \sigma_g^2}{N |\mG|}}\\
        &\quad= (1 - \alpha)^2 \frac{\sigma^2}{|\mG| \max_{g' \in \mG} n_{g'}} + \alpha^2 \frac{\sigma'^2}{N |\mG|}
        = (1 - \alpha)^2 \frac{\sigma^2}{n_\T} (1 - \rho) + \alpha^2 \frac{\sigma'^2}{N |\mG|} = R_{n_\T,N}.
    \end{align*}
Write $T_2 =: R_{n_\T,N} \sum_\q \{s(\q)\}^2$. Note that when $\lambda < 1$ and $2p > d$, we have
    \begin{align*}
        \sum_\q \{s(\q)\}^2 &\leq \sum_\q \frac{1}{1 + \lambda (1 + \|\q\|^{2p})} 
        \leq C_{1,d} \int \frac{1}{1 + \lambda \|\q\|^{2p}} \dd{\q}
        = C_{2,d} \int_0^\infty \frac{\eta^{d-1}}{1 + \lambda \eta^{2p}} \dd{\eta}\\
        &\leq C_{2,d} \qty{\int_0^{\lambda^{-1/(2p)}} \frac{\eta^{d-1}}{1 + \lambda \eta^{2p}} \dd{\eta} + \int_{\lambda^{-1/(2p)}}^\infty \frac{\eta^{d-1}}{1 + \lambda \eta^{2p}} \dd{\eta}}\\
        &\lesssim C_{3,d} \qty{\int_0^{\lambda^{-1/(2p)}} \eta^{d-1} \dd{\eta} + \frac{1}{\lambda} \int_{\lambda^{-1/(2p)}}^\infty \eta^{d-2p-1} \dd{\eta}}
        \lesssim C_{4,d} \lambda^{-d/(2p)},
    \end{align*}
    where $C_{1,d}$, $C_{2,d}$, $C_{3,d}$, and $C_{4,d}$ are some constants depending on the dimension $d$. Hence
    \begin{align}
        T_2 \lesssim C_{4,d} R_{n_\T,N} \lambda^{-d/(2p)}.\label{eq: T2}
    \end{align}

Next, we bound $T_1$. Note that for any $\delta > 0$, $(a + b)^2 \leq (1 + \delta) a^2 + (1 + \delta^{-1}) b^2$. Thus,
    \begin{align*}
        T_1 &\leq (1 + \delta) \underbrace{\sum_\q (s(\q) - 1)^2 \qty(\frac{1}{|\mG|} \sum_{g \in \mG} \theta^{(g)}(\q))^2}_{=: T_{1,1}}\\
        &\quad+ (1 + \delta^{-1}) \underbrace{\sum_\q \{s(\q)\}^2 \qty[\frac{1}{|\mG|} \sum_{g \in \mG} \qty{(1 - \alpha) \qty(1 - \frac{n_g}{\max_{g' \in \mG} n_{g'}}) + \alpha} (\tilde \theta^{(g)}(\q) - \theta^{(g)}(\q))]^2}_{=: T_{1,2}}.
    \end{align*}
    Let $\eta := \lambda^{-1/(2p \vee r)}$. For the term $T_{1,1}$, we have 
    \begin{align}
        T_{1,1} &= \sum_{\q: \|\q\| \leq \eta} \qty{\frac{\lambda (1 + \|\q\|^{2p})}{1 + \lambda (1 + \|\q\|^{2p})}}^2 \qty(\frac{1}{|\mG|} \sum_{g \in \mG} \theta^{(g)}(\q))^2\\
        &\quad+ \sum_{\q: \|\q\| > \eta} \qty{\frac{\lambda (1 + \|\q\|^{2p})}{1 + \lambda (1 + \|\q\|^{2p})}}^2 \qty(\frac{1}{|\mG|} \sum_{g \in \mG} \theta^{(g)}(\q))^2\nonumber\\
        &\leq \frac{1}{4} \sum_{\q: \|\q\| \leq \eta} \qty{\lambda (1 + \|\q\|^{2p})}^2 \qty(\frac{1}{|\mG|} \sum_{g \in \mG} \theta^{(g)}(\q))^2 + \sum_{\q: \|\q\| > \eta} \qty(\frac{1}{|\mG|} \sum_{g \in \mG} \theta^{(g)}(\q))^2\nonumber\\
        &\leq \frac{\lambda^2}{4} \qty(\max_{\q: \|\q\| \leq \eta} \frac{(1 + \|\q\|^{2p})^2}{1 + \|\q\|^{2r}}) \sum_{\q: \|\q\| \leq \eta} (1 + \|\q\|^{2r}) \qty(\frac{1}{|\mG|} \sum_{g \in \mG} \theta^{(g)}(\q))^2\nonumber\\
        &\quad+ \qty(\max_{\q: \|\q\| > \eta} \frac{1}{1 + \|\q\|^{2r}}) \sum_{\q: \|\q\| > \eta} (1 + \|\q\|^{2r}) \qty(\frac{1}{|\mG|} \sum_{g \in \mG} \theta^{(g)}(\q))^2\nonumber\\
        &\lesssim \lambda^2 \max_{\q: \|\q\| \leq \eta} \qty(1 \vee \|\q\|^{4p-2r}) + \max_{\q: \|\q\| > \eta} \frac{1}{1 + \|\q\|^{2r}}\nonumber\\
        &\lesssim \lambda^2 + \lambda^{2-(4p-2r)/(2p \vee r)} + \lambda^{2r/(2p \vee r)} \lesssim \lambda^{2r/(2p \vee r)},\label{eq: T11}
    \end{align}
    where the third-to-last inequality follows from $\|F_\rw\|_{r,2} \leq 1$.
    Denote $\alpha_g := (1 - \alpha) \rho_g + \alpha$.
    For the term $T_{1,2}$, write $T_{1,2} =: \sum_{\q} \{s(\q)\}^2 \{b(\q)\}^2$, where
    \begin{align}
        b(\q) := \frac{1}{|\mG|} \sum_{g \in \mG} \alpha_g \{\tilde \theta^{(g)}(\q) - \theta^{(g)}(\q)\}.\label{eq: b(q) definition}
    \end{align}
    Define
    \begin{align*}
        G(\x) := \frac{1}{|\mG|} \sum_{g \in \mG} \alpha_g F^{(g)}(\x), \ \ \tilde G(\x) := \frac{1}{|\mG|} \sum_{g \in \mG} \alpha_g \tilde F^{(g)}(\x).
    \end{align*}
    Note that $b(\q)$ is the Fourier coefficient of $\tilde G(\x) - G(\x)$.  Because $s(\q) \leq 1$, we have
    \begin{align}
        T_{1,2} \leq \sum_\q \{b(\q)\}^2 = \|\tilde G - G\|^2,\label{eq: T12}
    \end{align} 
    
    Using \eqref{eq: T2}, \eqref{eq: T11}, and \eqref{eq: T12}, we have
    \begin{align*}
        \E[\|\hat F - F_\rw\|^2] \leq C_{5,d} \qty{ (1 + \delta) \lambda^{2r/(2p \vee r)} +  R_{n_\T,N} \lambda^{-d/(2p)} } + (1 + \delta^{-1}) \|\tilde G - G\|^2,
    \end{align*}
    where $C_{5,d}$ is some constant depending on the dimension $d$.
    Choosing $\lambda \asymp R_{n_\T,N}^{2p/(2r' + d) + 2\upsilon p/d}$ for some $\upsilon > 0$ with $r' := 2 p \wedge r$ gives
    \begin{align*}
        \E[\|\hat F - F_\rw\|^2] &\leq (1 + \delta) C_{6,d} R_{n_\T,N}^{2r'/(2r' + d) - \upsilon} + \qty(1 + \frac{1}{\delta}) \|\tilde G - G\|^2,
    \end{align*}
where $C_{6,d}$ is some constant depending on the dimension $d$. This completes the proof of the upper bound.

\paragraph{Proof of the lower bound.}
Recall that $\E[\|\hat F - F_\rw\|^2] = T_1 + T_2$ by \eqref{eq: risk decomposition nonparametric}. We first bound $T_2$ from below. Write $T_2 =: R_{n_\T,N} \sum_\q \{s(\q)\}^2$. Note that when $\lambda < 1$ and $2p > d$, we have
    \begin{align*}
        \sum_\q \{s(\q)\}^2 &= \sum_\q \frac{1}{\{1 + \lambda (1 + \|\q\|^{2p})\}^2} 
        \geq C_{1,d} \int \frac{1}{\{1 + \lambda (1 + \|\q\|^{2p})\}^2} \dd{\q}\\
        &= C_{2,d} \int_0^\infty \frac{\eta^{d-1}}{\{1 + \lambda (1 + \eta^{2p})\}^2} \dd{\eta}
        \geq C_{2,d} \int_0^{\lambda^{-1/(2p)}} \frac{\eta^{d-1}}{\{1 + \lambda (1 + \eta^{2p})\}^2} \dd{\eta}\\
        &\gtrsim C_{2,d} \int_0^{\lambda^{-1/(2p)}} \eta^{d-1} \dd{\eta}
        \gtrsim C_{3,d} \lambda^{-d/(2p)},
    \end{align*}
    where $C_{1,d}$, $C_{2,d}$, and $C_{3,d}$ are some constants depending on the dimension $d$. Hence
    \begin{align}
        T_2 \gtrsim C_{3,d} R_{n_\T,N} \lambda^{-d/(2p)}.\label{eq: T2 lb}
    \end{align}

    Next we bound $T_1$ from below. For any $a, b \in \R$, $(a + b)^2 \geq - a^2 + (1/2) b^2$. Thus,
    \begin{align*}
        T_1 &\geq -\underbrace{\sum_\q (s(\q) - 1)^2 \qty(\frac{1}{|\mG|} \sum_{g \in \mG} \theta^{(g)}(\q))^2}_{= T_{1,1}}\\
        &\quad+ \frac{1}{2} \underbrace{\sum_\q \{s(\q)\}^2 \qty[\frac{1}{|\mG|} \sum_{g \in \mG} \qty{(1 - \alpha) \qty(1 - \frac{n_g}{\max_{g' \in \mG} n_{g'}}) + \alpha} (\tilde \theta^{(g)}(\q) - \theta^{(g)}(\q))]^2}_{= T_{1,2}}.
    \end{align*}

For the term $T_{1,1}$, we have $T_{1,1} \lesssim \lambda^{2r/(2p \vee r)}$ from \eqref{eq: T11}.

For the term $T_{1,2}$, write $T_{1,2} =: \sum_{\q} \{s(\q)\}^2 \{b(\q)\}^2$, where $b(\q)$ is defined in \eqref{eq: b(q) definition}. Then,
\begin{align*}
        T_{1,2} &\geq \sum_{\q: \|\q\| \leq \lambda^{-1/(2p \vee r)}} \frac{1}{\{1 + \lambda(1 + \|q\|^{2p})\}^2} \{b(\q)\}^2\\
        &\geq \frac{1}{9} \sum_{\q: \|\q\| \leq \lambda^{-1/(2p \vee r)}} \{b(\q)\}^2
        \geq \frac{1}{9} \sum_{\q} \{b(\q)\}^2 - \frac{1}{9} \sum_{\q: \|\q\| > \lambda^{-1/(2p \vee r)}} \{b(\q)\}^2.
\end{align*}
when $\lambda \leq 1$. Using Cauchy-Schwarz inequality, we have
    \begin{align*}
        \sum_{\q: \|\q\| > \lambda^{-1/(2p \vee r)}} \{b(\q)\}^2 &\leq \sum_{\q: \|\q\| > \lambda^{-1/(2p \vee r)}} \frac{1}{|\mG|} \sum_{g \in \mG} \alpha_g^2 (\tilde \theta^{(g)}(\q) - \theta^{(g)}(\q))^2\\
        &\leq \sum_{\q: \|\q\| > \lambda^{-1/(2p \vee r)}} \frac{1}{|\mG|} \sum_{g \in \mG} (\tilde \theta^{(g)}(\q) - \theta^{(g)}(\q))^2.
    \end{align*}

Note that $\|\tilde F^{(g)} - F^{(g)}\|_{r,2} = O(1)$, we have $\sum_{\q} \|\q\|^{2r} |\tilde \theta^{(g)}(\q) - \theta^{(g)}(\q)|^2 = O(1)$. This gives
\begin{align*}
\sum_{\q: \|\q\| > \lambda^{-1/(2p \vee r)}} & \frac{1}{|\mG|} \sum_{g \in \mG} (\tilde \theta^{(g)}(\q) - \theta^{(g)}(\q))^2 \\
& \leq \frac{1}{|\mG|} \sum_{g \in \mG} \sum_{\q: \|\q\| > \lambda^{-1/(2p \vee r)}} \qty(\frac{\|\q\|}{\lambda^{-1/(2p \vee r)}})^{2r} (\tilde \theta^{(g)}(\q) - \theta^{(g)}(\q))^2 
= O(\lambda^{2r/(2p \vee r)}).
\end{align*}
Thus, by \eqref{eq: T12}, we have
    \begin{align}
        T_{1,2} &\geq \frac{1}{9} \sum_{\q} \{b(\q)\}^2 - O(\lambda^{2r/(2p \vee r)}) = \frac{1}{9} \|\tilde G - G\|^2 - O(\lambda^{2r/(2p \vee r)}),\label{eq: T12 lb}
    \end{align}

Combining \eqref{eq: T2 lb}, \eqref{eq: T11}, and \eqref{eq: T12 lb}, we have
    \begin{align*}
        \E[\|\hat F - F_\rw\|^2] \gtrsim - O(\lambda^{2r/(2p \vee r)}) + \sum_{\q} \|\tilde G - G\|^2 - O(\lambda^{2r/(2p \vee r)}) + C_{3,d} R_{n_\T,N} \lambda^{-d/(2p)}.
    \end{align*}
    Choosing $\lambda \asymp R_{n_\T,N}^{2p/(2r' + d) + 2 \upsilon p/d}$ for some $\upsilon > 0$ gives
    \begin{align*}
        \E[\|\hat F - F_\rw\|^2] &\gtrsim C_{3,d} R_{n_\T,N}^{2r'/(2r' + d) - \upsilon} - O(R_{n_\T,N}^{2r'/(2r'+d)} R_{n_\T,N}^{2\upsilon r'/d}) - O(R_{n_\T,N}^{2r'/(2r'+d)} R_{n_\T,N}^{2\upsilon r'/d}) + \sum_{\q} \|\tilde G - G\|^2\\
        &\gtrsim C_{3,d} R_{n_\T,N}^{2r'/(2r' + d) - \upsilon} + \sum_{\q} \|\tilde G - G\|^2,
    \end{align*}
where the last inequality follows because $R_{n_\T,N} = o(1)$.

This completes the proof of Theorem~\ref{thm: nonparametric scaling law balanced}.
\end{proof}

\section{Theory about transformers}
\label{supp-sec: theory transformers proof}

\subsection{Preparations}
\label{supp-sec:tf_preliminaries}

We first provide in Table~\ref{tab: notation transformer} a summary of notations used in Section~\ref{sec: theory transformers}.

\begin{table}[t!]
\centering
\caption{Notations used in Section~\ref{sec: theory transformers}.}
\label{tab: notation transformer}
\begin{tabular}{ll} \hline
\textbf{Symbol} & \textbf{Meaning} \\ \hline
        $n$ & Number of seed pairs (real data used to construct input tokens) \\
        $D$ & Hidden/token dimension \\
        $d$ & Number of discrete tokens (size of codebook) \\
        $r$ & Embedding dimension / logit dimension \\
        $h^X_i, h^Y_i \in \R^D$ & Token vectors for feature $X_i$ and label $Y_i$ \\
        $H_n \in \R^{D \times 2n}$ & Matrix of $2n$ input tokens $[h_{1}^X;h_{1}^Y;\ldots;h_{n}^X;h_{n}^Y]$\\
        $p_{s,n}$ & Positional encoding for the $s$-th token \\
        $\u_1,\ldots,\u_d \in \R^r$ & Token vectors; $d$ is number of discrete token values \\
        $\tilde \h_{k} \in \R^D$ & Hidden state at position $k$ after transformer layers \\
        $\v_{k} \in \R^r$ & Logit vector whose softmax gives next-token distribution \\
        $\tau > 0$ & Temperature parameter in softmax sampling \\
        $\psi=(\mu,\nu)$ & Parameters of a single transformer layer \\
        $\Psi=(\psi_1,\ldots,\psi_L)$ & Parameters of an $L$-layer transformer \\
        $\mathrm{TF}_\psi$ & Mapping induced by one transformer layer \\
        $\mathrm{TF}_{(\psi_1,\ldots,\psi_L)}$ & Mapping induced by $L$ stacked layers \\
        $Q_{\tilde X_s,\tilde Y_s;\Psi,\tau,\mathcal{D}_n}$ & Distribution of generated pair $(\tilde X_s,\tilde Y_s)$ \\ \hline
\end{tabular}
\end{table}

We next introduce the set of subject embeddings and functions, and the settings used in our proofs. Denote the joint distribution of $X_1$ and $Y_1$ by $P_{X_1,Y_1;T=t,M=m,U,\eta}$, and the distributions of $X_1$ and $Y_1|X_1$ under model~\ref{model: dgp} by $P_{X_1;T=t,U,\eta}$ and $P_{Y_1|X_1;M=m,U,\eta}$, respectively. Given the initial input tokens $H_n = [\h_1^X; \h_1^Y; \h_2^X; \h_2^Y; \dots; \h_n^X; \h_n^Y] \in \R^{D \times 2n}$, a transformer parameterized by $\Psi$ sequentially outputs $\h_{2n+1}, \h_{2n+2}, \dots$ corresponding to synthetic data from a categorical distribution given the last output from the transformer layers.  As in \eqref{eq: Q joint}, we write the joint distribution of $\v_{2n+2s-1}$ and $\v_{2n+2s}$ by:
\begin{align*}
    Q_{\tilde X_s,\tilde Y_s;\Psi,\tau,\mD_n}(x,y) := \P(\v_{2n+2s-1} = \u_x, \v_{2n+2s} = \u_y).
\end{align*}
Similarly, we denote the marginal and conditional distributions of $\tilde X_s$ and $\tilde Y_s | \tilde X_s$ by $Q_{\tilde X_s;\Psi,\tau,\mD_n}(x)$ and $Q_{\tilde Y_s|\tilde X_s;\Psi,\tau,\mD_n}(y)$, respectively.

\subsection{Auxiliary lemmas}
\label{supp-sec:auxiliary}

We first provide Lemmas~\ref{lem: subject identification} and \ref{lem: function identification} to show that selecting the subject or function embeddings by maximizing the cosine similarity gives the ground-truth subject or function with high probability. More specifically, Lemmas~\ref{lem: subject identification} is for the convergence of $(1/n) \sum_{i \in [n]} z^{(t') \top} \u_{X_i}$, and Lemma~\ref{lem: function identification} for the convergence of $(1/n) \sum_{i \in [n]} \langle f^{(m')}(\u_{X_i}), \u_{Y_i} \rangle$.

\begin{lemma}\label{lem: subject identification}
For any $a \geq 0$, $a \geq 0$, $t \in \mT$ and $\eta \geq (1/\sqrt{r}) \log d$, it holds that
\begin{small}
\begin{align*}
        &\P\qty(\max_{t' \in \mT} \abs{\z^{(t') \top} \qty(\frac{1}{n} \sum_{i \in [n]} \u_{X_i}) - \frac{1}{\eta r} \z^{(t') \top} \z^{(t)}} \lesssim a + \frac{\log^2 d}{\sqrt{n r}} + \frac{\log d}{\sqrt{d r}} \middle| T=t, U, \eta)\\
        &\quad= 1 - \exp(-\Omega\qty(\frac{ra^2 n}{\log^2 d} + \log^2 d)).
\end{align*}
\end{small}\noindent
Here the $\Omega$-notation does not depend on $t$.
\end{lemma}

\begin{proof}[\textbf{Proof of Lemma~\ref{lem: subject identification}}]
Define $E_i^{(t')} := \sum_{x \in \mX} \z^{(t') \top} \u_x \Id\{X_i = x\}$. By Lemma~\ref{lem: good event U}, there exists an event $\mathcal{E}$ for a random matrix $U = [\u_1; \dots; \u_d]^\top$ with $\P(\mathcal{E}) = 1 - \exp(-\Omega(\log^2 d))$, such that, the event,
\begin{align}
&\max_{T \in \mT, x \in \mX} |\z^{(t) \top} \u_x| \lesssim \frac{1}{\sqrt{r}} \log d, \ \ 
\max_{t \in \mT} \abs{\frac{d \E[C_x^{(t)}]}{\sum_{x \in \mX} C_x^{(t)}} - 1} \lesssim \frac{1}{\sqrt{d}} \log d, \label{eq: U condition 1}\\
&\max_{t, t' \in \mT} \abs{\sum_{x \in \mX} (D_x^{(t,t')} - \E[D_x^{(t,t')}])} \lesssim \sqrt{\frac{d}{r}} \log d \label{eq: U condition 2}\\
&\max_{t,t' \in \mT} \abs{\sum_{x \in \mX} \frac{D_x^{(t,t')}}{\sum_{x' \in \mX} C_{x'}^{(t)}} - \frac{\E[D_x^{(t,t')}]}{\E[C_x^{(t)}]}} \lesssim \frac{1}{\sqrt{dr}} \log d, \label{eq: U condition 3}
\end{align}
hold, where $C_x^{(t)} = \sum_{x \in \mX} \exp(\eta^{-1} \z^{(t) \top} \u_x)$ and $D_x^{(t,t')} = \sum_{x \in \mX} \z^{(t') \top} \u_x \exp(\eta^{-1} \z^{(t) \top} \u_x)$.

For now we fix any $U$ satisfying \eqref{eq: U condition 1} and \eqref{eq: U condition 2}. We also fix any $t \in \mT$ and $\eta \geq (1/\sqrt{r}) \log d$. Then, there exists some constant $C > 0$, such that $|E_i^{(t')}| \leq \max_{x \in \mX} |\z^{(t') \top} \u_x| \leq (C/\sqrt{r}) \log d$ holds. From Hoeffding's inequality, it follows that
    \begin{align*}
        \P\qty(\abs{\sum_{i \in [n]} (E_i^{(t')} - \E[E_i^{(t')} | U, T = t, \eta])} \geq \epsilon \middle| U, T = t, \eta) \leq 2\exp(-\frac{r \epsilon^2}{2 C^2 n \log^2 d}).
    \end{align*}
Choosing $\epsilon \gets an \vee \sqrt{n/r} \log^2 d$, and using $(1/n) \sum_{i \in [n]} \z^{(t') \top} \u_{X_i} = (1/n) \sum_{i \in [n]} E_i^{(t')}$,
    \begin{align*}
        &\P\qty(\abs{\z^{(t') \top} \qty(\frac{1}{n} \sum_{i \in [n]} \u_{X_i}) - \sum_{x \in \mX} \frac{D_x^{(t,t')}}{\sum_{x' \in \mX} C_{x'}^{(t)}}} \leq a \vee \frac{\log^2 d}{\sqrt{n r}} \middle| U, T = t, \eta)\\
        &\quad= 1 - \exp(-\Omega\qty(\frac{ra^2 n}{\log^2 d} \vee \log^2 d)),
    \end{align*}
where we use the fact that
    \begin{align*}
        \E[E_i^{(t')} | U, T = t, \eta] = \sum_{x \in \mX} \z^{(t') \top} \u_x \P(X_i = x | U, T=t,\eta) = \frac{\sum_{x \in \mX} D_x^{(t,t')}}{\sum_{x' \in \mX} C_{x'}^{(t)}}.
    \end{align*}
    By a union bound argument,
    \begin{align}
        &\P\qty(\max_{t' \in \mT} \abs{\z^{(t') \top} \qty(\frac{1}{n} \sum_{i \in [n]} \u_{X_i}) - \sum_{x \in \mX} \frac{D_x^{(t,t')}}{\sum_{x' \in \mX} C_{x'}^{(t)}}} \leq a \vee \frac{\log^2 d}{\sqrt{n r}} \middle| U, T = t, \eta)\\
        &\quad= 1 - \exp(-\Omega\qty(\frac{ra^2 n}{\log^2 d} \vee \log^2 d)). \label{eq: triangle 1}
    \end{align}
    
Note that \eqref{eq: D tilde mean} and $\E[C_{x'}^{(t)}] = \exp(\|\z^{(t)}\|^2/(2\eta^2 r))$ yield
    \begin{align*}
        \frac{\E[D_x^{(t,t')}]}{\E[C_x^{(t)}]} = \frac{1}{\eta r} \z^{(t') \top} \z^{(t)}.
    \end{align*}
    Therefore, combining \eqref{eq: triangle 1} and \eqref{eq: U condition 3}, we have
    \begin{align}
        \P\qty(\max_{t' \in \mT} |\Delta^{(t,t')}| \lesssim a + \frac{\log^2 d}{\sqrt{n r}} + \frac{\log d}{\sqrt{dr}} \middle| U, T = t, \eta) = 1 - \exp(-\Omega\qty(\frac{ra^2 n}{\log^2 d} \vee \log^2 d)),\label{eq: u mean claim 1}
    \end{align}
    where $\Delta^{(t,t')} := (1/n) \sum_{i \in [n]} \z^{(t') \top} \u_{X_i} - (1/\eta r) \z^{(t') \top} \z^{(t)}$.

This completes the proof of Lemma~\ref{lem: subject identification}.
\end{proof}

\begin{lemma}\label{lem: function identification}
For any $a > 0$, $m \in \mM$, $t \in \mT$ and $\eta \geq (1/\sqrt{r}) \log d$, it holds that
\begin{align*}
        &\P\Biggl(\max_{m' \in \mM} \abs{\frac{1}{n} \sum_{i \in [n]} \langle f^{(m')}(\u_{X_i}), \u_{Y_i} \rangle - \E\qty[\langle f^{(m')}(\u_{X_1}), \u_{Y_1} \rangle | M=m, T=t, U, \eta]}\\
        &\quad\quad\quad\quad> a + \frac{\log^2 d}{\sqrt{n r}} \bigg| M=m, T=t, U, \eta\Biggr) = \exp(-\Omega\qty(\frac{ra^2 n}{\log^2 d} + \log^2 d)).
\end{align*}
with high probability. Here the $\Omega$-notation does not depend on $t$.
\end{lemma}

\begin{proof}[\textbf{Proof of Lemma~\ref{lem: function identification}}]
To ease the notation, let $p_x^{(t)} := \P(X_1 = x | T=t, U, \eta)$. For now we fix any $m' \in \mM$. We also fix any $t \in \mT$, $m \in \mM$ and $\eta \geq (1/\sqrt{r}) \log d$. By Lemma~\ref{lem: good event U},
\begin{align}
        \max_{m' \in \mM} \max_{x,y \in \mX} |H_{x,y}^{(m')}| \leq \max_{y \in \mX} \|\u_y\| \leq C (1/\sqrt{r})\log d\label{eq: temp}
    \end{align}
    holds for some constant $C > 0$. For any $U$ satisfying \eqref{eq: temp}, Hoeffding's inequality gives
    \begin{align*}
        &\P\qty(\abs{\sum_{i \in [n]} \qty(H_{X_i, Y_i}^{(m')} - \E\qty[H_{X_1, Y_1}^{(m')} | M=m, T=t, U, \eta])} > \epsilon \middle| M=m, T=t, U, \eta)\\
        &\quad\leq 2 \exp(-\frac{r \epsilon^2}{2 n C^2 \log^2 d}).
    \end{align*}
    Choosing $\epsilon \gets a n \vee \sqrt{n/r} \log^2 d$ and a union bound argument gives
    \begin{small}
    \begin{align}
        &\P\qty(\max_{m' \in \mM} \abs{\frac{1}{n} \sum_{i \in [n]} H_{X_i, Y_i}^{(m')} - \E\qty[H_{X_1, Y_1}^{(m')} | M=m, T=t, U, \eta]} > a + \frac{\log^2 d}{\sqrt{n r}} \middle| M=m, T=t, U, \eta)\\
        &\quad\lesssim |\mM| |\mT| \exp(-\Omega\qty(\frac{ra^2 n}{\log^2 d} + \log^2 d)) = \exp(-\Omega\qty(\frac{ra^2 n}{\log^2 d} + \log^2 d)),\label{eq: H - E H}
    \end{align}
    \end{small}\noindent
where we use $|\mM| \lesssim d^\alpha$ for some constant $\alpha = O(1)$. 

This completes the proof of Lemma~\ref{lem: function identification}.
\end{proof}

Next, we provide a lemma for the existence of a good event for the prior of $U = [\u_1, \dots, \u_x]^\top$.

\begin{lemma}\label{lem: good event U}
Let $\u_x \sim N(0, (1/r) I_r)$ i.i.d. Suppose $\eta \geq (1 / \sqrt{r}) \log d$. If $r = o(\log d)$, then $\max_{x \in \mX} \|\u_x\| \lesssim {\log d}/{\sqrt{r}}$, and
\begin{align}
        \sup_{\z \in \mathbb{B}_r(1)} \abs{\sum_{x \in \mX} \qty{\exp(\eta^{-1} \z^\top \u_x) - \exp(\frac{\|\z\|^2}{2 \eta^2 r})} } &\lesssim \sqrt{d} \log d,\label{eq: good event U claim 0}\\
        \sup_{\z \in \mathbb{B}_r(1)} \abs{\frac{d \exp(\|\z\|^2/(2\eta^2 r))}{\sum_{x' \in \mX} \exp(\eta^{-1} \z^\top \u_{x'})} - 1} &\lesssim \frac{1}{\sqrt{d}} \log d,\label{eq: good event U claim 1}\\
        \sup_{\z, \z' \in \mathbb{B}_r(1)} \abs{\sum_{x \in \mX} \qty{\z'^\top \u_x \exp(\eta^{-1} \z^\top \u_x) - \E\qty[\z'^\top \u_x \exp(\eta^{-1} \z^\top \u_x)]}} &\lesssim \sqrt{\frac{d}{r}} \log d,\label{eq: good event U claim 2}\\
        \sup_{\z, \z' \in \mathbb{B}_r(1)} \abs{\sum_{x \in \mX} \frac{\z'^\top \u_x \exp(\eta^{-1} \z^\top \u_x)}{\sum_{x' \in \mX} \exp(\eta^{-1} \z^\top \u_{x'})} - \frac{\E\qty[\z'^\top \u_x \exp(\eta^{-1} \z^\top \u_x)]}{\exp(\|\z\|^2/(2\eta^2 r))}} &\lesssim \frac{1}{\sqrt{dr}} \log d,\label{eq: good event U claim 3}
\end{align}
hold with probability $1 - \exp(-\Omega(\log^2 d))$.
\end{lemma}

\begin{proof}[\textbf{Proof of Lemma~\ref{lem: good event U}}]
    We first derive the concentration for $(\|\u_x\|)_{x \in \mX}$.
    Fix any $x \in \mX$.
    Note that $r \|\u_x\|^2 \sim \chi_2(r)$. The concentration inequality for chi-squared distribution (see, for example, Lemma 1 in \citet{laurent2000adaptive}) gives
    \begin{align*}
        \P(r \|\u_x\|^2 \geq r + 2\sqrt{r\epsilon} + 2\epsilon) \leq \exp(-\epsilon).
    \end{align*}
    Choosing $\epsilon \gets \log^2 d$ gives $r \|\u_x\|^2 \leq (\sqrt{r} + 2\log d)^2$ with high probability. By a union bound argument, we have
    \begin{align}
        \max_{x \in \mX} \|\u_x\| \leq \frac{\sqrt{r} + 2\log d}{\sqrt{r}} \leq \frac{C \log d}{\sqrt{r}}\label{eq: u norm concentration}
    \end{align}
    for some constant $C > 0$
    with probability $1 - |\mX| \exp(-\Omega(\log^2 d)) = 1 - \exp(-\Omega(\log^2 d))$, where we used $r = o(\log d)$.
    Let $\mathcal{E}$ be the event where \eqref{eq: u norm concentration} holds.
    Let $\z_1, \z_2, \dots, \z_J \in \mathbb{B}_r(1)$ be the centers of a $\delta$-covering ($\delta > 0$) of a ball in $\R^r$ with radius $1$, that is,
    \begin{align*}
        \mathbb{B}_r(1) \subset \bigcup_{j \in [J]} \{\z \in \R^r : \|\z - \z_j\| \leq \delta\}.
    \end{align*}
From a standard argument of covering number, we can take the $\delta$-covering $(\z_j)_{j \in [J]}$ with $\log J = O(r \log (1 + 1/\delta))$. We specifically choose $\delta = 1/\sqrt{d}$.

We next prove \eqref{eq: good event U claim 0} to \eqref{eq: good event U claim 3} one by one.

\paragraph{Step 1: Proof of \eqref{eq: good event U claim 0} and \eqref{eq: good event U claim 1}.}
We first derive the uniform convergence of $\sup_{\z \in \mathbb{B}_r(1)} |\tilde C_x(\z) - \E[\tilde C_x(\z)]|$, where $\tilde C_x(\z) := \exp(\eta^{-1} \z^\top \u_x) \Id_\mathcal{E}$. Since $\eta^2 r \geq \log^2 d$, we have
    \begin{align*}
        \max_{x \in \mX} \max_{j \in [J]} |\tilde C_x(\z_j)| \leq \exp(\frac{C \log d}{\eta \sqrt{r}}) \leq e^C.
    \end{align*}
    Now fix any $j \in [J]$. From Hoeffding's inequality,
    \begin{align*}
        \P\qty(\abs{\sum_{x \in \mX} (\tilde C_x(\z_j) - \E[\tilde C_x(\z_j)])} > \epsilon) \leq 2 \exp(-\frac{\epsilon^2}{2 d e^{2C}}).
    \end{align*}
    Choosing $\epsilon \gets \sqrt{d} \log d$ gives $|\sum_{x \in \mX} (\tilde C_x(\z_j) - \E[\tilde C_x(\z_j)])| \leq \sqrt{d} \log d$ with probability $1 - \exp(-\Omega(\log^2 d))$.
    By a union bound, we have $\max_{j \in [J]} |\sum_{x \in \mX} (\tilde C_x(\z_j) - \E[\tilde C_x(\z_j)])| \leq \sqrt{d} \log d$ with probability
    \begin{align*}
        1 - J \exp(-\Omega(\log^2 d)) = 1 - \exp(-\Omega(\log^2 d) + O(r \log d)) = 1 - \exp(-\Omega(\log^2 d)).
    \end{align*}

Let $C_x(\z) := \exp(\eta^{-1} \z^\top \u_x)$. For any fixed $\z \in \mathbb{B}_r(1)$, there exists some $j' \in [J]$ such that $\|z - \z_{j'}\| \leq \delta$. Note that on the event $\mathcal{E}$, $C_x(\z) = \tilde C_x(\z)$ and thus
    \begin{align*}
        |C_x(\z) - \tilde C_x(\z_j)| &\leq \abs{\exp(\eta^{-1} \z^\top \u_x) - \exp(\eta^{-1} \z_j^\top \u_x)} \\
        & = \exp(\eta^{-1} \z_j^\top \u_x) \abs{\exp(\eta^{-1} (\z - \z_j)^\top \u_x) - 1} \\
        & \leq 2 e \eta^{-1} |(\z_j - \z)^\top \u_x|
        \lesssim \delta \frac{\log d}{\eta \sqrt{r}} \leq \frac{1}{\sqrt{d}},
    \end{align*}
where we use $\exp(\eta^{-1} \z_j^\top \u_x) \leq \exp(\eta^{-1} \max_{x \in \mX} \|\u_x\|) \lesssim 1$ by \eqref{eq: u norm concentration}, $|1 - e^x| \leq 2|x|$ for $x \in [-1, 1]$ and $|\eta^{-1} (\z - \z_j)^\top \u_x| \leq \eta^{-1} \delta \max_{x \in \mX} \|\u_x\| \leq 1$ for sufficiently large $d$ in the second inequality. 
    Furthermore, a similar argument combined with Cauchy-Schwarz inequality gives
    \begin{align*}
        &|\E[C_x(\z)] - \E[\tilde C_x(\z_{j'})]| \leq |\E[C_x(\z)] - \E[C_x(\z_{j'})]| +  |\E[C_x(\z_{j'})] - \E[\tilde C_x(\z_{j'})]|\\
        &\quad= \abs{\exp(\frac{\|\z\|^2}{2\eta^2 r}) - \exp(\frac{\|\z_{j'}\|^2}{2\eta^2 r})} + \E\qty[\exp(\eta^{-1} \z_{j'}^\top \u_x) \Id_{\mathcal{E}^c}]\\
        &\quad\leq \exp(\frac{\|\z\|^2}{2\eta^2 r}) \abs{1 - \exp(\frac{\|\z_{j'}\|^2 - \|\z\|^2}{2\eta^2 r})} + \sqrt{\E\qty[\exp(2 \eta^{-1} \z_{j'}^\top \u_x)] \P(\mathcal{E}^c)}\\
        &\quad\leq \exp(\frac{\|\z\|^2}{2\eta^2 r}) \abs{1 - \exp(\frac{(\|\z_{j'}\| - \|\z\|)(\|\z_{j'}\| + \|\z\|)}{2\eta^2 r})} + \exp(\frac{\|\z_{j'}\|^2}{\eta^2 r}) \exp(-\Omega(\log^2 d))\\
        &\quad\leq \frac{2\sqrt{e} \delta}{\eta^2 r} + \exp(-\Omega(\log^2 d))
        \lesssim \frac{1}{\sqrt{d} \log^2 d} + \exp(-\Omega(\log^2 d)).
    \end{align*}
    Hence by a union bound argument,
    \begin{align}
        \sup_{\z \in \mathbb{B}_r(1)} \abs{\sum_{x \in \mX} (C_x(\z) - \E[C_x(\z)])} &\lesssim d \qty(\frac{1}{\sqrt{d}} + \frac{1}{\sqrt{d} \log^2 d} + \exp(-\Omega(\log^2 d))) + \sqrt{d} \log d \nonumber\\
        &\lesssim \sqrt{d} \log d\label{eq: good event U part 1}
    \end{align}
    holds with high probability.
Using $\E[C_x(\z)] = \exp(\|\z\|^2/(2 \eta^2 r)) \geq 1$ and \eqref{eq: good event U part 1}, we obtain
    \begin{align}
        \sup_{\z \in \mathbb{B}_r(1)} \abs{\frac{1}{d \E[C_x(\z)]} \sum_{x \in \mX} C_x(\z) - 1} \lesssim \frac{1}{\sqrt{d}} \log d.\label{eq: C concentration}
    \end{align}
    Thus
    \begin{align*}
        \sup_{\z \in \mathbb{B}_r(1)} \abs{\frac{d \E[C_x(\z)]}{\sum_{x \in \mX} C_x(\z)} - 1} \lesssim \frac{1}{\sqrt{d}} \log d
    \end{align*}
    holds with high probability for sufficiently large $d$.

\paragraph{Step 2: Proof of \eqref{eq: good event U claim 2}}
    We first derive a concentration inequality for $\sum_{x \in \mX} \z_{j'}^\top \u_x \exp(\eta^{-1} \z_j^\top \u_x)$ uniformly over all $j,j' \in [J]$.
    Fix any $j,j' \in [J]$. Define $D_x(\z,\z') := \z'^\top \u_x \exp(\eta^{-1} \z^\top \u_x)$ and $\tilde D_x(\z,\z') := \z'^\top \u_x \exp(\eta^{-1} \z^\top \u_x) \Id_\mathcal{E}$.
    Note that since $\eta \geq \log d / \sqrt{r}$,
    \begin{align*}
        |\tilde D_x(\z_j,\z_{j'})| \leq \frac{1}{2 \sqrt{r}} \exp(\frac{1}{2 \eta \sqrt{r}} \log d) \log d \leq \frac{\sqrt{e}}{2\sqrt{r}} \log d
    \end{align*}
    holds for $d \geq 4$. 
    Furthermore,
    \begin{align}
        \E[\tilde D_x(\z_j,\z_{j'})] &\leq \E[D_x(\z_j,\z_{j'})] = \z_{j'}^\top \E\qty[\u_x \exp(\eta^{-1} \z_j^\top \u_x)]\nonumber\\
        &= \eta \z_{j'}^\top \E\qty[\eval{\dv{\z} \exp(\eta^{-1} \z^\top \u_x)}_{\z = \z_j}] 
        = \eta \z_{j'}^\top \eval{\dv{\z} \E\qty[\exp(\eta^{-1} \z^\top \u_x)]}_{\z = \z_j}\nonumber\\
        &= \eta \z_{j'}^\top \eval{\dv{\z} \exp(\frac{\|\z\|^2}{2\eta^2 r})}_{\z = \z_j}
        = \frac{1}{\eta r} \z_{j'}^\top \z_j \exp(\frac{\|\z_j\|^2}{2\eta^2 r}),\label{eq: D tilde mean}
    \end{align}
    and
    \begin{align}
        \E[\tilde D_x(\z_j,\z_{j'})^2] &\leq \E[D_x(\z_j,\z_{j'})^2] = \E\qty[(\z_{j'}^\top \u_x)^2 \exp(2 \eta^{-1} \z_j^\top \u_x)]\nonumber\\
        &= \E\qty[\tr(\z_{j'} \z_{j'}^\top \u_x \u_x^\top \exp(2 \eta^{-1} \z_j^\top \u_x))]\nonumber\\
        &= \frac{\eta^2}{4} \E\qty[\tr(\z_{j'} \z_{j'}^\top \eval{\dv{\z} \dv{\z^\top} \exp(2 \eta^{-1} \z^\top \u_x)}_{\z=\z_j})]\nonumber\\
        &= \frac{\eta^2}{4} \tr(\z_{j'} \z_{j'}^\top \eval{\dv{\z} \dv{\z^\top} \E\qty[\exp(2 \eta^{-1} \z^\top \u_x)]}_{\z=\z_j})\nonumber\\
        &= \frac{\eta^2}{4} \tr(\z_{j'} \z_{j'}^\top \eval{\dv{\z} \dv{\z^\top} \exp(\frac{2 \|\z\|^2}{\eta^2 r})}_{\z=\z_j})\nonumber\\
        &= \frac{\eta^2}{4} \tr(\z_{j'} \z_{j'}^\top \qty(\frac{4}{\eta^2 r} I + \frac{16}{\eta^4 r^2} \z_j \z_j^\top) \exp(\frac{2 \|\z_j\|^2}{\eta^2 r}))\nonumber\\
        &= \frac{\eta^2}{4} \qty(\frac{4}{\eta^2 r} + \frac{16}{\eta^4 r^2} (\z_{j'} \z_j)^2) \exp(\frac{2\|\z_{j'}\|}{\eta^2 r})
        \leq \frac{5 e^2}{r},\label{eq: D tilde second moment}
    \end{align}
    where we used $\eta^2 r \geq \log^2 d \geq 1$ for $d \geq 4$, and $\|\z_j\| \leq 1$ in the last inequality.
    From Bernstein inequality, we obtain
    \begin{align*}
        \P\qty(\abs{\sum_{x \in \mX} (\tilde D_x(\z_j,\z_{j'}) - \E[\tilde D_x(\z_j,\z_{j'}))} \geq \epsilon) &\leq 2\exp(-\frac{(1/2) \epsilon^2}{5 d e^2 / r + (1/6) \sqrt{e/r} \log d \epsilon}).
    \end{align*}
    Choosing $\epsilon \gets \sqrt{d/r} \log d$ gives 
    \begin{align*}
        \P\qty(\abs{\sum_{x \in \mX} (\tilde D_x(\z_j,\z_{j'}) - \E[\tilde D_x(\z_j,\z_{j'})])} 
        \leq \sqrt{\frac{d}{r}} \log d) 
        &\leq \exp(-\Omega\qty(\frac{(d/r) \log^2 d}{d/r + (\sqrt{d} /r) \log^2 d}))\\
        &= \exp(-\Omega(\log^2 d)).
    \end{align*}
    By a union bound argument, we have $\max_{j,j' \in [J]} |\sum_{x \in \mX} (\tilde D_x(\z_j,\z_{j'}) - \E[\tilde D_x(\z_j,\z_{j'})])| \leq \sqrt{d/r} \log d$ with high probability.
    From Cauchy-Schwarz inequality, we have
    \begin{align*}
        0 &\leq \E[D_x(\z_j,\z_{j'})] - \E[\tilde D_x(\z_j,\z_{j'})] = \E\qty[\z_{j'}^\top \u_x \exp(\eta^{-1} \z_j^\top \u_x) \Id_{\mathcal{E}^c}]\\
        &\leq \sqrt{\E[D_x(\z_j,\z_{j'})^2] \P(\mathcal{E}^c)} \leq \sqrt{\frac{5 e^2}{r}} \exp(-\Omega(\log^2 d)),
    \end{align*}
    where the last inequality follows from \eqref{eq: D tilde second moment}.
    Since $\tilde D_x(\z_j,\z_{j'}) = D_x(\z_j,\z_{j'})$ on the event $\mathcal{E}$, we have
    \begin{align}
        \max_{j,j' \in [J]} \abs{\sum_{x \in \mX} (D_x(\z_j,\z_{j'}) - \E[D_x(\z_j,\z_{j'})])} \lesssim \sqrt{\frac{d}{r}} \log d + \frac{1}{\sqrt{r}} \exp(-\Omega(\log^2 d)) \lesssim \sqrt{\frac{d}{r}} \log d\label{eq: D concentration}
    \end{align}
with probability $1 - \exp(-\Omega(\log^2 d)) - \P(\mathcal{E}^c) = 1 - \exp(-\Omega(\log^2 d))$. A similar argument as in Step1 gives \eqref{eq: good event U claim 2}.

\paragraph{Step 3: Proof of \eqref{eq: good event U claim 3}.}
Fix any $\z, \z' \in \mathbb{B}_r(1)$. Combining \eqref{eq: good event U claim 1} and \eqref{eq: good event U claim 2} yield
    \begin{small}
    \begin{align*}
        &\abs{\sum_{x \in \mX} \frac{D_x(\z,\z')}{\sum_{x' \in \mX} C_{x'}(\z)} - \frac{\E[D_x(\z,\z')]}{\E[C_x(\z)]}}\\
        &\quad\leq \frac{1}{d \E[C_{x'}(\z)]} \abs{\sum_{x \in \mX} D_x(\z,\z') \qty(1 - \frac{d\E[C_{x'}(\z)]}{\sum_{x' \in \mX} C_{x'}(\z)})} + \frac{1}{d \E[C_{x'}(\z)]} \abs{\sum_{x \in \mX} (D_x(\z,\z') - \E[D_x(\z,\z')])}\\
        &\quad\leq \abs{\frac{1}{d} \sum_{x \in \mX} D_x(\z,\z')} \abs{1 - \frac{d\E[C_{x'}(\z)]}{\sum_{x' \in \mX} C_{x'}(\z)}} + \frac{1}{d} \abs{\sum_{x \in \mX} (D_x(\z,\z') - \E[D_x(\z,\z')])}\\
        &\quad\lesssim \qty(\frac{\z'^\top \z}{\eta r} \exp(\frac{\|\z\|^2}{2\eta^2 r}) + \frac{1}{\sqrt{dr}} \log d) \qty(\frac{1}{\sqrt{d}} \log d) + \frac{1}{\sqrt{dr}} \log d
        \lesssim \frac{1}{\sqrt{dr}} \log d, 
    \end{align*}
    \end{small}
    where we use $\E[C_{x'}(\z)] \geq 1$ in the second inequality, and \eqref{eq: D tilde mean} in the third inequality. The last inequality follows since $\|\z\|^2 \leq 1$, $|\z'^\top \z| \leq 1$ and $\eta^2 r \geq \log^2 d$. Since $\z, \z' \in \mathbb{B}_r(1)$ are arbitrary, this gives \eqref{eq: good event U claim 3}.

This completes the proof of Lemma~\ref{lem: good event U}. 
\end{proof}

\subsection{Construction of transformer layers}
\label{supp-sec:construction}

We construct a transformer that outputs synthetic covariates or labels depending on the current task. We only consider the case where $|\mM| \geq |\mT|$. The proof is similar for the other case. 

We first define identity feedforward neural networks and transformers.
\begin{definition}
Define the identity attention $\Attn_{\mu_\id}$ and identity feed-forward neural network $\FFN_{\nu_\id}$ as $\mu_\id = \{(O, O, O)\}$ and $\nu_\id = \{O, O\}$.
 \end{definition}
We also introduce a useful function that combines the product and indicator functions.
Define $\phi_B: \R^3 \to \R$ as
\begin{align} \label{eq: phi c}
\begin{split}
\phi_B(x; s, t) &:= -4 B \sigma\qty(\frac{1}{4 B} x + t - s + \frac{1}{2}) + 8 B \sigma\qty(\frac{1}{4 B} x + t - s + \frac{1}{4}) \\
&\quad- 8 B \sigma\qty(\frac{1}{4 B} x + t - s - \frac{1}{4}) + 4 B \sigma\qty(\frac{1}{4 B} x + t - s - \frac{1}{2}). 
\end{split}
\end{align}

The next lemma facilitates the implementation of an attention layer that filters some input tokens. Its proof is straightforward and is omitted. 
\begin{lemma}\label{lem: phi}
For any $B > 0$ and $x \in [-B, B]$, and $s, t \in \Z$, we have $\phi_B(x; s, t) = x \Id\{s = t\}$. 
\end{lemma}

We also introduce the notation to simplify the statement. We omit the second subscript $n$ from $\p_{s,n}$ for the positional encoding. Define $\h_{2i-1} = \h_i^X$ and $\h_{2i} = \h_i^Y$ for all $i \in [n]$. Let $(\tilde X_s, \tilde Y_s)$ ($s \geq 1$) be the synthetic pair of data generated as the $(2s-1)$-th and $(2s)$-th output from the transformer. Let $X_{n+s} := \tilde X_s$ and $Y_{n+s} := \tilde Y_s$. Define 
\begin{align} \label{eq: generated token}
    \h_{2n+2s-1} = \begin{pmatrix}
        \u_{X_{n+s}}\\
        \zero\\
        \p_{2n+2s-1}
    \end{pmatrix}, \ \ 
    \h_{2n+2s-1} = \begin{pmatrix}
        \u_{Y_{n+s}}\\
        \zero\\
        \p_{2n+2s}
    \end{pmatrix}
\end{align}
for $s \geq 1$. As such, we have input tokens $H_n = [\h_1, \h_2, \dots, \h_{2n}]$ and previously generated tokens $[\h_{2n+1}, \h_{2n+2}, \dots, \h_{2n+\ell}]$ for some $\ell \in \{0\} \cup \N$.

We present the following proposition regarding the construction of a good transformer, that interchangeably generates $\hat f(\u_{X_{n+1}}), \hat \z, \hat f(\u_{X_{n+2}}), \hat \z, \dots$. 
\begin{proposition}\label{prop: transformer as good data generator}
Fix any $d, r, r_0, L_0 \in \N$, $\omega > 0$, $(\z^{(t)})_{t \in \mT}$ and $(f^{(m)})_{m \in \mM} \subset \mF(L_0, r_0)$. Then, there exist transformer layers $\TF_{\Psi^*}$ with $\Psi^* = \Psi^*((\z^{(t)})_{t \in \mT}, (f^{(m)})_{m \in \mM}, d, r, \omega, L_0, r_0)$, such that
\begin{enumerate}[(i)]
        \item the dimension of token embeddings is $r+r|\mM| + |\mM| + 4$;
        \item it consists of $L_0+9$ transformer layers with the width of FNN $O(|\mM|^2 \vee |\mM| r_0)$, and the number of heads of attention layers $O(|\mM|)$;
        \item given inputs $H_n$ and $\h_{2n+1}, \dots, \h_{2n+\ell}$ defined in \eqref{eq: generated token}, it outputs
    \begin{align*}
        \TF_{\Psi^*}([H_n; \h_{2n+1}; \dots; \h_{2n+\ell}])_{2n+\ell} = \begin{cases}
            (\hat f(\u_{X_{n+s}})^\top, \zero^\top)^\top & \text{ if $\ell=2s-1$},\\
            (\hat \z^\top, \zero^\top)^\top & \text{ if $\ell=2s$},
        \end{cases}
    \end{align*}
    for all $s \in \N$, where $\hat f(\u_{X_{n+s}})$ and $\hat \z$ satisfy
    \begin{small}
    \begin{align*}
        \hat f(\u_{X_{n+s}}) &\in \conv\biggl\{ f^{(m')}(\u_{X_{n+s}}): m' \in \mM, \frac{1}{n} \sum_{i \in [n]} \langle \u_{Y_i}, f^{(m')}(\u_{X_i}) \rangle\\
        &\quad\quad\quad\quad\geq \max_{m'' \in \mM} \frac{1}{n} \sum_{i \in [n]} \langle \u_{Y_i}, f^{(m'')}(\u_{X_i}) \rangle - \frac{\omega}{n}\biggr\},\\
        \hat \z &\in \conv\biggl\{ \z^{(t')}: t' \in \mT, \frac{1}{n} \sum_{i \in [n]} \langle \u_{X_i}, \z^{(t')} \rangle \geq \max_{t'' \in \mT} \frac{1}{n} \sum_{i \in [n]} \langle \u_{X_i}, \z^{(t'')} \rangle - \frac{\omega}{n}\biggr\},
    \end{align*}
    \end{small}\noindent
where $\conv(S)$ denotes the convex hull of a set $S$.
\end{enumerate}
\end{proposition}

\begin{proof}[\textbf{Proof of Proposition~\ref{prop: transformer as good data generator}}]
Let $B = \max_{x \in \mX} \|\u_x\|$ so that $\sup_{\z: \|\z\| = 1} |\z^\top \u_x| \leq B$ holds for all $x \in \mX$. For brevity, define $\bar m = |\mM|$, $\bar t = |\mT|$, and write $\mM = [\bar m]$, $\mT = [\bar t]$. Recall that $\bar m \geq \bar t$ by assumption. Fix $\ell \in \{0\} \cup \N$. To ease the notation, let $N := 2n+\ell$, $D := r + r \bar m + \bar m + 4$ and $H^{(\ell)}_n := [\h_1; \dots; \h_{2n}; \h_{2n+1}; \dots; \h_{2n+\ell}] := [H_n; \h_{2n+1}; \dots; \h_{2n+\ell}]$. 

We show the existence of transformer layers $\TF_{\Psi^*}$, such that for any $\ell \geq 0$, given the input $H_n^{(\ell)} \in \R^{D \times (2n+\ell)}$ at the $\ell$-th step, the last token is transformed to :
\begin{align*}
        (\TF_{\Psi^*}(H_n^{(\ell)}))_{2n+\ell} &= \begin{cases}
            \begin{pmatrix}
                \hat f(\u_{X_i})\\
                {\bm *}
            \end{pmatrix} & \text{if $\ell$ is even},\\
            \begin{pmatrix}
                \hat \z\\
                {\bm *}
            \end{pmatrix} & \text{if $\ell$ is odd}.
        \end{cases}
\end{align*}
Define $\z^{(m)} = 0$ for $m \in [\bar m] \setminus [\bar t]$. We divide the construction of the desired transformer into 4 steps.

\paragraph{Step 1.}
In Step 1, we aim to construct transformer layers with parameter $\Psi_1^*$, such that it outputs
    \begin{align*}
        \TF_{\Psi_1^*}(H_n^{(\ell)})_{2i-1} &= \begin{pmatrix}
            \u_{X_i}\\
            f^{(1)}(\u_{X_i})\\
            \vdots\\
            f^{(\bar m)}(\u_{X_i})\\
            \zero_{\bar m}\\
            \p_{2i-1}
        \end{pmatrix}, \ \ 
        \TF_{\Psi_1^*}(H_n^{(\ell)})_{2i} =
        \begin{pmatrix}
            \u_{Y_i}\\
            \z^{(1)}\\
            \vdots\\
            \z^{(\bar m)}\\
            \zero_{\bar m}\\
            \p_{2i}
        \end{pmatrix}.
    \end{align*}
    Note that by assumption, there exist weights $\{(W_{\pre,k,1}^{(m)}, W_{\pre,k,2}^{(m)})\}_{k \in [L_0], m \in [\bar m]} \subset \R^{r_0 \times r} \times \R^{r \times r_0}$, such that
    \begin{align*}
        f^{(m)}(\u) = g_{L_0}^{(m)} \circ g_{L_0-1}^{(m)} \circ \dots \circ g_1^{(m)}(\u), \ \ g_k^{(m)}(\u) = W_{\pre,k,2}^{(m)} \sigma(W_{\pre,k,1}^{(m)} \u).
    \end{align*}
    
    For $k \in [L_0]$, choose parameters $\nu_{1,k}^* = (W_{1,k,1}^*, W_{1,k,2}^*) \in \R^{(\bar m r_0) \times D} \times \R^{D \times (\bar m r_0)}$ with
    \begin{align*}
        W_{1,k,1}^* &= \begin{bmatrix}
            W_{\pre,k,1}^{(1)} & O_{r_0\times (D-r)}\\
            \vdots\\
            W_{\pre,k,1}^{(\bar m)} & O_{r_0\times (D-r)}
        \end{bmatrix}, \ \ 
        W_{1,k,2}^* = \begin{bmatrix}
            \multicolumn{3}{c}{O_{r\times (\bar m r_0)}}\\
            & W_{\pre,k,2}^{(1)} & O_{r \times ((\bar m-1) r_0)}\\
            O_{r \times r_0} & W_{\pre,k,2}^{(2)} & O_{r \times ((\bar m-2) r_0)}\\
            & \vdots &\\
            O_{r \times ((\bar m-1) r_0)} & W_{\pre,k,2}^{(\bar m)}\\
            \multicolumn{3}{c}{O_{(D-r(1+\bar m)) \times (\bar m r_0)}}
        \end{bmatrix}.
    \end{align*}
    Let $\psi_{1,k}^* = (\mu_\id, \nu_{1,k}^*)$ and $H_n^{(\ell)[0.5]} = [\h_1^{[0.5]}; \dots; \h_N^{[0.5]}] := \TF_{(\psi_{1,1}^*, \dots, \psi_{1,L_0}^*)}(H_n^{(\ell)})$.
    Note that
    \begin{align*}
        \h^{[0.5]}_{2i-1} &= \begin{pmatrix}
            \u_{X_i}\\
            f^{(1)}(\u_{X_i})\\
            \vdots\\
            f^{(\bar m)}(\u_{X_i})\\
            \zero_{\bar m}\\
            \p_{2i-1}
        \end{pmatrix}, \ \ 
        \h^{[0.5]}_{2i} =
        \begin{pmatrix}
            \u_{Y_i}\\
            f^{(1)}(\u_{Y_i})\\
            \vdots\\
            f^{(\bar m)}(\u_{Y_i})\\
            \zero_{\bar m}\\
            \p_{2i}
        \end{pmatrix}.
    \end{align*}
    Choose $\psi_{1,L_0+1}^* = (\mu_{L_0+1}^*, \nu_\id)$ with $\mu_{L_0+1}^* := \{(Q_{1,L_0+1,j}^*, K_{1,L_0+1,j}^*, V_{1,L_0+1,j}^*)\}_{j \in [4]}$, such that
    \begin{align*}
        Q_{1,L_0+1,1}^* \h^{[0.5]}_s = \begin{pmatrix}
            (\p_s)_2/4\\
            -2(\p_s)_1-(\p_s)_2\\
            1\\
            1\\
            \zero_{D-4}
        \end{pmatrix}, \ \ 
        K_{1,L_0+1,1}^* \h^{[0.5]}_s = \begin{pmatrix}
            1\\
            1\\
            2(\p_s)_1+(\p_s)_2\\
            1/2\\
            \zero_{D-4}
        \end{pmatrix},\\
        Q_{1,L_0+1,2}^* \h^{[0.5]}_s = \begin{pmatrix}
            (\p_s)_2/4\\
            -2(\p_s)_1-(\p_s)_2\\
            1\\
            1\\
            \zero_{D-4}
        \end{pmatrix}, \ \ 
        K_{1,L_0+1,2}^* \h^{[0.5]}_s = \begin{pmatrix}
            1\\
            1\\
            2(\p_s)_1+(\p_s)_2\\
            1/4\\
            \zero_{D-4}
        \end{pmatrix},\\
        Q_{1,L_0+1,3}^* \h^{[0.5]}_s = \begin{pmatrix}
            (\p_s)_2/4\\
            -2(\p_s)_1-(\p_s)_2\\
            1\\
            1\\
            \zero_{D-4}
        \end{pmatrix}, \ \ 
        K_{1,L_0+1,3}^* \h^{[0.5]}_s = \begin{pmatrix}
            1\\
            1\\
            2(\p_s)_1+(\p_s)_2\\
            -1/4\\
            \zero_{D-4}
        \end{pmatrix},\\
        Q_{1,L_0+1,4}^* \h^{[0.5]}_s = \begin{pmatrix}
            (\p_s)_2/4\\
            -2(\p_s)_1-(\p_s)_2\\
            1\\
            1\\
            \zero_{D-4}
        \end{pmatrix}, \ \ 
        K_{1,L_0+1,4}^* \h^{[0.5]}_s = \begin{pmatrix}
            1\\
            1\\
            2(\p_s)_1+(\p_s)_2\\
            -1/2\\
            \zero_{D-4}
        \end{pmatrix},
    \end{align*}
    and 
    \begin{align*}
        V_{1,L_0+1,1}^* \h^{[0.5]}_s &= -4 \bar \h^{[0.5]}_s, \ \ 
        V_{1,L_0+1,2}^* \h^{[0.5]}_s = 8 \bar \h^{[0.5]}_s,\\
        V_{1,L_0+1,3}^* \h^{[0.5]}_s &= -8 \bar \h^{[0.5]}_s,\ \ 
        V_{1,L_0+1,4}^* \h^{[0.5]}_s = 4 \bar \h^{[0.5]}_s,
    \end{align*}
    where
    \begin{align}
        \bar \h^{[0.5]}_s = \begin{pmatrix}
            \zero_r\\
            \z^{(1)} - (\h^{[0.5]}_{s'})_{(r+1):(2r)}\\
            \vdots\\
            \z^{(\bar m)} - (\h^{[0.5]}_{s'})_{(r\bar m + 1):(r\bar m + r)}\\
            \zero_{\bar m+4}
        \end{pmatrix}.
    \end{align}
    Then,
    \begin{align*}
        &\TF_{\psi_{1,L_0+1}^*}(H_n^{(\ell)[0.5]})_s\\
        &\quad= \h_s^{[0.5]} + \sum_{s' \in [N]} \phi_1((\p_s)_2; 2(\p_s)_1+(\p_s)_2, 2(\p_{s'})_1+(\p_{s'})_2) \begin{pmatrix}
            \zero_r\\
            \z^{(1)} - (\h^{[0.5]}_{s'})_{(r+1):(2r)}\\
            \vdots\\
            \z^{(\bar m)} - (\h^{[0.5]}_{s'})_{(r\bar m + 1):(r\bar m + r)}\\
            \zero_{\bar m+4}
        \end{pmatrix}\\
        &\quad= \h_s^{[0.5]} + (\p_s)_2 \begin{pmatrix}
            \zero_r\\
            \z^{(1)} - (\h^{[0.5]}_s)_{(r+1):(2r)}\\
            \vdots\\
            \z^{(\bar m)} - (\h^{[0.5]}_s)_{(r\bar m + 1):(r\bar m + r)}\\
            \zero_{\bar m+4}
        \end{pmatrix},
    \end{align*}
    where we used $2(\p_s)_1+(\p_s)_2 = 2(\p_{s'})_1+(\p_{s'})_2$ if and only if $s = s'$.
    Hence $\TF_{\Psi_1^*}$ with $\Psi_1^* = (\psi_{1,1}^*, \dots, \psi_{1,L_0}^*, \psi_{1,L_0+1}^*)$ is the desired transformer.

\paragraph{Step 2.}
Let $H_n^{(\ell)[1]} = [\h_1^{[1]}; \dots; \h_N^{[1]}] := \TF_{\Psi_1^*}(H_n^{(\ell)})$ be the output from $\TF_{\Psi_1^*}$ constructed in Step 1. In Step 2, we aim to construct transformer layers with parameter $\Psi_2^*$ satisfying
    \begin{align*}
        \TF_{\Psi_2^*}(H_n^{(\ell)[1]})_{2i-1} &= \begin{pmatrix}
            \u_{X_i}\\
            f^{(1)}(\u_{X_i})\\
            \vdots\\
            f^{(\bar m)}(\u_{X_i})\\
            \u_{Y_i}^\top f^{(1)}(\u_{X_i})\\
            \vdots\\
            \u_{Y_i}^\top f^{(\bar m)}(\u_{X_i})\\
            \p_{2i-1}
        \end{pmatrix}, \ \ 
        \TF_{\Psi_2^*}(H_n^{(\ell)[1]})_{2i} = \begin{pmatrix}
            \u_{Y_i}\\
            \z^{(1)}\\
            \vdots\\
            \z^{(\bar m)}\\
            \u_{X_i}^\top \z^{(1)}\\
            \vdots\\
            \u_{X_i}^\top \z^{(\bar m)}\\
            \p_{2i}
        \end{pmatrix}.
    \end{align*}
    Let $\phi_B$ be a function defined in \eqref{eq: phi c}.
    Let $\Attn_{\mu_2^*}$ be an attention layer with parameters $\mu_2^* = \{(Q_{2,j,j'}^*, K_{2,j,j'}^*, V_{2,j,j'}^*)\}_{j \in [\bar m], j' \in [4]}$ defined as
    \begin{align*}
        Q_{2,j,1}^* \h_s^{[1]} &= \begin{pmatrix}
            (\h_s^{[1]})_{1:r}/(4B)\\
            -2 (\p_s)_1 - (\p_s)_2\\
            1\\
            1/2\\
            \zero_{D-3-r}
        \end{pmatrix},\ \ 
        Q_{2,j,2}^* \h_s^{[1]} = \begin{pmatrix}
            (\h_s^{[1]})_{1:r}/(4B)\\
            -2 (\p_s)_1 - (\p_s)_2\\
            1\\
            1/4\\
            \zero_{D-3-r}
        \end{pmatrix},\\
        Q_{2,j,3}^* \h_s^{[1]} &= \begin{pmatrix}
            (\h_s^{[1]})_{1:r}/(4B)\\
            -2 (\p_s)_1 - (\p_s)_2\\
            1\\
            -1/4\\
            \zero_{D-3-r}
        \end{pmatrix},\ \ 
        Q_{2,j,4}^* \h_s^{[1]} = \begin{pmatrix}
            (\h_s^{[1]})_{1:r}/(4B)\\
            -2 (\p_s)_1 - (\p_s)_2\\
            1\\
            -1/2\\
            \zero_{D-3-r}
        \end{pmatrix},\\
        K_{2,j,1}^* \h_s^{[1]} &= \begin{pmatrix}
            (\h_s^{[1]})_{(jr+1):((1+j)r)}\\
            1\\
            2 (\p_s)_1 + 1 - (\p_s)_2\\
            1\\
            \zero_{D-3-r}
        \end{pmatrix},\ \ 
        K_{2,j,2}^* \h_s^{[1]} = \begin{pmatrix}
            (\h_s^{[1]})_{(jr+1):((1+j)r)}\\
            1\\
            2 (\p_s)_1 + 1 - (\p_s)_2\\
            1\\
            \zero_{D-3-r}
        \end{pmatrix},\\
        K_{2,j,3}^* \h_s^{[1]} &= \begin{pmatrix}
            (\h_s^{[1]})_{(jr+1):((1+j)r)}\\
            1\\
            2 (\p_s)_1 + 1 - (\p_s)_2\\
            1\\
            \zero_{D-3-r}
        \end{pmatrix},\ \ 
        K_{2,j,4}^* \h_s^{[1]} = \begin{pmatrix}
            (\h_s^{[1]})_{(jr+1):((1+j)r)}\\
            1\\
            2 (\p_s)_1 + 1 - (\p_s)_2\\
            1\\
            \zero_{D-3-r}
        \end{pmatrix},\\
        V_{2,j,1}^* \h_s^{[1]} &= -4 B \e_{r(1+\bar m)+j},
        V_{2,j,2}^* \h_s^{[1]} = 8 B \e_{r(1+\bar m)+j},\\
        V_{2,j,3}^* \h_s^{[1]} &= -8 B \e_{r(1+\bar m)+j},
        V_{2,j,4}^* \h_s^{[1]} = 4 B \e_{r(1+\bar m)+j}.
    \end{align*}
    
    Then, the $s$-th column of the output of $\Attn_{\mu_2^*}$ is
    \begin{small}
    \begin{align*}
        &\Attn_{\mu_2^*}(H_n^{(\ell)[1]})_s\\
        &\quad= \h_s^{[1]} + \sum_{j \in [\bar m]} \sum_{j' \in [4]} \sum_{s' \in [N]} \sigma(\langle Q_{2,j,j'}^* \h_s^{[1]}, K_{2,j,j'}^* \h_{s'}^{[1]} \rangle) V_{2,j,j'}^* \h_{s'}^{[1]}\\
        &\quad= \h_s^{[1]} + \sum_{j \in [\bar m]} \sum_{s' \in [N]} \phi_{B}((\h_s^{[1]})_{1:r}^\top (\h_{s'}^{[1]})_{(jr+1):((1+j)r)}; 2(\p_s)_1+(\p_s)_2, 2(\p_{s'})_1+1-(\p_{s'})_2) \e_{r(1+\bar m)+j}.
    \end{align*}
    \end{small}\noindent
    Since $2(\p_s)_1+(\p_s)_2 = 2(\p_{s'})_1+1-(\p_{s'})_2$ if and only if $(\p_s)_1 = (\p_{s'})_1$ and $(\p_s)_2 = 1 - (\p_{s'})_2$, 
    \begin{align*}
        \Attn_{\mu_2^*}(H_n^{(\ell)[1]})_s = \h_s^{[1]} + \sum_{s' \in [N]} \begin{pmatrix}
            \zero_{r(1+\bar m)}\\
            (\h_s^{[1]})_{1:r}^\top (\h_{s'}^{[1]})_{(jr+1):((1+j)r)}\\
            \vdots\\
            (\h_s^{[1]})_{1:r}^\top (\h_{s'}^{[1]})_{(jr+1):((1+j)r)}\\
            \zero_4
        \end{pmatrix} \Id\{(\p_s)_1 = (\p_{s'})_1, (\p_s)_2 = 1-(\p_{s'})_2\},
    \end{align*}
    where we used $|(\h_s^{[1]})_{1:r}^\top (\h_{s'}^{[1]})_{(jr+1):((1+j)r)}| \leq B$.
    The desired transformer is obtained by choosing parameter $\Psi_2^* = (\psi_2^*)$, where $\psi_2^* = (\mu_2^*, \nu_\id)$.
    Define $H_n^{(\ell)[2]} = [\h_1^{[2]}; \dots; \h_n^{[2]}] = \TF_{\Psi_2^*}(H_n^{(\ell)[1]})$.

\paragraph{Step 3.}
In Step 3, we aim to construct transformer layers $\TF_{\Psi_3^*}$ satisfying
    \begin{align*}
        \TF_{\Psi_3^*}(H_n^{(\ell)[2]})_{2i-1} &= \begin{pmatrix}
            (\h_{2i-1}^{[2]})_{1:(r(\bar m + 1))}\\
            \sum_{i' \in [n]} \u_{Y_{i'}}^\top f^{(1)}(\u_{X_{i'}})\\
            \vdots\\
            \sum_{i' \in [n]} \u_{Y_{i'}}^\top f^{(\bar m)}(\u_{X_{i'}})\\
            (\h_{2i-1}^{[2]})_{(D-3):D}
        \end{pmatrix}, \ \ 
        \TF_{\Psi_3^*}(H_n^{(\ell)[2]})_{2i} &= \begin{pmatrix}
            (\h_{2i}^{[2]})_{1:(r(\bar m + 1))}\\
            \sum_{i' \in [n]} \u_{X_{i'}}^\top \z^{(1)}\\
            \vdots\\\
            \sum_{i' \in [n]} \u_{X_{i'}}^\top \z^{(\bar m)}\\
            (\h_{2i}^{[2]})_{(D-3):D}
        \end{pmatrix}.
    \end{align*}
    Note that the summations are over all $i' \in [n]$.
    We choose the parameter $\psi_3^* = (\mu_\id, \nu_3^*)$ with $\nu_3^* = (W_{3,1}^*, W_{3,2}^*)$ defined as
    \begin{align*}
        W_{3,2}^* \sigma(W_{3,1}^* \h_s) &= W_{3,2}^* \begin{pmatrix}
            \sigma((\p_s)_2)\\
            \sigma((\p_s)_3)\\
            \sigma(2(\p_s)_1 - (\p_s)_3)
        \end{pmatrix}\\
        &= \begin{pmatrix}
            \zero_{D-2}\\
            - \sigma((\p_s)_3) + \sigma((\p_s)_2) + \sigma(2(\p_s)_1 - (\p_s)_3)\\
            0
        \end{pmatrix}.
    \end{align*}
    Then, it follows that
    \begin{align*}
        \TF_{\psi_3^*}(H_n^{(\ell)[2]})_{2i-1} &= \begin{pmatrix}
            (\h_{2i-1}^{[2]})_{1:(D-4)}\\
            i\\
            0\\
            \sigma(2(i-n))\\
            1
        \end{pmatrix} =: \begin{pmatrix}
            (\h_{2i-1}^{[2]})_{1:(D-4)}\\
            \tilde \p_{2i-1}
        \end{pmatrix},\\
        \TF_{\psi_3^*}(H_n^{(\ell)[2]})_{2i} &= \begin{pmatrix}
            (\h_{2i}^{[2]})_{1:(D-4)}\\
            i\\
            1\\
            1 + \sigma(2(i-n))\\
            1
        \end{pmatrix} =: \begin{pmatrix}
            (\h_{2i}^{[2]})_{1:(D-4)}\\
            \tilde \p_{2i}
        \end{pmatrix}.
    \end{align*}
    Let $H_n^{(\ell)[2.5]} = [\h_1^{[2.5]}, \dots, \h_N^{[2.5]}] := \TF_{\psi_3^*}(H_n^{(\ell)[2]})$.
    By a similar argument as in Step 1 and Step 2,
    we can choose $\psi_4^* = (\mu_4^*, \nu_\id)$ with $\mu_4^* = \{(Q_{4,j}^*, K_{4,j}^*, V_{4,j}^*)\}_{j \in [8]}$ such that
    \begin{align*}
        \TF_{\psi_4^*}(H_n^{(\ell)[2.5]})_s &= \h_s^{[2.5]} + \sum_{s' \in [N]} \phi_1(1; (\tilde \p_s)_2, (\tilde \p_{s'})_3) \begin{pmatrix}
            \zero_{r(\bar m+1)}\\
            (\h_{s'}^{[2.5]})_{(r(\bar m+1)+1):(D-4)}\\
            \zero_4
        \end{pmatrix}\\
        &\quad- \sum_{s' \in [N]} \phi_1(1; 2(\tilde \p_s)_1+(\tilde \p_s)_2, 2(\tilde \p_{s'})_1+(\tilde \p_{s'})_2) \begin{pmatrix}
            \zero_{r(\bar m+1)}\\
            (\h_{s'}^{[2.5]})_{(r(\bar m+1)+1):(D-4)}\\
            \zero_4
        \end{pmatrix}\\
        &= \begin{pmatrix}
            (\h_s^{[2.5]})_{1:(r(\bar m + 1))}\\
            \zero_{r \bar m}\\
            (\h_s^{[2.5]})_{(D-3):D}
        \end{pmatrix} + \sum_{s' \in [N]} \phi_1(1; (\tilde \p_s)_2, (\tilde \p_{s'})_3) \begin{pmatrix}
            \zero_{r(\bar m+1)}\\
            (\h_{s'}^{[2.5]})_{(r(\bar m+1)+1):(D-4)}\\
            \zero_4
        \end{pmatrix}.
    \end{align*}
    Note that $(\tilde \p_s)_2 = (\p_s)_2 \in \{0, 1\}$, $(\tilde \p_s)_3 = (\tilde \p_s)_2$ for $s \leq 2n$, and $(\tilde \p_s)_3 \geq 2$ for $s \geq 2n+1$. Thus,
    \begin{align*}
        \TF_{\psi_4^*}(H_n^{(\ell)[2.5]})_s
        &= \begin{pmatrix}
            (\h_s^{[2.5]})_{1:(r(\bar m + 1))}\\
            \zero_{r \bar m}\\
            (\h_s^{[2.5]})_{(D-3):D}
        \end{pmatrix} + \sum_{s' \in [2n]} \phi_1(1; (\p_s)_2, (\p_{s'})_3) \begin{pmatrix}
            \zero_{r(\bar m+1)}\\
            (\h_{s'}^{[2.5]})_{(r(\bar m+1)+1):(D-4)}\\
            \zero_4
        \end{pmatrix}.
    \end{align*}
    This gives
    \begin{align*}
        \TF_{\psi_4^*}(H_n^{(\ell)[2.5]})_{2i-1} &= \begin{pmatrix}
            (\h_{2i-1}^{[2.5]})_{1:(r(\bar m + 1))}\\
            \sum_{i' \in [n]} \u_{Y_{i'}}^\top f^{(1)}(\u_{X_{i'}})\\
            \vdots\\
            \sum_{i' \in [n]} \u_{Y_{i'}}^\top f^{(\bar m)}(\u_{X_{i'}})\\
            (\h_{2i-1}^{[2.5]})_{(D-3):D}
        \end{pmatrix}, \ \ 
        \TF_{\psi_4^*}(H_n^{(\ell)[2.5]})_{2i} &= \begin{pmatrix}
            (\h_{2i}^{[2.5]})_{1:(r(\bar m + 1))}\\
            \sum_{i' \in [n]} \u_{X_{i'}}^\top \z^{(1)}\\
            \vdots\\\
            \sum_{i' \in [n]} \u_{X_{i'}}^\top \z^{(\bar m)}\\
            (\h_{2i}^{[2.5]})_{(D-3):D}
        \end{pmatrix}.
    \end{align*}
    $\TF_{\Psi_3^*}$ with $\Psi_3^* = (\psi_3^*, \psi_4^*)$ is the desired transformer.
    Let $H_n^{(\ell)[3]} = \TF_{\Psi_3^*}(H_n^{(\ell)[2]})$.

\paragraph{Step 4.}
In Step 4, we aim to construct transformer layers satisfying
    \begin{align*}
        \TF_{\Psi_4^*}(H_n^{(\ell)[3]})_{2i-1} &= \begin{pmatrix}
            \hat f(\u_{X_i})\\
            \zero_{(r+1)\bar m}\\
            \p_{2i-1}
        \end{pmatrix}, \ \ 
        \TF_{\Psi_4^*}(H_n^{(\ell)[3]})_{2i} = \begin{pmatrix}
            \hat \z\\
            \zero_{(r+1)\bar m}\\
            \p_{2i}
        \end{pmatrix}.
    \end{align*}
    where $\hat f$ and $\hat \z$ satisfy
    \begin{small}
    \begin{align*}
        \hat f(\u_{X_i}) &\in \conv\qty{ f^{(m')}(\u_{X_i}): \frac{1}{n} \sum_{i \in [n]} \langle \u_{Y_i}, f^{(m')}(\u_{X_i}) \rangle \geq \max_{m'' \in [\bar m]} \frac{1}{n} \sum_{i \in [n]} \langle \u_{Y_i}, f^{(m'')}(\u_{X_i}) \rangle - \frac{\omega}{n} },\\
        \hat \z &\in \conv\qty{ \z^{(t')}: \frac{1}{n} \sum_{i \in [n]} \langle \u_{X_i}, \z^{(t')} \rangle \geq \max_{t'' \in [\bar t]} \frac{1}{n} \sum_{i \in [n]} \langle \u_{X_i}, \z^{(t'')} \rangle - \frac{\omega}{n} },
    \end{align*}
    \end{small}\noindent
    where $\conv(S)$ denotes the convex hull of a set $S$.
    Toward this end, we directly apply Lemma~\ref{lem: minimum by transformer}. Let $\Psi_4^* = \Psi^\mn$.

    The desired transformer is $\TF_{\Psi_4^*} \circ \TF_{\Psi_3^*} \circ \TF_{\Psi_2^*} \circ \TF_{\Psi_1^*}$. This completes the proof of Proposition~\ref{prop: transformer as good data generator}. 
\end{proof}

Now we provide the following key theorem, which separately states the generative and discriminative capacities of the constructed transformers. It is a detailed version of Theorem~\ref{thm: generative and discriminative} in the paper. 

\begin{theorem}[Detailed version of Theorem~\ref{thm: generative and discriminative}]
\label{thm: transformer as good data generator 2}
Suppose Assumptions~\ref{asm: known subjects} and \ref{asm: separability} hold. Fix $dr \in \N$, $(\z^{(t)})_{t \in \mT}$ and $(f^{(m)})_{m \in \mM} \subset \mF(L_0, r_0)$. Take $n$ and $d$ sufficiently large, such that
\begin{align*}
C_0 \qty( \eta \sqrt{r} \frac{\log^2 d}{\sqrt{n}} + \eta r \frac{\log d}{\sqrt{d}} ) &< \delta_\mM \wedge \delta_\mT
\end{align*}
holds for some constant $C_0 > 0$. Choose $\omega = \log^2 d/\sqrt{r}$. Let $\Psi^* = \Psi^*((\z^{(t)})_{t \in \mT}, (f^{(m)})_{m \in \mM},$  $d,r,\omega, L_0, r_0)$ be the parameter of transformer layers as in Proposition~\ref{prop: transformer as good data generator}. Then, for any $m \in \mM$, $t \in \mT$, $\eta \geq (1/\sqrt{r}) \log d$, and any $P_{X_1,Y_1|M=m,T=t,U,\eta}$ described in Section~\ref{sec: dgp}, with the choice $\tau = \eta$,
\begin{small}
\begin{align}
        \E[\KL(P_{X_1;T=t,U,\eta} \| Q_{\tilde X_s;\Psi^*,\tau,\mD_n}) | T=t,\eta] &= \exp(-\Omega\qty(\frac{\delta_\mT^2 n}{\eta^2 r})),\label{eq: KL 1}\\
        \E\qty[\KL(P_{Y_1|X_1;M=m,U,\eta} \| Q_{\tilde Y_s|X_1;\Psi^*,\tau,\mD_n})|M=m,T=t,\eta] &= \exp(-\Omega\qty(\frac{\delta_\mM^2 n}{\eta^2 r}))\label{eq: KL 2}
\end{align}
\end{small}
hold for all $s \in \N$. Henceforth,
\begin{small}
\begin{align}
        \E[\KL(P_{X_1,Y_1;T=t,M=m,U,\eta} \| Q_{\tilde X_s,\tilde Y_s;\Psi^*,\tau,\mD_n}) | M=m,T=t,\eta] = \exp(-\Omega\qty(\frac{(\delta_\mT^2 \wedge \delta_\mM^2) n}{\eta^2 r}))\label{eq: KL 3}
\end{align}
\end{small}
holds for all $s \in \N$.
\end{theorem}

\begin{proof}[\textbf{Proof of Theorem~\ref{thm: transformer as good data generator 2}}]
We divide the proof into two steps, where Step 1 provides the bound for \eqref{eq: KL 1}, Step 2 provides the bound for \eqref{eq: KL 2}, and Step 3 combines the previous steps to give the bound for \eqref{eq: KL 3}. Fix any distributions $P_{X_1;T=t,U,\eta}$ and $P_{Y_1|X_1;M=m,U,\eta}$ introduced in Section~\ref{sec: dgp}. Let $\bar m = |\mM|$, $\bar t = |\mT|$ and write $\mM = [\bar m]$, $\mT = [\bar t]$. Recall that $\bar m \geq \bar t$.

\paragraph{Step 1.}
In this step, we prove \eqref{eq: KL 1}. For any $T = t$ and $U$, from Proposition~\ref{prop: transformer as good data generator}, there exists some $\mT' \subset \mT = [\bar t]$ with $\mT' \ni t$, such that
    \begin{align}
        \hat \z = \sum_{t' \in \mT'} \alpha^{(t')} \z^{(t')},\label{eq: hat z z}
    \end{align}
    where $\alpha^{(t')} \geq 0$ for any $t' \in \mT'$, $\sum_{t' \in \mT'} \alpha^{(t')} = 1$, and
    \begin{align}
        \frac{1}{n} \sum_{i \in [n]} \langle \u_{X_i}, \z^{(t')} \rangle \geq \max_{t'' \in [\bar t]} \frac{1}{n} \sum_{i \in [n]} \langle \u_{X_i}, \z^{(t'')} \rangle - \frac{\omega}{n}\label{eq: u z}
    \end{align}
    holds for all $t' \in \mT'$.
    From Lemma~\ref{lem: subject identification} with $a> 0$ chosen later, there exists some constant $C > 0$, such that
    \begin{align*}
        &\max_{t \in [\bar t]} \P\qty(\max_{t' \in \mT} \abs{\z^{(t') \top} \qty(\frac{1}{n} \sum_{i \in [n]} \u_{X_i}) - \frac{1}{\eta r} \z^{(t') \top} \z^{(t)}} \leq C \epsilon \middle| T=t,U,\eta)\\
        &= 1 - \exp(-\Omega\qty(\frac{ra^2 n}{\log^2 d} + \log^2 d)) =: 1 - R_{n,d,r},
    \end{align*}
    where $\epsilon := a + \log^2 d/\sqrt{n r} + \log d/\sqrt{d r}$. Combined with \eqref{eq: u z}, we have
    \begin{align*}
        \frac{1}{\eta r} \langle \z^{(t)}, \z^{(t')} \rangle + C \epsilon &\geq \max_{t'' \in [\bar t]} \frac{1}{\eta r} \langle \z^{(t)}, \z^{(t'')} \rangle - C \epsilon - \frac{\omega}{n}\\
        &= \frac{1}{\eta r} - C \epsilon - \frac{\omega}{n}
    \end{align*}
    with high probability for all $t' \in \mT'$. Thus we have
    \begin{align}
        \min_{t' \in \mT'} \z^{(t') \top} \z^{(t)} \geq 1 - 2 C \eta r \epsilon - \frac{\eta r \omega}{n}\label{eq: hat z z inner product}
    \end{align}
    with high probability. Set $a = \delta_\mT/(4C\eta r)$. Then, $R_{n,d,r} = \exp(-\Omega(\delta_\mT^2 n / (\eta^2 r \log^2 d) + \log^2 d))$.
    By taking $C_0$ in Theorem~\ref{thm: transformer as good data generator 2} sufficiently large, $n$ and $d$ satisfy
    \begin{align}
        (2C + 1) \frac{\eta \sqrt{r} \log^2 d}{\sqrt{n}} + 2C \frac{\eta \sqrt{r} \log d}{\sqrt{d}} < \frac{\delta_\mT}{2}.
    \end{align}
    Henceforth,
    \begin{align}
        1 - 2 C \eta r \epsilon - \frac{\eta r \omega}{n} &= 1 - \frac{\delta_\mT}{2} - 2C \frac{\eta \sqrt{r} \log^2 d}{\sqrt{n}} - 2C \frac{\eta \sqrt{r} \log d}{\sqrt{d}} - \frac{\eta \sqrt{r} \log^2 d}{n}\\
        &= 1 - \frac{\delta_\mT}{2} - (2C + 1) \frac{\eta \sqrt{r} \log^2 d}{\sqrt{n}} - 2C \frac{\eta \sqrt{r} \log d}{\sqrt{d}} > 1 - \delta_\mT,\label{eq: recovery regime}
    \end{align}
    Assumption~\ref{asm: separability} and \eqref{eq: hat z z inner product} imply that $\mT'$ is a singleton set, i.e., $\mT' = \{t\}$ and $\hat \z = \z^{(t)}$ with probability $1 - R_{n,d,r}$.
    We next bound $\E[\KL(P_{X_1;T=t,U,\eta} \| Q_{\tilde X_s;\Psi^*,\tau,\mD_n})]$. We specifically choose $\tau \gets \eta$. Then,
    \begin{small}
    \begin{align*}
        &\KL(P_{X_1;T=t,U,\eta} \| Q_{\tilde X_s;\Psi^*,\tau,\mD_n})\\
        &\quad= \sum_{x \in \mX} P_{X_1=x|T=t,U,\eta} \qty(\log \frac{\exp(\eta^{-1} \z^{(t) \top} \u_x)}{\sum_{x' \in \mX} \exp(\eta^{-1} \z^{(t) \top} \u_{x'})} - \log \frac{\exp(\eta^{-1} \hat \z^\top \u_x)}{\sum_{x' \in \mX} \exp(\eta^{-1} \hat \z^\top \u_{x'})}).
    \end{align*}
    \end{small}\noindent
    Let $T_x(\z) := \log \frac{\exp(\eta^{-1} \z^\top \u_x)}{\sum_{x' \in \mX} \exp(\eta^{-1} \z^\top \u_{x'})}$.
    Observe that
    \begin{align}
        \|\nabla T_x(\z)\| &= \norm{\eta^{-1} \u_x - \eta^{-1} \sum_{x' \in \mX} \u_{x'} \frac{\exp(\eta^{-1} \z^\top \u_{x'})}{\sum_{x'' \in \mX} \exp(\eta^{-1} \z^\top \u_{x''})}} \leq 2\eta^{-1} \max_{x \in \mX} \|\u_x\|.
    \end{align}
    Thus $T_x(\z)$ is $2\eta^{-1} \max_{x \in \mX} \|\u_x\|$-Lipschitz in $\z$. By the mean value theorem,
    \begin{align}
        |T_x(\z^{(t)}) - T_x(\hat\z)| \leq 2\eta^{-1} \max_{x \in \mX} \|\u_x\| \|\z^{(t)} - \hat \z\|.
    \end{align}
    Since this holds for all $x \in \mX$, we obtain
    \begin{align}
        \KL(P_{X_1;T=t,U,\eta} \| Q_{\tilde X_s;\Psi^*,\tau,\mD_n}) \leq 2\eta^{-1} \max_{x \in \mX} \|\u_x\| \|\z^{(t)} - \hat \z\|.
    \end{align}
    By Cauchy-Schwarz inequality, we have
    \begin{align}
        \E[\KL(P_{X_1;T=t,U,\eta} \| Q_{\tilde X_s;\Psi^*,\tau,\mD_n}) | T=t, \eta] \lesssim \eta^{-1} \qty(\E[\max_{x \in \mX} \|\u_x\|^2] \E[\|\z^{(t)} - \hat \z\|^2 | T=t, \eta])^{1/2}.\label{eq: term 1 and 2}
    \end{align}

Now we derive $\E[\max_{x \in \mX} \|\u_x\|^2]$. Since $r \|\u_x\|^2 \sim \chi_2(r)$, we employ the concentration inequality for chi-squared distribution (see Lemma 1 in \citet{laurent2000adaptive}):
    \begin{align*}
        \P(r \|\u_x\|^2 \geq r + 2\sqrt{r\epsilon} + 2\epsilon) \leq \exp(-\epsilon).
    \end{align*}
    By a union bound argument, and by $2\sqrt{r \epsilon} \leq r + \epsilon$, we have
    \begin{align*}
        \P\qty(\frac{r \max_{x \in \mX} \|\u_x\|^2 - 2r}{3} \geq \epsilon) &\leq \P(r \max_{x \in \mX} \|\u_x\|^2 \geq r + 2\sqrt{r \epsilon} + 2 \epsilon) \leq 1 \wedge d \exp(-\epsilon).
    \end{align*}
    Using $\E[|X|] = \int_0^\infty \P(|X| > \epsilon) \dd{\epsilon}$, we obtain 
    \begin{align*}
        \E\qty[\frac{r \max_{x \in \mX} \|\u_x\|^2 - 2r}{3}] \leq \int_0^{\log d} 1 \dd{\epsilon} + \int_{\log d}^\infty d \exp(-\epsilon) \dd{\epsilon} = O(\log d).
    \end{align*}
    Henceforth, 
    \begin{align}
        \E[\max_{x \in \mX} \|\u_x\|^2] \lesssim 1 + \frac{\log d}{r} \lesssim \frac{\log d}{r}.\label{eq: term 1}
    \end{align}
    For the term $\E[\|\z^{(t)} - \hat \z\|^2 | T=t, \eta]$, using $\hat \z = z^{(t)}$ with probability $1 - R_{n,d,r}$ and $\|\z^{(t)}\| \vee \|\hat \z\| \leq 1$, we have 
    \begin{align}
        \E[\|\z^{(t)} - \hat \z\|^2 | T=t, \eta] \lesssim \exp(-\Omega\qty(\frac{\delta_\mT^2}{\eta^2 r \log^2 d} n + \log^2 d)).\label{eq: term 2}
    \end{align}
    Combining \eqref{eq: term 1} and \eqref{eq: term 2} into \eqref{eq: term 1 and 2}, we obtain
    \begin{align*}
        \E[\KL(P_{X_1;T=t,U,\eta} \| Q_{\tilde X_s;\Psi^*,\tau,\mD_n}) | T=t, \eta] &\lesssim \frac{\sqrt{\log d}}{\eta \sqrt{r}} \exp(-\Omega\qty(\frac{\delta_\mT^2}{\eta^2 r \log^2 d} n + \log^2 d))\\
        &\leq \exp(-\Omega\qty(\frac{\delta_\mT^2}{\eta^2 r \log^2 d} n + \log^2 d))\\
        &\leq \exp(-\Omega\qty(\frac{\delta_\mT^2}{\eta^2 r} n)),
    \end{align*}
where we use $\eta \geq r^{-1/2} \log d$ in the second inequality, and $|a| + |b| \geq 2\sqrt{|ab|}$ in the third inequality. This yields \eqref{eq: KL 1}.

\paragraph{Step 2.}
In this step, we prove \eqref{eq: KL 2}. For any fixed $M = m$ and $U$, from Proposition~\ref{prop: transformer as good data generator}, there exists some $\mM' \subset \mM = [\bar m]$ with $\mM' \ni m$, such that
    \begin{align}
        \hat f(\u_x) = \sum_{m' \in \mM'} \beta^{(m')} f^{(m')}(\u_x),\label{eq: hat f f}
    \end{align}
    where $\beta^{(m')} \geq 0$ for any $m' \in \mM'$, $\sum_{m' \in \mM'} \beta^{(m')} = 1$, and
    \begin{align*}
        \frac{1}{n} \sum_{i \in [n]} \langle \u_{Y_i}, f^{(m')}(\u_x) \rangle &\geq \max_{m'' \in [\bar m]} \frac{1}{n} \sum_{i \in [n]} \langle \u_{Y_i}, f^{(m'')}(\u_x) \rangle - \frac{\omega}{n}\\
        &\geq \frac{1}{n} \sum_{i \in [n]} \langle \u_{Y_i}, f^{(m)}(\u_x) \rangle - \frac{\omega}{n}
    \end{align*}
    holds for all $m' \in \mM'$.
    From Lemma~\ref{lem: good event U} and \eqref{eq: D tilde mean}, 
    Note that $\|\u_x\| \lesssim (1/\sqrt{r}) \log d$ holds with high probability from Lemma~\ref{lem: good event U}.
    Thus $\max_{m \in [\bar m], x \in \mX} \|f^{(m)}(\u_x)\| \leq 1$ holds with high probability by assumption.
    \begin{align}
        \sup_{\z, \z' \in \mathbb{B}_r(1)} \abs{\sum_{x \in \mX} \z'^\top \u_x \exp(\eta^{-1} \z^\top \u_x) - \frac{1}{\eta r} \z'^\top \z \exp(\frac{\|\z\|^2}{2\eta^2 r})} &\lesssim \sqrt{\frac{d}{r}} \log d,\label{eq: good event 1}\\
        \sup_{\z \in \mathbb{B}_r(1), x \in \mX} \abs{\z^\top \u_x} \lesssim \frac{\log d}{\sqrt{r}}, \ \ \sup_{\z \in \mathbb{B}_r(1)} \abs{\frac{d \exp(\|\z\|^2/(2\eta^2 r) )}{\sum_{x' \in \mX} \exp(\eta^{-1} \z^\top \u_{x'})} - 1} &\lesssim \frac{1}{\sqrt{d}} \log d\label{eq: good event 0}
    \end{align}
    hold with high probability with respect to $U$. Hereafter we focus on the event for $U$, where \eqref{eq: good event 1} and \eqref{eq: good event 0} hold, and $\max_{m \in [\bar m], x \in \mX} \|f^{(m)}(\u_x)\| \leq 1$ holds. 
    From Lemma~\ref{lem: function identification} with $a > 0$ chosen later, we have
    \begin{align*}
        &\max_{m \in \mM, t \in \mT} \P\biggl(\max_{m' \in \mM} \abs{\frac{1}{n} \sum_{i \in [n]} \langle f^{(m')}(\u_{X_i}), \u_{Y_i} \rangle - \E\qty[\langle f^{(m')}(\u_{X_1}), \u_{Y_1} \rangle | M=m, T=t, U, \eta]}\\
        &\quad\quad\quad\quad> a + \frac{\log^2 d}{\sqrt{n r}} \bigg| M=m, T=t, U, \eta\biggr)\\
        &\quad= \exp(-\Omega\qty(\frac{ra^2 n}{\log^2 d} + \log^2 d)).
    \end{align*}
    Thus, we have
    \begin{align}
        &\E\qty[\langle f^{(m')}(\u_{X_1}), \u_{Y_1} \rangle | M=m, T=t, U, \eta]\nonumber\\
        &\quad\geq \max_{m'' \in [\bar m]} \E\qty[\langle f^{(m'')}(\u_{X_1}), \u_{Y_1} \rangle | M=m, T=t, U, \eta] - \frac{\omega}{n} - \frac{2 \log^2 d}{\sqrt{nr}} - 2a\nonumber\\
        &\quad\geq \E\qty[\langle f^{(m)}(\u_{X_1}), \u_{Y_1} \rangle | M=m, T=t, U, \eta] - \frac{\omega}{n} - \frac{2 \log^2 d}{\sqrt{nr}} - 2a\label{eq: inner prod f raw}
    \end{align}
    with probability $1 - \exp(-\Omega(ra^2n/\log^2 d + \log^2 d))$.
    We next compute $\E[\langle f^{(m')}(\u_{X_1}), \u_{Y_1} \rangle | M=m, T=t, U, \eta]$ for any $m' \in \mM$. Note that
    \begin{align*}
        &\E[\langle f^{(m')}(\u_{X_1}), \u_{Y_1} \rangle | X_1=x, T=t, M=m, U, \eta]\\
        &\quad= \sum_{y \in \mX} \langle f^{(m')}(\u_x), \u_y \rangle \frac{\exp(\eta^{-1} \langle \u_y, f^{(m)}(\u_x) \rangle)}{\sum_{y' \in \mX} \exp(\eta^{-1} \langle \u_{y'}, f^{(m)}(\u_x) \rangle)}\\
        &\quad= \frac{\sum_{y \in \mX} \langle f^{(m')}(\u_x), \u_y \rangle \exp(\eta^{-1} \langle \u_y, f^{(m)}(\u_x) \rangle)}{\sum_{y' \in \mX} \exp(\eta^{-1} \langle \u_{y'}, f^{(m)}(\u_x) \rangle)}.
    \end{align*}
    We bound the denominator and numerator separately. For the denominator, \eqref{eq: good event 0} gives
    \begin{align*}
        &\sum_{y' \in \mX} \exp(\eta^{-1} \langle \u_{y'}, f^{(m)}(\u_x) \rangle)\\
        &\quad= \exp(\eta^{-1} \langle \u_x, f^{(m)}(\u_x) \rangle) + \sum_{y'\neq x} \exp(\eta^{-1} \langle \u_{y'}, f^{(m)}(\u_x) \rangle)\\
        &\quad= \exp(O(\eta^{-1}\|\u_x\|)) + (d-1) \exp(\frac{\|f^{(m)}(\u_x)\|^2}{2\eta^2 r}) \qty(1 + O\qty(\frac{1}{\sqrt{d}} \log d)).
    \end{align*}
    with high probability. For the numerator, \eqref{eq: good event 1} gives
    \begin{align*}
        &\sum_{y \in \mX} \langle f^{(m')}(\u_x), \u_y \rangle \exp(\eta^{-1} \langle \u_y, f^{(m)}(\u_x) \rangle)\\
        &\quad= \langle f^{(m')}(\u_x), \u_x \rangle \exp(\eta^{-1} \langle \u_x, f^{(m)}(\u_x) \rangle) + \sum_{y\neq x} \langle f^{(m')}(\u_x), \u_y \rangle \exp(\eta^{-1} \langle \u_y, f^{(m)}(\u_x) \rangle)\\
        &\quad= O(\|\u_x\|) \exp(O(\eta^{-1}\|\u_x\|)) + (d-1) \frac{1}{\eta r} \langle f^{(m')}(\u_x), f^{(m)}(\u_x) \rangle \exp(\frac{\|f^{(m)}(\u_x)\|^2}{2\eta^2 r})\\
        &\quad\quad + O\qty(\sqrt{\frac{d}{r}} \log d)
    \end{align*}
    with high probability. Since $\|\u_x\| \lesssim (1/\sqrt{r}) \log d$ holds with high probability from Lemma~\ref{lem: good event U}, $r = o(\log d)$, and $\eta \geq (1/\sqrt{r}) \log d$, the following holds with high probability:
    \begin{align*}
        &\E[\langle f^{(m')}(\u_{X_1}), \u_{Y_1} \rangle | X_1=x, T=t, M=m, U, \eta]\\
        &\quad= \frac{(d-1) \frac{1}{\eta r} \langle f^{(m')}(\u_x), f^{(m)}(\u_x) \rangle \exp(\frac{\|f^{(m)}(\u_x)\|^2}{2\eta^2 r}) + O\qty(\sqrt{\frac{d}{r}} \log d)}{(d-1) \exp(\frac{\|f^{(m)}(\u_x)\|^2}{2\eta^2 r}) + O(\sqrt{d} \log d)}\\
        &\quad= \frac{1}{\eta r} \langle f^{(m')}(\u_x), f^{(m)}(\u_x) \rangle + O\qty(\frac{\log d}{\sqrt{d r}}) + O\qty(\frac{\log d}{\sqrt{d}})\\
        &\quad= \frac{1}{\eta r} \langle f^{(m')}(\u_x), f^{(m)}(\u_x) \rangle + O\qty(\frac{\log d}{\sqrt{d}}).
    \end{align*}
    Note that the constant hidden in the big-$O$ notation is independent of $x$. Hence
    \begin{align*}
        &\E[\langle f^{(m')}(\u_{X_1}), \u_{Y_1} \rangle | T=t, M=m, U, \eta]\\
        &\quad= \E\qty[\frac{1}{\eta r} \langle f^{(m')}(\u_x), f^{(m)}(\u_x) \rangle | T=t, U, \eta] + O\qty(\frac{\log d}{\sqrt{d}}).
    \end{align*}
    Combined with \eqref{eq: inner prod f raw}, we have
    \begin{align}
        \frac{1}{\eta r} \E\qty[\|f^{(m)}(\u_{X_1})\|^2 - \langle f^{(m')}(\u_{X_1}), f^{(m)}(\u_{X_1}) \rangle | T=t, U, \eta] &\leq \frac{\omega}{n} + \frac{2 \log^2 d}{\sqrt{nr}} + 2a + C' \frac{\log d}{\sqrt{d}}\label{eq: inner prod bound on f}
    \end{align}
    with probability $1 - \exp(-\Omega(ra^2n/\log^2 d + \log^2 d))$ for all $m' \in \mM'$, where $C' > 0$ is some constant.
    Set $a = \delta_\mM/(4\eta r)$.
    By taking $C_0$ in Theorem~\ref{thm: transformer as good data generator 2} sufficiently large, $n$ and $d$ satisfy
    \begin{align*}
        3 \eta r \frac{\log^2 d}{\sqrt{nr}} + C' \eta r \frac{\log d}{\sqrt{d}} &< \frac{\delta_\mT}{2}.
    \end{align*}
Henceforth, 
    \begin{align}
        \frac{\omega}{n} + \frac{2 \log^2 d}{\sqrt{nr}} + 2a + C' \frac{\log d}{\sqrt{d}} &= \frac{\log^2 d}{n \sqrt{r}} + \frac{2 \log^2 d}{\sqrt{nr}} + \frac{\delta_\mM}{2\eta r} + C' \frac{\log d}{\sqrt{d}}\nonumber\\
        &\leq \frac{3 \log^2 d}{\sqrt{nr}} + \frac{\delta_\mM}{2\eta r} + C' \frac{\log d}{\sqrt{d}} 
        < \frac{1}{\eta r} \delta_\mM.\label{eq: recovery regime 2}
    \end{align}
    Assumption~\ref{asm: separability} combined with \eqref{eq: inner prod bound on f} implies that $\mM' = \{m\}$ and thus $\hat f = f^{(m)}$ with probability $1 - \exp(-\Omega(ra^2n/\log^2 d + \log^2 d))$.\\

We next bound $\E[\KL(P_{Y_1|X_1;M=m,U,\eta} \| Q_{\tilde Y_s|X_1;\Psi^*,\tau,\mD_n})| M=m,T=t,U,\eta]$ with the choice $\tau \gets \eta$. 
    \begin{align*}
        &\E[\KL(P_{Y_1|X_1;M=m,U,\eta} \| Q_{\tilde X_s;\Psi^*,\tau,\mD_n}) | M=m, T=t, U, \eta]\\
        &\quad= \sum_{x \in \mX} p_x^{(t)} \sum_{y \in \mX} P_{Y_1|X_1=x, M=m,U,\eta}(y) \log \frac{P_{Y_1|X_1=x,M=m,U,\eta}(y)}{Q_{\tilde X_s;\Psi^*,\tau,\mD_n}(y)},
    \end{align*}
    where $p_x^{(t)} := \P(X_1 = x | T=t, U, \eta)$. By a similar argument as in Step 1,
    \begin{align*}
        &\E[\KL(P_{Y_1|X_1;M=m,U,\eta} \| Q_{\tilde Y_s|X_1;\Psi^*,\mD_n}) | M=m, T=t, U, \eta]\\
        &\quad= \sum_{x \in \mX} p_x^{(t)} \sum_{y \in \mX} P_{Y_1|X_1=x, M=m,U,\eta}(y)\\
        &\quad\quad\times \qty(\log\frac{\exp(\eta^{-1} \langle f^{(m)}(\u_x), \u_y \rangle)}{\sum_{y' \in \mX} \exp(\eta^{-1} \langle f^{(m)}(\u_x), \u_y \rangle)} - \log\frac{\exp(\eta^{-1} \langle \hat f(\u_x), \u_y \rangle)}{\sum_{y' \in \mX} \exp(\eta^{-1} \langle \hat f(\u_x), \u_y \rangle)})\\
        &\quad\leq 2 \eta^{-1} \sum_{x \in \mX} p_x^{(t)} \max_{y \in \mX} \|\u_x\| \|f^{(m)}(\u_x) - \hat f(\u_x)\|\\
        &\quad= 2 \eta^{-1} \max_{y \in \mX} \|\u_x\| \E[\|f^{(m)}(\u_{X_1}) - \hat f(\u_{X_1})\| | T=t, U, \eta].
    \end{align*}
    Hence, taking expectation with respect to $U$ combined with Cauchy-Schwarz inequality, we have
    \begin{align*}
        &\E[\KL(P_{Y_1|X_1;M=m,U,\eta} \| Q_{\tilde Y_s|X_1;\Psi^*,\mD_n}) | M=m, T=t, \eta]\\
        &\quad\lesssim \eta^{-1} \qty(\E[\max_{y \in \mX} \|\u_x\|^2] \E[\|f^{(m)}(\u_{X_1}) - \hat f(\u_{X_1})\|^2 | M=m, T=t, \eta])^{1/2}.
    \end{align*}
    Using \eqref{eq: term 1} and the fact that $\hat f = f^{(m)}$ with probability $1 - \exp(-\Omega(\delta_\mM^2n/(\eta^2 r) + \log^2 d))$ combined with $\|f^{(m)}\| \vee \|\hat f\| \leq 1$, we obtain
    \begin{align*}
        &\E[\KL(P_{Y_1|X_1;M=m,U,\eta} \| Q_{\tilde Y_s|X_1;\Psi^*,\mD_n}) | M=m, T=t, \eta]\\
        &\quad\lesssim \eta^{-1} \qty(\frac{\log d}{r} \exp(-\Omega\qty(\frac{\delta_\mM^2 n}{\eta^2 r \log^2 d} + \log^2 d)))^{1/2}\\
        &\quad= \exp(-\Omega\qty(\frac{\delta_\mM^2 n}{\eta^2 r \log^2 d} + \log^2 d)) = \exp(-\Omega\qty(\frac{\delta_\mM^2 n}{\eta^2 r})),
    \end{align*}
    where we used $\eta \geq r^{-1/2} \log d$ and $|a| + |b| \geq 2\sqrt{|ab|}$. This gives \eqref{eq: KL 2}.

\paragraph{Step 3.}
In this step, we prove \eqref{eq: KL 3}. Taking expectation on both sides of 
    \begin{align*}
        &\E[\KL(P_{X_1,Y_1;T=t,M=m,U,\eta} \| Q_{\tilde X_s,\tilde Y_s;\Psi^*,\tau,\mD_n}) | M=m,T=t, \eta]\\
        &\quad= \E[\KL(P_{X_1;T=t,U,\eta} \| Q_{\tilde X_s;\Psi^*,\tau,\mD_n}) | T=t, \eta]\\
        &\quad\quad+ \E[\KL(P_{Y_1|X_1;M=m,U,\eta} \| Q_{\tilde Y_s|X_1;\Psi^*,\mD_n}) | M=m, T=t, \eta],
    \end{align*}
and combining the results from Steps 1 and 2 gives \eqref{eq: KL 3}. 

This completes the proof of Theorem~\ref{thm: transformer as good data generator 2}
\end{proof}

We implement the selection layer that performs argmin operation. The next lemma is an application of Proposition M.2 of \citet{bai2023transformers}.

\begin{lemma}[Minimum by Transformer]\label{lem: minimum by transformer}
Fix $\omega > 0$, $r \in \N$ and $\bar m \geq 2$. Then, there exist transformer layers $\TF_{\Psi^\mn}$ with $\Psi^\mn = \Psi^\mn(\omega, \bar m, r)$ satisfying that
\begin{enumerate}[(i)]
        \item it consists of $5$ transformer layers with the width of FNN $O(\bar m^2)$, and the number of heads of attention layers $O(\bar m)$;
        \item for any $N \in \N$ and $H = [\h_1, \dots, \h_N] \in \R^{r(\bar m + 1) + \bar m + 4}$ of the form
        \begin{align*}
            \h_s = \begin{pmatrix}
                \x_0\\
                \x_1\\
                \vdots\\
                \x_{\bar m}\\
                v_1\\
                \vdots\\
                v_{\bar m}\\
                \p_s
            \end{pmatrix},
        \end{align*}
        where $\x_m \in \R^r$ for $m \in \{0\} \cup [\bar m]$, $v_m \in [-1, 1]$ for $m \in [\bar m]$, and $\p_s$ is defined in Section~\ref{sec: dgp}, $\TF_{\Psi^\mn}$ outputs
        \begin{align*}
             \TF_{\Psi^\mn}(H)_s = (\x^\top, \zero^\top)^\top,
        \end{align*}
        where $\x \in \conv\{\x_m : m \in [\bar m], v_m \leq \min_{m' \in [\bar m]} v_{m'} + \omega\}$, with $\conv(S)$ being the convex hull of $S$.
\end{enumerate}
\end{lemma}

\begin{proof}[\textbf{Proof of Lemma~\ref{lem: minimum by transformer}}]
Write $D = r + r\bar m + \bar m + 4$. We divide the proof into 4 steps.

\paragraph{Step 1.}
Let $\FFN_{\nu^\mn_{1}}$ be a feedforward neural network with $\nu^\mn_{1} = (W^\mn_{1,1}, W^\mn_{1,2})$, such that
    \begin{align*}
        W^\mn_{1,2} \sigma(W^\mn_{1,1} \h)
        &= W^\mn_{1,2} \begin{pmatrix}
            \sigma((\h)_{r(1+\bar m)+1} - (\h)_{r(1+\bar m)+2})\\
            \sigma((\h)_{r(1+\bar m)+1} - (\h)_{r(1+\bar m)+3})\\
            \vdots\\
            \sigma((\h)_{r(1+\bar m)+1} - (\h)_{r(1+\bar m)+\bar m})\\
            \vdots\\
            \sigma((\h)_{r(1+\bar m)+\bar m} - (\h)_{r(1+\bar m)+1})\\
            \vdots\\
            \sigma((\h)_{r(1+\bar m)+\bar m} - (\h)_{r(1+\bar m)+\bar m-1})\\
            \sigma((\h)_{(r(1+\bar m)+1):(r(1 +\bar m)+\bar m)})\\
            \sigma(-(\h)_{(r(1+\bar m)+1):(r(1 +\bar m)+\bar m)})
        \end{pmatrix}\\
        &= \begin{pmatrix}
            \zero_{r(1+\bar m)}\\
            - (\h)_{r(1+\bar m)+1} + \sum_{m' : m'\neq 1} \sigma((\h)_{r(1+\bar m)+1} - (\h)_{r(1+\bar m)+m'})\\
            \vdots\\
            - (\h)_{r(1+\bar m)+\bar m} + \sum_{m' : m'\neq \bar m} \sigma((\h)_{r(1+\bar m)+\bar m} - (\h)_{r(1+\bar m)+m'})\\
            \zero_4
        \end{pmatrix}.
    \end{align*}
    Then,
    \begin{align*}
        \FFN_{\nu^\mn_{1}}(\h_s) &= \begin{pmatrix}
            \x_0\\
            \x_1\\
            \vdots\\
            \x_{\bar m}\\
            v^{(1)}_1\\
            \vdots\\
            v^{(1)}_{\bar m}\\
            \p_s
        \end{pmatrix},
    \end{align*}
    where $v^{(1)}_m = \sum_{m': m' \neq m} \sigma(v_m - v_{m'})$. Note that $v^{(1)}_m \leq \omega$ implies $v_m \leq \min_{m': m' \neq m} v_{m'} + \omega$.
    Choose $\psi^\mn_{1} := (\mu_\id, \nu_1^\mn)$.

\paragraph{Step 2.} 
Define $\FFN_{\nu^\mn_{2}}$ as a feedforward neural network with $\nu^\mn_{2} = (W^\mn_{2,1}, W^\mn_{2,2})$, such that
\begin{small}
\begin{align*}
        W^\mn_{2,2} \sigma(W^\mn_{2,1} \h)
        &= W^\mn_{2,2} 
        \begin{pmatrix}
            \sigma((\h)_{r(1+\bar m)+1})\\
            \vdots\\
            \sigma((\h)_{r(1+\bar m)+\bar m})\\
            \sigma(-(\h)_{r(1+\bar m)+1})\\
            \vdots\\
            \sigma(-(\h)_{r(1+\bar m)+\bar m})\\
            \sigma((\h)_D - (\h)_{r(1+\bar m)+1} / \omega)\\
            \vdots\\
            \sigma((\h)_D - (\h)_{r(1+\bar m)+\bar m} / \omega)
        \end{pmatrix}
        = \begin{pmatrix}
            \zero_{r(1+\bar m)}\\
            -(\h)_{r(1+\bar m)+1} + \sigma((\h)_D - (\h)_{r(1+\bar m)+1} / \omega)\\
            \vdots\\
            -(\h)_{r(1+\bar m)+\bar m} + \sigma((\h)_D - (\h)_{r(1+\bar m)+\bar m} / \omega)\\
            \zero_4
        \end{pmatrix},
    \end{align*}
    Then,
    \begin{align*}
        \FFN_{\nu^\mn_{2}}(\h_s)
        &= \begin{pmatrix}
            \x_0\\
            \x_1\\
            \vdots\\
            \x_{\bar m}\\
            v^{(2)}_1\\
            \vdots\\
            v^{(2)}_{\bar m}\\
            \p_s
        \end{pmatrix},
\end{align*}
\end{small}
where $v^{(2)}_m = \sigma(1 - v^{(1)}_m/\omega)$. Let $\psi^\mn_{2} := (\mu_\id, \nu_2^\mn)$.

\paragraph{Step 3.}
Define $\FFN_{\nu^\mn_{3}}$ as a feedforward neural network with $\nu^\mn_{3} = (W^\mn_{3,1}, W^\mn_{3,2})$, such that
    \begin{align*}
        W^\mn_{3,2} \sigma(W^\mn_{3,1} \h) &= W^\mn_{3,2}
        \begin{pmatrix}
            \sigma((\h)_{r(1+\bar m)+1})\\
            \vdots\\
            \sigma((\h)_{r(1+\bar m)+\bar m})\\
            \sigma(-(\h)_{r(1+\bar m)+1})\\
            \vdots\\
            \sigma(-(\h)_{r(1+\bar m)+\bar m})\\
            \sigma((\h)_D)\\
            \sigma((\h)_D - (\h)_{r(1+\bar m)+1})\\
            \sigma((\h)_D - (\h)_{r(1+\bar m)+1} - (\h)_{r(1+\bar m)+2})\\
            \vdots\\
            \sigma((\h)_D - (\h)_{r(1+\bar m)+1} - (\h)_{r(1+\bar m)+2} - \dots - (\h)_{r(1+\bar m)+\bar m})
        \end{pmatrix}\\
        &= \begin{pmatrix}
            \zero_{r(1 + \bar m)}\\
            -(\h)_{r(1+\bar m)+1} + V_1^{(3)}(\h)\\
            -(\h)_{r(1+\bar m)+2} + V_2^{(3)}(\h)\\
            \vdots\\
            -(\h)_{r(1+\bar m)+\bar m} + V_{\bar m}^{(3)}(\h)\\
            \zero_4\\
        \end{pmatrix},
    \end{align*}
where $V_m^{(3)}(\h) := \sigma((\h)_D - (\h)_{r(1+\bar m)+1} - \dots - (\h)_{r(1+\bar m)+m-1}) - \sigma((\h)_D - (\h)_{r(1+\bar m)+1} - \dots - (\h)_{r(1+\bar m)+m})$. Then,
    \begin{align*}
        \FFN_{\nu^\mn_{3}}(\h_s) &= \h_s + W^\mn_{3,2} \sigma(W^\mn_{3,1} \h_s)
        = \begin{pmatrix}
            \x_0\\
            \x_1\\
            \vdots\\
            \x_{\bar m}\\
            v^{(3)}_1\\
            \vdots\\
            v^{(3)}_{\bar m}\\
            \zero\\
            *
        \end{pmatrix},
    \end{align*}
where $v^{(3)}_m = \sigma(1 - \sum_{m' \in [m-1]} v^{(2)}_{m'}) - \sigma(1 - \sum_{m' \in [m]} v^{(2)}_{m'})$. Note that, since there always exists some $m^* \in [\bar m]$ with $v_{m^*} = \min_{m \in [\bar m]} v_m$, $v^{(1)}_{m^*} = 0$ and $v^{(2)}_{m^*} = 1$. This implies $v^{(2)}_{\bar m} = 0$. Thus, 
    \begin{align*}
        \sum_{m \in [\bar m]} v^{(3)}_m = \sigma(1) - \sigma\qty(1 - \sum_{m' \in [\bar m]} v^{(2)}_{m'}) = 1 - 0 = 1.
    \end{align*}
We also have $v^{(3)}_m \geq 0$ for all $m \in [\bar m]$. Furthermore, $v^{(3)}_m > 0$ implies $v^{(1)}_m < \omega$ and thus $v_m \leq \min_{m' \in [\bar m]} v_{m'} + \omega$. Then, $\sum_m v^{(3)}_m \x_m$ is a convex combination of $\{\x_m : v_m \leq \min_{m' \in [\bar m]} v_{m'} + \omega\}$. Let $\psi_{\mn,3} := (\mu_\id, \nu^\mn_{ 3})$.

\paragraph{Step 4.}
Define $\TF_{\mu^\mn_4}$ with a parameter $\mu^\mn_4 = \{(Q^\mn_{4,j,j'}, K^\mn_{4,j,j'}, V^\mn_{4,j,j'})\}_{j \in [\bar m], j' \in [8]}$, such that
    \begin{align*}
        Q^\mn_{4,j,1} \h &= \begin{pmatrix}
            (\h)_{r(1+\bar m)+j}/4\\
            -2 (\h)_{D-3} - (\h)_{D-2}\\
            (\h)_D\\
            (\h)_D/2\\
            \zero_{D-4}
        \end{pmatrix},
        Q^\mn_{4,j,2} \h = \begin{pmatrix}
            (\h)_{r(1+\bar m)+j}/4\\
            -2 (\h)_{D-3} - (\h)_{D-2}\\
            (\h)_D\\
            (\h)_D/4\\
            \zero_{D-4}
        \end{pmatrix},\\
        Q^\mn_{4,j,3} \h &= \begin{pmatrix}
            (\h)_{r(1+\bar m)+j}/4\\
            -2 (\h)_{D-3} - (\h)_{D-2}\\
            (\h)_D\\
            -(\h)_D/4\\
            \zero_{D-4}
        \end{pmatrix},
        Q^\mn_{4,j,4} \h = \begin{pmatrix}
            (\h)_{r(1+\bar m)+j}/4\\
            -2 (\h)_{D-3} - (\h)_{D-2}\\
            (\h)_D\\
            -(\h)_D/2\\
            \zero_{D-4}
        \end{pmatrix},\\
        Q^\mn_{4,j,5} \h &= \begin{pmatrix}
            (\h)_D/4\\
            -2 (\h)_{D-3} - (\h)_{D-2}\\
            (\h)_D\\
            (\h)_D/2\\
            \zero_{D-4}
        \end{pmatrix},
        Q^\mn_{4,j,6} \h = \begin{pmatrix}
            (\h)_D/4\\
            -2 (\h)_{D-3} - (\h)_{D-2}\\
            (\h)_D\\
            (\h)_D/4\\
            \zero_{D-4}
        \end{pmatrix},\\
        Q^\mn_{4,j,7} \h &= \begin{pmatrix}
            (\h)_D/4\\
            -2 (\h)_{D-3} - (\h)_{D-2}\\
            (\h)_D\\
            -(\h)_D/4\\
            \zero_{D-4}
        \end{pmatrix},
        Q^\mn_{4,j,8} \h = \begin{pmatrix}
            (\h)_D/4\\
            -2 (\h)_{D-3} - (\h)_{D-2}\\
            (\h)_D\\
            -(\h)_D/2\\
            \zero_{D-4}
        \end{pmatrix},\\
        K^\mn_{4,j,1} \h &= \begin{pmatrix}
            (\h)_D\\
            (\h)_D\\
            2 (\h)_{D-3} + (\h)_{D-2}\\
            (\h)_D\\
            \zero_{D-4}
        \end{pmatrix},
        K^\mn_{4,j,2} \h = \begin{pmatrix}
            (\h)_D\\
            (\h)_D\\
            2 (\h)_{D-3} + (\h)_{D-2}\\
            (\h)_D\\
            \zero_{D-4}
        \end{pmatrix},\\
        K^\mn_{4,j,3} \h &= \begin{pmatrix}
            (\h)_D\\
            (\h)_D\\
            2 (\h)_{D-3} + (\h)_{D-2}\\
            (\h)_D\\
            \zero_{D-4}
        \end{pmatrix},
        K^\mn_{4,j,4} \h = \begin{pmatrix}
            (\h)_D\\
            (\h)_D\\
            2 (\h)_{D-3} + (\h)_{D-2}\\
            (\h)_D\\
            \zero_{D-4}
        \end{pmatrix},\\
        K^\mn_{4,j,5} &= K^\mn_{4,j,1}, \ \ K^\mn_{4,j,6} = K^\mn_{4,j,2}, \ \ K^\mn_{4,j,7} = K^\mn_{4,j,3}, \ \ K^\mn_{4,j,8} = K^\mn_{4,j,4},\\
        V^\mn_{4,j,1} \h &= -4 \begin{pmatrix}
            (\h)_{(jr+1):((j+1)r)}\\
            \zero_{D-r}
        \end{pmatrix},\ \ 
        V^\mn_{4,j,2} \h = 8\begin{pmatrix}
            (\h)_{(jr+1):((j+1)r)}\\
            \zero_{D-r}
        \end{pmatrix},\\
        V^\mn_{4,j,3} \h &= -8 \begin{pmatrix}
            (\h)_{(jr+1):((j+1)r)}\\
            \zero_{D-r}
        \end{pmatrix},\ \ 
        V^\mn_{4,j,4} \h = 4 \begin{pmatrix}
            (\h)_{(jr+1):((j+1)r)}\\
            \zero_{D-r}
        \end{pmatrix}\\
        V^\mn_{4,j,5} \h &= 4 \begin{pmatrix}
            (\h)_{1:(r(1+\bar m))}\\
            \zero_{D-r(1+\bar m)}
        \end{pmatrix},\ \ 
        V^\mn_{4,j,6} \h = -8 \begin{pmatrix}
            (\h)_{1:(r(1+\bar m))}\\
            \zero_{D-r(1+\bar m)}
        \end{pmatrix},\\
        V^\mn_{4,j,7} \h &= 8 \begin{pmatrix}
            (\h)_{1:(r(1+\bar m))}\\
            \zero_{D-r(1+\bar m)}
        \end{pmatrix},\ \ 
        V^\mn_{4,j,8} h = -4 \begin{pmatrix}
            (\h)_{1:(r(1+\bar m))}\\
            \zero_{D-r(1+\bar m)}
        \end{pmatrix}.
    \end{align*}
Let $\tilde H = [\tilde \h_1; \dots; \tilde \h_N] := \TF_{(\psi^\mn_1,\psi^\mn_2,\psi^\mn_3)}(H)$. Then,
    \begin{small}
    \begin{align*}
        &\Attn_{\mu^\mn_4}(\tilde H)_s = \tilde \h_s + \sum_{s' \in [N]} \phi_1((\tilde \h_s)_{r(1+\bar m)+j}; 2(\tilde \h_s)_{D-3}+(\tilde \h_s)_{D-2}, 2(\tilde \h_{s'})_{D-3}+(\tilde \h_{s'})_{D-2}) \begin{pmatrix}
            (\tilde \h_{s'})_{(jr+1):((j+1)r)}\\
            \zero_{D-r}
        \end{pmatrix}\\
        &\quad\quad- \sum_{s' \in [N]} \phi_1(1; 2(\tilde \h_s)_{D-3}+(\tilde \h_s)_{D-2}, 2(\tilde \h_{s'})_{D-3}+(\tilde \h_{s'})_{D-2}) \begin{pmatrix}
            (\tilde \h_{s'})_{1:(r(1+\bar m))}\\
            \zero_{D-r}
        \end{pmatrix}\\
        &\quad= \begin{pmatrix}
            \sum_{m \in [\bar m]} v^{(3)}_m \x_m\\
            \zero_{(r+1)\bar m}\\
            \p_s
        \end{pmatrix}.
    \end{align*}
    \end{small}\noindent
Define $\psi^\mn_4 = (\mu^\mn_4, \nu_\id)$. We can easily implement a transformer layer $\TF_{\psi_5^\mn}$, such that $\TF_{\psi_5^\mn}(H)_s = ((\h_s)_{1:D-4}^\top, \zero_4^\top)^\top$. 

The desired transformer is obtained by $\TF_{\Psi^\mn}$ with $\Psi^\mn = (\psi^\mn_1,\psi^\mn_2,\psi^\mn_3,\psi^\mn_4,\psi^\mn_5)$. This completes the proof of Lemma~\ref{lem: minimum by transformer}.
\end{proof}

\section{Additional Numerical Results}
\label{supp-sec: add-numerical}

\subsection{More details about general Algorithm~\ref{alg}}
\label{supp-sec: add-algo}

We present a schematic plot of our LLM-powered synthetic oversampling and augmentation algorithm in Figure~\ref{fig: diagram}, and another schematic plot of the transformer workflow in Figure~\ref{fig: diagram2}. 

\begin{figure}[b!]
\centering
\includegraphics[width=0.75\linewidth, height=2.3in]{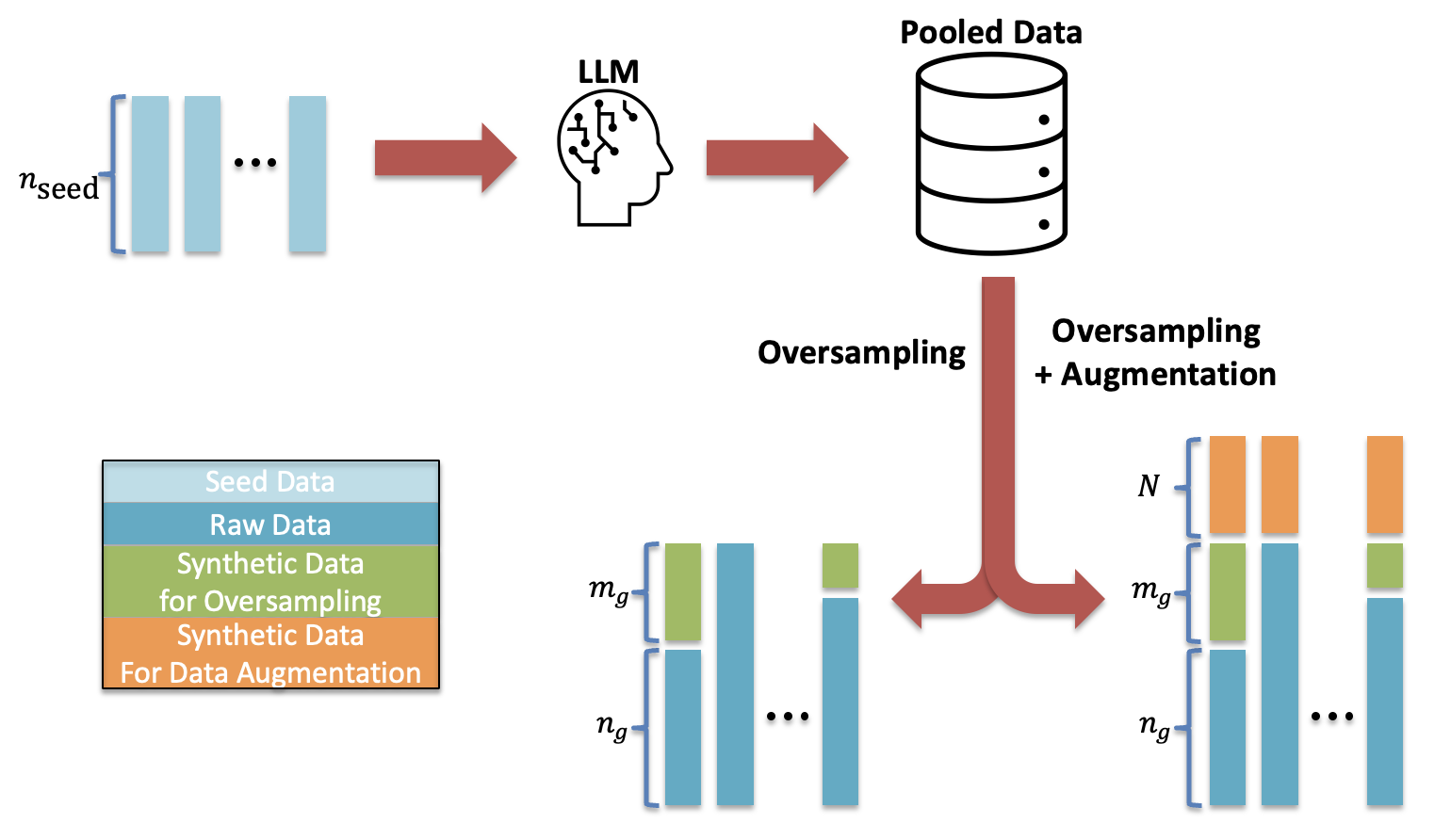}
\caption{Schematic plot of Algorithm~\ref{alg}.}
\label{fig: diagram}
\end{figure}

\begin{figure}[t!]
\centering
\includegraphics[width=0.75\linewidth, height=2.3in]{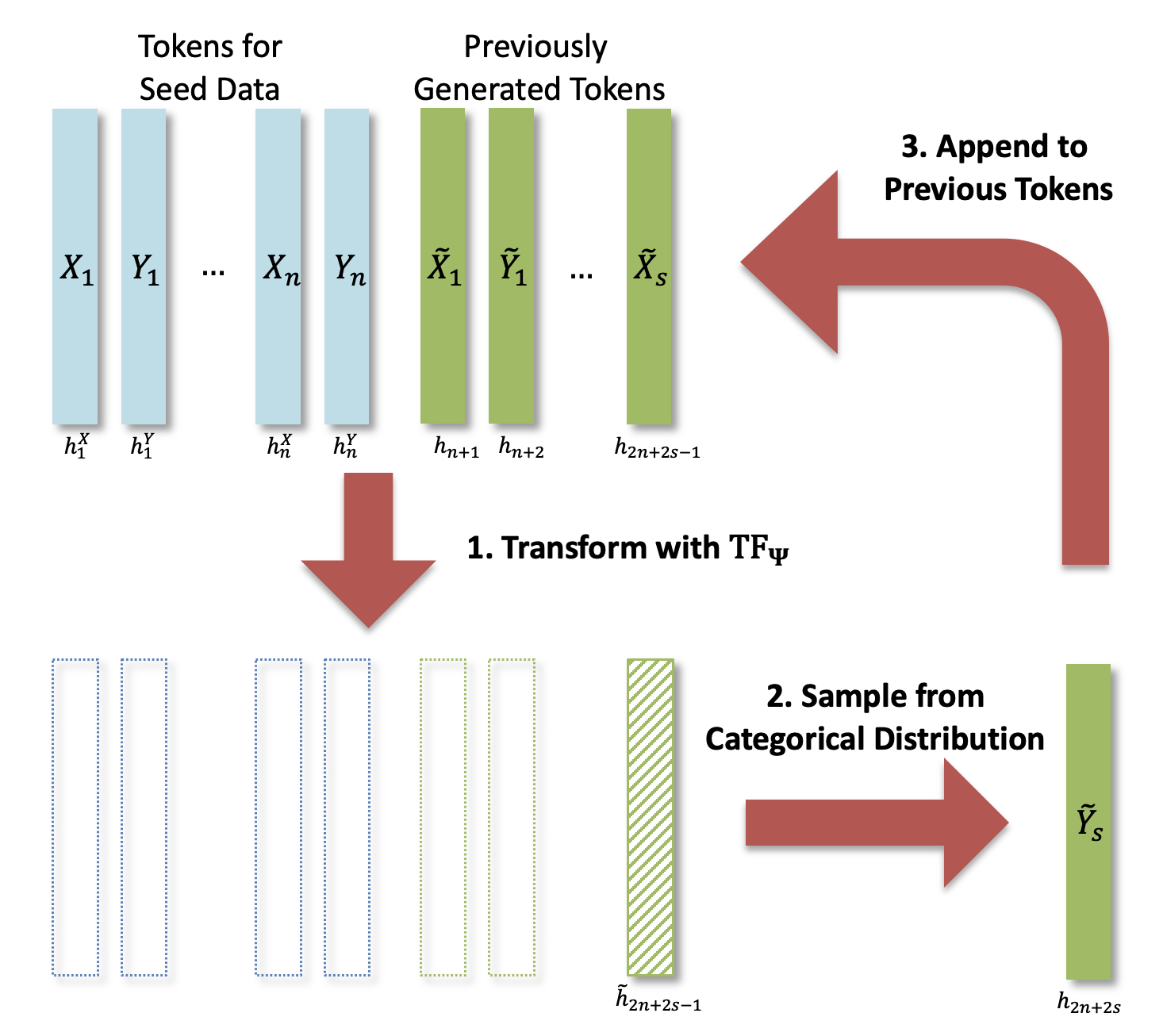}
\caption{Schematic plot of transformer workflow.}
\label{fig: diagram2}
\end{figure}

We clarify that, in Algorithm~\ref{alg}, we first generate a sufficiently large pool of synthetic samples using the LLM, and then randomly select $m_g$ samples for each group $g$ according to their generated labels. We also comment that, incorporating \emph{all} synthetic samples \emph{without} group-level adjustment may produce an imbalanced dataset, as the distribution of the generated labels may not match the desired balance for the downstream classification task. As such, our algorithm introduces a controlled sampling step to enforce equal representation across groups and maintain balance. Furthermore, after synthetic oversampling, we perform an additional synthetic augmentation step, by adding $N$ synthetic samples per group, As demonstrated in Theorem~\ref{thm: scaling law}, along with Figures~\ref{fig:scale-imb-xgb-loglog} and \ref{fig:scale-spurious-xgb-loglog}, this augmentation step further improves the downstream performance, by reducing the variance term at a polynomial rate while incurring an negligible bias.

\subsection{More results with GPT-2}
\label{supp-sec: add-gpt2}

\subsubsection{Implementation details}
\label{supp-sec: gpt2-details}

We provide additional implementation details for GPT-2. Specifically, the GPT-2 model (a distilGPT-2 variant) was fine-tuned using the be-great (v0.0.9) framework, which offers a stable implementation of the GReaT method for tabular data synthesis. All experiments were run on a single NVIDIA A10G GPU with 24 GB of memory, with each fine-tuning run completed in about ten minutes. Fine-tuning was performed for 100 epochs with a batch size of 32 using the AdamW optimizer and the default learning-rate schedule. The decoding temperature was set to 0.7. All other optimization hyperparameters, including weight decay, gradient clipping, and warm-up steps, followed the framework's default configuration. The maximum sequence length and tokenizer settings were automatically adapted to the feature dimension of each dataset. To ensure robustness, five distinct random seeds were used, and results were averaged across those replicates. For reproducibility, random seeds were fixed across all components, including model initialization, data shuffling, and sampling procedures. During data generation, each synthetic record was produced independently via top-$p$ nucleus sampling at $p=0.9$, with a separate decoding run and unique random seed for every sample to guarantee inter-sample independence. Continuous variables were discretized using quantile-based binning prior to tokenization, ensuring consistent input structure across training and generation. The bin size was automatically determined by the quantile procedure implemented in the be-great framework.

\subsubsection{Seed data}
\label{supp-sec: seeddata}

In our framework, the seed data is used to prompt and fine-tune a large language model, and to generate synthetic samples. The raw data is used to train the downstream classifier, along with the generated synthetic data. The seed data serve as representative examples that initialize the in-context learning process of the transformer-based generator. They define the conditional prompt from which synthetic samples are generated, thereby anchoring the learned distribution to the true data distribution. In our numerical experiments, both the seed data and the raw data are sampled from the training data, i.e., the 70\% subset of a given dataset. 

\begin{figure}[t!]
\centering
\begin{subfigure}{0.24\textwidth}
  \centering
  \includegraphics[width=\linewidth]{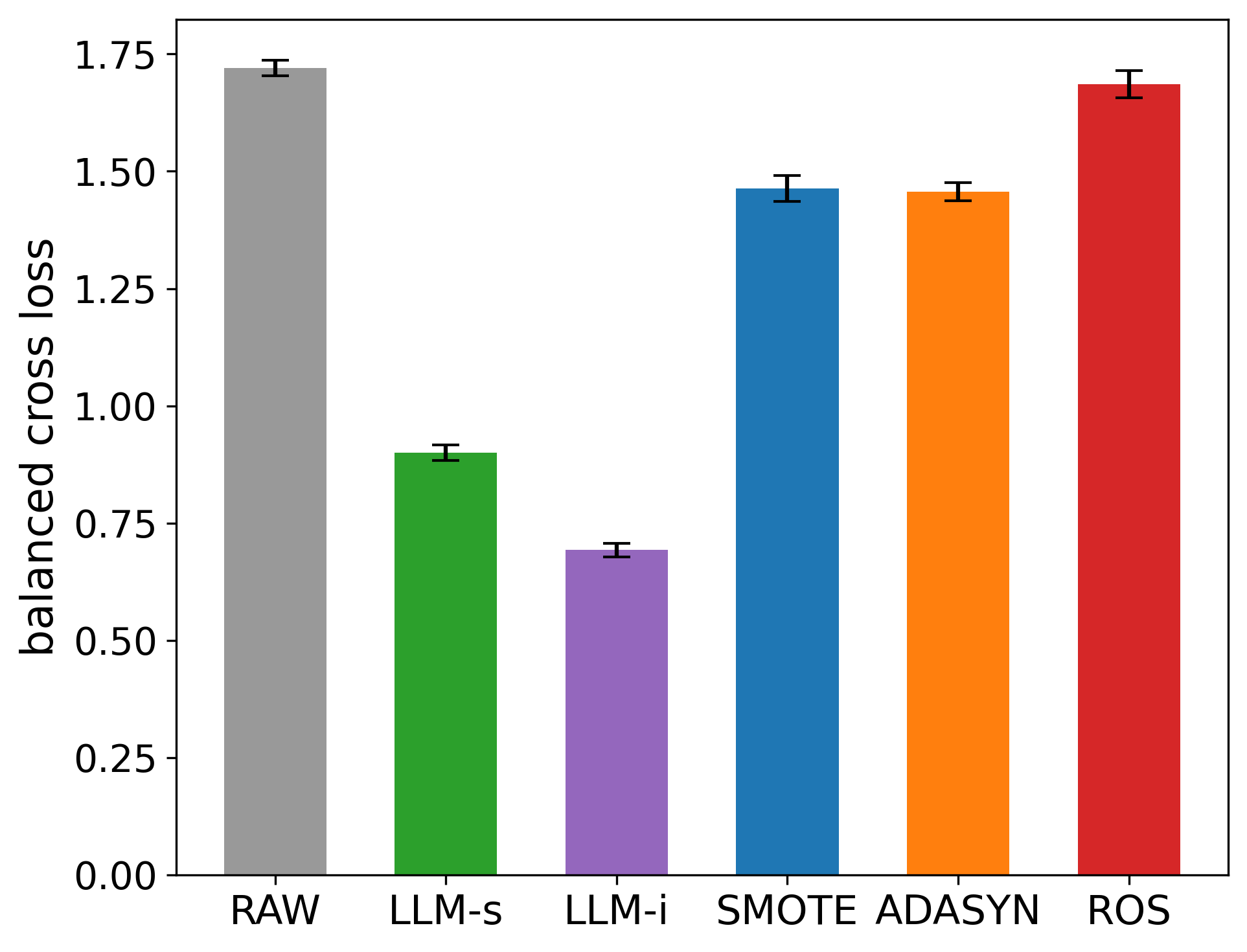}
  \caption{CRAFT (imb.class, bal.loss)}
\end{subfigure}
\begin{subfigure}{0.24\textwidth}
  \centering
  \includegraphics[width=\linewidth]{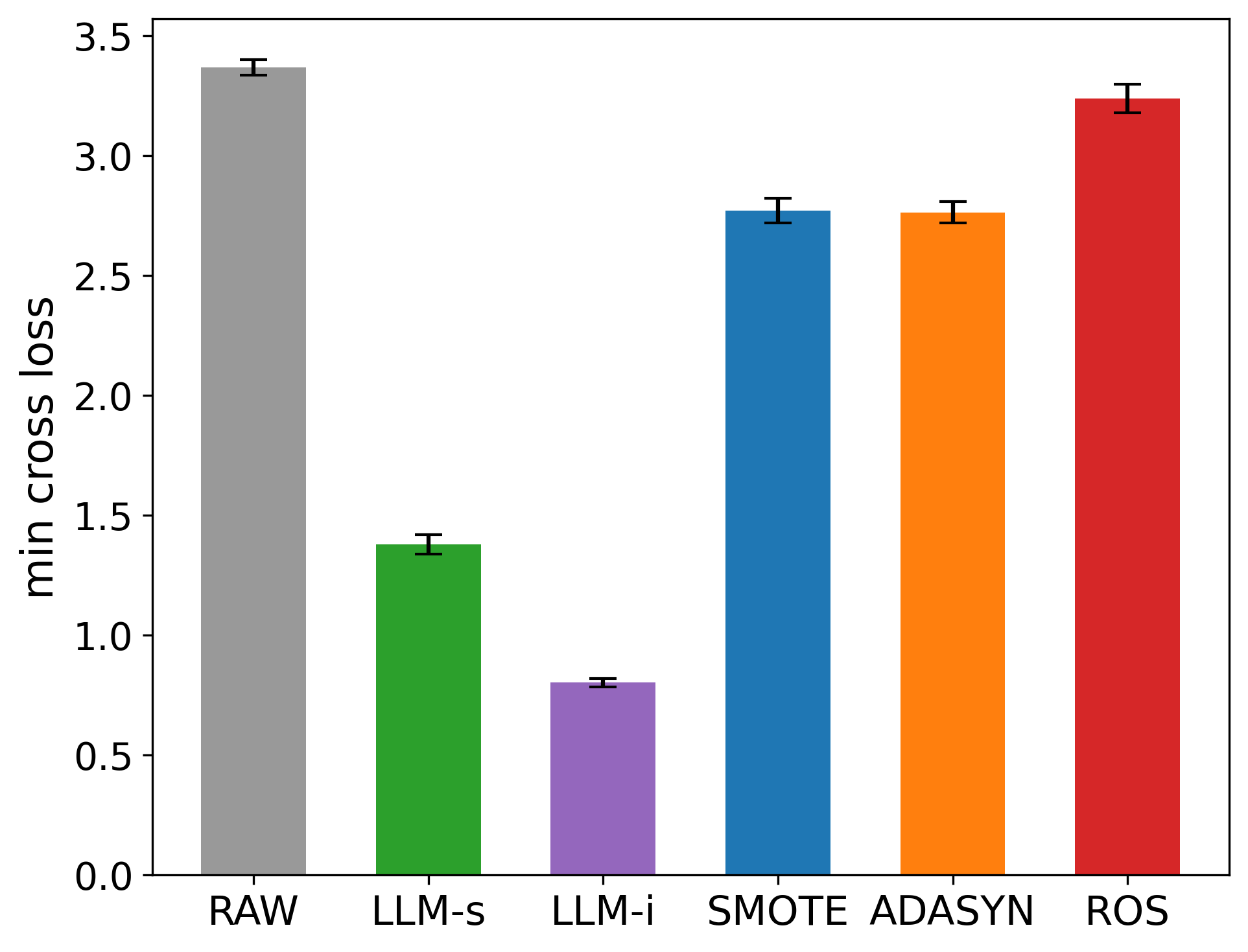}
  \caption{CRAFT (imb.class, min.loss)}
\end{subfigure}
\begin{subfigure}{0.24\textwidth}
  \centering
  \includegraphics[width=\linewidth]{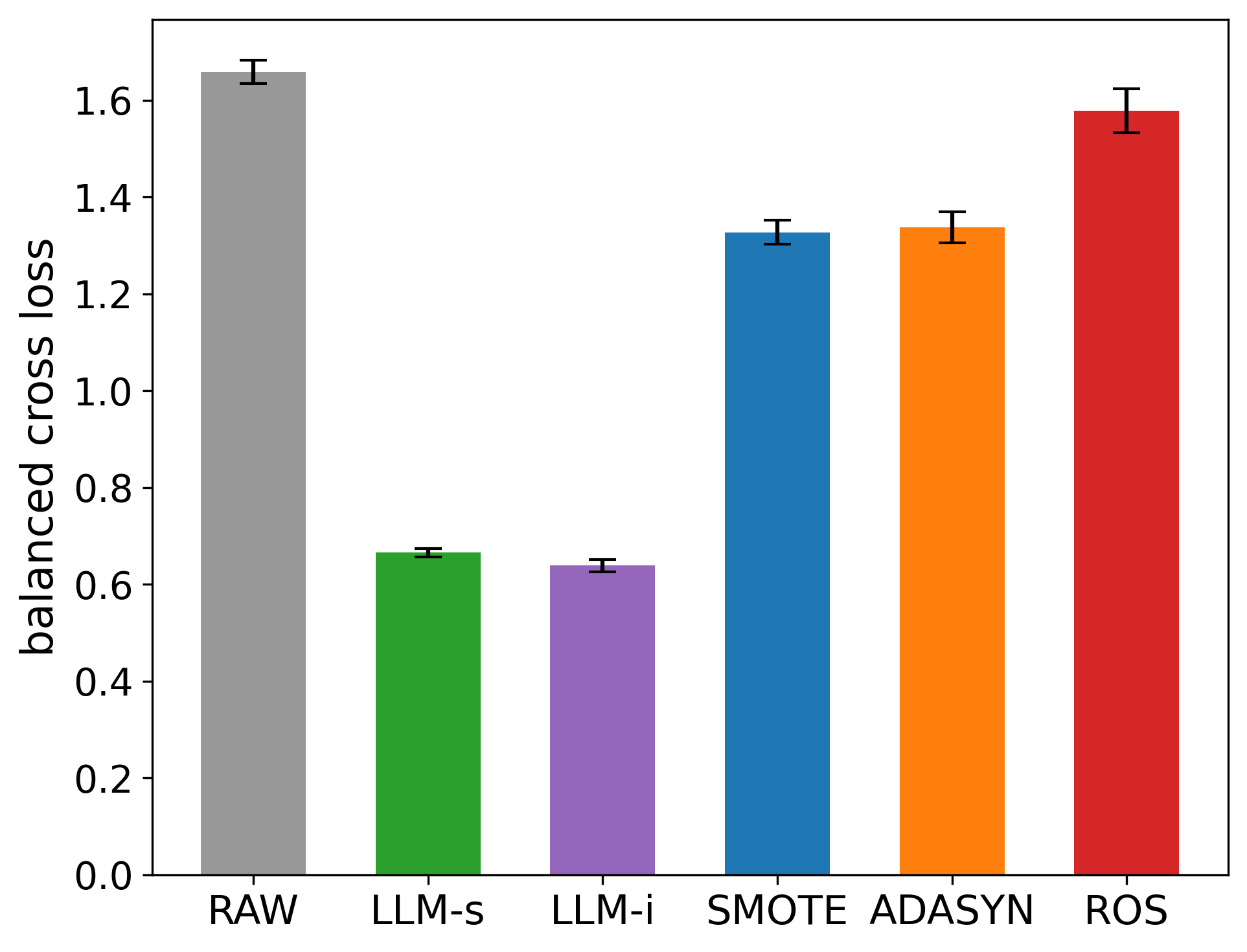}
  \caption{CRAFT (spur.cor, bal.loss.)}
\end{subfigure}
\begin{subfigure}{0.24\textwidth}
  \centering
  \includegraphics[width=\linewidth]{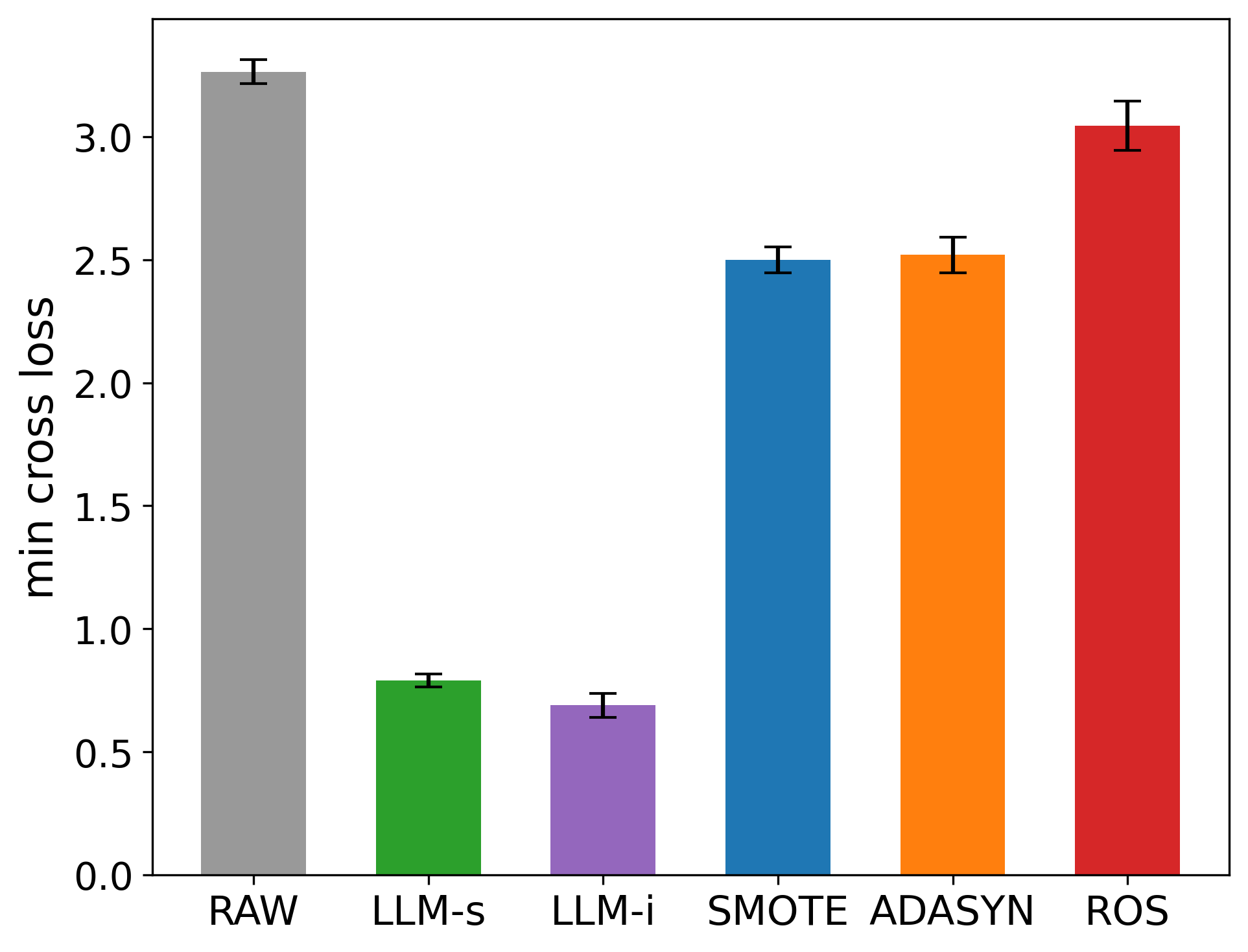}
  \caption{CRAFT (spur.cor., min.loss)}
\end{subfigure}

\vspace{4pt}
\begin{subfigure}{0.24\textwidth}
  \centering
  \includegraphics[width=\linewidth]{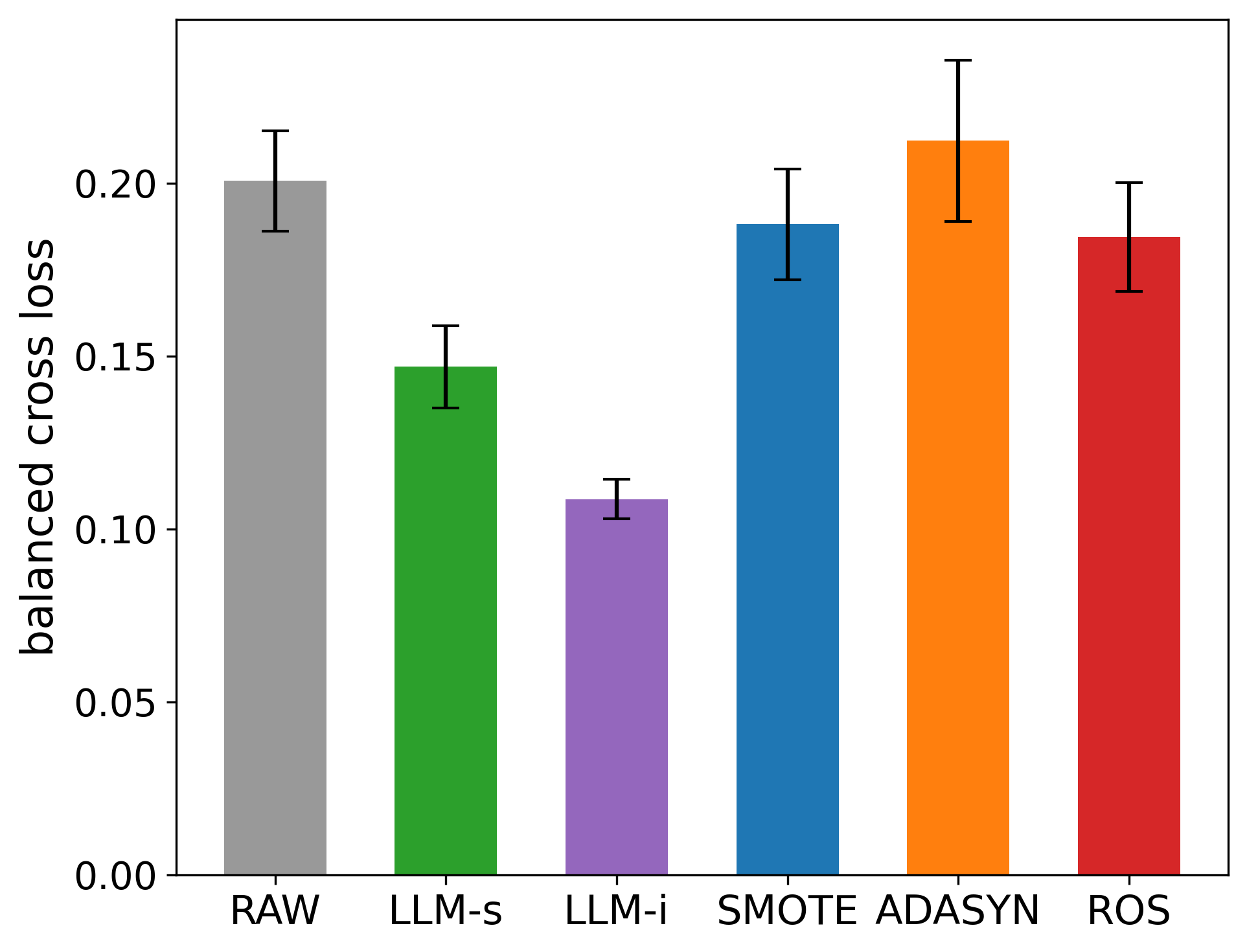}
  \caption{Gender (imb.class, bal.loss)}
\end{subfigure}
\begin{subfigure}{0.24\textwidth}
  \centering
  \includegraphics[width=\linewidth]{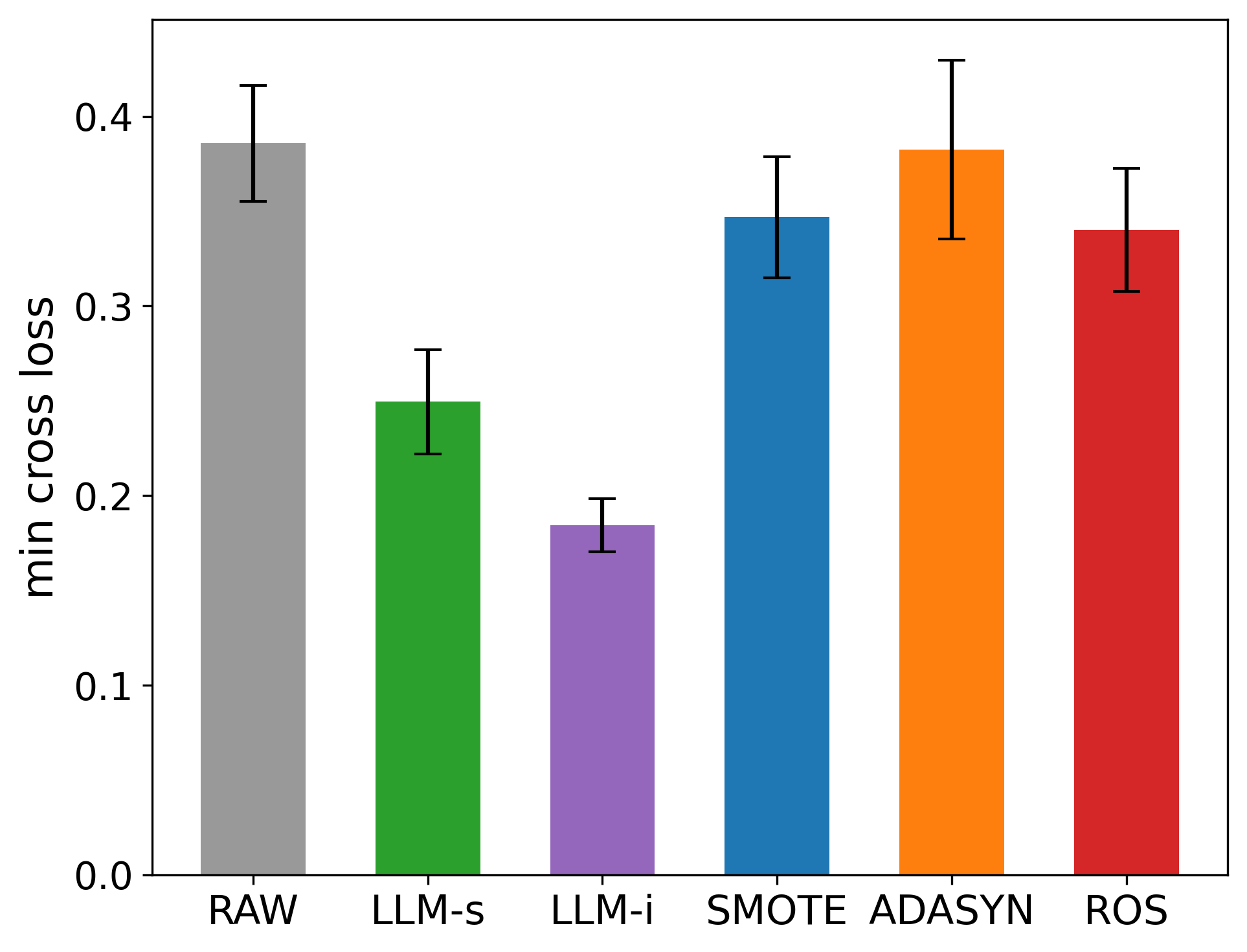}
  \caption{Gender (imb.class, min.loss)}
\end{subfigure}
\begin{subfigure}{0.24\textwidth}
  \centering
  \includegraphics[width=\linewidth]{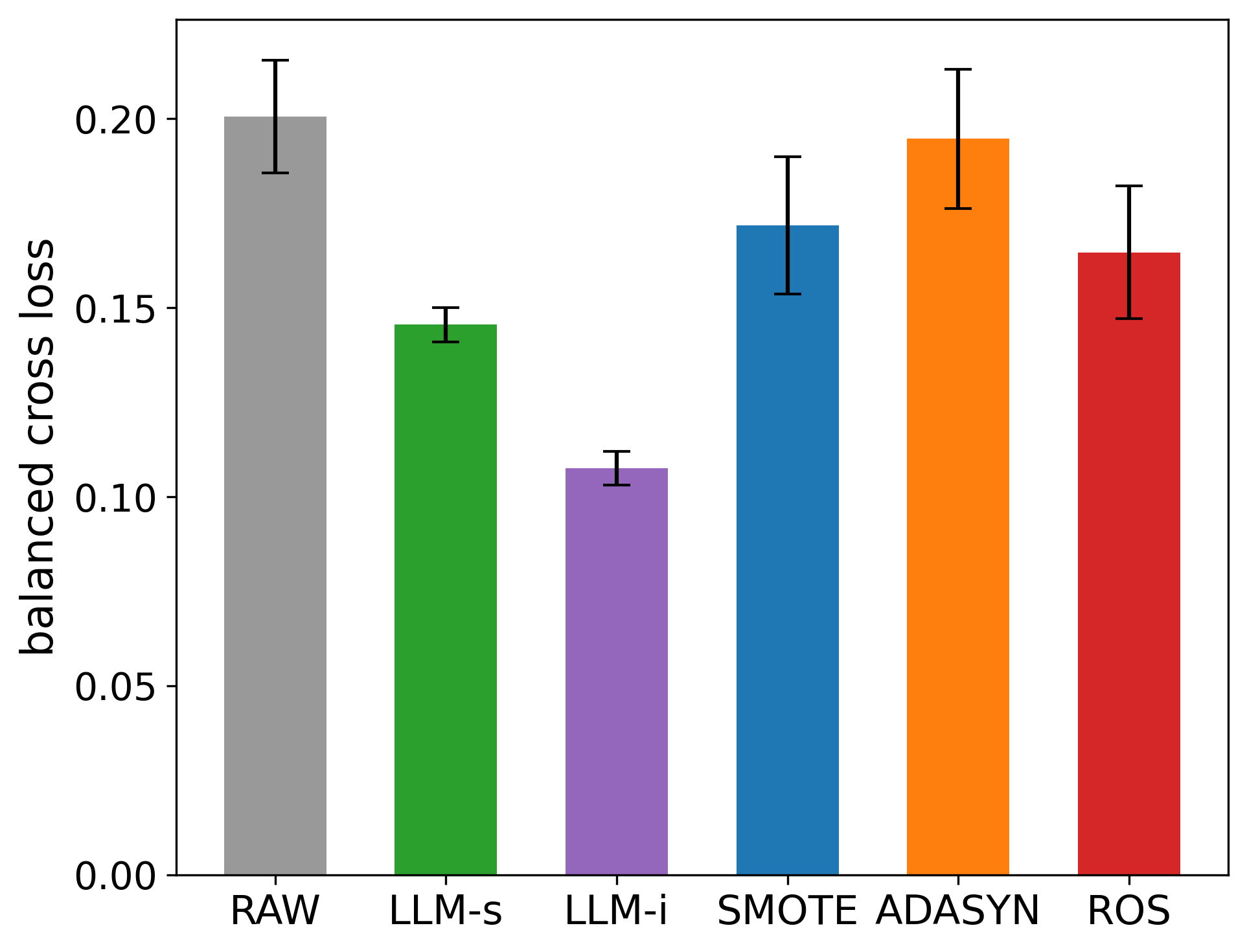}
  \caption{Gender (spur.cor, bal.loss.)}
\end{subfigure}
\begin{subfigure}{0.24\textwidth}
  \centering
  \includegraphics[width=\linewidth]{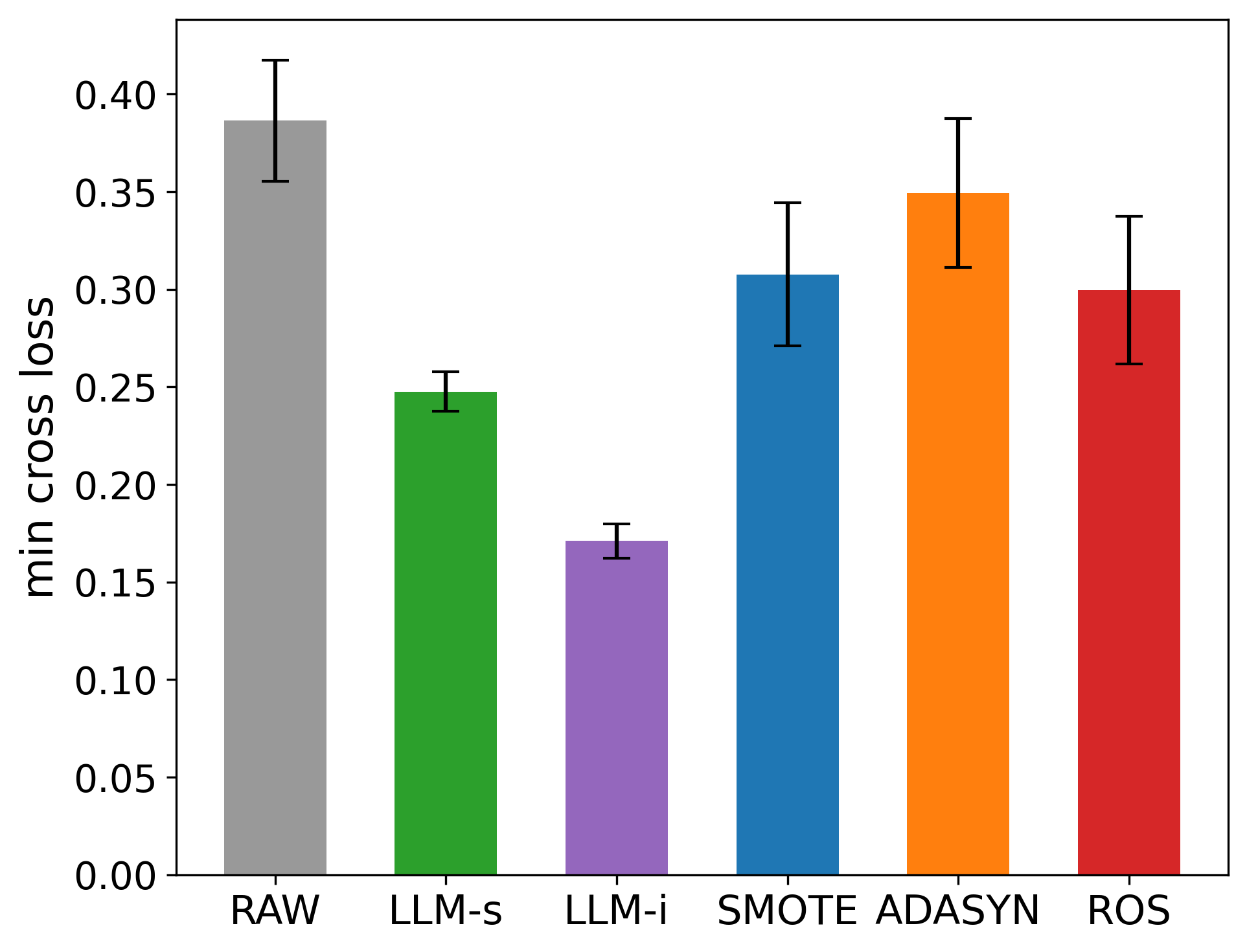}
  \caption{Gender (spur.cor, min.loss.)}
\end{subfigure}

\vspace{4pt}
\begin{subfigure}{0.24\textwidth}
  \centering
  \includegraphics[width=\linewidth]{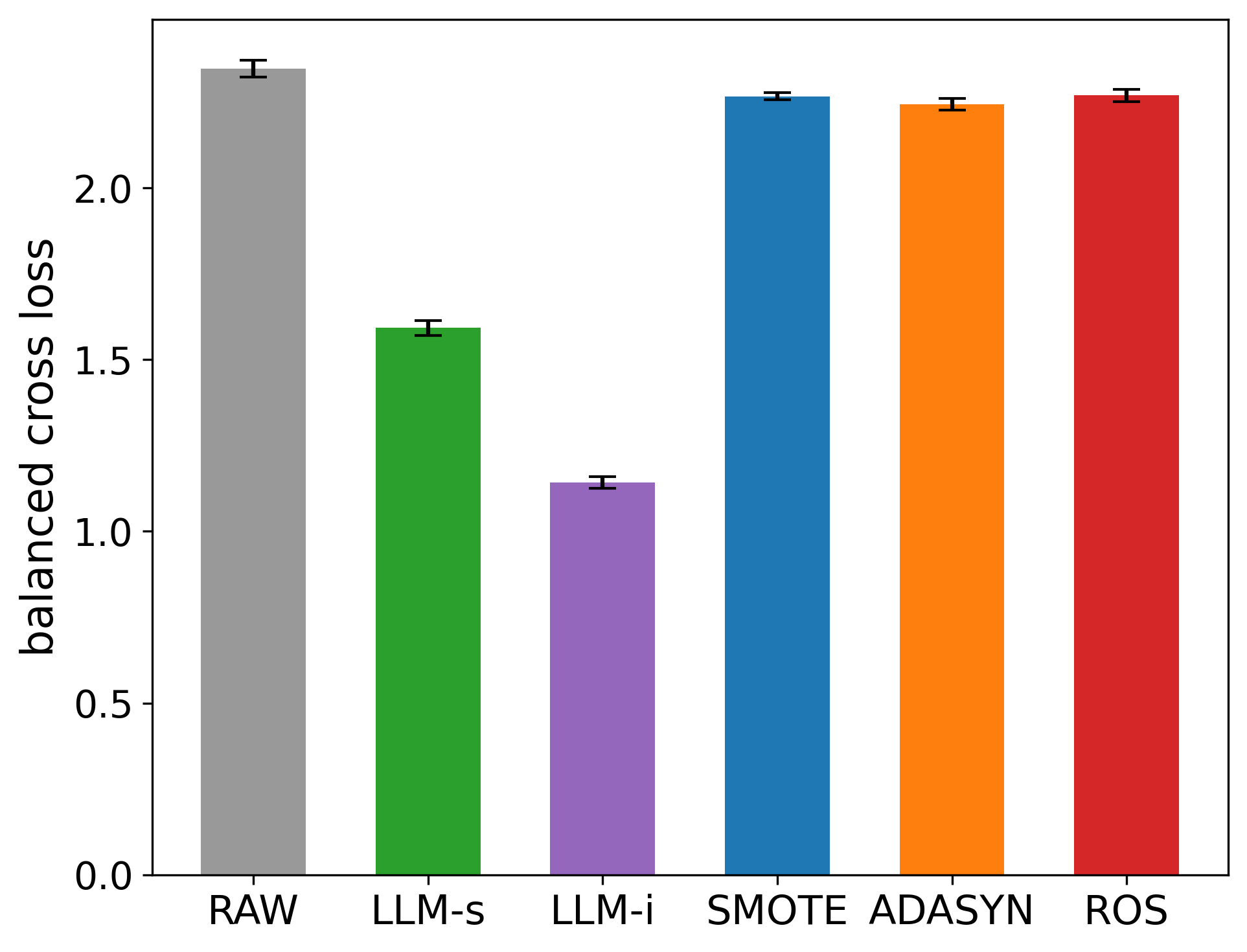}
  \caption{Diabetes (imb.class, bal.loss)}
\end{subfigure}
\begin{subfigure}{0.24\textwidth}
  \centering
  \includegraphics[width=\linewidth]{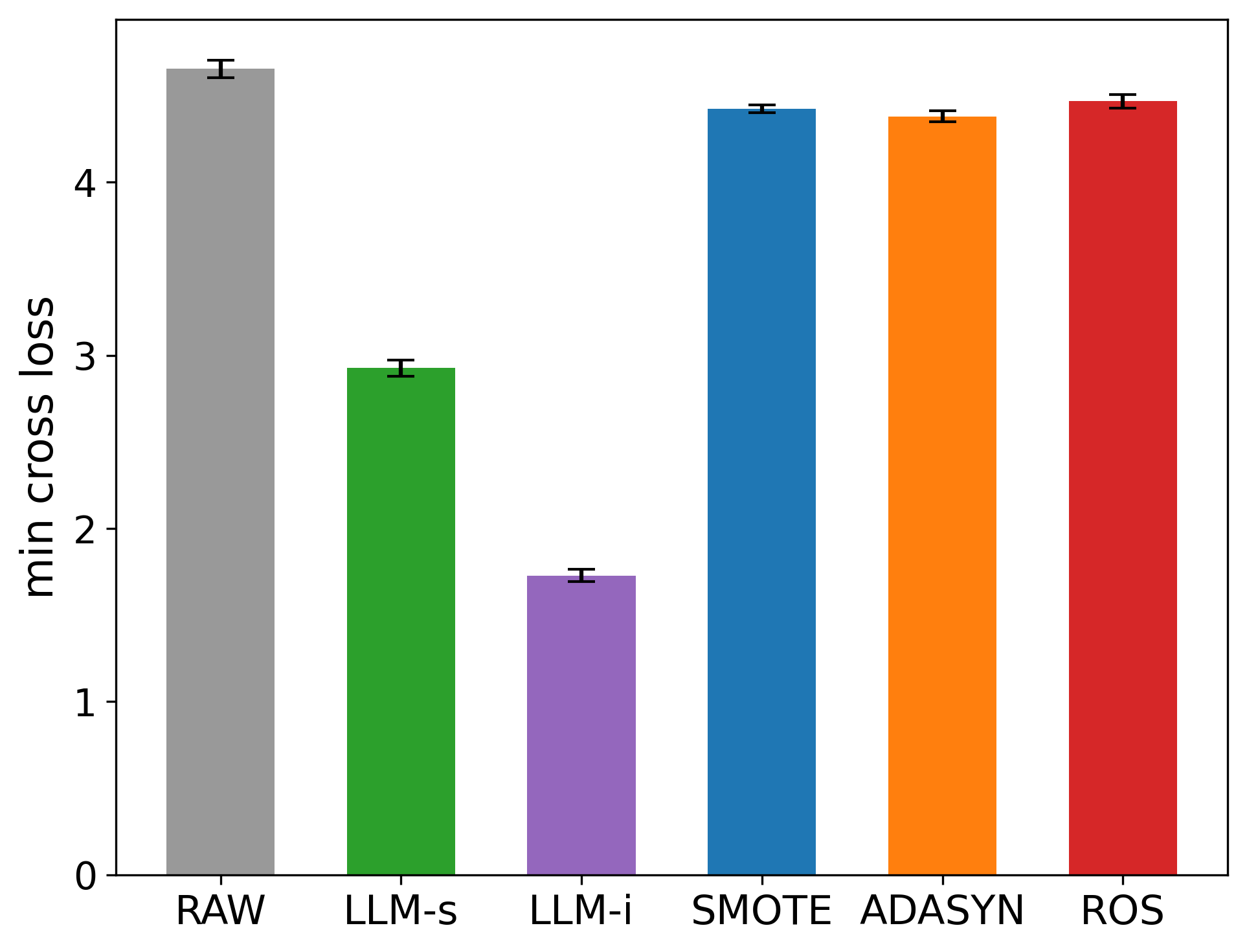}
  \caption{Diabetes (imb.class, min.loss))}
\end{subfigure}
\begin{subfigure}{0.24\textwidth}
  \centering
  \includegraphics[width=\linewidth]{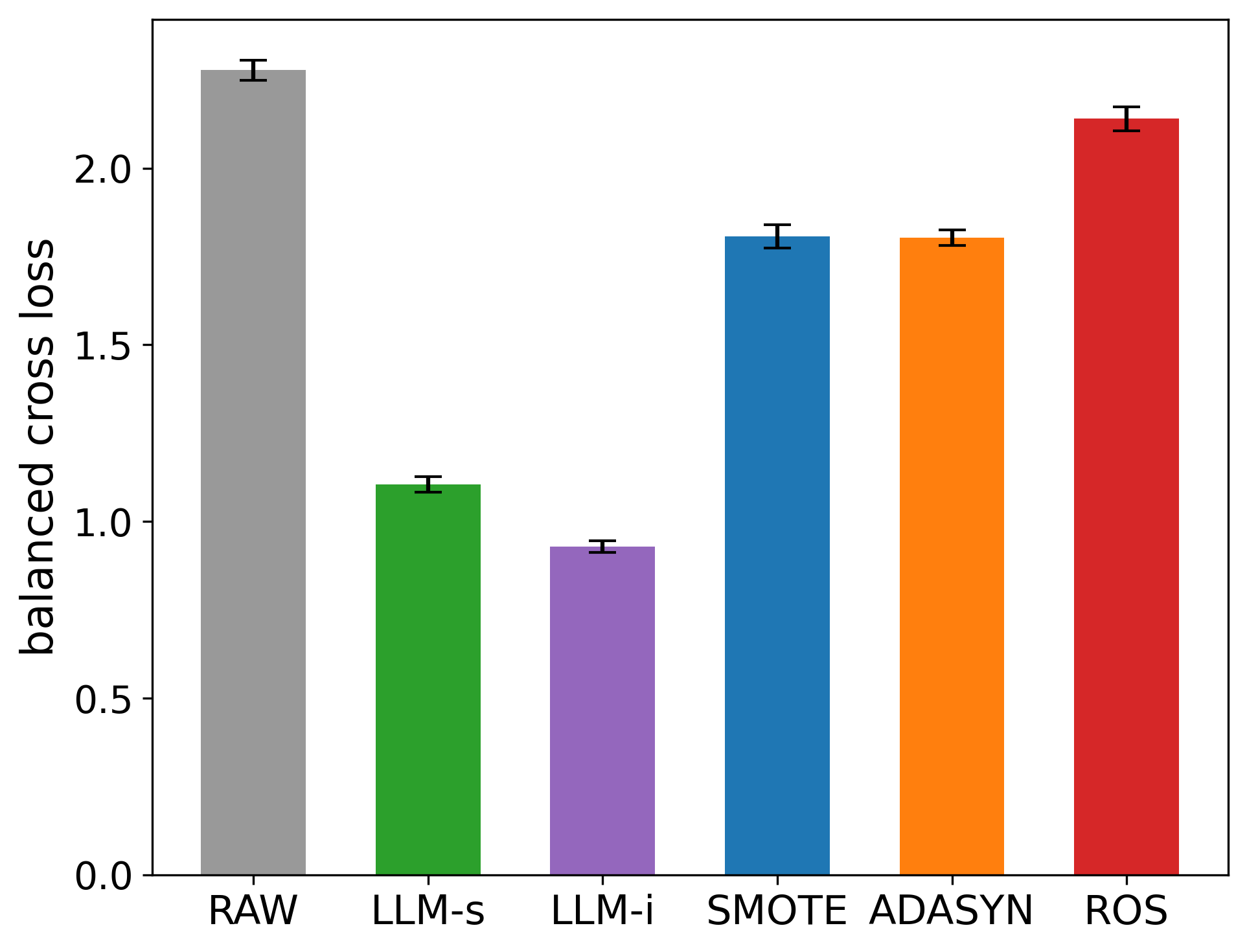}
  \caption{Diabetes (spur.cor, bal.loss.)}
\end{subfigure}
\begin{subfigure}{0.24\textwidth}
  \centering
  \includegraphics[width=\linewidth]{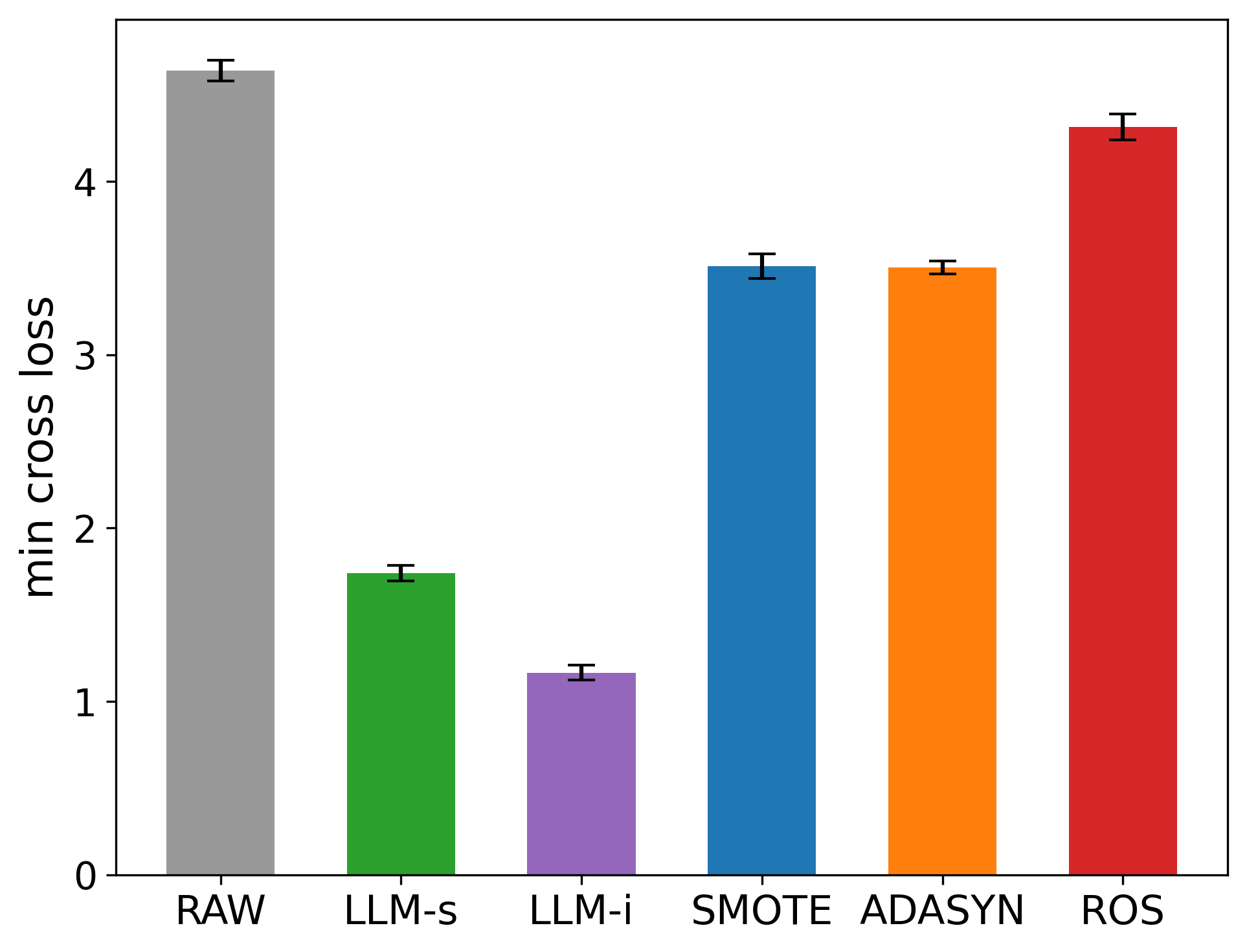}
  \caption{Diabetes (spur.cor, min.loss.)}
\end{subfigure}

\vspace{4pt}
\begin{subfigure}{0.24\textwidth}
  \centering
  \includegraphics[width=\linewidth]{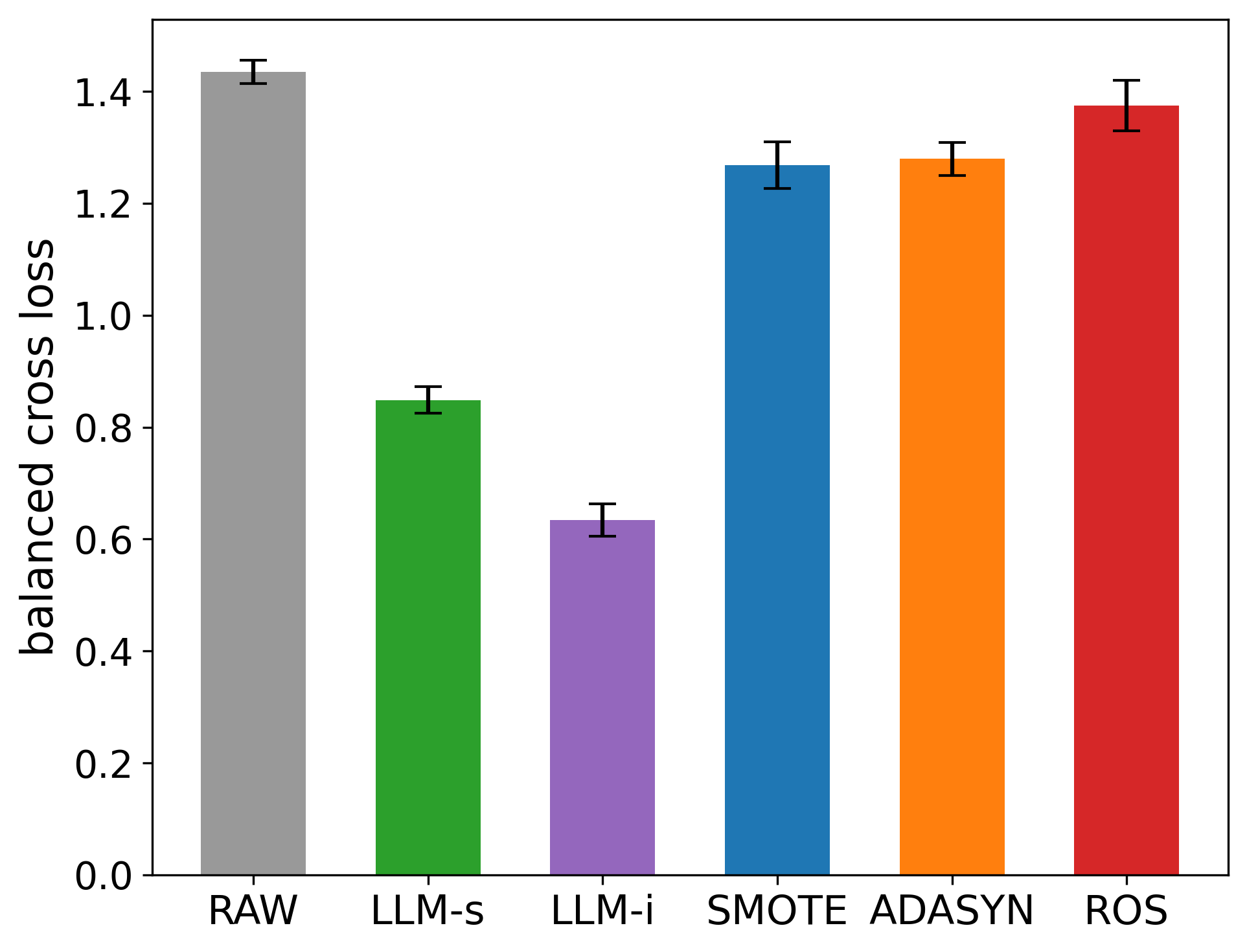}
  \caption{Adult (imb.class, bal.loss)}
\end{subfigure}
\begin{subfigure}{0.24\textwidth}
  \centering
  \includegraphics[width=\linewidth]{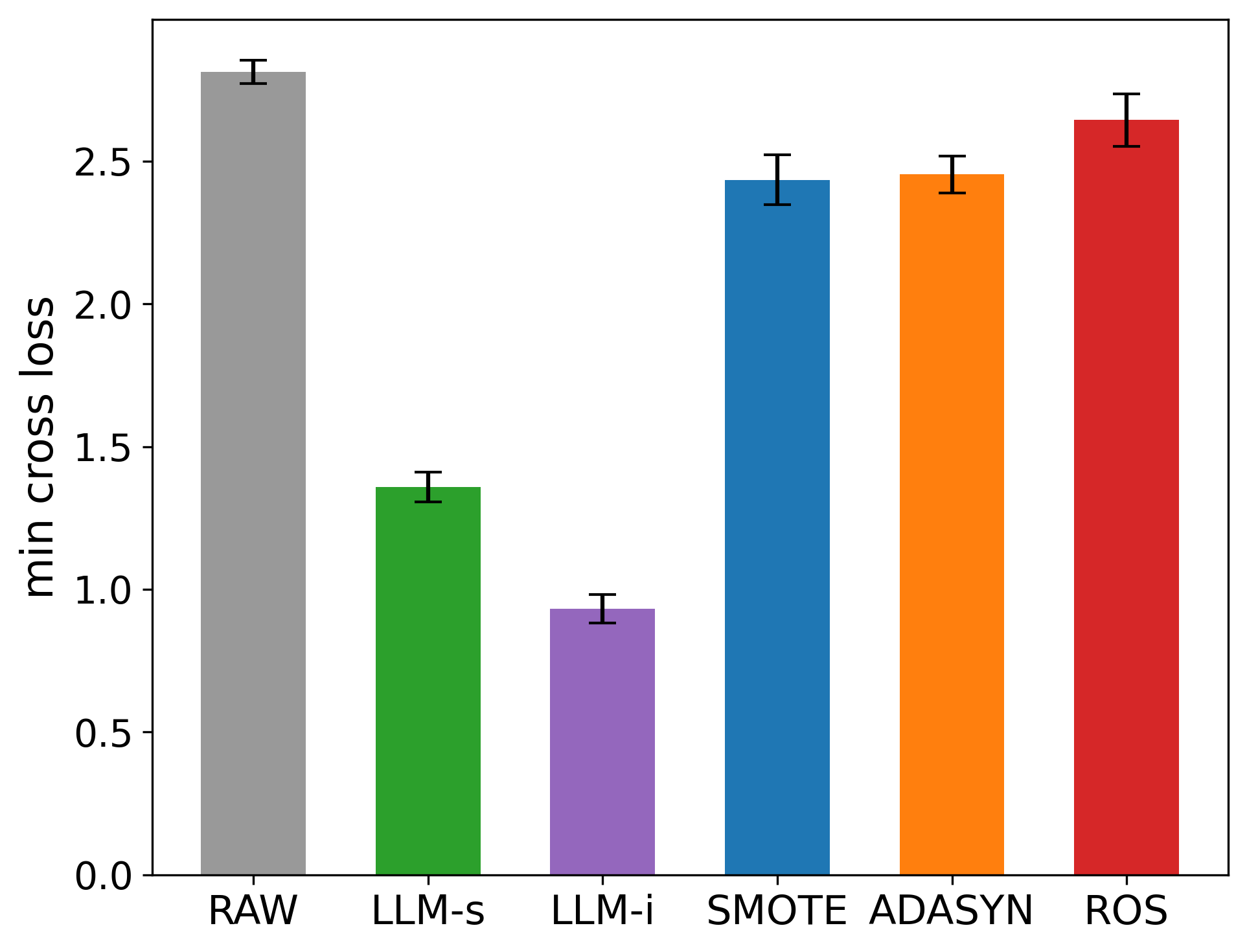}
  \caption{Adult (imb.class, min.loss)}
\end{subfigure}
\begin{subfigure}{0.24\textwidth}
  \centering
  \includegraphics[width=\linewidth]{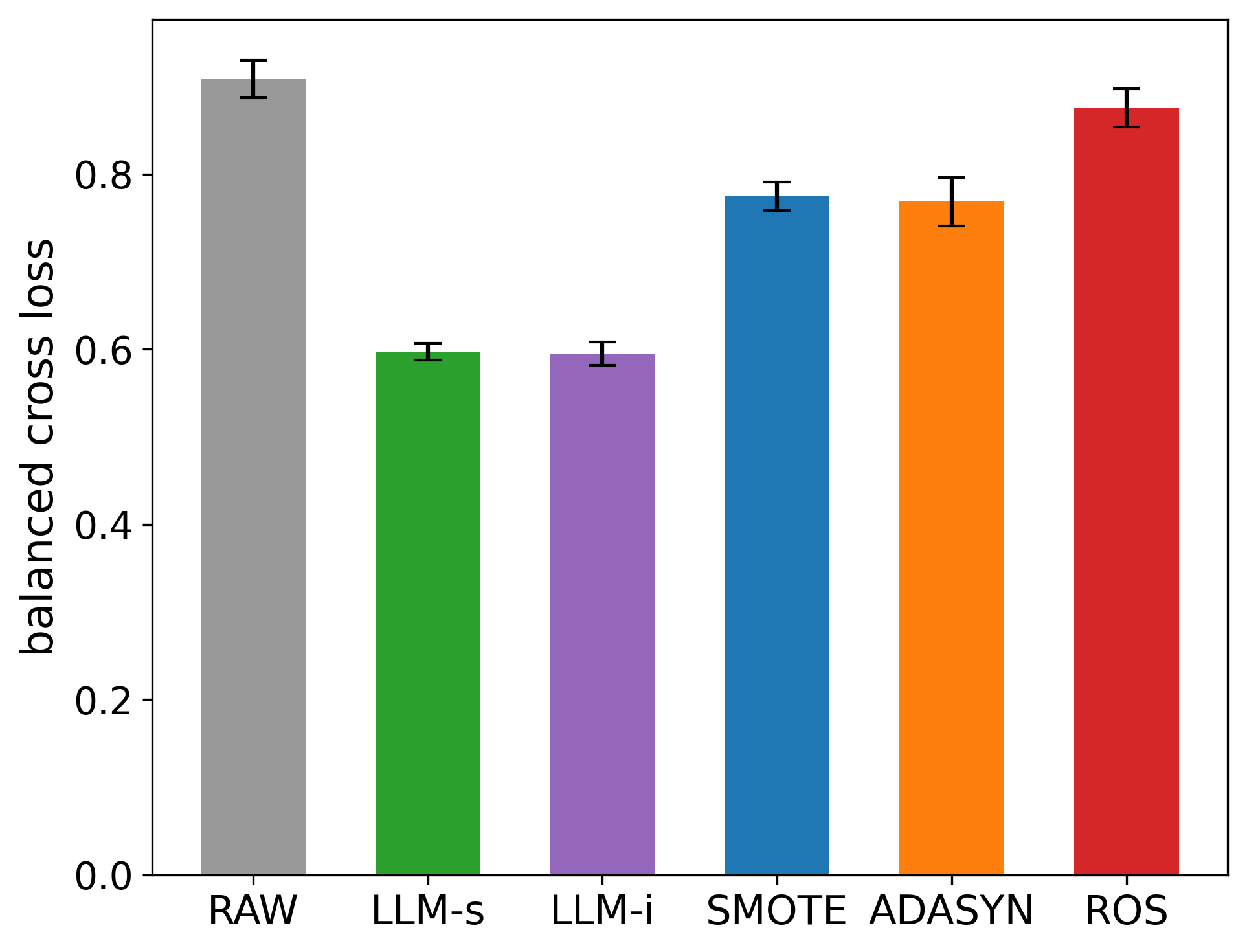}
  \caption{Adult (spur.cor, bal.loss.)}
\end{subfigure}
\begin{subfigure}{0.24\textwidth}
  \centering
  \includegraphics[width=\linewidth]{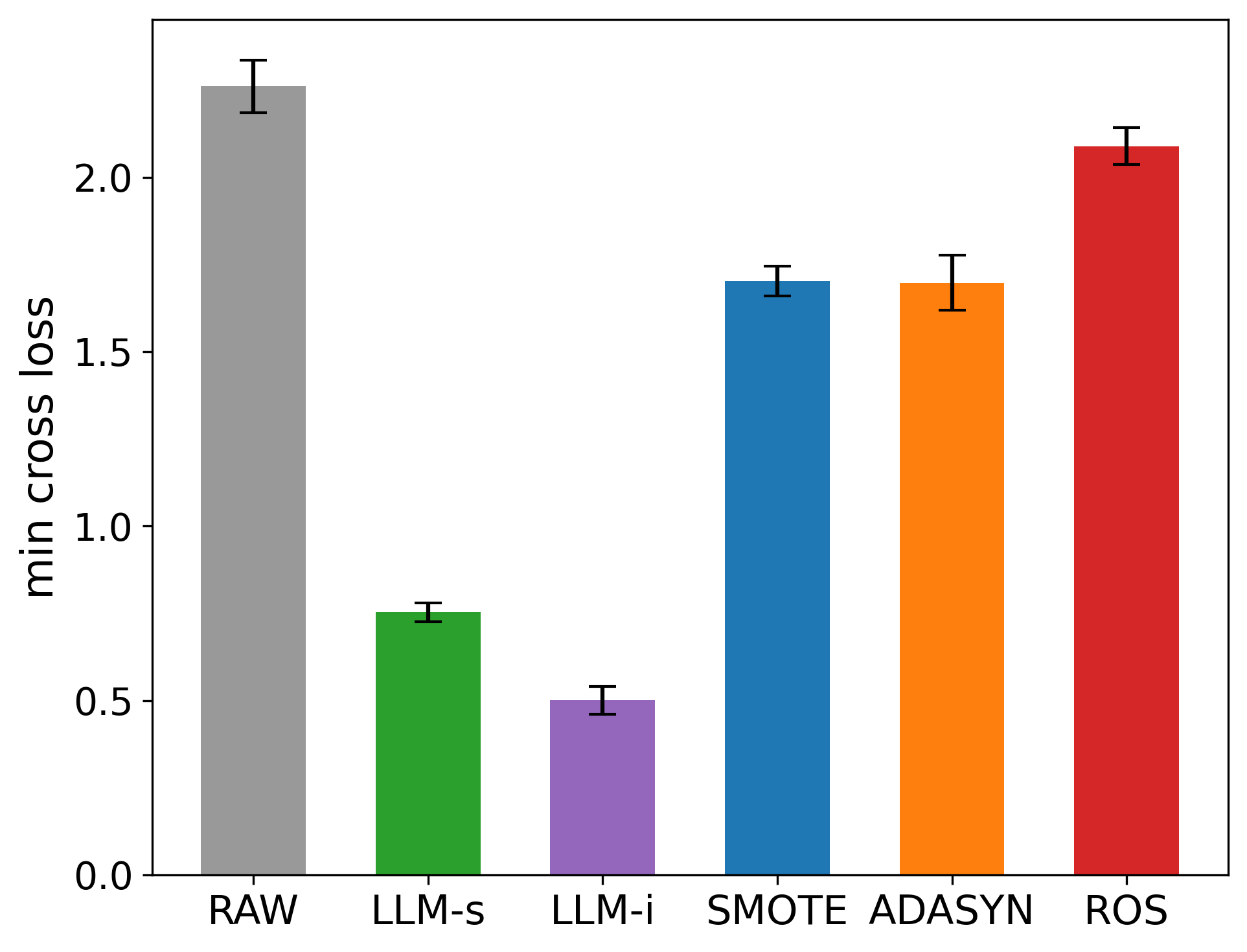}
  \caption{Adult (spur.cor, min.loss.)}
\end{subfigure}
\caption{Synthetic oversampling. Rows correspond to datasets, and columns to imbalanced classification and spurious correlation under balanced and minority cross-entropy loss.}
\label{fig:ablate_all}
\end{figure}

In Section~\ref{sec:setup} of the paper, we use \emph{all} training data as the seed data. In this section, we consider two variations of creating the seed data. In the first variation, we randomly choose 200 samples, with 100 each from the minority and majority groups, out of \emph{all} training samples to serve as the seed data. We refer to this approach as LLM-s, reflecting that only a small subset of samples are used as the seed data. In the second variation, we still randomly choose 200 samples, but out of the training samples excluding the raw data. This way, the generated synthetic data are to be independent of the raw data. We refer to this approach as LLM-i, reflecting that the synthetic and raw data are independent. We consider the same datasets, Craft, Gender, Diabetes, Adult as in Section~\ref{sec:setup}. For each dataset, we choose $n_{\text{min}} = 100$ samples from the minority group, and ${n_{\text{maj}}} = 1100$ samples from the majority group, to serve as the raw data. We continue to compare with the same alternative solutions as in Section~\ref{sec:setup}, including the baseline method with no oversampling (RAW), random oversampling (ROS), SMOTE, and ADASYN. We employ the same evaluation criteria, i.e., the balanced cross-entropy loss and the cross-entropy loss for the minority group. 

Figure~\ref{fig:ablate_all} reports the results. We observe that both variations of LLM-based synthetic oversampling method continue to outperform all the alternative solutions. Moreover, LLM-s achieves comparable and in some cases slightly worse performance than LLM-i, but still much better performance than other methods. These results, along with the results reported in Section~\ref{sec:exp-oversampling},  show that, even with a small amount of seed data, the LLM-based generator manages to produce high-quality synthetic samples that enhance downstream classification performance. In addition, the dependence between the seed and raw data has relatively limited effect empirically.

\subsection{Results with GPT-4}
\label{supp-sec: add-gpt4}

\subsubsection{Experiment setup}

For synthetic data generation, in addition to a fine-tuned GPT-2 model, we also consider a pretrained GPT-4 model without any additional fine-tuning. Specifically, we use the GPT-4 Turbo model (gpt-4-1106-preview) accessed through the OpenAI API. Synthetic data are generated through few-shot prompting, where the seed examples are formatted as tabular text followed by an instruction requesting additional records with the same schema. Each API call produces one sample record using temperature 0.7 and top-$p$ nucleus sampling with $p=0.9$. 

We consider two datasets, Diabetes and Gender from Section~\ref{sec:setup}. Due to the token limit of the GPT-4 Turbo, the size of the seed data to prompt GPT-4 is limited. We randomly choose 20 samples, with 10 each from the minority and majority groups, out of all training samples to serve as the seed data. We randomly choose $n_{\text{min}} = 10$ samples from the minority group, and ${n_{\text{maj}}} = 40$ samples from the majority group, to serve as the raw data. Slightly different from Section~\ref{sec:setup}, we examine the empirical performance by varying the synthetic-to-raw ratio, i.e., the number of synthetic samples relative to the number of raw samples, among $\{0\%, 20\%, 40\%, 60\%, 80\%, 100\%\}$. We also consider three different classifiers, logistic regression, CatBoost \citep{prokhorenkova2018catboost}, and random forest \citep{breiman2001random}, with all hyperparameters tuned by three-fold cross-validation. For synthetic oversampling, we compare with RAW, ROS and SMOTE. For synthetic augmentation, we compare with two popular deep generative methods, conditional tabular generative adversarial networks \citep[CTGAN,][]{xu19ctgan}, and tabular variational autoencoder \citep[TVAE,][]{Tazwar2024}, along with a continuous data-mixing analogue of SMOTE \citep[Mixup,][]{zhang2018mixup}, and a resampling analogue of ROS (Bootstrap). We use the misclassification error rate as the evaluation criterion. We repeat each experiment five times.

\subsubsection{Example prompt}
\label{supp-sec: gpt4-prompt}

The following prompt is used for GPT-4-based data generation. The same structure is applied across datasets, with variable descriptions adapted accordingly.
\medskip

\begin{lstlisting}[basicstyle=\ttfamily\small,
                   breaklines=true,
                   breakatwhitespace=false,
                   frame=single,
                   columns=fullflexible,
                   escapeinside={(*@}{@*)}]
(*@\textcolor{blue}{Sys:}@*) You are an expert statistician in analyzing diabetes condition. Your objective is to predict/guess new chunk of records that closely mirrors the statistical properties of a provided real-world records. This predicted/guessed records collection will be instrumental for downstream tasks such as developing personalized treatment plans, conducting epidemiological studies, and optimizing healthcare resource allocation. You are good at in-context learning. You always think step-by-step, use chain-of-thoughts, and your common sense. 
(*@\textcolor{blue}{User:}@*) The following is the text of the observed records of diabetes condition. Investigate it carefully. Each row represents the number of times pregnant, plasma glucose concentration, diastolic blood pressure, triceps skin fold thickness, 2-Hour serum insulin, body mass index, diabetes pedigree function, age, class (class value 1 is interpreted as tested positive for diabetes) . Guess and craft new 20 records of textural representation as if they were from the same source of the given records. Do not replicate the real records and the good example predicted records I will give you. Discover the pattern and trends of the real records. Your guess should preserve statistical properties. All pairs of correlation of variables should be very close to real-world records. All variables marginal distribution should be closely align with the real dataset. Learn complicated associations and interplays. Introduce interpretable variation. The guess should closely resemble real records in terms of trends and patterns. Use your domain knowledge and understanding of diabetes and other factors when you are predicting. Output predicted records in the same format as real-world records format. Do not order the guessed records.
(*@\textcolor{blue}{User:}@*) [serialized seed data]
(*@\textcolor{blue}{User:}@*) Your response must only exclusively contain your guessed records with the same format as the provided example (e.g. object is value). No other words. Please always think step-by-step, use chain-of-thoughts, and your common sense. The guessed 20 records are:
\end{lstlisting}

An example of the serialized seed data provided to GPT-4 is shown below.

\begin{lstlisting}[basicstyle=\ttfamily\small,
                   breaklines=true,
                   breakatwhitespace=false,
                   frame=single,
                   columns=fullflexible]
preg is 3.0, plas is 128.0, pres is 68.0, skin is 25.0, insu is 155.0, mass is 34.3, pedi is 0.372, age is 29.0, class is 0.0.
preg is 1.0, plas is 85.0, pres is 66.0, skin is 29.0, insu is 0.0, mass is 26.6, pedi is 0.351, age is 31.0, class is 0.0.
preg is 4.0, plas is 112.0, pres is 78.0, skin is 39.0, insu is 0.0, mass is 37.6, pedi is 0.412, age is 22.0, class is 0.0.
preg is 0.0, plas is 137.0, pres is 40.0, skin is 35.0, insu is 168.0, mass is 43.1, pedi is 2.288, age is 33.0, class is 1.0.
preg is 3.0, plas is 173.0, pres is 82.0, skin is 48.0, insu is 465.0, mass is 38.4, pedi is 2.137, age is 25.0, class is 1.0.
preg is 10.0, plas is 115.0, pres is 70.0, skin is 30.0, insu is 0.0, mass is 35.3, pedi is 0.134, age is 29.0, class is 0.0.
\end{lstlisting}

\subsubsection{GPT-4 Results}

Figure~\ref{fig:gpt4_imbalanced_3x3} reports the results for synthetic oversampling for imbalanced classification and Figure~\ref{fig:gpt4_spurious_overall_3x3} for synthetic augmentation. We see that our LLM-based synthetic oversampling method continues to outperform all the alternative solutions. Table~\ref{tab:aug-rf} reports the results for synthetic augmentation. We see that our method is the only one that consistently reduces misclassification error rate as the amount of synthetic data increases, and it consistently outperforms all the alternative solutions. Together, these results demonstrate the capacity of GPT-4 to generate high-quality synthetic data.

\begin{figure}[t!]
\centering
\begin{subfigure}{0.31\textwidth}
  \centering
  \includegraphics[width=\linewidth]{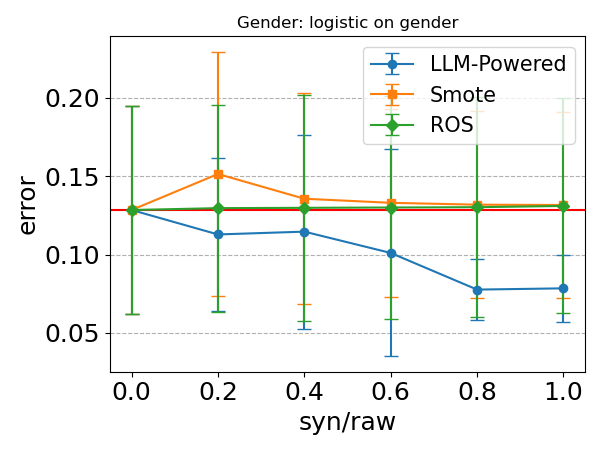}
  \caption{Gender (Logistic)}
\end{subfigure}\hfill
\begin{subfigure}{0.31\textwidth}
  \centering
  \includegraphics[width=\linewidth]{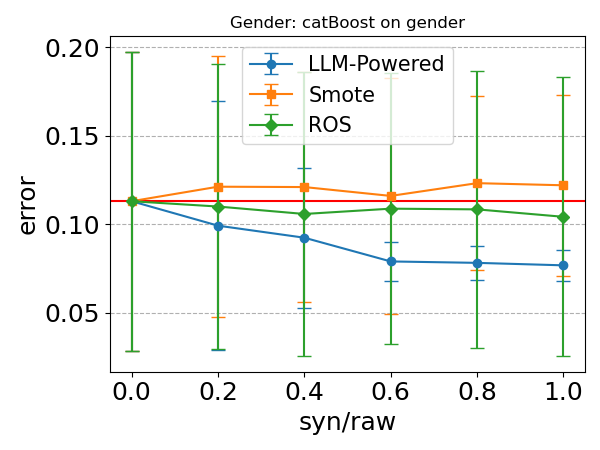}
  \caption{Gender (CatBoost)}
\end{subfigure}\hfill
\begin{subfigure}{0.31\textwidth}
  \centering
  \includegraphics[width=\linewidth]{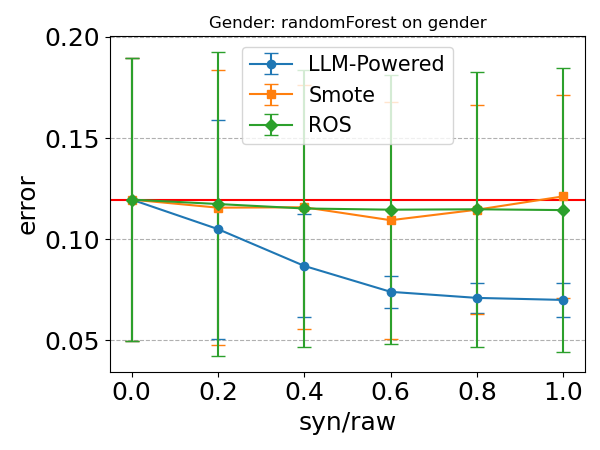}
  \caption{Gender (Random Forest)}
\end{subfigure}

\vspace{4pt}
\begin{subfigure}{0.31\textwidth}
  \centering
  \includegraphics[width=\linewidth]{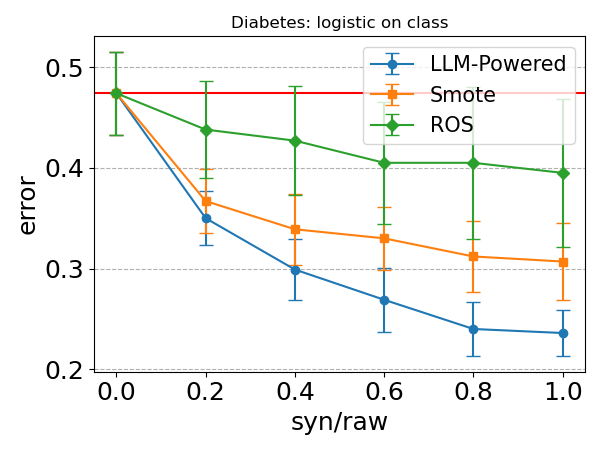}
  \caption{Diabetes (Logistic)}
\end{subfigure}\hfill
\begin{subfigure}{0.31\textwidth}
  \centering
  \includegraphics[width=\linewidth]{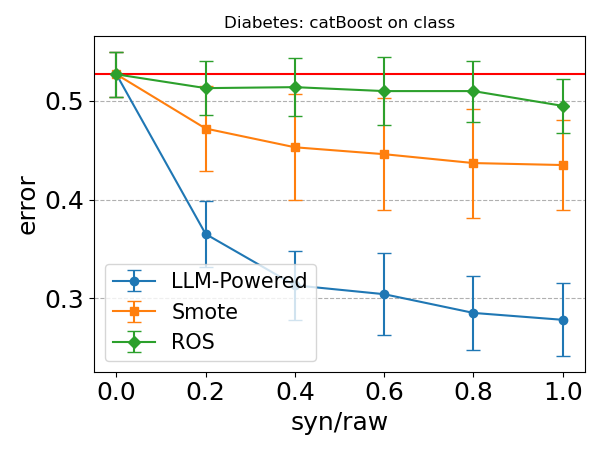}
  \caption{Diabetes (CatBoost)}
\end{subfigure}\hfill
\begin{subfigure}{0.31\textwidth}
  \centering
  \includegraphics[width=\linewidth]{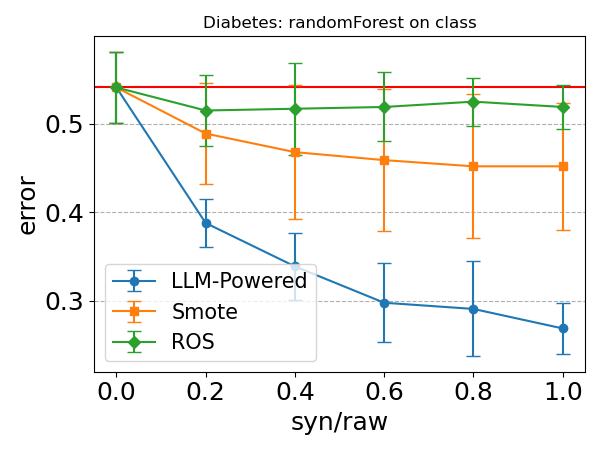}
  \caption{Diabetes (Random Forest)}
\end{subfigure}
\vspace{4pt}
\caption{Synthetic oversampling for imbalanced classification using GPT-4. Rows correspond to datasets, and columns to classification methods. The red line corresponds to the baseline method with no oversampling (RAW).}
\label{fig:gpt4_imbalanced_3x3}
\end{figure}

\begin{figure}[H]
\centering
\begin{subfigure}{0.31\textwidth}\centering
\includegraphics[width=\linewidth]{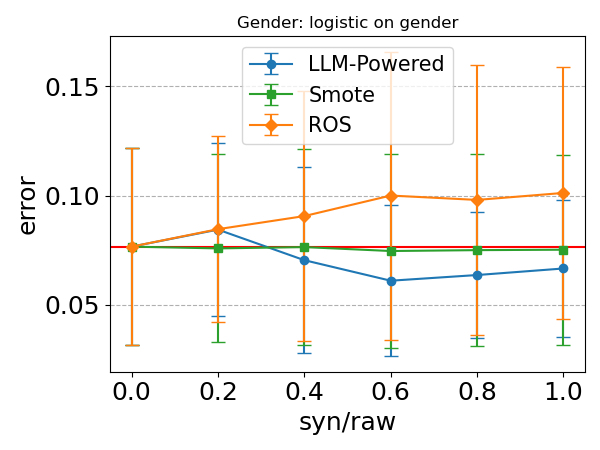}
\caption{Gender (Logistic)}
\end{subfigure}\hfill
\begin{subfigure}{0.31\textwidth}\centering
\includegraphics[width=\linewidth]{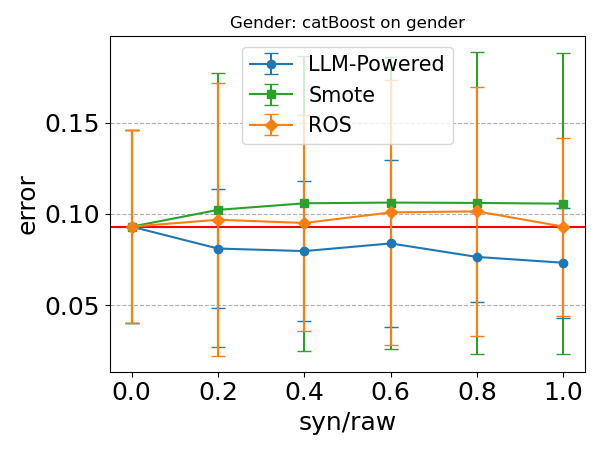}
\caption{Gender (CatBoost)}
\end{subfigure}\hfill
\begin{subfigure}{0.31\textwidth}\centering
\includegraphics[width=\linewidth]{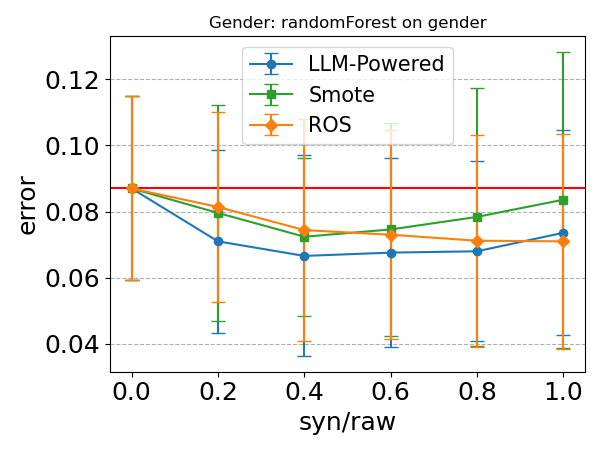}
\caption{Gender (Random Forest)}
\end{subfigure}

\vspace{4pt}
\begin{subfigure}{0.31\textwidth}\centering
\includegraphics[width=\linewidth]{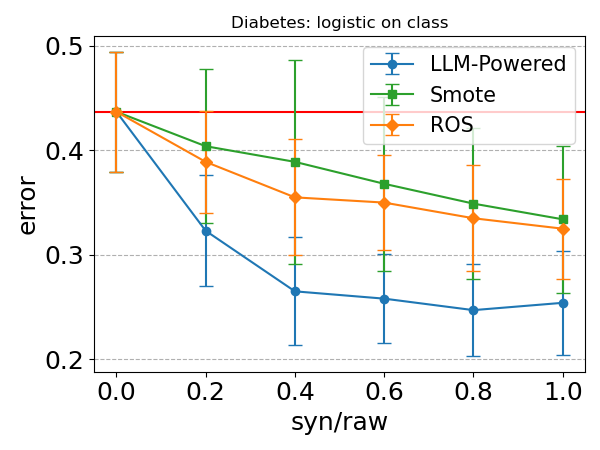}
\caption{Diabetes (Logistic)}
\end{subfigure}\hfill
\begin{subfigure}{0.31\textwidth}\centering
\includegraphics[width=\linewidth]{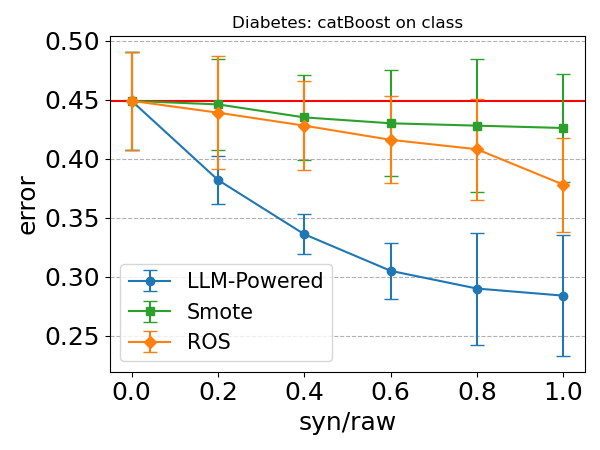}
\caption{Diabetes (CatBoost)}
\end{subfigure}\hfill
\begin{subfigure}{0.31\textwidth}\centering
\includegraphics[width=\linewidth]{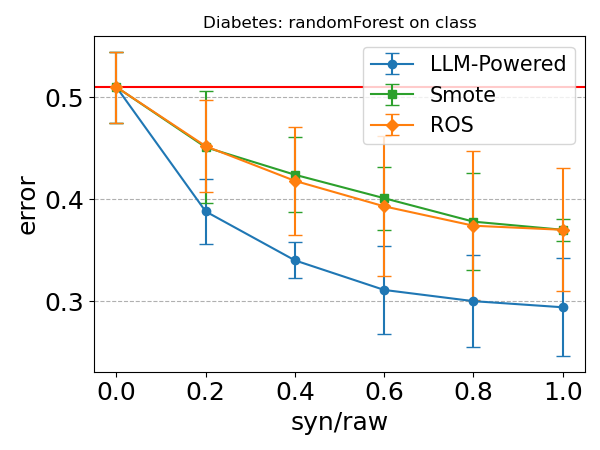}
\caption{Diabetes (Random Forest)}
\end{subfigure}
\caption{Synthetic oversampling for spurious correlation using GPT-4. Rows correspond to datasets, and columns to classification methods. The red line corresponds to the baseline method with no oversampling (RAW).}
\label{fig:gpt4_spurious_overall_3x3}
\end{figure}

\begin{table}[H]
\centering
\caption{Synthetic augmentation using data generated from different methods.}
\label{tab:aug-rf}
\resizebox{\textwidth}{!}{%
\begin{tabular}{llcccccc}
\toprule
dataset & method & \multicolumn{6}{c}{syn/raw ratio in synthetic augmentation} \\
\cmidrule(lr){3-8}
 &  & Raw & 20\% & 40\% & 60\% & 80\% & 100\% \\
\midrule
\addlinespace
Gender 
& GPT4 & $0.0914 \pm 0.0547$ & $\mathbf{0.083 \pm 0.0358}${\color{green}$\downarrow$} & $\mathbf{0.0796 \pm 0.0285}${\color{green}$\downarrow$} & $\mathbf{0.075 \pm 0.0217}${\color{green}$\downarrow$} & $\mathbf{0.073 \pm 0.0210}${\color{green}$\downarrow$} & $\mathbf{0.075 \pm 0.0221}${\color{green}$\downarrow$} \\
& CTGAN & $0.0914 \pm 0.0547$ & $0.117 \pm 0.0967${\color{red}$\uparrow$} & $0.152 \pm 0.103${\color{red}$\uparrow$} & $0.177 \pm 0.124${\color{red}$\uparrow$} & $0.187 \pm 0.122${\color{red}$\uparrow$} & $0.203 \pm 0.123${\color{red}$\uparrow$} \\
& TVAE & $0.0914 \pm 0.0547$ & $0.093 \pm 0.047${\color{red}$\uparrow$} & $0.098 \pm 0.052${\color{red}$\uparrow$} & $0.098 \pm 0.054${\color{red}$\uparrow$} & $0.101 \pm 0.059${\color{red}$\uparrow$} & $0.100 \pm 0.061${\color{red}$\uparrow$} \\
& Mixup & $0.0914 \pm 0.0547$ & $0.0896 \pm 0.0448${\color{green}$\downarrow$} & $0.0922 \pm 0.0519${\color{red}$\uparrow$} & $0.0906 \pm 0.0473${\color{green}$\downarrow$} & $0.0896 \pm 0.0457${\color{green}$\downarrow$} & $0.0904 \pm 0.0427${\color{green}$\downarrow$} \\
& Bootstrap & $0.0914 \pm 0.0547$ & $0.0992 \pm 0.0758${\color{red}$\uparrow$} & $0.105 \pm 0.0828${\color{red}$\uparrow$} & $0.105 \pm 0.0895${\color{red}$\uparrow$} & $0.105 \pm 0.0879${\color{red}$\uparrow$} & $0.107 \pm 0.0860${\color{red}$\uparrow$} \\
\midrule
\addlinespace
Diabetes 
& GPT4 & $0.239 \pm 0.0185$ & $\mathbf{0.233 \pm 0.0251}${\color{green}$\downarrow$} & $\mathbf{0.225 \pm 0.0209}${\color{green}$\downarrow$} & $\mathbf{0.225 \pm 0.0209}${\color{green}$\downarrow$} & $\mathbf{0.214 \pm 0.0139}${\color{green}$\downarrow$} & $\mathbf{0.210 \pm 0.0245}${\color{green}$\downarrow$} \\
& CTGAN & $0.239 \pm 0.0185$ & $0.241 \pm 0.0179${\color{red}$\uparrow$} & $0.252 \pm 0.0202${\color{red}$\uparrow$} & $0.269 \pm 0.0263${\color{red}$\uparrow$} & $0.274 \pm 0.0266${\color{red}$\uparrow$} & $0.279 \pm 0.0258${\color{red}$\uparrow$} \\
& TVAE & $0.239 \pm 0.0185$ & $0.250 \pm 0.0127${\color{red}$\uparrow$} & $0.252 \pm 0.0120${\color{red}$\uparrow$} & $0.248 \pm 0.0076${\color{red}$\uparrow$} & $0.252 \pm 0.0104${\color{red}$\uparrow$} & $0.255 \pm 0.0158${\color{red}$\uparrow$} \\
& Mixup & $0.239 \pm 0.0185$ & $0.238 \pm 0.0299${\color{green}$\downarrow$} & $0.242 \pm 0.0220${\color{red}$\uparrow$} & $0.250 \pm 0.0345${\color{red}$\uparrow$} & $0.248 \pm 0.0305${\color{red}$\uparrow$} & $0.252 \pm 0.0299${\color{red}$\uparrow$} \\
& Bootstrap & $0.239 \pm 0.0185$ & $0.237 \pm 0.0192${\color{green}$\downarrow$} & $0.238 \pm 0.0231${\color{green}$\downarrow$} & $0.241 \pm 0.0258${\color{red}$\uparrow$} & $0.241 \pm 0.0258${\color{red}$\uparrow$} & $0.241 \pm 0.0258${\color{red}$\uparrow$} \\
\midrule
\end{tabular}%
}
\end{table}

\end{document}